\documentclass[12pt]{cmuthesis}
\usepackage[utf8]{inputenc}
\usepackage{times}
\usepackage{fullpage}
\usepackage{eufrak}
\usepackage{amssymb}
\usepackage{amsmath}
\usepackage{amsthm}
\usepackage{esvect}
\usepackage{tikz}
\usepackage{booktabs} 
\usepackage{graphicx}
\usepackage{float}
\usepackage{pgfplots}
\usetikzlibrary{arrows.meta}
\usepackage{xcolor}
\usepackage[ruled,linesnumbered]{algorithm2e}
\usepackage{enumitem}
\usepackage[numbers,sort]{natbib}
\usepackage[backref,pageanchor=true,plainpages=false, pdfpagelabels, bookmarks,bookmarksnumbered,
]{hyperref}
\usepackage{subfigure}
\pgfplotsset{compat=1.15}
\usepackage[letterpaper,twoside,vscale=.8,hscale=.75,nomarginpar,hmarginratio=1:1]{geometry}

\newtheorem{theorem}{Theorem}[section]

\newtheorem{lemma}[theorem]{Lemma}
\newtheorem*{remark*}{Remark}

\newtheorem{proposition}[theorem]{Proposition}
\newcommand{\bigA}{\mbox{\normalfont\Large\bfseries A}}
\newcommand{\bigB}{\mbox{\normalfont\Large\bfseries B}}
\newcommand{\bigC}{\mbox{\normalfont\Large\bfseries C}}
\newcommand{\bigD}{\mbox{\normalfont\Large\bfseries D}}

\newcommand{\nindep}{\not\!\perp\!\!\!\perp}
\newcommand{\alignnewline}{\textcolor{white}{newline}}

\newcommand{\indep}{\raisebox{0.05em}{\rotatebox[origin=c]{90}{$\models$}}}

\begin{document}
\frontmatter

%initialize page style, so contents come out right (see bot) -mjz
\pagestyle{empty}
\title{{Nonlinearity, Feedback and Uniform Consistency in Causal Structural Learning}}

\author{Shuyan Wang}
\date{August 2023}
\Year{2023}
\trnumber{}

\committee{
Peter Spirtes (Chair)\\ 
Joseph Ramsey\\ 
Kun Zhang \\ 
Greg Cooper\\ 
}

\support{}
\disclaimer{}


\maketitle

\begin{dedication}
To My Mom, Hui Liu.
\end{dedication}

\pagestyle{plain} % for toc, was empty

%% Obviously, it's probably a good idea to break the various sections of your thesis
%% into different files and input them into this file...

\begin{abstract}
The goal of Causal Discovery is to find automated search methods for learning causal structures from observational data.  In some cases all variables of the interested causal mechanism are measured, and the task is to predict the effects one measured variable has on another.  In contrast, sometimes the variables of primary interest are not directly observable but instead inferred from their manifestations in the data. These are referred to as latent variables. One commonly known example is the psychological construct of intelligence, which cannot directly measured so researchers try to assess through various indicators such as IQ tests. In this case, casual discovery algorithms can uncover underlying patterns and structures to reveal the causal connections between the latent variables and between the latent and observed variables. \\ 
\textcolor{white}{abc}This thesis focuses on two questions in causal discovery: providing an alternative definition of \textbf{$k$-Triangle Faithfulness} that (i) is weaker than strong faithfulness when applied to the Gaussian family of distributions,  (ii) can be applied to non-Gaussian families of distributions, and (iii) under the assumption that the modified version of \textbf{Strong Faithfulness} holds, can be used to show the uniform consistency of a modified causal discovery algorithm; relaxing the \textit{sufficiency} assumption to learn causal structures with latent variables.  Given the importance of inferring cause-and-effect relationships for understanding and forecasting complex systems, the work in this thesis of relaxing various simplification assumptions is expected to extend the causal discovery method to be applicable in a wider range with diversified causal mechanism and statistical phenomena.\\ 
\textcolor{white}{abc}Chapter 2 will present the Generalized version of $k$-Triangle Faithfulness, which can be applied to any smooth distribution; I will then return to the linear Gaussian model, provide an investigation of the probability of the violation of $k$-Triangle-Faithfulness and compared the $k$-Triangle-Faithfulness with the Strong Faithfulness assumption\cite{JMLR:v8:kalisch07a} to quantify how much weaker the $k$-Triangle-Faithfulness assumption is, both by a mathematical analysis and a simulation study. 
Chapter 3 will show various ways of applying two constraints, GIN and rank constraints, to learn causal structure with latent variables and with some cycles and some nonlinearity. Chapter 4 will focus on the tensor constraint, a generalization of rank constraint under linear but non-Gaussian distribution\cite{robeva2020multitrek} and show that there is a hitherto unnoticed limitation of when tensor constraints are implied by causal structure but
the implication of tensor constraints by causal structure is preserved with linearity only between the observed variables and their latent common causes.
\end{abstract}
\newpage
\begin{acknowledgments}
First and foremost, I am immensely grateful to my esteemed advisor, Peter Spirtes, for his exceptional guidance, mentorship, and unwavering support throughout the entire research process. It is incredible that he managed to bear with me for such a long time without kicking me out, given I am notorious for being outrageous. I extend my heartfelt appreciation to the members of my thesis committee, Kun Zhang, Joseph Ramsey, and Greg Cooper, for their thoughtful critiques and invaluable suggestions that have significantly strengthened the quality of this work.\\ 
\textcolor{white}{abc}Special thanks to my two coffee drinking buddies: Joseph Ramsey and Fernando Larrain Langlois. Joe has provided me with so many espressos and always tolerated my dramatic and judgy comments on many of them. Fernando always managed to pull the chitchats I started while making coffee back to down-to-earth academic discussions. They provided encouragement, camaraderie, and a stimulating academic environment.\\
\textcolor{white}{abc}Additionally, I wish to acknowledge the dedication of the administrative staff at the Philosophy Department for their assistance in various aspects of my grad student life: Mary Grace Joseph, Patrick Doyle, Lisa Everett, Rosemarie Commisso, and Jacqueline DeFazio. Mary Grace helped me graduate on time through various roadblocks, was always available for support and on top of things. Patrick managed to match me with my favorite class to grade every semester. Lisa responded to me about my reimbursement request so promptly, she makes being a penniless grad student much more enjoyable.\\
\textcolor{white}{abc}I would also like to thank my best friend from the department, Biwei Huang. She not only provided sincere and heart-warming friendship but also showed me what a good researcher should be. I hope I can sustain even a fraction of her drive and dedication in my own career.\\
\textcolor{white}{abc}Furthermore, I am deeply grateful to my Master thesis advisor, Clark Glymour, whose invitation to read a statistics book opened the whole Causal Discovery area to me. He has always been a cool, wise, fun friend, and mentor. There is not a single time when I reflected on my Ph.D. journey without feeling extremely lucky to have been advised by him. His belief in me and his support throughout this journey have made a significant impact on my life.\\ 
\textcolor{white}{abc}None of this would have been possible without the firmest love of my boyfriend, Peter Oostema. Your continuous support, cheer, and belief in my abilities have been the cornerstone of my academic achievements. It might happen that one day you are not the unique Peter Oostema in Google search anymore, but you will always be my forever and unique love.\\ 
\textcolor{white}{abc}Lastly, I would like to dedicate this thesis to my parents, Hui Liu, and Anbiao Wang, without whom I wouldn’t have existed. They might not fully grasp my research, but their full understanding of me made my Ph.D. journey go smoother.\\ \textcolor{white}{abc}
\end{acknowledgments}
\tableofcontents
\listoffigures

\mainmatter
\chapter{Introduction}
Efficient and accurate data analysis techniques are more and more important as the volume of data continues to expand exponentially. This technique has proven to be instrumental in numerous fields, such as bioinformatics, finance, climate studies, and social sciences.\\ \textcolor{white}{abc} 
In some cases all variables of the interested causal mechanism are measured, and the task is to predict the effects one measured variable has on another.  In contrast, sometimes the variables of primary interest are not directly observable but instead inferred from their manifestations in the data. These are referred to as latent variables. One commonly known example is the psychological construct of intelligence, which cannot directly measured so researchers try to assess through various indicators such as IQ tests. In this case, casual discovery algorithms can uncover underlying patterns and structures to reveal the causal connections between the latent variables and between the latent and observed variables. \\ \textcolor{white}{abc}  
This thesis focuses on two questions in causal discovery: providing an alternative definition of \textbf{$k$-Triangle Faithfulness} that (i) is weaker than strong faithfulness when applied to the Gaussian family of distributions,  (ii) can be applied to non-Gaussian families of distributions, and (iii) under the assumption that the modified version of \textbf{Strong Faithfulness} holds, can be used to show the uniform consistency of a modified causal discovery algorithm; relaxing the \textit{sufficiency} assumption to learn causal structures with latent variables.  Given the importance of inferring cause-and-effect relationships for understanding and forecasting complex systems, the work in this thesis of relaxing various simplificity assumptions is expected to extend the causal discovery method to be applicable in a wider range with diversified causal mechanism and statistical phenomena.\\ \textcolor{white}{abc}
In this chapter I will briefly review the definitions and assumptions in causal discovery.  Most of the assumptions and terms will occur frequently throughout the thesis.  I will then outline the structure of the thesis.  

\section{Background for Causal Discovery}
\subsection{DAG and Causal Graph}
Directed graphs are used to represent causal relations between variables.  A directed graph $G=\langle \mathbf{V,E}\rangle$ consists of a set of \textit{nodes} \textbf{V} and a set of \textit{edges} $\mathbf{E\subset V\times  V}$. If there is an edge $\langle A,B\rangle \in \mathbf{E}$ , we write $A\rightarrow B$. $A$ is a \textbf{parent} of $B$, and $B$ is a \textbf{child} of $A$, the edge is \textbf{out} of $A$ and \textbf{into} $B$, and $A$ is the \textbf{source} and $B$ is the \textbf{target}.  A \textbf{directed path} from $X$ to $Y$ is a 
sequence of edges where the source of the first edge is $X$, the target of the last edge is $Y$, and if there are $n$ edges in the sequence, for $1\leq i <n$, the target of the $i$th edge is the source of the $i+1$st edge. If there is a directed path from $X$ to $Y$, then $X$ is an \textbf{ancestor}  of $Y$, and $Y$ is a \textbf{descendant} of $X$.  A directed path is \textit{acyclic} if no vertex occurs more than once on the path.  A directed \textit{acyclic} graph (DAG) is a directed graph if all directed paths are acyclic.  Notice that $X$ cannot be its own ancestor nor descendant in a DAG. \\ \textcolor{white}{abc}  
If a variable $Y$ is in a structure $X\rightarrow Y\leftarrow Z$, and there is no edge between $X$ and $Z$, we call $\langle X,Y,Z\rangle$ an \textit{unshielded collider}; if there is also an edge between $X$ and $Z$, then $\langle X,Y,Z\rangle$ is a \textit{triangle} and we call $\langle X,Y,Z\rangle$ a \textit{shielded collider}.  If $\langle X,Y,Z\rangle$ is a triangle but $Y$ is not a child of both $X$ and $Z$, we call $\langle X,Y,Z\rangle$ a \textit{shielded non-collider}; if there is no edge between $X$ and $Z$, then $\langle X,Y,Z\rangle$
is an \textit{unshielded non-collider.}\\ \textcolor{white}{abc}
 A \textit{Bayesian network} is an ordered pair $\langle P,G\rangle$ where $P$ is a probability distribution over a set of variables $\textbf{V}$ in $G$. A distribution $P$ over a set of variables $\mathbf{V}$ satisfies the \textbf{(local) Markov condition} for $G$ if and only if each variable in $\mathbf{V}$ is independent of the set of its non-parents and non-descendants, conditional on its parents. A distribution $P$ is $faithful$ to a $DAG$ $G$ if every conditional independence true in $P$ is entailed by the local Markov condition for $G$.   \\ \textcolor{white}{abc}
Given $M=\langle P,G\rangle$, $P_M$ denotes $P$ and $G_M$ denotes $G$.  Two acyclic directed graphs (DAG) $G_1$ and $G_2$ are \textit{Markov equivalent} if conditional independence relations entailed by the local Markov condition in $G_1$ are the same as in $G_2$.  It has been proven that two $DAG$s are \textit{Markov equivalent} if and only if they have the same variables, same adjacencies and same \textit{unshielded colliders} \cite{Verma}. A \textit{CPDAG} $O$ is an undirected graph that represents a set $M$ of Markov equivalent DAGs: an edge $X\rightarrow Y$ is in $O$ if it is in every DAG in $M$; if $X\rightarrow Y$ is in some DAG and $Y\rightarrow X$ in some other DAG in $M$, then $X-Y$ in $O$ \cite{Spirtes2016CausalDA}. \\ \textcolor{white}{abc}
Interpreting a directed graph in a causal sense, $X\rightarrow Y$ in a directed graph $G$ means $X$ is a direct cause of $Y$ (relative to the variables in the graph).  Intuitively, to accurately predict the effects of interventions on a variable, it is essential to know facts about the causal structure.  In reality what the researchers have assess to is the data generated by a causal graph, and in order to study the state of any interested variable, it is necessary to learn the causal graph structures from the data.

\subsection{Causal Markov, Faithfulness and Causal Sufficiency Assumption}

Given a \textit{Bayesian network} $\langle P,G\rangle$ that satisfies \textbf{Markov Condition}, we say that $P$ is \textit{faithful} to $G$ if any conditional independence relation that holds in $P$ is entailed by $G$ by the \textbf{Markov Condition}.  Formally the Causal Sufficiency, Causal Markov, and Faithfulness assumption can be defined as  
\cite{CausationPredictionSearch}:\\ 
\textbf{Causal Sufficiency Assumption:} A set of variables $\mathbf{V}$ is causally sufficient if $\textbf{V}$ contains all the variables $X$ s.t. there exist at least two variables $Y_i$ and $Y_j$ in $\mathbf{V}$ s.t. $X\rightarrow Y_i$ and $X\rightarrow Y_j$ are in the true causal graph $G_M$.\\ 
\textbf{Causal Markov Assumption:} If a set of variables in a causal model $M = \langle P_M, G_M\rangle$ of a population is causally sufficient, each variable in $\mathbf{V}$ is independent of the set of its non-parents and non-descendants, conditional on its parents in $G_M$ \cite{Spirtes2016CausalDA}.\\ 
\textbf{Causal Faithfulness Assumption: } all conditional independence relations that hold in the population are consequences of the Markov condition from the underlying true causal DAG.\\ \textcolor{white}{abc}

\subsection{SEM and Linear Gaussian Models}
\textbf{\textit{Fixed Parameter Structral Equation Model (SEM) $S$}}\label{def:SEM}: The causal relation and effect can be quantified by structural equation model \cite{doi:10.1177/0049124198027002004SEM}. In a SEM we divide random variables into two sets, the variables of interest and error/noise terms.  A fixed parameter SEM $S$ has two parts $\langle \phi, \theta\rangle$, where $\phi$ is a set of equations in which each random variable of interest $V$ is written as a function of other substantive random variables and a unique error variable, together with $\theta$, the joint distributions over the error variables\cite{LAchoke}.\\ \textcolor{white}{abc}
 A common parameterization of a causal graph $G$ is linear structural equation model.  Specifically, given a model $M = \langle P_M, G_M\rangle$, where $M$ is a Bayesian network with $P_M$ over $\mathbf{V}$ for $G=\langle \mathbf{V,E}\rangle$, the $M$ is a linear model if for any $X_i\in\mathbf{V}$ can be written as:
\begin{center}
    $X_i =  \underset{X_j\in Pa_M(X_i)}{\sum} a_{j,i}X_j+\epsilon_i$  
\end{center}
where $Pa_M(X)$ denotes the set of parents of $X$ in $G_M$ and $\epsilon_i$ is independent of each $X_j$. Based on the equation above, we define in the linear case the \textit{edge strength} $e_M(X_j\rightarrow X_i)$ as the corresponding coefficient $a_{j,i}$.  The error/noise terms are jointly independent.  A linear Gaussian Model is a linear model when the vector of noises $\epsilon$ follows multivariate Gaussian distribution.   Having $\mathbf{A}$ as the coefficient matrix such that $a_{ji} = e_M(X_j\rightarrow X_i)$ and $\mathbf{\epsilon} = \left(\epsilon_1,...,\epsilon_{\mathbf{|V|}}\right)\sim\mathcal{N}\left(0, \mathbf{\Sigma_\epsilon}\right)$, we rewrite the linear equation as:
\begin{center}
    $\mathbf{(I-A)}^T\mathbf{X}=\epsilon$
\end{center}
and have the $\mathbf{X}$ follows the distribution:
\begin{center}
    $\mathbf{X}\sim\mathcal{N}\left(0,\left[\mathbf{(I-A)\mathbf{\Sigma_\epsilon}^{-1}(I-A)}^T\right]^{-1}\right)$
\end{center}
 We assume the model $M = \langle P_M, G_M\rangle$, where $P_M$ over $\mathbf{V}$ for $G=\langle \mathbf{V,E}\rangle$ respects the \textbf{Causal Markov Assumption}.
\subsection{Trek, T-Separation and D-Separation}

In a directed graph $G$, a \textit{trek} $\tau$ in $\mathfrak{G}$ between $X$ and $Y$ is a pair of directed paths $(\pi_1,\pi_2)$ such that $\pi_1$ ends at $X$ and $\pi_2$ ends at $Y$ and  $\pi_1$ and $\pi_2$ begin at some common vertex, $S$, denoted as $top(\tau)$. A \textit{trek} $\tau$ is \textit{simple} if $\pi_1$ and $\pi_2$ only intersect at one $S$.\\ \textcolor{white}{abc}
For sets $\mathbf{A, B, C_A, C_B}\subset \mathbf{V}$, $\mathbf{A,B}$ are \textbf{\textit{t-separated}} by $(\mathbf{C_A,C_B})$ if for every trek $(\pi_1,\pi_2)$ from some $A\in\mathbf{A}$ to some $B\in\mathbf{B}$, either $\pi_1$ includes some $C\in \mathbf{C_A}$ or $\pi_2$ includes some $C\in\mathbf{C_B}$.  We say $\mathbf{(C_A,C_B)}$ are \textit{choke sets}.
\begin{figure}[H]
    \centering
    \includegraphics[width=100mm]{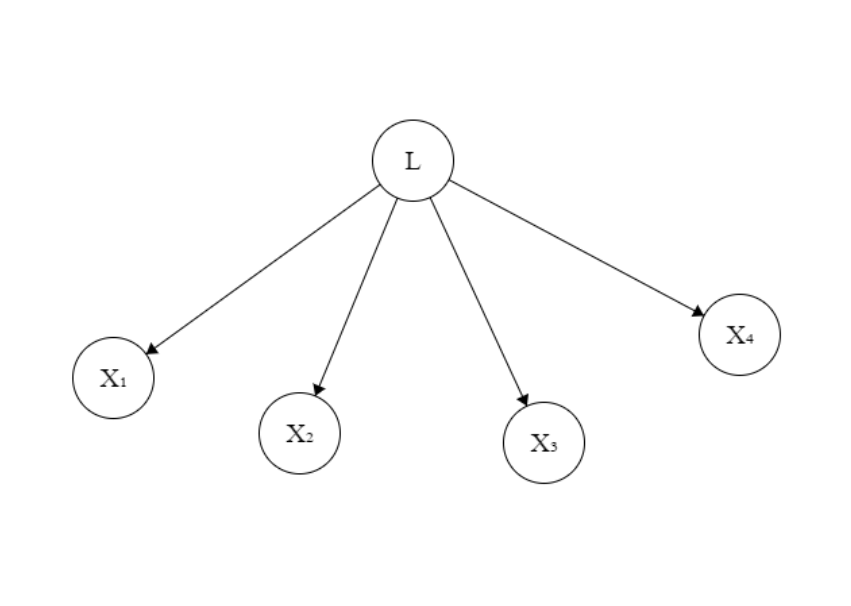}
    \caption{a 4-trek on $\{X_1, X_2, X_3, X_4\}$}%For any four variables $X_1, X_2, X_3$ and $X_4$, if $\sum_\mathbf{A,B}$ has a zero rank for every partition $\mathbf{A,B}$ of $\{X_1, X_2, X_3, X_4\}$, FOFC puts a latent common cause among them}
     \label{fig:FOFC}
\end{figure}
For three disjoint sets $\mathbf{A, B, C}\subset \mathbf{V}$,  $\mathbf{A}$ and $\mathbf{B}$ are \textbf{\textit{d-connected}} by $\mathbf{C}$ in $G$ if and only if there exists an undirected path $\pi$ between some vertex in $\mathbf{A}$ and some vertex in $\mathbf{B}$ such that for every collider $W$ on $\pi$, either $W$ or a descendent of $W$ is in $\mathbf{C}$, and no non-collider on $\pi$ is in $\mathbf{C}$.\\ \textcolor{white}{abc}
 $\mathbf{A}$ and $\mathbf{B}$ are \textit{\textbf{d-separated}} by $\mathbf{C}$ in $G$ if and only if they are not \textit{\textbf{d-connected}} by $\mathbf{C}$ in $G$.\\ \textcolor{white}{abc}
If $M = \langle P, G\rangle$ respects the \textbf{Causal Markov Assumption}, then $\mathbf{A}$ and $\mathbf{B}$ being \textit{\textbf{d-separated}} by $\mathbf{C}$ in $G$ implies that $\mathbf{A}$ and $\mathbf{B}$ are independent conditioning on $\mathbf{C}$. Assuming faithfulness, the conditional independences learnt from the data be used to derive \textbf{\textit{d-separation}} in the causal structure, which can then be used to search for the causal graph\cite{CausationPredictionSearch}.\\ \textcolor{white}{abc} 
\textbf{\textit{D-separation}} is a special case of \textbf{\textit{t-separation}}.  Later in this thesis we are going to introduce that for linear models, \textit{\textbf{t-separation}} implies rank constraints on submatrices of covariance matrix of $\mathbf{V}$ and can also be used for learning the true causal structure.

\section{Overview of the thesis}
This thesis is organized as follows:\\ \textcolor{white}{abc}
Chapter 2 presents the Generalized version of $k$-Triangle Faithfulness, which can be applied to any smooth distribution. In addition, under the Generalized $k$-Triangle Faithfulness Assumption, we describe the Edge Estimation Algorithm that provides uniformly consistent estimators of causal effects and the \textit{Very Conservative }$SGS$ Algorithm that is a uniformly consistent estimator of the Markov equivalence class of the true DAG.  We further return to the linear Gaussian model, provide an investigation of the probability of the violation of $k$-Triangle-Faithfulness and compared the $k$-Triangle-Faithfulness with the Strong Faithfulness assumption\cite{JMLR:v8:kalisch07a} to quantify how much weaker the $k$-Triangle-Faithfulness assumption is, both by a mathematical analysis and a simulation study.\\ \textcolor{white}{abc}
Chapter 3 seeks various ways of applying two constraints used to learn causal structure with latent variables with linear acyclic model to models with cycles and nonlinearity. In particular, we will show that with linearity and acyclicity between the observed variables between their latent common causes and nonlinearity between latents, the Generalized Independence Noise (GIN) can be used to cluster observed variables that are sharing the same set of latent variables\cite{xie2020generalized}.  We will also show that with only linearity between the observed variables and their latent common causes (in most of the cases represented by \textit{choke sets}), rank constraints can be used to cluster variables in a similar fashion\cite{LAchoke}\cite{Sullivant_2010}. Assuming the noises of variables following non-Gaussianity distribution, we will first present a method for identifying cycles under \textit{choke sets} using GIN and rank constraint and design an causal clustering algorithm with this method.  We further implement this algorithm and show its simulation result.  Finally we will look into the cases where in an acyclic latent causal structure, the edges connecting latent variables from two different clusters are with the opposite directions (cycles between `latent blocks').  We will describe the pseudo-code identifying the cycles between latent blocks using GIN and rank constraints and an example of how the pseudo-algorithm works. \\ \textcolor{white}{abc}
Chapter 4 focuses on the tensor constraint, a generalization of rank constraint under non-Gaussian distribution\cite{robeva2020multitrek}.  Similar to Chapter 3, we will show that the validity of tensor constraint is preserved with only linearity between the observed variables and their latent common causes. We will then show a mathematical property of the hyperdeterminant and the limitation of the tensor constraint. \\ \textcolor{white}{abc}
%Chapter 5 will provide an empirical study of portfolio optimization using GIN.  Specifically speaking, I use causal clustering result to select a list of stocks as well as their weights to achieve an equal performance of the state-of-art graphical modeling methods or gain a portfolio with higher return.  The study shows that GIN is effective in portfolio optimization in the sense that assigning weights to stocks based on the clustering result of GIN increases the return and without decreasing the lowest value point of the portfolio.\\ \textcolor{white}{abc}
 Chapter 5 will summarize the thesis and discuss future research directions.
\newpage

\chapter{The Uniform Consistency of \textit{k}-Triangle-Faithfulness Assumption}
\section{Introduction}
It has been proved that under the Causal Markov, Faithfulness assumptions and Causal Sufficiency Assumption, there are no uniformly consistent estimators of Markov equivalence classes of causal structures represented by directed acyclic graphs (DAG)\cite{pointconsisitencyrobins}. It has been shown that for linear Gaussian models, under the Causal Markov Assumption, the Strong Causal Faithfulness Assumption, and the assumption of causal sufficiency, the PC algorithm is a uniformly consistent estimator of the Markov Equivalence Class of the true causal DAG for linear Gaussian models; it follows from this that for the identifiable causal effects in the Markov Equivalence Class, there are uniformly consistent estimators of causal effects as well\cite{JMLR:v8:kalisch07a}. The \textbf{\textit{k-}Triangle-Faithfulness Assumption} is a strictly weaker assumption for some values of $k$ is strictly weaker than than the strong faithfulness assumption that avoids some implausible implications of the Strong Causal Faithfulness Assumption and also allows for uniformly consistent estimates of Markov Equivalence Classes (in a weakened sense), and of identifiable causal effects.\\ \textcolor{white}{abc} 
Our work about faithfulness assumption consists of two parts: since the original \textbf{\textit{k-}Triangle-Faithfulness Assumption} is restricted to linear Gaussian models, we first propose the Generalized $k$-Triangle Faithfulness, which can be applied to any smooth distribution and describe the Edge Estimation Algorithm that provides uniformly consistent estimates of causal effects in some cases (and otherwise outputs ``can't tell"); we then provide upper bounds and lower bounds of the violation of \textbf{\textit{k-}Triangle-Faithfulness Assumption} in the linear case, and compare it with the upper and lower bounds of the violation of the Strong Causal Faithfulness validated with simulation results \cite{Uhler_2013}.

\subsection{The Faithfulness Assumptions for Causal Discovery}

It has been proved that under the Causal Markov and Faithfulness assumptions, there are no uniformly consistent estimators of Markov equivalence classes of causal structures represented by DAG \cite{pointconsisitencyrobins}.  It has been shown that such uniform consistency is achieved by the $PC$ algorithm if the underlying DAG is sparse relative to the sample size under a strengthened version of Faithfulness Assumption\cite{JMLR:v8:kalisch07a}.  This \textbf{Strong Causal Faithfulness Assumption} in the linear Gaussian case bounds the absolute value of any partial correlation not entailed to be zero by the true causal DAG away from zero by some positive constants. It has the implausible consequence that it puts a lower bound on the strength of edges: it is often implicitly assumed that there may be many weak causes of a variable, but even if they exist, that they can be safely ignored in estimating causal effects; the \textbf{$k$-Triangle Faithfulness Assumption} basically formalizes that assumption, whereas the Kalisch-B\"{u}hlmann results require the stronger assumption that there are no such weak causes. This assumption is not necessary to obtain uniform consistency, and limits the applicability of their results.\\ \textcolor{white}{abc} 
However, the \textbf{Strong Causal Faithfulness Assumption} can be weakened to the strictly weaker (for some values of $k$)  \textbf{\textit{k-}Faithfulness Assumption} while still achieving uniform consistency. Furthermore, at the cost of having a smaller subset of edges oriented, the \textbf{\textit{k-}Faithfulness Assumption} can be weakened to the \textbf{\textit{k-}Triangle-Faithfulness Assumption}, while still achieving uniform consistency and can be relaxed while preserving the uniform consistency: the \textbf{\textit{k-}Triangle-Faithfulness Assumption} \cite{Spirtes_2014} only bounds the conditional correlation between variables in a triangle structure from below by some functions of the corresponding edge strength:\\ 
\textbf{\textit{k-}Triangle-Faithfulness Assumption:}  
Given a set of variables \textbf{V}, suppose the true causal model over \textbf{V} is $M=\langle P,G\rangle$, where $P$ is a Gaussian distribution over \textbf{V}, and \textit{G} is a DAG with vertices \textbf{V}. For any variables \textit{X}, \textit{Y}, \textit{Z} that form a triangle in \textit{G}:
\begin{itemize}
    \item if \textit{Y} is a non-collider on the path $\langle X,Y,Z\rangle$, then $|\rho_M(X,Z|\mathbf{W})|\geq k\times |e_M(X-Z)|$ for all $\mathbf{W\subset V}$ that do not contain $Y$; and
    \item if \textit{Y} is a collider on the path $\langle X,Y,Z\rangle$, then $|\rho_M(X,Z|\mathbf{W})|\geq k\times |e_M(X-Z)|$ for all $\mathbf{W\subset V}$ that do contain $Y$.
\end{itemize}
where the $X-Z$ represents the edge between $X$ and $Z$ but the direction is not determined\cite{Spirtes_2014}.\\ \textcolor{white}{abc}
The \textbf{\textit{k-}Triangle-Faithfulness Assumption} is strictly weaker than the \textbf{Strong Faithfulness Assumption }in several respects: the \textbf{Strong faithfulness Assumption }does not allow edges to be weak any where in a graph, while the \textbf{\textit{k-}Triangle-Faithfulness Assumption} only excludes conditional correlations $\rho(X,Z|\mathbf{W})$ from being too small if $X$ and $Z$ are in some triangle structures $\langle X,Y,Z\rangle$ and $X-Z$ is not a weak edge; for every $\epsilon$ used in the \textbf{Strong Faithfulness Assumption }as the lower bound for any partial correlation, there is a $k$ for the \textbf{\textit{k-}Triangle-Faithfulness Assumption} that gives a lower bound smaller than $\epsilon$.

\subsubsection{VCSGS Algorithm}
The algorithm we use to infer the structure of the underlying true causal graph is \textit{Very Conservative SGS} ($VCSGS$) algorithm, which takes uniformly consistent tests of conditional independence as input:
\RestyleAlgo{ruled}
\SetKwComment{Comment}{/* }{ */}
\begin{algorithm}
\caption{VCSGS}\label{alg:VCSGS}
\KwIn{Data from the set of variables $\mathbf{V}$; }
\KwOut{Graph $\mathcal{H}$}
 $\mathcal{H}\gets$ Complete undirected graph over the given $\mathbf{V}$;
 
 \For{$\{X, Y\}\subset\mathbf{V}$}{
 \For{ $\mathbf{S} \subset\mathbf{V}\setminus\{X, Y \}$}{\If{$X\indep Y|\mathbf{S}$}{Remove $X-Y$ from $\mathcal{H}$}}
 }
 \For{unshielded triple $\langle X, Y, Z\rangle$ in $\mathcal{H}$ \Comment*[r]{$X$ and $Z$ are not adjacent}}{
 \eIf{$\forall \mathbf{S\subset V}\setminus\{X, Z\}$ s.t. $Y\in\mathbf{S}$, $X\nindep Z|\mathbf{S}$}{Orient $X-Y-Z$ in $\mathcal{H}$ as $X \rightarrow Y \leftarrow Z$}{
 \eIf{$\forall \mathbf{S\subset V}\setminus\{X, Y, Z\}$, $X\nindep Z|\mathbf{S}$}{Mark $X-Y-Z$ in $\mathcal{H}$ as \textit{noncollider}}{Mark $X-Y-Z$ in $\mathcal{H}$ as \textit{ambiguous}}
 }
 }
 {$\#$\textit{Executing More Orientation Rules}}\\ \textcolor{white}{abc}
 \For{ $X \rightarrow Y-Z$ in $\mathcal{H}$}{\If{ $\langle X, Y, Z\rangle$ is marked as \textit{noncollider} }{Orient $Y-Z$ as $ Y\rightarrow Z$ in $\mathcal{H}$}
 }
 \For{ $X \rightarrow Y\rightarrow Z$ in $\mathcal{H}$}{\If{ $X-Z$ exists in $\mathcal{H}$}{Orient $X-Z$ as $ X\rightarrow Z$ in $\mathcal{H}$}
 }
 \For{ $X \rightarrow Y\leftarrow Z$ in $\mathcal{H}$}{\If{ $\langle X, W, Z\rangle$ is marked as \textit{noncollider} in $\mathcal{H}$ and $W-Y$ exists in $\mathcal{H}$}{Orient $W-Y$ as $ W\rightarrow Y$ in $\mathcal{H}$}
 }
 \For{ \textit{ambiguous triple} in $\mathcal{H}$}{\If{ 
 $\forall$ pattern consistent with $\mathcal{H}$ with the triple orientated, \textit{Markov Condition} is satisfied}{Mark all the `apparently non-adjacent’ pairs as
`definitely non-adjacent’}
 }
\textbf{Return} $\mathcal{H}$
\end{algorithm}

\newpage
It has been proved that under the \textbf{\textit{k-}Triangle-Faithfulness Assumption}, $VCSGS$ algorithm is uniformly consistent (while allowing the possibility of missing weak edges as correct) in the inference of graph structure that is a subgraph of the true causal graph.  Notice that $VCSGS$ is a very conservative algorithm, with complexity of $\mathcal{O}(|\mathbf{V}|^22^{|\mathbf{V}|-2})$ with $|\mathbf{V}|$ being the cardinality of $\mathbf{V}$ or the number of variables in the dataset.  A follow-up algorithm that estimates edge strength given the output of $VCSGS$ also reaches uniform consistency. We are going to prove that the uniform consistency of the estimation of the causal influences under the \textbf{\textit{k-}Triangle-Faithfulness Assumption} can be extended to discrete and nonparametric cases as long as there are uniformly consistent tests of conditional independence (which in the general case requires a smoothness assumption), by showing that missed edges in the inference of causal structure are so weak that the estimations of the causal influences are still uniformly consistent. 
\section{Nonparametric Case}
For nonparametric case, we consider variables supported on $[0,1]$.
We first define the strength of the edge $X\rightarrow Y$ as the maximum change in $L1$ norm of the probability of $Y$ when we condition on different values of $X$ while holding everything else constant:

\begin{center}
    
If $x\in Pa(Y):$
$e(X\rightarrow Y) :=max_{pa_{\setminus\{x\}}(Y)\in [Pa(Y)\setminus\{X\}]}max_{x_1,x_2\in[X]}||p_{Y|x_1,pa_{\setminus\{x\}}(Y)}- p_{Y|x_2,pa_{\setminus\{x\}}(Y)}||_1$

\end{center}
where $[X]$ denotes the set of values that $X$ takes,  $[Pa(Y)]\subset [0,1]^{|Pa(Y)|}$ the set of values that parents of Y take, $p_{Y|x_1,pa_{\setminus\{x\}}(Y)}$ the probability distribution of $Y|X=x_1, Pa(Y)\setminus\{X\}=pa_{\setminus\{x\}}(Y)$ and $p_{y|x_1,pa_{\setminus\{x\}}(Y)}$ the density of $p_{Y|x_1,pa_{\setminus\{x\}}(Y)}$ 
for $Y=y$. Since we are conditioning on the set of parents, the conditional probability is equal to the manipulated probability. 

Now we can make the \textbf{\textit{k-}Triangle-Faithfulness Assumption}: given a set of variables \textbf{V}, where the true causal model over \textbf{V} is $M=\langle P,G\rangle$, $P$ is a distribution over \textbf{V}, and \textit{G} is a DAG with vertices \textbf{V}, for any variables \textit{X}, \textit{Y}, \textit{Z} that form a triangle in \textit{G}:
\begin{itemize}
    \item if \textit{Z} is a non-collider on the path $\langle X,Z,Y\rangle$, given any subset $\mathbf{W\subset V}\setminus \{X,Y,Z\}$,\newline $\max_{\mathbf{w\in [W]}}\max_{x_1, x_2\in[X]}||p_{Y|\mathbf{w},x_1}- p_{Y|\mathbf{w},x_2}||_1\geq K_Y e(X\rightarrow Y)$ for some $K_Y>0$
    \item if \textit{Z} is a collider  on the path $\langle X,Z,Y\rangle$, then for every $y\in [Y]$, given any subset $\mathbf{W\subset V}\setminus \{X,Y,Z\}$,$\max_{\mathbf{w\in [W]}}\max_{z\in [Z]}\max_{x_1,x_2\in[X]}||p_{Y|\mathbf{w}, x_1, z}-p_{Y=y|\mathbf{w},x_2,z}||_1\geq K_Y e(X\rightarrow Y)$ for some $K_Y>0$
\end{itemize}
In this assumption we only set restrictions on single manipulation (focusing on one parent).  If we instead want to post similar restrictions for multiple manipulations, we just need to replace the $x_1$ and $x_2$ with different values a set of variables take and replace the strength of an edge with a strength of a collection of edges through a similar replacement in the definition of strength of an edge.\\ \textcolor{white}{abc}
The \textbf{\textit{max-max}} version is the weakest version of the \textbf{\textit{k-}Triangle-Faithfulness Assumption} for the nonparametric case, which may be too weak for practical uses considering the conditional independence test sometimes tests dependence on intervals of values.  If the \textbf{\textit{max-max}} versions of the \textbf{\textit{k-}Triangle-Faithfulness Assumption} satisfies uniform consistency, then all other version of this assumption, \textit{i.e.}, versions stricter than \textbf{\textit{max-max}} automatically satisfy the uniform consistency.\\ \textcolor{white}{abc}
In order to have uniformly consistent tests of conditional independence, we make smoothness assumption for continuous variables with the support on $[0,1]$:\newline
 \textbf{ TV (Total Variation) Smoothness(L): } Let $\mathcal{P}_{[0,1],TV(L)}$ be the collection of distributions $p_{Y,\mathbf{A}}$, such that for all $\mathbf{a},\mathbf{a'}\in [0,1]^{|\mathbf{A}|}$, we have:
\begin{center}
    
    $||p_{Y|\mathbf{A=a}}-p_{Y|\mathbf{A=a'}}||_1\leq L||\mathbf{a-a'}||_1$
\end{center}

 The smoothness assumption prevents the case where the probability measure over $Y$ dramatically changes while the value of the conditioned sets change only slightly.  However, it does not necessarily prevent the case when $\mathbf{A}$ takes a slightly different value, the observed value of $Y$ changes a lot.  In finite sample sizes, it can still leads to a missed edge in the reference of causal structure, if we failed to manipulate $\mathbf{A}$ within the small internal and observed the extraordinary change of $Y$, but not in the limit.
    
Given the TV smoothness(L), $p$ is continuous.  Furthermore, since $[0,1]^d$ ($d\in \mathbb{N}$) is compact, for any $\mathbf{W,U\subset V}$ (the set of all variables in the true causal graph) , $p_{U|W}$ attains its max and min on its support.  Since $\mathbf{|V|}$ is finite, we can further assume conditional densities are non-zero (NZ(T)):\\ 
 \textbf{ Non-Zero Assumption (NZ(T)): }
for any $X\in \mathbf{V}$, there exists some $1>T>0$, s.t. $\mathbf{U\subset V}$, $p_{X|\mathbf{U}}\geq T$ for all values of $X$.

Notice that by TV Smoothness(L) and that variables have support on $[0,1]$, we can derive an upper bound for probability of any variable given its parents: 

$p_{y|{Pa(Y)=pa(Y)}} \leq ||p_{Y|Pa(Y)=pa(Y)}||_1\leq ||p_{Y|Pa(Y)=pa'(Y)}||_1 + L||pa(Y)-pa'(Y)||_1\leq (1+L|Pa(Y)|) $ 
 
Although the discrete probability case does not have support on $[0,1]$, and its probability is not continuous, it still satisfies the TV smoothness(L) assumption: for instance, if the discrete variables have support only on integers, we can set $L=1$. By replacing the density $p_{X|\mathbf{U}}$ with the probability $P(X|\mathbf{U})$ in NZ(T) assumption, we have a NZ(T) assumption for the discrete case.  Therefore the proof of uniform consistency for the nonparametric case in the rest of the paper also works for the discrete case.

Since the nonparametric version can be viewed as a generalized \textbf{$k$-Triangle-Faithfulness}, we can plug in the linear Gaussian case into the generalized triangle faithfulness and see what implied about $K_Y$ and $L$. $T$ as a lower bound of a p.d.f is only possible in the linear Gaussian case if we restrict the value of variables in an interval.  We can restrict the value of variables in an interval then rescale all the probability density with a constant to have the integral of density to be 1. In the following discussion we are not considering such restriction, since the result can be easily generalized.

For the common parameteric case with the linear Gaussian model
    $\mathbf{(I-B)^TX}=\epsilon$\footnote{It was $\mathbf{(I-A)^TX}=\epsilon$ in chapter 1, but $\mathbf{A}$ has been used to denote the set of variables being conditioned on in this chapter.}, 
it is easy to check that \textit{TV Smoothness(L)} is satisfied. For instance, given $p_{Y,\mathcal{A}}$ as Gaussian and $\Sigma$ the covariance matrix of $\mathbf{X}$, we can find sufficient conditions for the values of $L$ such as:
\begin{center}
$L\geq\dfrac{2\phi(0)||\Sigma_{Y,\mathbf{A}}\Sigma_{\mathbf{A}\mathbf{A}}^{-1}||_1}{\sqrt{var(Y|\mathbf{A})}}$
\end{center}
The detailed derivation of this condition can be found in the Appendix.
Notice that by setting a value for $L$, we directly bound $var(Y|\mathbf{A})$ away from 0, which rules out determinism.\\ \textcolor{white}{abc} 
Now we apply the nonparametric, or generalized definition of the edge strength to the linear Gaussian model.  Quantifying causal influence is a question open to discussion, and different ways of defining causal strength may have their own pros and cons\cite{Janzing_2013}.  Here we are not trying to compare which definition, the maximum $L_1$ distance change or the linear coefficient, is better, but to provide a direct demonstration of our setup in this paper.  Recall that $L_1$ distance between two Gaussian distribution with only different means is:
\begin{center}
    $||Y_1-Y_2||_1 = 2\left(2\Phi(\dfrac{|\mu_1-\mu_2|}{2\sqrt{var(Y)}})-1\right)$
\end{center}
where $Y_1\sim\mathcal{N}\left(\mu_1,var(Y)\right)$ and $Y_2\sim\mathcal{N}\left(\mu_2,var(Y)\right)$. Therefore when plugging the linear Gaussian case into the definition of edge strength of the nonparametric version, we have for every $X_i$, $X_j$ such that in a graph $G = \langle \mathbf{V}, \mathbf{E}\rangle$ such that $X_i\rightarrow X_j\in\mathbf{E}$ with the linear coefficient $a_{ij}$, the edge strength in terms of $L_1$ distance manipulating only $X_i$ is:

\begin{center}
    $2\left(2\Phi(\dfrac{a_{ij}\mathbf{\delta}{X_i}}{2\sqrt{var(X_j|\mathbf{Pa(X_j)})}})-1\right)$
\end{center}
where $\mathbf{\delta}{X_i}$ is the change of value of $X_i$ that maximizes the $L_1$ probability change of $X_j$. 
 \\ \textcolor{white}{abc}
Even though for the general nonparametric case, the change of value of $X_i$ that maximizes the $L_1$ distance varies case by case, in the linear case apparently the distance gets bigger as the change of value of the corresponding parent gets bigger.  Similarly in the \textbf{$k$-Triangle Faithfulness Assumption}, when conditioning on a set where $X_i$ is included, the $L_1$ distance gets bigger when the change of value of $X_i$ gets larger.  Therefore, instead of directly applying the \textbf{$k$-Triangle Faithfulness Assumption}, we can first fix the range of manipulation of $X_i$, then compare the resulting distribution change conditioning on a set and on the parent set.  For instance, when $\langle X_i, Z, X_j\rangle$ is a triangle, the first condition in the \textbf{$k$-Triangle Faithfulness Assumption} is:
\begin{center}
     if \textit{Z} is a non-collider on the path $\langle X_i,Z,X_j\rangle$, given any subset $\mathbf{W\subset V}\setminus \{X_i,X_j,Z\}$,\newline$\left(2\Phi(\dfrac{\Sigma_{X_j,\mathbf{W},X_i}\Sigma^{-1(i)}_{\mathbf{W},X_i}\mathbf{\delta}{X_i}}{2\sqrt{var(X_j|\mathbf{W},X_i)}})-1\right)\geq K\left(2\Phi(\dfrac{a_{ij}\mathbf{\delta}{X_i}}{2})-1\right)$ for some $K>0$
\end{center}
where $\Sigma^{-1(i)}_{\mathbf{W},X_i}$ is the $i$th column of $\Sigma^{-1}_{\mathbf{W},X_i}$.
Setting $\delta X_i = 1$, we are looking for a $K$ such that:
\begin{center}
    $\left(2\Phi(\dfrac{\Sigma_{X_j,\mathbf{W},X_i}\Sigma^{-1(i)}_{\mathbf{W},X_i}}{2\sqrt{var(X_j|\mathbf{W},X_i)}})-1\right)\geq K\left(2\Phi(\dfrac{a_{ij}}{2})-1\right)$
\end{center}
similar to the \textbf{$k$-Triangle Faithfulness Assumption} in the linear Gaussian case, the stronger the edge is ($|a_{ij}|$ is large), the stronger the correlated $X_i$ and $X_j$ should be to satisfy the assumption.

\subsection{Uniform Consistency in the inference of structure}
We use $L1$ norm to characterize dependence :$\epsilon_{X,Y|\mathcal{A}}=||p_{X,Y,\mathcal{A}}-p_{X|\mathcal{A}}p_{Y|\mathcal{A}}p_{\mathcal{A}}||_1$.We want a test $\psi$ of $H_0: \epsilon = 0$ versus $H_1: \epsilon>0 $. $\psi$ is a family of functions: $\psi_0...\psi_n...$ one for each sample size, that takes an i.i.d sample $V_n$ from the joint distribution over $\mathbf{V}$.  Then the test is uniformly consistent w.r.t. a set of distributions $\Omega$ for :
\begin{itemize}
    \item $lim_{n\rightarrow \infty} sup_{P\in\mathcal{P}_{[0,1],TV(L)},\epsilon(P)=0}P^n(\psi_n(V_n)=1)=0$
    \item for every $\delta>0$, $lim_{n\rightarrow \infty} sup_{\epsilon(P)\geq\delta}P^n(\psi_n(V_n)=0)=0$
\end{itemize}

With the TV Smoothness(L) assumption, there are uniformly consistent tests of conditional independence, such as a minimax optimal conditional independence test proposed by Neykov et al.(2020).

Given any causal model $M=\langle P,G \rangle$ over $\mathbf{V}$, let $C(n,M)$ denote the (random) output of the $VCSGS$ algorithm given an$i.i.d.$ sample of size $n$ from the distribution $P_M$, then there are three types of errors that it can contain that will mislead the estimation of causal influences:
\begin{enumerate}
    \item $C(n,M)$ \textit{errs in kind I} if it has an adjacency not in $G_M$;
    \item  $C(n,M)$ \textit{errs in kind II} if every adjacency it has is in $G_M$, but it has a marked non-collider not in $G_M$;
    \item  $C(n,M)$ \textit{errs in kind III} if every adjacency and marked non-collider it has is in $G_M$, but it has an orientation not in $G_M$
\end{enumerate}

If $C(n.M)$ errs in either of these three way, there will be variables $X$ and $Y$ in $C(n,M)$ such that $X$ is treated as a parent of $Y$ but is not in the true graph $G_M$; if there is no undirected edge connecting $Y$ in this $C(n,M)$, the algorithm will estimate the causal influence of ``parents'' of $Y$ on $Y$, but such estimation does not bear useful information since intervening on $X$ does not really influence $Y$.  Notice that missing an edge is not listed as an mistake here, and we are going to prove later that the estimation of causal influence can still be used to correctly predict the effect of intervention even if the algorithm misses edges.

Let $\phi^{k,L,T}$ be the set of causal models over \textbf{V} under \textbf{\textit{k-}Triangle-Faithfulness Assumption}, TV smoothness(L) and the assumptions of NZ(T). 

We will prove that under the causal sufficiency of the measured variables \textbf{V}, causal Markov assumption, $k$-Triangle-Faithfulness, TV smoothness(L) assumption and NZ(T) assumption,
\begin{center}
    $\underset{n\rightarrow\infty}{lim}\underset{M\in \phi^{k,L,C}}{sup}P^n_M(C(n,M) errs)=0$
\end{center}

 We begin by proving a useful lemma that bounds $\epsilon_{X,Y|\mathbf{A}}$ with strengths of the edge $X\rightarrow Y$:
\begin{lemma}

Given an ancestral set $\mathbf{A\subset V}$ that contains the parents of $Y$ but not $Y$:

If $X$ is a parent of $Y$: 

$ T^{|\mathbf{A}|} e(X\rightarrow Y)\leq\epsilon_{X,Y|\mathbf{A}\setminus\{X\}}\leq e(X\rightarrow Y)$

\end{lemma}
\begin{proof}
\begin{flalign}
\epsilon_{X,Y|\mathbf{A}\setminus\{X\}}&=\int_{\mathbf{A}\setminus\{X\}}\int_X\int_Y|p_{x,y,\mathbf{a}\setminus\{x\}}-p_{y|\mathbf{a}\setminus\{x\}}p_{x|\mathbf{a}\setminus\{x\}}p_{\mathbf{a}\setminus\{x\}}|dydxd\mathbf{a}\setminus\{x\}\\ \textcolor{white}{abc}
&= \int_{\mathbf{A}} p_{\mathbf{a}}||p_{y|x,\mathbf{A}\setminus\{X\}}-p_{y|\mathbf{A}\setminus\{X\}}||_1d\mathbf{a}\\ \textcolor{white}{abc}
&\leq \mathbb{E}_{{\mathbf{A}}\sim p_{\mathbf{A}}}[e(X\rightarrow Y)]\\ \textcolor{white}{abc}
&= e(X\rightarrow Y)
\end{flalign}
\begin{flalign}
\epsilon_{X,Y|\mathbf{A}\setminus\{X\}}&=\int_{\mathbf{A}\setminus\{X\}}\int_X\int_Y|p_{x,y|\mathbf{a}\setminus\{x\}}-\newline p_{y|\mathbf{a}\setminus\{x\}}p_{x|\mathbf{a}\setminus\{x\}}|dydxd\mathbf{a}\setminus\{x\}\\ \textcolor{white}{abc}
&= \int_{\mathbf{A}\setminus\{X\}}p_{\mathbf{a}\setminus\{x\}}\int_x p_{x|\mathbf{a}\setminus\{x\}}||p_{Y|x,\mathbf{a}\setminus\{x\}}-\newline p_{Y|\mathbf{a}\setminus\{x\}}||_1dxd\mathbf{a}\setminus\{x\}\\ \textcolor{white}{abc}
&\geq T^{|\mathbf{A}|} e(X\rightarrow Y)
\end{flalign}

The step (7) is derived using a direct conclusion from NZ(T): 

for any $\mathbf{V\supset W}=\{W_1,W_2...W_n\}$,
by the Chain Rule:
$p_{\mathbf{W}}=\prod_{i=1}^{n}p_{W_i|W_{i+1}...W_n}\geq T^n$

\end{proof}
We are going to prove for each case that the probability for $C(n,M)$ to make each of the three kinds of mistakes uniformly converges to zero.  Since the proofs for the \textit{kind I} and \textit{kind III} errors are almost the same as the proof for the linear Gaussian case (Spirtes and Zhang, 2014), we are only going to prove \textit{kind II} here.

\begin{lemma}
Given causal sufficiency of the measured variables $\mathbf{V}$, the Causal Markov,$k$-Triangle-Faithfulness, TV smoothness(L) assumption and NZ(T) assumption:
\begin{center}
    
    $\underset{n\rightarrow\infty}{lim}\underset{M\in \phi^{k,L,C}}{sup}P^n_M(C(n,M) \textit{ errs in kind II})=0$\end{center}

\end{lemma}
\begin{proof}

For any $M\in \phi^{k,L,C}$, if it errs in kind II then it contains a marked non-collider $\langle X,Z,Y\rangle$ that is not in $G_M$.  Since it's been proved  (Spirtes and Zhang, 2014):
\begin{center}
    
    $\underset{n\rightarrow\infty}{lim}\underset{M\in \phi^{k,L,C}}{sup}P^n_M(C(n,M) \textit{ errs in kind I})=0$\end{center}

the errors of kind II can be one of the two cases:

$(I)$ $\langle X,Z,Y\rangle$ is an unshielded collider in $G_M$;

$(II)$ $\langle X,Z,Y\rangle$ is a shielded collider in $G_M$;

the proof for case (I) is the same as the proof for the $C(n,M)$ errs in kind I \cite{Spirtes_2014}, so we are going to prove here that the probability of case (II) uniformly converges to zero as sample size increases.

We are going to prove this by contradiction.  Suppose that the probability that $VCSGS$ making a mistake of kind II does not uniformly converge to zero.  Then there exists $\lambda >0$, such that for every sample size $n$, there is a model $M(n)$ such that the probability of $C(n,M(n))$ containing an unshielded non-collider that is  a shielded collider in $M(n)$ is greater than $\lambda$. Let that triangle be $\langle X^{M(n)}, Z^{M(n)}, Y^{M(n)}\rangle$ with $X^{M(n)}$ being the parent of $Y^{M(n)}$ in $M(n)$.

The algorithm will identify the triple as an unshielded non-collider only if:

$(i)$ there is a set $\mathbf{U}^{M(n)}\subset \mathbf{V}^{M(n)}$ containing $Z^{M(n)}$, such that the test of $\epsilon_{X^{M(n)},Y^{M(n)}|\mathbf{U}^{M(n)}} =0$ returns 0, call this test $\psi_{n0}$; 

$(ii)$  there is an ancestral set $\mathbf{W}^{M(n)}$ that contains $X^{M(n)}$ and $Y^{M(n)}$ but not $Z^{M(n)}$, such that for set $\mathbf{A}^{M(n)} = \mathbf{W}^{M(n)}\setminus \{X^{M(n)}, Y^{M(n)}\}$, the test $\epsilon_{X^{M(n)},Y^{M(n)}|\mathbf{A}^{M(n)} }=0$ returns 1; call this test $\psi_{n1}$.

If what we want to prove does not hold for the algorithm,  for all $n$ there is a model $M(n)$:\newline
(1)$P^n_{{M(n)}}(\psi_{n0}=0)>\lambda$\newline
(2)$P^n_{{M(n)}}(\psi_{n1}=1)>\lambda$\newline

(1) tells us that there is some $\delta_n$ such that $|\epsilon_{X^{M(n)},Y^{M(n)}|\mathbf{U}^{M(n)}} |<\delta_n$ and $\delta_n\rightarrow 0$ as $n\rightarrow \infty$ since the test is uniformly consistent.  So we have:
\begin{flalign}
\delta_n&>\epsilon_{X^{M(n)},Y^{M(n)}|\mathbf{U}^{M(n)}  }\\ \textcolor{white}{abc}
        &=\mathbb{E}_{\mathbf{U}^{M(n)}}\sim p_{\mathbf{U}^{M(n)}}[\int_{X^{M(n)}} p_{x^{M(n)}|\mathbf{U}^{M(n)}}||p_{Y^{M(n)}|x^{M(n)},\mathbf{U}^{M(n)}}-p_{Y^{M(n)}|\mathbf{U}^{M(n)}}||_1dx^{M(n)}]\\ \textcolor{white}{abc}
        &\geq \mathbb{E}_{\mathbf{U}^{M(n)}} \sim p_{\mathbf{U}^{M(n)}}\max_{x_3^{M(n)},x_4^{M(n)}\in [X^{M(n)}]} T||p_{Y^{M(n)}|x_3^{M(n)},\mathbf{U}^{M(n)}}-p_{Y^{M(n)}|x_4^{M(n)},\mathbf{U}^{M(n)}}||_1\\ \textcolor{white}{abc}
        &\geq  T^{|\mathbf{U}|+1}K_{Y^{M(n)}} e_M(X^{M(n)}\rightarrow Y^{M(n)})\newline 
\end{flalign}
The last step is by $k$-Triangle-Faithfulness and NZ(T).

So  $e_M(X^{M(n)}\rightarrow Y^{M(n)}) <\dfrac{\delta_n}{T^{|\mathbf{U}|+1}K_{Y^{M(n)}}}$.

By Lemma 2, $\epsilon_{X^{M(n)},Y^{M(n)}|\mathbf{A}^{M(n)}}<e_M(X^{M(n)}\rightarrow Y^{M(n)})$.  

Therefore, $\epsilon_{X^{M(n)},Y^{M(n)}|\mathbf{A}^{M(n)}}<\dfrac{\delta_n}{T^{|\mathbf{U}|+1}K_{Y^{M(n)}}}\rightarrow0$ as $n\rightarrow \infty$, which violates the condition $(ii)$, which says that the uniformly consistency test will reject that $\epsilon_{X^{M(n)},Y^{M(n)}|\mathbf{A}^{M(n)} }=0$. Contradiction.
\end{proof}
\begin{theorem}
Given causal sufficiency of the measured variables $\mathbf{V}$, the Causal Markov, $k$-Triangle-Faithfulness, TV smoothness(L) and NZ(T) assumptions:
\begin{center}
    
    $\underset{n\rightarrow\infty}{lim}\underset{M\in \phi^{k,L,C}}{sup}P^n_M(C(n,M) \textit{ errs })=0$\end{center}

\end{theorem}
\begin{proof}
Since we have proved that the probability for $C(n,M)$ to make any of the three kinds of mistakes uniformly converges to 0, the theorem directly follows.
\end{proof}

\subsection{Uniform consistency in the inference of causal effects}
\subsubsection{Edge Estimation Algorithm}
In this section we are going to prove the uniform consistency in the causal effect inference.  Notice that since this proof is built based on the uniform consistency of the inference of the causal structure, which has been proved to be satisfied for any version of  
\textbf{\textit{k-}Triangle-Faithfulness Assumption}, we are not going to distinguish versions of \textbf{\textit{k-}Triangle-Faithfulness Assumption} in this section.\\ \textcolor{white}{abc}
\begin{algorithm}
\caption{Edge Estimation Algorithm}\label{alg:eea}
\KwIn{Data from the set of variables $\mathbf{V}$; Graph $\mathcal{H}$ from $VCSGS$;}
\KwOut{$\hat P(Y|Pa(Y))$ where $Pa(Y)$ are the set of parents of $Y$ according to $\mathcal{H}$;}
\For{$Y\in\mathbf{V}$}{\eIf{$\forall X-Y $ in $\mathcal{H}$, $X-Y$ is oriented}{$\hat p(y_i|{Pa(Y)=pa(Y)})\gets$\footnote{we denote the density of $p_{Y|Pa(Y)=pa(Y)}$ at $Y=y$ as $p(y|{Pa(Y)=pa(Y)})$ in this section to match with the result of estimation.}the histogram estimation for  $y_i\in[Y]$ and ${pa(Y)\in [Pa(Y)]}$}{Mark $\hat P(Y|Pa(Y))$ as \textit{Unknown}}}
\For{estimated $\hat P(Y|Pa(Y))$}{\If{$\hat P(Y|Pa(Y))$ violates TV smoothness(L)}{Mark $\hat P(Y|Pa(Y))$ as \textit{Unknown}}}
\textbf{Return} $\hat P(Y|Pa(Y))$ where $Pa(Y)$ are the set of parents of $Y$ according to $\mathcal{H}$
\end{algorithm}
\subsubsection{Defining the distance between $M_1$ and $M_2$}

The method for estimation for $p(y|{Pa(Y)=pa(Y)})$ is:
we first get $\hat p(Y=y,{Pa(Y)=pa(Y)})$ and  $\hat p({Pa(Y)=pa(Y)})$ by histogram, then we get:

\begin{center}
     $\hat p(y|{Pa(Y)=pa(Y)})=\dfrac{\hat p(Y=y,{Pa(Y)=pa(Y)})}{\hat p({Pa(Y)=pa(Y)})}$
\end{center}

Let $M_1$ be the output of the Edge Estimation Algorithm\ref{alg:eea}, and $M_2$ be a causal model, we define the \textit{conditional probability distance}, $d[M_1,M_2]$, between $M_1$ and $M_2$ to be:
\begin{center}
$d[M_1,M_2]=\underset{\substack{{Y\in\mathbf{V}},\\ \textcolor{white}{abc}{y_i\in[Y]},\\ \textcolor{white}{abc}{pa_{M_1}(Y)}\\ \textcolor{white}{abc} {\in [Pa_{M_1}(Y)]},\\ \textcolor{white}{abc}{pa_{M_2}(Y)}\\ \textcolor{white}{abc}{\in [Pa_{M_2}(Y)]},\\ \textcolor{white}{abc}{pa_{M_1}\subset pa_{M_2}}}}{max}|\hat p_{M_1}(y_i|pa_{M_1}(Y))-p_{M_2}(y_i|pa_{M_2}(Y))|$
\end{center}where $Pa_{M}(Y)$ denotes the parent set of $Y$ in causal model $M$. By convention $|\hat P_{M_1}(y_i|pa_{M_1}(Y))-P_{M_2}(y_i|pa_{M_2}(Y))|=0$ if $\hat P_{M_1}(y_i|pa_{M_1}(Y))$ is ``Unknown".\\ \textcolor{white}{abc}
Now we want to show, the edge estimation algorithm is uniformly consistent in the sense that for every $\delta >0$,
\begin{center}
    $\underset{n\rightarrow\infty}{lim}$ $\underset{M\in \phi^{k,L,C}} {sup}$ $P^n_M(d[\hat O(M),M]>\delta)=0$
\end{center}
Here $M$ is any causal model satisfying causal sufficiency of the measured variables $\mathbf{V}$, the Causal Markov, $k$-Triangle-Faithfulness, TV smoothness(L) and NZ(T) assumptions and $\hat O(M)$ is the output of the algorithm given an iid sample from $P_M$.

\begin{proof}
Let $\mathcal{O}$ be the set of possible graphs of $VCSGS$.
Since given $\mathbf{V}$, there are only finitely many outputs in $\mathcal{O}$, it suffices to prove that for each output $O\in \mathcal{O}$,
\begin{center}
    
    $\underset{n\rightarrow\infty}{lim}\underset{M\in \phi^{k,L,C}}{sup}P^n_M(d[\hat O(M),M]>\delta|C(n,M)=O)P_M^n(C(n,M)=O)=0$
 
\end{center}
Now partition all the $M$ into three sets given $O$ :\newline
\begin{itemize}
    \item $\Psi_1=\{M|$ all adjacencies, non-adjacencies and orientations in O are true in$ M\}$;
    \item $\Psi_2=\{M|$  only some adjacencies, or orientations in O are not true in $M\}$;
    \item $\Psi_3=\{M|$ only some non-adjacencies in O are not true in $M\}$.
\end{itemize}
It suffices to show that for each $\Psi_i$,
\begin{center}
    $\underset{n\rightarrow\infty}{lim}$ $\underset{M\in \Psi_i} {sup}$ $P^n_M(d[\hat O(M),M]>\delta|C(n,M)=O)P_M^n(C(n,M)=O)=0$
\end{center}
$\Psi_1$:

For any $M\in \Psi_1$, if the conditional probabilities of a vertex $Y$  in $O$ can be estimated (so not ``Unknown"), it means that $Pa_O(Y)=Pa_M(Y)$.  Recall that the histogram estimator is close to the true density with high probability:
\begin{center}
for any $\lambda<1$, $sup_{P\in \mathcal{P}_{TV(L)}}P^n(||\hat p_h(x)-p(x)||_{\infty}\leq f(n,\lambda)\leq O((\dfrac{\log n}{n})^{\frac{1}{2+d}}))\geq 1-\lambda$
\end{center}
where $f(n,\lambda)$ is continuous and monotonically decreasing with respect to $n$ and $\epsilon$ and $h\propto n^{2/ (2+d)}$ (the number of bins) where $d$ is the dimentionality of the $x$.  For instance, $d=|Pa(Y)|+1$ when estimating $P(Y,Pa(Y))$.

Given a $\delta>0$, $f(n,\epsilon)=\delta$ entails that $sup_{P\in \mathcal{P}_{[0,1],TV(L)}}P^n(||\hat p_h(x)-p(x)||_{\infty}>\delta)<\lambda$. 

By monotonicity of $f(n,\lambda)$, when $n>n_f $ s.t. $f(n_f,\lambda)=\delta$, $sup_{P\in \mathcal{P}_{[0,1],TV(L)}}P^n(||\hat p_h(x)-p(x)||_{\infty}>\delta)<\lambda$.

Therefore the histogram estimators of $p_M(y,Pa_M(Y)=pa_M(Y))$  and $p_M(Pa_M(Y)=pa_M(Y))$ are uniformly consistent.  Next we are going to use the lemma below: 
\begin{lemma}
 If $\hat p(Y=y,{Pa(Y)=pa(Y)})$and $\hat p({Pa(Y)=pa(Y)})$ are uniformly consistent estimators of  $ p(Y=y,{Pa(Y)=pa(Y)})$and  $ p({Pa(Y)=pa(Y)})$,
then
\begin{center}
    $\hat p(y|{Pa(Y)=pa(Y)})=\dfrac{\hat p(Y=y,{Pa(Y)=pa(Y)})}{\hat p({Pa(Y)=pa(Y)})}$
\end{center}is a uniformly consistent estimator for $p(y|{Pa(Y)=pa(Y)})$.

\end{lemma}

\begin{proof}
Recall that:
\begin{center}
for any $\lambda<1$, $sup_{P\in \mathcal{P}_{TV(L)}}P^n(||\hat p_h(x)-p(x)||_{\infty}\leq f(n,\lambda)= O((\dfrac{\log n}{n})^{\frac{1}{2+d}}))\geq 1-\lambda$ $(*)$
\end{center}where $f(n,\lambda)$ is continuous and monotonically decreasing wrt $n$ and $\lambda$ and $h\propto n^{2/ (2+d)}$ (the number of bins) where $d$ is the dimentionality of the $x$.  For instance, $d=|Pa(Y)|+1$ when estimating $P(Y,Pa(Y))$.\\ \textcolor{white}{abc}
For any $\delta>0$, $f(n,\lambda)=\delta$ entails that $sup_{P\in \mathcal{P}_{[0,1],TV(L)}}P^n(||\hat p_h(x)-p(x)||_{\infty}>\delta)<\lambda$. \\ \textcolor{white}{abc}
By monotonicity of $f(n,\lambda)$, when $n>n_f $ s.t. $f(n_f,\lambda)\leq\delta$, $sup_{P\in \mathcal{P}_{[0,1],TV(L)}}P^n(||\hat p_h(x)-p(x)||_{\infty}>\delta)<\lambda$.\\ \textcolor{white}{abc}
Let $d = |Pa(Y)|$, notice that for any $\lambda$, with probability at least $1-\lambda$,\footnote{Here we use $\hat p (x)$ instead of $\hat p_h (x)$ because $h$ is dependent on the dimension of $x$.}\begin{flalign}
&|\frac{ \hat p(Y=y,{Pa(Y)=pa(Y)})}{ \hat p({Pa(Y)=pa(Y)})}-p(y|{Pa(Y)=pa(Y)})|\\ \textcolor{white}{abc}
&=\frac{| \hat p(Y=y,{Pa(Y)=pa(Y)})-p(y|{Pa(Y)=pa(Y)})\hat p({Pa(Y)=pa(Y)})|}{\hat p({Pa(Y)=pa(Y)})}\\ \textcolor{white}{abc}
&\leq \dfrac{1}{T^{d}}|O((\dfrac{\log n}{n})^{\frac{1}{3+d}})+O((\dfrac{\log n}{n})^{\frac{1}{2+d}})|\\ \textcolor{white}{abc}
&= O((\dfrac{\log n}{n})^{\frac{1}{3+d}})
\end{flalign}
Step (14) is derived according to $(*)$, the fact that the estimation result can only be valid if it satisfies TV smoothness(L))\footnote{Recall that $\mathbf{V}$ denotes the set of variables in the true graph.} and $p(y|{Pa(Y)=pa(Y)})$ is upper bounded by $(1+L|Pa(Y)|) $ by TV smoothness.\\ \textcolor{white}{abc}
We have:
\begin{center}
    $\underset{P\in \mathcal{P}_{TV(L)}}{sup}P^n(|\dfrac{ \hat p(Y=y,{Pa(Y)=pa(Y)})}{ \hat p({Pa(Y)=pa(Y)})}-p(y|{Pa(Y)=pa(Y)})|\leq O((\dfrac{\log n}{n})^{\frac{1}{3+d}}))\geq 1-\lambda$
\end{center}
which leads to the conclusion that $\hat p(y|{Pa(Y)=pa(Y)})$ is a uniform consistent estimator for $p(y|{Pa(Y)=pa(Y)})$.
\end{proof}
By lemma 3.4, we conclude that the $\hat p_M(y|Pa_M(Y)=pa_M(Y))$ is a uniformly consistent estimator for
\begin{center}
$p_M(y|Pa_M(Y)=pa_M(Y))=\dfrac{p_M(y,Pa_M(Y)=pa_M(Y))}{p_M(Pa_M(Y)=pa_M(Y))}$,  
\end{center}
So:

    $\underset{n\rightarrow\infty}{lim}\underset{M\in \Psi_1} {sup}P^n_M(d[\hat O(M),M]>\delta|C(n,M)=O)P_M^n(C(n,M)=O)\newline\leq\underset{n\rightarrow\infty}{lim}\underset{M\in \Psi_1} {sup}P^n_M(d[\hat O(M),M]>\delta|C(n,M)=O)=0$\\ \textcolor{white}{abc}
$\Psi_2:$ the proof is exactly the same as for the linear Gaussian case\cite{Spirtes_2014}.\newline
$\Psi_3:$Let $O(M)$ be the population version of $\hat O(M)$.  Since the histogram estimator is uniformly consistent over $||\hat p_h-p||_\infty$ and there are finitely many parent-child combinations, for every $\lambda >0$ there is a sample size $N_1$, such that for $n>N_1$, and all $M\in \Psi_3$,
\begin{center}
     $P^n_M(d[\hat O(M),O(M)]>\delta/2|C(n,M)=O)<\lambda$
\end{center}

Since only some non-adjacencies in $O$ are not true in $M$, we know that for any vertex $Y$ that has some estimated conditional probabilities given its parents in $O$, $Pa_{O(M)}(Y)\subset Pa_M(Y)$ where $Pa_{O(M)}(Y)$ denotes the parent set of $Y$ in the $O$ when the underlying probability is $P_M$(i.e., $M$ is the true causal model).  Since $Pa_M(Y)\not\subset Pa_{O(M)}(Y)$, for any $y_i\in [Y]$ and $pa_{O(M)}(Y)\in [Pa_{O(M)}(Y)]$, $P_O(y_i|Pa_{O(M)}(Y)=pa_{O(M)}(Y))$ is a marginalization of $p_M(y_i|Pa_M(Y)=pa_M(Y))$. Therefore, the distance between $O(M)$ and $M$ is:
\begin{center}
    
    $d[O(M),M]=\underset{\substack{{Y\in\mathbf{V}},\\ \textcolor{white}{abc}{y_i\in[Y]},\\ \textcolor{white}{abc} {pa_{O(M)}(Y)\in}\\ \textcolor{white}{abc}{[Pa_{O(M)}(Y)]},\\ \textcolor{white}{abc} {pa_{M}(Y)\in}\\ \textcolor{white}{abc}{ [Pa_{M}(Y)]},\\ \textcolor{white}{abc} {pa_{O(M)}\subset pa_{M}}}}{max}| p_{O(M)}(y_i|pa_{O(M)}(Y))-p_{M}(y_i|pa_{M}(Y))|$\end{center}

Given the $Y$ corresponding to the equation above, let $Pa_M(Y)=\{A_1...A_g..A_{g+h}\}$ and $Pa_{O(M)}(Y)=\{A_1...A_g\}$.
Since $P_{O(M)}(y_i|pa_{O(M)}(Y))$ is the marginalization of all $P_M(y_i|pa_M(Y))$, we have:
\begin{flalign}
|p_{O(M)}(y_i|pa_{O(M)}(Y))-p_{M}(y_i|pa_{M}(Y))|&\leq\underset{\substack{pa_{M}(Y)_1,\\ \textcolor{white}{abc}pa_{M}(Y)_2 \\ \textcolor{white}{abc} \in
    [Pa_{M}(Y)], \\ \textcolor{white}{abc}s.t. {pa_{O(M)}\subset} \\ \textcolor{white}{abc}{ pa_{M}(Y)_1\cap pa_{M}(Y)_2}}}{max}| p_{M}(y_i|pa_{M}(Y)_{1})-p_{M}(y_i|pa_{M}(Y)_{2})|\\ \textcolor{white}{abc}
    &<\sum_{j=1}^{j=h}e_M(A_{g+j}\rightarrow Y)
\end{flalign}

If $ \sum_{j=1}^{j=h}e_M(A_{g+j}\rightarrow Y)< \delta/2$, then we have:
\begin{center}
    $d[O(M),M]<\delta/2$
\end{center}
For all such $M$, there is a $N_1$, such that for any $n>N_1$:
\begin{flalign}
    P^n_M(d[\hat O(M),M]>\delta|C(n,M)=O)&\leq P^n_M(d[\hat O(M),O(M)]+d[ O(M),M]>\delta|C(n,M)=O)\\ \textcolor{white}{abc}
    &\leq P^n_M(d[\hat O(M),O(M)]>\delta/2|C(n,M)=O)<\epsilon
\end{flalign}

If $\sum_{j=1}^{j=h}e_M(A_{g+j}\rightarrow Y)\geq \delta/2$, then there is at least an $w\in \{1,2...h\}$, $s,t.$ $e_M(A_{g+w}\rightarrow Y)>\dfrac{\delta}{2h}$. By Lemma 3.1:
\begin{center}
    $\epsilon_{A_{g+w},Y|\mathbf{U} }\geq T^{|\mathbf{U}|+1}e_M(A_{g+w}\rightarrow Y)>T^{|\mathbf{U}|+1}\dfrac{\delta}{2h}$. 
\end{center}
where the $\mathbf{U}\cup \{A_{g+w},Y\}$ is some ancestral set not containing any descendant of $Y$. 

Since the density estimation does not turn ``unknown", we know that in step 5 of $VCSGS$ the test of $\epsilon_{A_{g+w},Y|U}=0$ returns 0 while $\epsilon_{A_{g+w},Y|U}\geq T^{|\mathbf{U}|+1}\dfrac{\delta}{2h}$.  Since the test is uniformly consistent, it follows that there is a sample size $N_2$ such that
\begin{center}
$P^{N_2}_M(\epsilon_{A_{g+w},Y|U}=0)<\lambda$
\end{center}
 for any $n>N_2$, and therefore for all such M,
\begin{center}
    $P^n_M(d[\hat O(M),M]>\delta|C(n,M)=O)<\lambda$
\end{center}
Let $N = max(N_1, N_2)$, for $n>N$,
 \begin{flalign}
 &\underset{n\rightarrow\infty}{lim}\underset{M\in \Psi_3} {sup}P^n_M(d[\hat O(M),M]>\delta|C(n,M)=O)P_M^n(C(n,M)=O)\\ \textcolor{white}{abc}
 &\leq \underset{n\rightarrow\infty}{lim}\underset{M\in \Psi_3} {sup} P^n_M(d[\hat O(M),M]>\delta|C(n,M)=O)=0 
 \end{flalign}
\end{proof}
\section{Upper and Lower Bound of the probability of violation of \textit{k-}Triangle-Faithfulness }
In this section we investigate more mathematical properties of the \textbf{\textit{k-}Triangle-Faithfulness Assumption}.  Specifically, we want to assess the relative weakness of \textbf{\textit{k-}Triangle-Faithfulness Assumption} compared to the \textbf{Strong Causal Faithfulness Assumption}.  
In order to assess how weak the \textbf{\textit{k-}Triangle-Faithfulness Assumption} is compared to the \textbf{Strong Causal Faithfulness Assumption}, which bounds the absolute value of any nonzero partial correlation in the linear Gaussian case, we analyze the \textbf{\textit{k-}Triangle-Faithfulness Assumption} in the linear Gaussian case.  Uhler et al. have provided the upper bounds of the probability of violation of \textbf{Strong Faithfulness Assumption} and lower bounds in certain types of graphs\cite{Uhler_2013}.  Similarly, we are going to bound the probability of the violation of \textbf{\textit{k-}Triangle-Faithfulness Assumption} and compare them with the strong faithfulness.

\subsection{Upper bound of the probability of violation of Strong Faithfulness}

We first revisit what the \textbf{Strong Faithfulness Assumption} is and the proof of the upper bound of the probability of its being violated:\\ 
\textbf{Strong Faithfulness Assumption:} Given $\lambda \in (0,1)$, a multivariate Gaussian distribution $\mathbb{P}$ is said to be \textit{$\lambda$-strong-faithful} to a DAG $G=\langle \mathbf{V,E}\rangle$ if for any $i,j\in V$ and for any $\mathbf{S}\in V\setminus \{i,j\}$:
\begin{center}
    $min\{|\rho(X_i, X_j|\mathbf{X_S})|$ $j$ is not $d$-separated from $i|\mathbf{S}$, $\forall \mathbf{S}, i, j\} > \lambda$
\end{center}
Adopting the aforementioned linear SEM 
\begin{center}
    $\mathbf{(I-A)}^TX=\epsilon$ 
\end{center}
and assuming that the strength of edges in a DAG $G$ uniformly distributes between $[-1, 1]$, we have $(a_{ij})_{(i,j)\in\mathbf{E}}\in [-1, 1]^{|\mathbf{E}|}$ (without standardizing).  We first denote the set of parameters that leads to correlations with absolute value smaller than $\lambda$: 
\begin{center}
    $\mathcal{P}^\lambda_{ij|\mathbf{S}} := \{(a_{u,v})\in [-1, 1]^{|\mathbf{E}|}||cov(X_i, X_j|X_{\mathbf{S}})|\leq\lambda\sqrt{var(X_i|X_{\mathbf{S}})var(X_j|X_{\mathbf{S})}}\}$
\end{center} 

Then we denote the distribution over that DAG that is $\lambda$-strong-unfaithful as:

\begin{center}
    $\mathcal{M}_{G,\lambda} := \underset{ \substack{
 i,j\in \mathbf{V}, \mathbf{S}\in V\setminus \{i,j\}\\ \textcolor{white}{abc} \text{$j$ is not $d$-separated from $i|\mathbf{S}$}}}{\bigcup}\mathcal{P}^\lambda_{ij|\mathbf{S}}$
\end{center}

The upper bound of the probability of the violation of \textbf{Strong Faithfulness Assumption} is equivalent to the getting an upper bound of $\dfrac{vol(\mathcal{M}_{G,\lambda})}{2^{|\mathbf{E}|}}$, which uses Crofton's formula to get an upper bound on the surface area of a real algebraic hypersurface defined by a degree $d$ polynomial and Lojasiewicz inequality to get an upper bound for the distance between a point in a compact area and the nearest zero of the given analytic function:\\
CROFTON'S FORMULA. \textit{The volume of a degree $d$ real algebraic hypersurface in the unit m-ball is bounded above by $C(m)d$, where $C(m)$ satisfies }
\begin{center}
    ${m+d \choose d}-1 \leq C(m)d^m$
\end{center}
LOJASIEWICZ INEQUALITY. \textit{Let $f:\mathbb{R}^p \rightarrow \mathbb{R}$ be a real-analytic function and $K\subset\mathbb{R}^p$ compact.  Let $V_f\subset \mathbb{R}^p$ denote the real zero locus of f, which is assumed to be nonempty.  Then there exist positive constants c,q such that for all $x\in K$:}
\begin{center}
    $dist(x,V_f)\leq c|f(x)|^q$
\end{center}
Now we revisit the theorem for the upper bound of $\dfrac{vol(\mathcal{M}_{G,\lambda})}{2^{|\mathbf{E}|}}$  and the proof of it from Ulher et. al \cite{Uhler_2013}.

\begin{theorem}\label{thm:upperboundsf}
Let $G = \langle \mathbf{V, E} \rangle$ be a DAG on p nodes. Then\newline
\begin{center}
    $\dfrac{vol(\mathcal{M}_{G,\lambda})}{2^{|\mathbf{E}|}}\leq \dfrac{C(|\mathbf{E}|)c\kappa^q\lambda^q}{2^{|\mathbf{E}|/2}}\underset{i,j\in\mathbf{V}}{\sum}\underset{\mathbf{ S\subset V}\setminus\{i,j\}}{\sum}deg(cov(X_i,X_j|X_\mathbf{S}))$
\end{center}
where $C(|\mathbf{E}|)$ is a positive constant coming from Crofton's formula, c, q are positive constants, depending on the polynomials characterizing exact unfaithfulness, and $\kappa$ denotes the maximal partial variance over all possible parameter values $(a_{st}) \in [-1, 1]^{|\mathbf{E}|}$, that is,
\begin{center}
    $\kappa = \underset{i,j\in \mathbf{V}, \mathbf{S}\in V\setminus \{i,j\}}{max}\underset{(a_{st}) \in [-1, 1]^{|\mathbf{E}|}}{max} var(X_i|X_{\mathbf{S}})$
\end{center}
\end{theorem}
\begin{proof}
By the definition of $\kappa$:
\begin{center}
    $vol(\mathcal{P}^\lambda_{ij|\mathbf{S}}) = vol(\{(a_{st})\in [-1, 1]^{|\mathbf{E}|}||cov(X_i, X_j|X_{\mathbf{S}})|\leq\lambda\kappa\})$
\end{center}
    By LOJASIEWICZ INEQUALITY:
\begin{center}
    $vol(\mathcal{P}^\lambda_{ij|\mathbf{S}})= vol(\{(a_{st})\in [-1,1]^{|\mathbf{E}|}|dist((a_{st}), V_{ij|\mathbf{S}})\leq c_{ij|\mathbf{S}}\lambda^{q_{ij|\mathbf{S}}}\kappa^{q_{ij|\mathbf{S}}}\})$
\end{center}
where the $V_{ij|\mathbf{S}}$ denotes the real algebraic hypersurface that corresponds to zero covariance; $c_{ij|\mathbf{S}}$, $q_{ij|\mathbf{S}}$ are positive constants from LOJASIEWICZ INEQUALITY. Applying Crofton's formula on an $|\mathbf{E}|$-dimensional ball of radius $\sqrt{2}$ to get an upper bound on the surface area of the hypersurface in the hypercube $[-1,1]^{|\mathbf{E}|}$ that vanish on $cov(X_i, X_j|X_{\mathbf{S}})$:
\begin{center}
     $vol(\mathcal{P}^\lambda_{ij|\mathbf{S}})\leq c_{ij|\mathbf{S}}\lambda^{q_{ij|\mathbf{S}}}\kappa^{q_{ij|\mathbf{S}}}2^{\mathbf{E}/2}C(\mathbf{|E|})deg(cov(X_i, X_j|X_{\mathbf{S}}))$
\end{center}
Recall that:
\begin{center}
    $\mathcal{M}_{G,\lambda} := \underset{ \substack{
 i,j\in \mathbf{V}, \mathbf{S}\in \mathbf{V}\setminus \{i,j\}\\ \textcolor{white}{abc} \text{$j$ is not $d$-separated from $i|\mathbf{S}$}}}{\bigcup}\mathcal{P}^\lambda_{ij|\mathbf{S}}$
\end{center}
By union bound, we get:
\begin{center}
    $vol(\mathcal{M}_{G,\lambda})\leq \underset{ \substack{
 i,j\in \mathbf{V}, \mathbf{S}\in \mathbf{V}\setminus \{i,j\}\\ \textcolor{white}{abc} \text{$j$ is not $d$-separated from $i|\mathbf{S}$}}}{\sum}vol(\mathcal{P}^\lambda_{ij|\mathbf{S}})$
\end{center}
Setting $c = \underset{i,j\in \mathbf{V}, \mathbf{S}\in \mathbf{V}\setminus \{i,j\}}{max}c_{ij|\mathbf{S}}$  and  $q =  \underset{i,j\in \mathbf{V}, \mathbf{S}\in \mathbf{V}\setminus \{i,j\}}{max}q_{ij|\mathbf{S}}$, the theorem is obtained.
\end{proof}

\subsection{Upper bound of the probability of violation of \textit{k-}Triangle-Faithfulness}
We now compute the upper bound of \textbf{\textit{k-}Triangle-Faithfulness}.  Similar to the work for \textbf{Strong Faithfulness}, we assume that the strength of edges in a DAG $G$ uniformly distributes between $[-1, 1]$, we have $(a_{ij})_{(i,j)\in\mathbf{E}}\in [-1, 1]^{|\mathbf{E}|}$.  Notice that by assuming that the edge coefficients uniformly distributing across $[-1, 1]^{|\mathbf{E}|}$ we are not covering the set of all correlation matrix.\footnote{The same is true for strong faithfulness.} Such assumption is made for the sake of simpler computation and more straightforward comparison with the probability of violation of \textbf{Strong Faithfulness}.\\ \textcolor{white}{abc} 
We first denote the set of parameters that leads to correlations with absolute value smaller than $|Ke(X_i\rightarrow X_j)|$: 
\begin{center}
    $\mathcal{P}^K_{ij|\mathbf{S}} := \{(a_{u,v})\in [-1, 1]^{|\mathbf{E}|}||cov(X_i, X_j|X_{\mathbf{S}})|\leq K|e(X_i\rightarrow X_j)|\sqrt{var(X_i|X_{\mathbf{S}})var(X_j|X_{\mathbf{S})}}\}$
\end{center} Then we denote the distribution over that DAG that is $k$-Triangle-unfaithful as:

\begin{center}
    $\mathcal{Q}_{G,K} := \underset{ \substack{
 i,j,h\in \mathbf{V}\\ \textcolor{white}{abc} i\rightarrow j \in \mathbf{E} \\ \textcolor{white}{abc} i,j,h \text{ form a triangle in $G$}\\ \textcolor{white}{abc} \mathbf{S}\subset \mathbf{V}\setminus \{i,j\}\\ \textcolor{white}{abc} \text{ $h\in\mathbf{S}$ if and only if $h$ is the collider}}}{\bigcup}\mathcal{P}^K_{ij|\mathbf{S}}$
\end{center}

\begin{theorem}\label{thm:ratiosktf}
Let $G = \langle \mathbf{V, E} \rangle$ be a DAG on p nodes. Then\newline
\begin{center}
    $\dfrac{vol(\mathcal{Q}_{G,K})}{vol(\mathcal{M}_{G,\lambda})}\leq\dfrac{K^q\mathbf{|V|}}{\lambda^q(q_m+1)(|\mathbf{V}|+\dfrac{1}{(2^{n_{max}}-1)})}$
\end{center}
where $n_{max} $ is the maximum neighborhood size\footnote{A variable $Y$ is a neighbor of a variable $X$ in graph $G$, there is an edge between $X$ and $Y$ in $G$; neighborhood size of $X$ is the total number of neighbors of $X$ in $G$.} of the graph $G$, q and $q_m$ are positive constants.
\end{theorem}
\begin{proof}
By the definition of $\kappa$:
\begin{center}
    $vol(\mathcal{P}^K_{ij|\mathbf{S}}) = vol(\{(a_{st})\in [-1, 1]^{|\mathbf{E}|}||cov(X_i, X_j|X_{\mathbf{S}})|\leq K|e(X_i\rightarrow X_j)|\kappa\})$
\end{center}
    By LOJASIEWICZ INEQUALITY:
\begin{center}
    $vol(\mathcal{P}^K_{ij|\mathbf{S}})= vol(\{(a_{st})\in [-1,1]^{|\mathbf{E}|}|dist((a_{st}), V_{ij|\mathbf{S}})\leq c_{ij|\mathbf{S}}K^{q_{ij|\mathbf{S}}}|e(X_i\rightarrow X_j)|^{q_{ij|\mathbf{S}}}\kappa^{q_{ij|\mathbf{S}}}\})$
\end{center}
Applying Crofton's formula on an $|\mathbf{E}|$-dimensional ball of radius $\sqrt{2}$ to get an upper bound on the surface area of the hypersurface in the hypercube $[-1,1]^{|\mathbf{E}|}$ that vanish on $cov(X_i, X_j|X_{\mathbf{S}})$:
\begin{flalign}
       vol(\mathcal{P}^K_{ij|\mathbf{S}})&\leq \mathcal{C}^K_{ij|\mathbf{S}}\int_0^1 |e(X_i\rightarrow X_j)|^{q_{ij|\mathbf{S}}}\,d|e(X_i\rightarrow X_j)|\\ \textcolor{white}{abc}
    &\leq\dfrac{\mathcal{C}^K_{ij|\mathbf{S}} }{q_{ij|\mathbf{S}}+1}
\end{flalign}
\begin{center}
    where $\mathcal{C}^K_{ij|\mathbf{S}}:=c_{ij|\mathbf{S}}K^{q_{ij|\mathbf{S}}}\kappa^{q_{ij|\mathbf{S}}}2^{\mathbf{E}/2}C(\mathbf{|E|})deg(cov(X_i, X_j|X_{\mathbf{S}}))$.
\end{center}  
Recall that:
\begin{center}
  $\mathcal{Q}_{G,K} := \underset{ \substack{
 i,j,h\in \mathbf{V}\\ \textcolor{white}{abc}  i,j,h \text{ form a triangle in $G$}\\ \textcolor{white}{abc} \mathbf{S_{\Delta}}\subset \mathbf{V}\setminus \{i,j\}}}{\bigcup}\mathcal{P}^K_{ij|\mathbf{S_{\Delta}}}$
 
where for each triangle$\langle i, j, h\rangle$, $h\in \mathbf{S_{\Delta}}$ if and only if $h$ is the collider.\end{center}
By union bound, we get:
\begin{center}
    $vol(\mathcal{Q}_{G,K})\leq \underset{ \substack{\substack{
 i,j,h\in \mathbf{V}\\ \textcolor{white}{abc}  i,j,h \text{ form a triangle in $G$}\\ \textcolor{white}{abc} \mathbf{S_{\Delta}}\in \mathbf{V}\setminus \{i,j\}}}}{\sum}vol(\mathcal{P}^K_{ij|\mathbf{S_{\Delta}}})$
\end{center}
Setting $c = \underset{i,j\in \mathbf{V}, \mathbf{S}\in \mathbf{V}\setminus \{i,j\}}{max}c_{ij|\mathbf{S}}$, $q =  \underset{i,j\in \mathbf{V}, \mathbf{S}\in \mathbf{V}\setminus \{i,j\}}{max}q_{ij|\mathbf{S}}$ and $q_m =  \underset{i,j\in \mathbf{V}, \mathbf{S}\in \mathbf{V}\setminus \{i,j\}}{min}q_{ij|\mathbf{S}}$ , we have:
\begin{center}
$\dfrac{vol(\mathcal{Q}_{G,K})}{2^{|\mathbf{E}|}}\leq \dfrac{C(|\mathbf{E}|)c\kappa^qK^q}{2^{|\mathbf{E}|/2}(q_m+1)}\underset{i\rightarrow j\in\mathbf{E}}{\sum}\underset{\mathbf{ S_{\Delta}\subset V}\setminus\{i,j\}}{\sum}deg(cov(X_i,X_j|X_{S_{\Delta}}))$
\end{center}
By \textbf{Theorem \ref{thm:upperboundsf}}, we get:
\begin{flalign}
    \dfrac{vol(\mathcal{Q}_{G,K})}{vol(\mathcal{M}_{G,\lambda})}&\leq \dfrac{K^q}{\lambda^q(q_m+1)} \dfrac{\underset{i\rightarrow j\in\mathbf{E}}{\sum}\underset{\mathbf{ S_{\Delta}\subset V}\setminus\{i,j\}}{\sum}deg(cov(X_i,X_j|X_{S_{\Delta}}))}{\underset{i,j\in\mathbf{V}}{\sum}\underset{\mathbf{ S\subset V}\setminus\{i,j\}}{\sum}deg(cov(X_i,X_j|X_S))}\\ \textcolor{white}{abc}
    &\leq \dfrac{K^q}{\lambda^q(q_m+1)} \dfrac{\underset{i\rightarrow j\in\mathbf{E}}{\sum}\underset{\mathbf{ S_{\Delta}\subset V}\setminus\{i,j\}}{\sum}deg(cov(X_i,X_j|X_{S_{\Delta}}))}{\underset{i\rightarrow j\in\mathbf{E}}{\sum}\underset{\mathbf{ S\subset V}\setminus\{i,j\}}{\sum}deg(cov(X_i,X_j|X_S))}    
\end{flalign}

We are using the lemma below to proceed. The proof of this lemma can be found in the appendix:
\begin{lemma}\label{lemma:deg_cov}
    Following the definition of $i,j$ and $\mathbf{S_\Delta}$, we have:
    \begin{center}
        $deg(cov(X_i,X_j|X_{S_{\Delta}}))\leq 2(|\mathbf{V}|-|\mathbf{S_\Delta}|)$
    \end{center}
    where $\mathbf{Q} =\{i,j\}\bigcup\mathbf{S_\Delta}$.
\end{lemma}
Notice that if $i\rightarrow j\in\mathbf{E}$ then their conditional covariance on any subset of $\mathbf{V}$ cannot have a degree of zero.  Furthermore, if $i\rightarrow j$ is in $n_{ij}$ many triangles, denoted by $\langle i, j, h_i\rangle$ for $i\in\{1,...n\}$, then there are two sets:
\begin{itemize}
    \item $\mathbf{H^0_\Delta}\subset \mathbf{V}$ such that $\mathbf{H^0_\Delta} = \{h_i|h_i \text{ is not a collider in} \langle i, j, h_i\rangle\}$
    \item $\mathbf{H^1_\Delta}\subset \mathbf{V}$ such that $\mathbf{H^1_\Delta} = \{h_i|h_i \text{ is a collider in} \langle i, j, h_i\rangle\}$
\end{itemize}
such that given any $i\rightarrow j$ that is in a triangle, $\mathbf{S_\Delta}$ should not be a superset of $\mathbf{H^0_\Delta}$ while having an empty intersection with $\mathbf{H^1_\Delta}$ which suggests that there should be at most $2^{|\mathbf{V}|-2}-2^{|\mathbf{V}|-n_{ij}-2}$ many $\mathbf{S_\Delta}$.\\ \textcolor{white}{abc}
Now we can continue deriving an upper bound:

%\begin{flalign}
    %\dfrac{\mathcal{Q}_{G,K}}{\mathcal{M}_{G,\lambda}}&\leq \dfrac{K^q}{\lambda^q(q_m+1)} \dfrac{\underset{i\rightarrow j\in\mathbf{E}}{\sum}\underset{\mathbf{ S_{\Delta}\subset V}\setminus\{i,j\}}{\sum}deg(cov(X_i,X_j|X_{S_{\Delta}})}{\underset{i\rightarrow j\in\mathbf{E}}{\sum}\underset{\mathbf{ S\subset V}\setminus\{i,j\}}{\sum}deg(cov(X_i,X_j|X_S)}\\ \textcolor{white}{abc}
    %&\leq\dfrac{K^q}{\lambda^q(q_m+1)} \dfrac{\underset{i\rightarrow j\in\mathbf{E}}{\sum}2^{|\mathbf{V}|-2}-2^{|\mathbf{V}|-n_{ij}-2}\mathbf{|V|}}{1+\underset{i\rightarrow j\in\mathbf{E}}{\sum}2^{\mathbf{|V|}-n_{ij}-2}\dfrac{1}{|\mathbf{V}|}}\\ \textcolor{white}{abc}
    %&\leq\dfrac{(1-2^{-n_{ij}})K^q|\mathbf{|V|}}{\lambda^q(q_m+1)}\\ \textcolor{white}{abc}
    %&\leq\dfrac{(1-2^{-n_{max}})K^q|\mathbf{|V|}}{\lambda^q(q_m+1)}
%\end{flalign}

\begin{flalign}
    \dfrac{vol(\mathcal{Q}_{G,K})}{vol(\mathcal{M}_{G,\lambda})}&\leq \dfrac{K^q}{\lambda^q(q_m+1)} \dfrac{\underset{i\rightarrow j\in\mathbf{E}}{\sum}\underset{\mathbf{ S_{\Delta}\subset V}\setminus\{i,j\}}{\sum}deg(cov(X_i,X_j|X_{S_{\Delta}}))}{\underset{i\rightarrow j\in\mathbf{E}}{\sum}\underset{\mathbf{ S\subset V}\setminus\{i,j\}}{\sum}deg(cov(X_i,X_j|X_S))}\\ \textcolor{white}{abc}
    &\leq \dfrac{K^q}{\lambda^q(q_m+1)} \dfrac{\underset{i\rightarrow j\in\mathbf{E}}{\sum}\underset{\mathbf{ S_{\Delta}\subset V}\setminus\{i,j\}}{\sum}deg(cov(X_i,X_j|X_{S_{\Delta}}))}{\underset{i\rightarrow j\in\mathbf{E}}{\sum}\underset{\mathbf{ S_\Delta\subset V}\setminus\{i,j\}}{\sum}deg(cov(X_i,X_j|X_S))+\underset{i\rightarrow j\in\mathbf{E}}{\sum}2^{|\mathbf{V}|-n_{ij}-2}}\\ \textcolor{white}{abc}    
    &\leq\dfrac{K^q}{\lambda^q(q_m+1)} \dfrac{1}{1+\dfrac{1}{(2^{n_{max}}-1)|\mathbf{V}|}}\\ \textcolor{white}{abc}
    &\leq\dfrac{K^q|\mathbf{|V|}}{\lambda^q(q_m+1)(|\mathbf{V}|+\dfrac{1}{(2^{n_{max}}-1)})}
\end{flalign}
\end{proof}
\textbf{Theorem \ref{thm:ratiosktf}} suggests that the \textbf{$k$-Triangle-Faithfulness} is more likely to be satisfied when there are fewer variables in the DAG (smaller $|\mathbf{V}|$) or the DAG is sparse (smaller $n_{max}$).  Such implication aligns with the logic that, when there are no triangles in a graph, the \textbf{$k$-Triangle-Faithfulness} is vacuously satisfied, which is one aspect that makes the \textbf{$k$-Triangle-Faithfulness} weaker.  Another aspect is that, for every $\lambda$ as the lower bound of nonzero conditional covariance, there is a $k$ that gives a lower bound of nonzero conditional covariances smaller than $\lambda$.    We are going to see in the simulation section about how large the $k$ can be to have a faithfulness assumption easier to be satisfied than a strong faithfulness with a small $\lambda$. 

\subsection{Lower bound of the probability of violation of \textit{k-}Triangle-Faithfulness}
In this section we are providing a lower bound of the $k$-triangle-unfaithful distribution in the linear Gaussian case in a parameter cube $[-1, 1]^{|\mathbf{E}|}$ for DAGs where each edge is in at most one triangle.  The lower bounds of the $\lambda$-strong-unfaithful distribution is provided for three special types of DAGs: tree, single cycle and bipartite \cite{Uhler_2013}.  
\begin{figure}[h]
    \centering
    \includegraphics[scale=0.5]{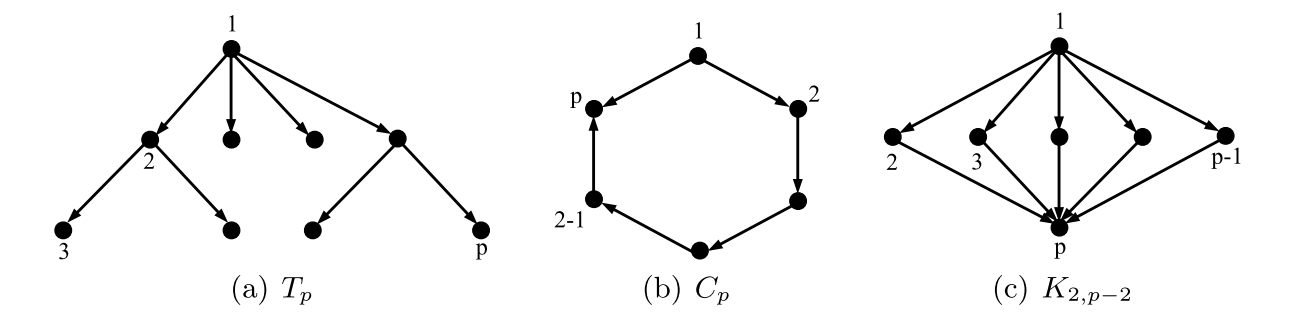}
    \caption{\textit{tree $T_p$, single cycle $C_p$ and bipartite $K_{2,p-2}$ with $p$ many nodes \cite{Uhler_2013}}}
    \label{fig:threetype}
\end{figure}
As shown in Figure \ref{fig:threetype}, only DAGs with the skeleton as a cycle can have triangles.  Therefore, we cannot directly compare the lower bound of $k$-triangle-unfaithful distribution with $\lambda$-strong-unfaithful distribution.  Instead, we are going to apply Uhler's result on a DAG with three nodes and a cycle as the skeleton to get a lower bound for a triangle DAG, then derive a lower bound for a DAG as a function of $|\mathbf{V}|$ and graph density $d = \frac{\mathbf{|E|}}{\mathbf{|V|(|V|-1)}}$ for a more general type of DAG.
\begin{theorem}\label{thm:3nodes}
    Let $C_p$ be a directed cycle with 3 nodes.  Then:
    \begin{center}
        $\dfrac{vol(\mathcal{M}_{C_p,\lambda})}{2^{|\mathbf{E}|}}\geq 1-(1-\lambda)^{2p-1}$
    \end{center}
    where $\mathbf{E}$ is the set of edges of $C_p$.
\end{theorem}
Uhler et al. provided \textbf{theorem \ref{thm:3nodes}} as a lower bound of a \textit{restricted version} of $\lambda$-strong-unfaithful distribution for $C_p$, but a cycle with only 3 nodes is also a complete graph, in which case the general \textbf{$\lambda$-strong Faithfulness} is equivalent to the \textbf{ \textit{restricted} $\lambda$-strong Faithfulness}\cite{Uhler_2013}. 

\begin{lemma}\label{lemma:triangle}
    Let $G_{\Delta}$ be a triangle.  Then:
    \begin{center}
        $\dfrac{vol(\mathcal{Q}_{G,K})}{2^{|\mathbf{E}|}}\geq 1-(1-\dfrac{K}{2})^4$
    \end{center}
    where $\mathbf{E}$ is the set of edges of $G_{\Delta}$.
\end{lemma}This lemma is almost a direct application of \textbf{theorem \ref{thm:3nodes}}, with slight adjustments according to the restriction that the only variable in the triangle that can be conditioned on should be the collider. \\ \textcolor{white}{abc} 
Uhler et al. only discussed the lower bounds for the three types of DAGs because there are too many cases of the conditional covariance 
$cov(X_i, X_j|X_S)$ to be discussed once the $i$ and $j$ are connected by more than one \textit{self-avoiding} path\cite{Uhler_2013}.  A path is \textit{self-avoiding} if no node is visited more than once. Similarly, it's generally hard to derive a lower bound of $k$-triangle unfaithful distribution for all DAGs.  Luckily, \textbf{lemma \ref{lemma:triangle}} can be used not only in $G_{\Delta}$, but also in any triangle as a subgraph as long as all triangles in the graph are not sharing edges:

\begin{theorem}
    \label{lemma:probtrainglevio}
    Let $G_{\{\Delta\}} = \langle \mathbf{V, E}\rangle$ be a DAG where each edge is in at most one triangle and density $d$.  Then,
    \begin{center}
$\mathbb{P}\left((a_{st})_{(s,j)\in\mathbf{E}}\in\mathcal{Q}_{G,K}\right)\geq (1-(1-d^3)^{\lfloor\frac{|\mathbf{V}|}{3}\rfloor}) (1-(1-\dfrac{K}{2})^4)$    
    \end{center}
\end{theorem}
The detailed proof of this theorem can be found in the appendix. This lower bound of the probability of the violation of \textbf{$k$-Triangle Faithfulness} consists of two parts:
\begin{enumerate}

    \item the probability of $G_{\Delta}$ having triangles:\newline $\mathbb{P}(\exists  i,j,h\in\mathbf{|V|}s.t.i\rightarrow j, i\rightarrow h, j\rightarrow h\in\mathbf{|E|} )\geq(1-(1-d^3)^{\lfloor\frac{|\mathbf{V}|}{3}\rfloor})$ 
    \item the probability that a triangle in $G_{\{\Delta\}}$ violates the \textbf{$k$-Triangle Faithfulness}:\newline
$\mathbb{P}\left((a_{st})_{(s,j)\in\mathbf{E}}\in\mathcal{Q}_{G,K}|\exists  i,j,h\in\mathbf{|V|}s.t.i\rightarrow j, i\rightarrow h, j\rightarrow h\in\mathbf{|E|}\right)\geq1-(1-\dfrac{K}{2})^4$
\end{enumerate}
Similar to theorem \ref{thm:ratiosktf}, theorem \ref{lemma:probtrainglevio} shows that the probability of the violation of $k$-triangle faithfulness converges to one as $\mathbf{V}$ increases and $K$ decreases ($k$-triangle faithfulness getting more strict).  Furthermore, as the graph gets larger and lager, the rate of convergence to one is more and more dominated by the size and density of the graph.  We are going to verify such observation in our simulation.

\section{Simulation Results}
In this section, we are presenting simulation results to visualize how much weaker \textbf{$k$-Triangle-Faithfulness} compared to \textbf{Strong Faithfulness Assumption}.\\ \textcolor{white}{abc}  
Similar to the simulation in Uhler et al.\cite{Uhler_2013}, we generated random DAGs given a expected neighborhood size and edge weights following uniform distribution in $[-1,1]$.  We generated graph with the 3,5 and 10 many nodes and with 5-20 different expected neighborhood sizes.  For each total number of nodes and expected neighborhood size, we generated 10,000 different graphs.\\ \textcolor{white}{abc}
In order to compare the strength of \textbf{$k$-Triangle-Faithfulness} compared to \textbf{Strong Faithfulness Assumption}, for each simulation we counted the violation for both of these two faithfulness assumptions. For the \textbf{$k$-Triangle-Faithfulness}, we first counted how many triangles were there in the graph.  If the graph contained no triangle then the assumption would be vacuously satisfied; then in every triangle we check for each pair of variables whether there were partial correlations between them that violate the assumption with a given $K$.  For the \textbf{Strong Faithfulness Assumption}, similar to Uhler et al.'s procedure, we assumed any partial correlation smaller than $10^{-5}$ were actual zeros due to the round-off errors and counted the number of simulations where there were partial correlations between $10^{-5}$ and a given $\lambda$.\\ \textcolor{white}{abc}
We considered three values for $\lambda$ and $K$: 0.1, 0.01 and 0.001.  The resulting plots are given in Figure \ref{fig:three comparison}.
\begin{figure}[H]
\begin{tabular}{ccc}
  \includegraphics[width=48mm]{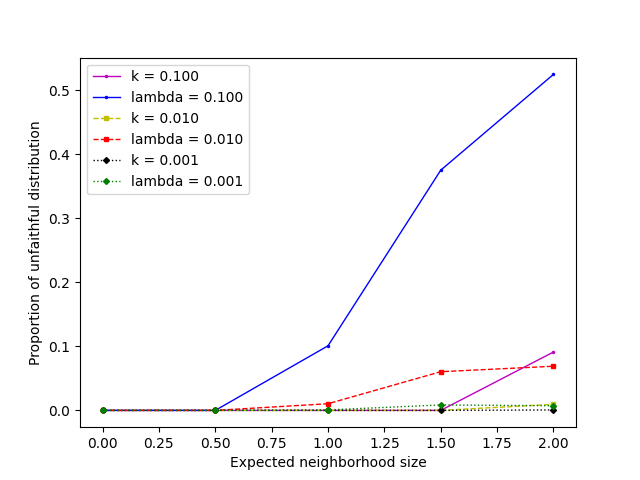} &\includegraphics[width=48mm]{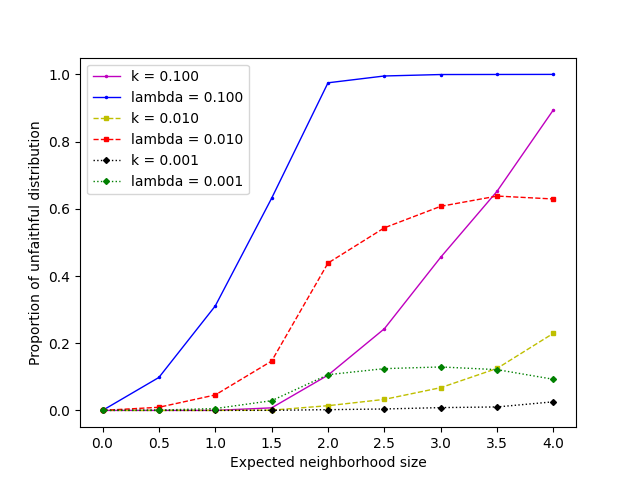}&\includegraphics[width=48mm]{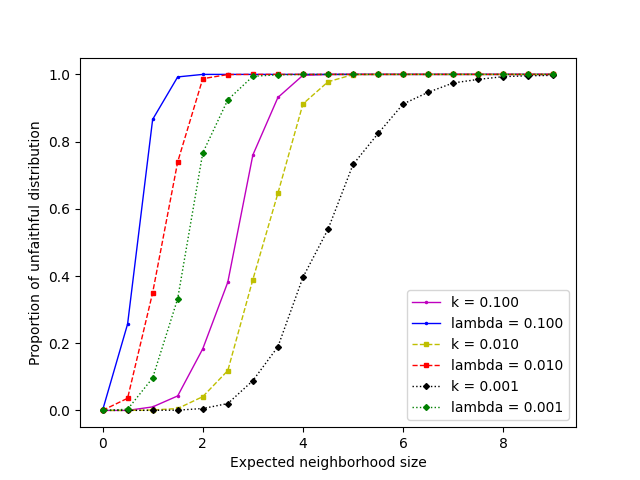}\\ \textcolor{white}{abc}
(a) 3-nodes DAGs & (b) 5-nodes DAGs & (c) 10-nodes DAGs\\ \textcolor{white}{abc}
\end{tabular}
\caption{\textit{Proportion of $\lambda$-strong-unfaithful distributions and $k$-triangle-unfaithful distributions for 3 values of $\lambda$ and $k$}}
\label{fig:three comparison}
\end{figure}In order to see how strong the \textbf{$k$-Triangle-Faithfulness} Assumption is, we further simulated 10,000 graphs with 10 \textit{nodes} and an \textit{expected neighborhood size} from 5 to 9, or density from 0.5 to 1, and looked through all the triangles to see the maximum value of $k$ that can be required to have all the simulations satisfy the assumption. The average $k$ is $1.65\times 10^{-7}$. Fig \ref{fig:max_k} shows that the change of required $k$ value as we increase the density to 1, indicating that once the graph is so dense that each edge would participate in more than one triangle, the assumption is extremely hard to be satisfied.

\begin{figure}[h]
    \centering
    \includegraphics[scale = 0.5 ]{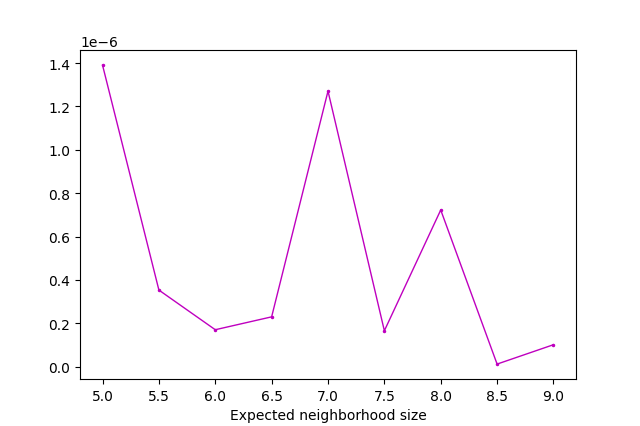}
    \caption{\textit{The maximum value of $k$ for the 10-nodes DAGs with 5 to 9 expected neighborhood size to fully satisfy the \textbf{$k$-Triangle-Faithfulness Assumption}}.}
    \label{fig:max_k}
\end{figure}
The $k$-Triangle-Faithfulness sets lower bounds of nonzero correlation proportional to the edge strength of the edge between the two variables, which is expected to have the property that the weaker the direct causal influence is, the closer to zero the correlations can be without violating the Faithfulness. In order to visualize the relation between the strength of the edge (the scale of the absolute value of the edge coefficient) and the probability of the violation of \textbf{$k$-Triangle-Faithfulness} Assumption is, we further simulated 1000 graphs with 8 \textit{nodes} and an \textit{expected neighborhood size} from 2 to 6 with $k=0.1$, and looked at the edge strength and the proportion of the violation of \textbf{$k$-Triangle-Faithfulness} that happens under such edge strength. Fig \ref{fig:es_vio} shows that as the causal influence is getting stronger, the \textbf{$k$-Triangle-Faithfulness} becomes more and more likely to be violated then becomes less likely when the absolute value of edge strength exceeds 0.5.  Furthermore, as the density of the graph increases, the violation of \textbf{$k$-Triangle-Faithfulness} becomes more uniform across different edge strengths, agreeing with the derivation of the lower bound of the probability of the violation of \textbf{$k$-Triangle-Faithfulness}, which indicates that graph density dominantly influences the probability of violation.

\begin{figure}[H]
    \centering
    \includegraphics[scale = 0.35 ]{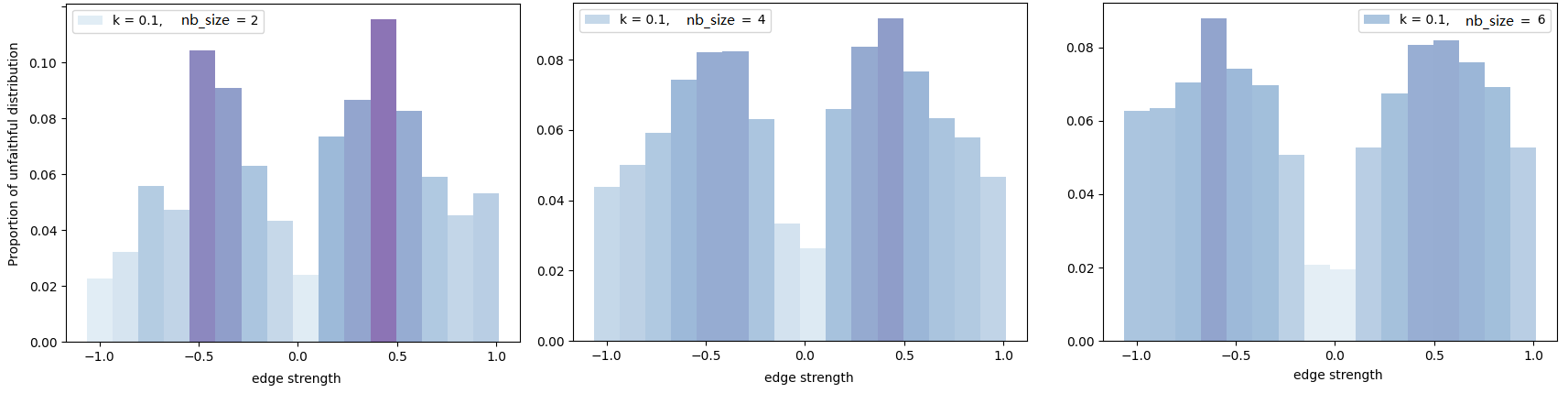}
    \caption{\textit{Proportion of $k$-triangle-unfaithful distributions for different average neighborhood size for graphs with 8 vertices and $k=0.1$ \textbf{nb\textunderscore size} stands for \textit{average neighborhood size}.}.}
    \label{fig:es_vio}
\end{figure}

\section{Discussion}

We have shown that there is a uniformly consistent estimator of causal structure and some causal effects for nonparametric distributions under the \textbf{\textit{k-}Triangle-Faithfulness Assumption}, which is sometimes stronger than the Faithfulness Assumption and weaker than the Strong Faithfulness Assumption.  We further provided the upper and lower bound of the violation of \textbf{\textit{k-}Triangle-Faithfulness Assumption} in the linear Gaussian case.\\ \textcolor{white}{abc}  
There are a number of open questions, such as whether the causal sufficiency assumption can be relaxed, so we allow the existence of latent variables and whether there are similar results under assumptions weaker than the Causal Faithfulness Assumption, such as the Sparsest Markov Representation Assumption \cite{solus2021consistency}.
\newpage

\chapter{Latent Causal Structure Learning}

\section{Introduction}
Causal mechanisms can be modeled from observed data and background assumptions. One way of representing the causal model is by directed graphs. A major challenge of finding causal models from data is that there might be confounders (direct causes of more than two variables present in the dataset) that are not measured in the dataset: sometimes one uses the causal model to predict the effect of changing states of a measured variable $A$ on $B$, in which case the confouders of $A$ and $B$ should be taken into account to estimate the effect; sometimes the unmeasured variables participate in the causal relation of interest, in which case it is important to detect the existence of the unmeasured and study the causal connection between the unmeasured.\\ \textcolor{white}{abc}
 In many scientific contexts, variables of interest that are not directly measured are called ``latent'' variables. In such cases, researchers often measure ``indicators" of the variables instead. For instance, researchers in psychometrics measure reaction times to study intelligence \cite{DEARY2001389}; personalities are identified by asking behavioral questions  \cite{personality}. In these cases, causal relations among the latent variables are of primary interest. For instance, WAIS (Wechsler Adult Intelligence Scale)-IV standardization data is a dataset with sample size 1800 that measures various cognitive behaviors from three group of people divided by ages: group 1 between 20 and 54 years of age (n = 1000), group 2 between 16 and 19 (n = 400), and group 3 between 55 and 69 (n = 400) and fig \ref{fig:WAIS} are two latent models as results from factor analysis on the covariance of the measured behaviors \cite{WAIS-IVorigin}  \cite{WAIS-IV}. Notice that their models do not contain causal connections between the latent factors. In some other cases, the measured variables are of primary interest, but there are unmeasured variables that influence the measured variables, and these unmeasured common causes will either lead to wrong conclusions about causal relations or measurement error.  One typical case is the Simpson's Paradox \cite{simpson}. The kidney stone case is an example of the Simpson's paradox, in which two treatments A and B are used to treat large and small kidneys stones, with large stone being harder to treat. As shown in fig \ref{fig:kidney}, while A is more successful than B in both sizes, when combining the large and small sizes together and calculate the overall success rate, B has a success rate higher than A. It is because the confounder, that A is more likely to be used to treat the large kidney stone, is not considered \cite{kidney}.
\begin{figure}
    \centering
    \includegraphics[width=0.8\textwidth]{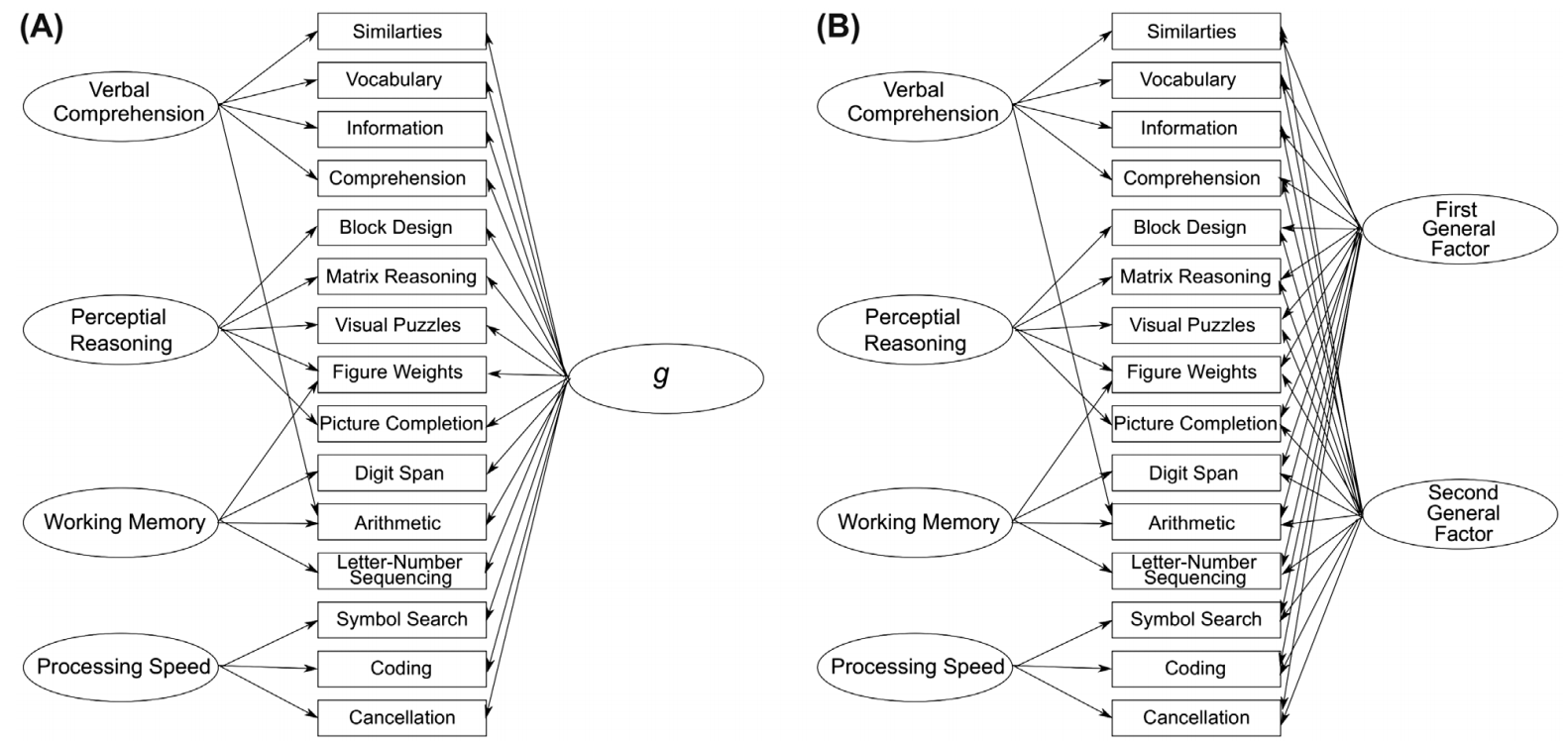}
    \caption{(A) A bi-factor model of the WAIS-IV. (B) Corresponding model with two general factors\cite{WAIS-IV}}
    \label{fig:WAIS}
\end{figure}
\begin{figure}
    \centering
    \includegraphics[width=0.8\textwidth]{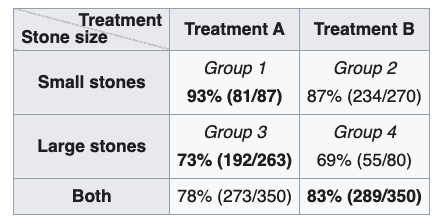}
    \caption{Treatment A is better in treating kidneys in both sizes, but when combining the two sizes together it has a lower successful rate}
    \label{fig:kidney}
\end{figure}\\ \textcolor{white}{abc}
Methods to study latent variables have been developing for decades. Recent developments, such as rank constraint, higher order moments (tensor constraint) and Generalized independence Noise (GIN) Condition, can be used to model latent variables with more details than classical methods such as factor analysis, but they all have limitations.  All these methods assume the linearity of the model over directed acyclic graphs.  Assuming non-Gaussianity of the noise distribution, the tensor constraint and GIN reveals information that cannot be uncovered by the rank constraint, but each of these methods so far has only been used alone under the linearity and acyclicity.  This chapter shows that rank constraint and GIN can be used to learn latent causal structures while relaxing the assumptions of linearity and acyclicity.  In section 3.2, I will introduce the various methods for causal modeling with latent variables, including the rank constraint; in section 3.3, I will introduce the the recent development of latent causal structures using non-Gaussianity and such methods can still cluster measured variables with the same latents with nonlinearity between latents; in section 3.4, I will present the main work of this chapter: I will first show how the rank constraint identifies latents with cycles under the latent then I will show that under the non-Gaussian distribution how combining GIN and rank constraint identifies cyclic structures.  For the rest of the thesis, we are going to refer to latent variables with $L$, and $L$ only means latent variables.

\section{Various Approaches to Causal Inference in the Presence of Latent Variables}

In this section I will introduce some classical approaches to causal inference with latent variables. Notice that many methods assume linear models, which requires variables to be continuous. If the measured variables are categorical which is common in psychometrics, some of the methods, such as the factor analysis and FOFC, can still be performed: regarding the discrete variables as discretized from multivariate normal distributions, certain rank constraints still hold for their covariance matrix \cite{cox}\cite{binaryTetrad}\cite{wang2020causal}. 
\subsection{Factor Analysis} 
    The factor analytic model is a kind of structural equation method that models the relation between the measured and latent variables. It assumes linear models, and studies the relation between the measured and latent variables by patterns of covariance matrix of the measured variables \cite{FALV}. This method estimates the number of latents in a heuristic way; it can only identify structures up to an orthogonal transformation, which means that given any square matrix $Q$ such that $QQ^T=I$, if the relation between the measured random vector $\mathbf{X}$ and the latent variable vector $\mathbf{U}$ can be written as a linear equation $\mathbf{X=AU}$, then the coefficient matrix $\mathbf{A'=AQ}$ and $\mathbf{U'=Q^TU}$ give the same result.  Furthermore, it cannot handle the case where some measured variables are causing latents and cannot identify the direct causal relations between measured variables, which is an important to study the latents since the direct causal relation between measured variables can mislead the process of estimating the structure of latents. Glymour (1998) argued that factor analysis and regression used to measure intelligence and its influence on people's behavior then were neither reliable nor uniquely correct \cite{Glymour1998-GLYWWW}.
\subsection{Structural Equation Modeling (SEM)}
As introduced in Chapter 1, Structral Equation Model (SEM) represents causal connections with equations.  It involves a variety of ways to model information about some phenomenons by equations with parameters which will be estimated \cite{SEM}. Construction of SEM is often heuristic: people need background knowledge to start constructing the model and certain assumptions are needed, such as linearity, to form the equation and estimate parameters.  
\subsection{Extension of LiNGAM with Overcomplete ICA:}
    Independent Component Analysis (ICA) assumes that the data is generated by:
    \begin{center}
        $\mathbf{x=Ae}$
    \end{center}
    where $\mathbf{x\in}\mathbb{R}^p$, $\mathbf{e\in}\mathbb{R}^d$, $\mathbf{A\in}\mathbb{R}^{p\times d}$.  Elements in $\mathbf{e}$ are assumed to be independent from each other and at most one of them is Gaussian \cite{ding2019likelihoodfreeICA}. Given $\mathbf{x}$ as the observed data, the purpose of ICA is to estimate $\mathbf{A}$ and $\mathbf{e}$.  When $d=p$, this process is called Complete ICA; when $d>p$, it is Overcomplete ICA;  when $d<p$, it is Undercomplete ICA.\\ \textcolor{white}{abc}
LiNGAM (short for Linear Non-Gaussian Acyclic Model) is an algorithm that finds causal orders by finding a permutation matrix $\mathbf{P}$ that makes $\mathbf{I-\Tilde{A}^{-1}}$ lower triangular, with $\mathbf{\Tilde{A}}$ being the estimate of $\mathbf{A}$ from processing the data with Complete ICA \cite{JMLR:v7:shimizu06a}.  This algorithm can also be used to find causal structures with latent variables, using Overcomplete ICA instead \cite{HOYER2008362}.  However, this method usually makes rather strong parametric assumptions on the distribution of independent components, which may be violated on real data, leading to sub-optimal or even wrong solutions \cite{ding2019likelihoodfreeICA}.\\ \textcolor{white}{abc}  
Furthermore, notice that ICA assume all elements in $\mathbf{e}$ are independent.  In the latent variable case, it assumes that the latents are independent from each other, so LiNGAM with the Overcomplete ICA cannot be used to recover connections between the latents.  In many studies the main interest is rather to learn how different latents causally influence each other, in which case the Overcomplete ICA is not helpful.
    
    \subsection{Autoencoder}
    An autoencoder is a type of neural network that learns a hidden layer $\mathbf{h}$ from the input data.  The network has two parts, the encoder $\mathbf{h=f(x)}$ and decoder $\mathbf{o=g(h)}$, and the purpose is to learn $f$ and $g$ to minimize the distance between $\mathbf{x}$ and $\mathbf{o}$ \cite{Goodfellow-et-al-2016}.
    When the number of nodes in the code layer is smaller than the input layer like fig \ref{fig:autoencoder}, it is called an Undercomplete Autoencoder.\\ \textcolor{white}{abc}
\begin{figure}     
\begin{center}
\begin{tikzpicture}[scale=0.15]
\tikzstyle{every node}+=[inner sep=0pt]
\draw [black] (14,-46.8) circle (3);
\draw (14,-46.8) node {$X_4$};
\draw [black] (14,-37.8) circle (3);
\draw (14,-37.8) node {$X_3$};
\draw [black] (14,-15.6) circle (3);
\draw (14,-15.6) node {$X_1$};
\draw [black] (61.7,-15.6) circle (3);
\draw (61.7,-15.6) node {$O_1$};
\draw [black] (61.7,-25.6) circle (3);
\draw (61.7,-25.6) node {$O_2$};
\draw [black] (14,-26.5) circle (3);
\draw (14,-26.5) node {$X_2$};
\draw [black] (61.7,-35.8) circle (3);
\draw (61.7,-35.8) node {$O_3$};
\draw [black] (61.7,-46.8) circle (3);
\draw (61.7,-46.8) node {$O_4$};
\draw [black] (37.1,-18.9) circle (3);
\draw (37.1,-18.9) node {$h_1$};
\draw [black] (38.3,-28.5) circle (3);
\draw (38.3,-28.5) node {$h_2$};
\draw [black] (38.3,-37.8) circle (3);
\draw (38.3,-37.8) node {$h_3$};
\draw [black] (16.97,-16.02) -- (34.13,-18.48);
\fill [black] (34.13,-18.48) -- (33.41,-17.87) -- (33.27,-18.86);
\draw [black] (16.65,-17.01) -- (35.65,-27.09);
\fill [black] (35.65,-27.09) -- (35.18,-26.28) -- (34.71,-27.16);
\draw [black] (16.85,-25.56) -- (34.25,-19.84);
\fill [black] (34.25,-19.84) -- (33.33,-19.61) -- (33.65,-20.56);
\draw [black] (16.32,-35.9) -- (34.78,-20.8);
\fill [black] (34.78,-20.8) -- (33.84,-20.92) -- (34.48,-21.69);
\draw [black] (16.21,-17.62) -- (36.09,-35.78);
\fill [black] (36.09,-35.78) -- (35.83,-34.87) -- (35.16,-35.61);
\draw [black] (16.8,-36.73) -- (35.5,-29.57);
\fill [black] (35.5,-29.57) -- (34.57,-29.39) -- (34.93,-30.33);
\draw [black] (17,-37.8) -- (35.3,-37.8);
\fill [black] (35.3,-37.8) -- (34.5,-37.3) -- (34.5,-38.3);
\draw [black] (16.4,-45) -- (35.9,-30.3);
\fill [black] (35.9,-30.3) -- (34.96,-30.39) -- (35.57,-31.19);
\draw [black] (40.96,-36.41) -- (59.04,-26.99);
\fill [black] (59.04,-26.99) -- (58.1,-26.91) -- (58.56,-27.8);
\draw [black] (40.93,-27.05) -- (59.07,-17.05);
\fill [black] (59.07,-17.05) -- (58.13,-17) -- (58.61,-17.87);
\draw [black] (40.07,-18.5) -- (58.73,-16);
\fill [black] (58.73,-16) -- (57.87,-15.61) -- (58,-16.6);
\draw [black] (39.99,-19.69) -- (58.81,-24.81);
\fill [black] (58.81,-24.81) -- (58.16,-24.12) -- (57.9,-25.08);
\draw [black] (41.16,-29.39) -- (58.84,-34.91);
\fill [black] (58.84,-34.91) -- (58.22,-34.19) -- (57.92,-35.15);
\draw [black] (40.66,-30.35) -- (59.34,-44.95);
\fill [black] (59.34,-44.95) -- (59.01,-44.07) -- (58.4,-44.85);
\draw [black] (41.29,-37.54) -- (58.71,-36.06);
\fill [black] (58.71,-36.06) -- (57.87,-35.63) -- (57.96,-36.62);
\draw [black] (39.57,-20.6) -- (59.23,-34.1);
\fill [black] (59.23,-34.1) -- (58.85,-33.24) -- (58.28,-34.06);
\draw [black] (16.81,-45.76) -- (35.49,-38.84);
\fill [black] (35.49,-38.84) -- (34.56,-38.65) -- (34.91,-39.59);
\draw [black] (41.1,-38.88) -- (58.9,-45.72);
\fill [black] (58.9,-45.72) -- (58.33,-44.97) -- (57.97,-45.9);
\end{tikzpicture}
\end{center}
        \caption{An Example of Undercomplete Autoencoder:the left side is the encoder, mapping the input data to the code layer $\mathbf{h}$, the right side is the decoder, outputting regenerated data}
        \label{fig:autoencoder}
\end{figure}
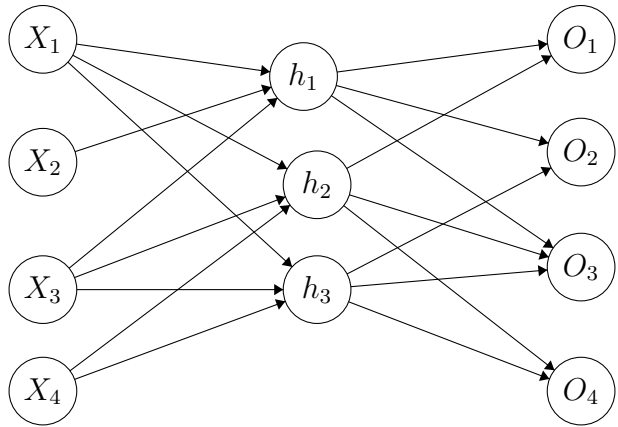Apparently the code layer and its connection from the input layer represents information about the latent variables. Here $f$ and $g$ do not have to be linear, which enables the autoencoder to learn more general cases.  But when the encoder and decoder have too few constraint, the autoencoder can generate data that are close to the original input without learning useful information about the latents \cite{Goodfellow-et-al-2016}.
    
   \subsection{FCI}
    Some causal discovery algorithms, such as Fast Causal Inference (FCI), are able to identify the existence of latent variables under certain assumptions. Roughly speaking, FCI starts with a complete undirected graph in which every variable in the dataset is a vertex; based on the causal Markov and Faithful assumptions, which will be explained in the next section, it then uses conditional independence tests to detect causal relations and latent variables.  For instance, the causal graph fig \ref{fig:FCIex}  can be identified by FCI by first identifying $Y$ and $W$ as colliders, then finding the existence of the latent variable from the double headed edge $Y\leftrightarrow W$.\\ \textcolor{white}{abc} 
    The output of FCI is a Markov equivalence class, which is a set if causal graphs that entail the same set of conditional independence relations, instead of a single causal graph, and there could be many graphs in one equivalence class. While FCI can sometimes detect the existence of latent variables, in other cases it cannot determine whether there is a latent variable or not. For example, whenever $X$ is a (possibly indirect) cause of $Y$, FCI cannot tell whether there is also a latent confounder of $X$ and $Y$ or not. FCI cannot identify the number of latents that a group of measured variables, are sharing and its performance is not good in simulations.
    
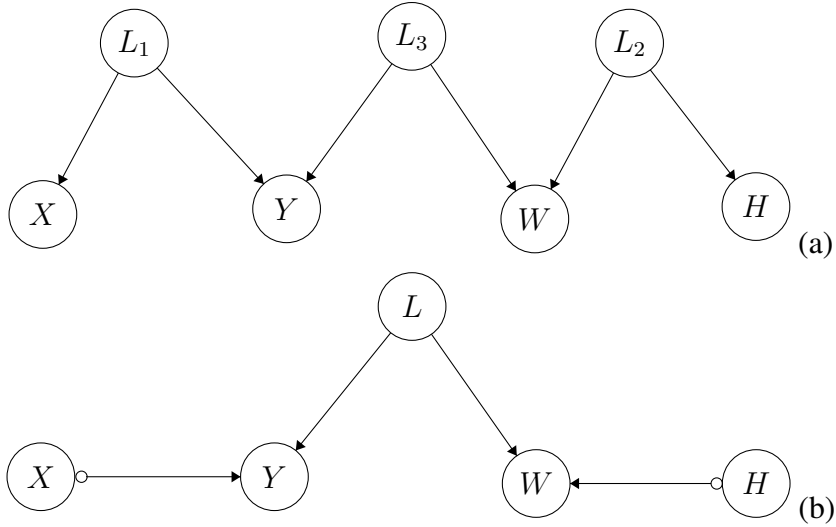
\begin{figure}
\begin{center}
\begin{tikzpicture}[scale=0.15]
\tikzstyle{every node}+=[inner sep=0pt]
\draw [black] (14.2,-11) circle (3);
\draw (14.2,-11) node {$L_1$};
\draw [black] (38.8,-10.4) circle (3);
\draw (38.8,-10.4) node {$L_3$};
\draw [black] (58.1,-11) circle (3);
\draw (58.1,-11) node {$L_2$};
\draw [black] (6.1,-26.2) circle (3);
\draw (6.1,-26.2) node {$X$};
\draw [black] (49.7,-26.6) circle (3);
\draw (49.7,-26.6) node {$W$};
\draw [black] (69.3,-25.5) circle (3);
\draw (69.3,-25.5) node {$H$};
\draw [black] (27.7,-25.7) circle (3);
\draw (27.7,-25.7) node {$Y$};
\draw [black] (12.79,-13.65) -- (7.51,-23.55);
\fill [black] (7.51,-23.55) -- (8.33,-23.08) -- (7.45,-22.61);
\draw [black] (59.93,-13.37) -- (67.47,-23.13);
\fill [black] (67.47,-23.13) -- (67.37,-22.19) -- (66.58,-22.8);
\draw [black] (56.68,-13.64) -- (51.12,-23.96);
\fill [black] (51.12,-23.96) -- (51.94,-23.49) -- (51.06,-23.02);
\draw [black] (40.47,-12.89) -- (48.03,-24.11);
\fill [black] (48.03,-24.11) -- (47.99,-23.17) -- (47.16,-23.73);
\draw [black] (37.04,-12.83) -- (29.46,-23.27);
\fill [black] (29.46,-23.27) -- (30.34,-22.92) -- (29.53,-22.33);
\draw [black] (16.23,-13.21) -- (25.67,-23.49);
\fill [black] (25.67,-23.49) -- (25.5,-22.56) -- (24.76,-23.24);
\end{tikzpicture}
(a)    
\end{center}

\begin{center}
\begin{tikzpicture}[scale=0.15]
\tikzstyle{every node}+=[inner sep=0pt]
\draw [black] (38.8,-10.4) circle (3);
\draw (38.8,-10.4) node {$L$};
\draw [black] (49.8,-26.1) circle (3);
\draw (49.8,-26.1) node {$W$};
\draw [black] (69.3,-26.1) circle (3);
\draw (69.3,-26.1) node {$H$};
\draw [black] (65.8,-26.1) circle (0.5);
\draw (65.8,-26.1) node {};
\draw [black] (5.9,-25.5) circle (3);
\draw (5.9,-25.5) node {$X$};
\draw [black]  (9.5,-25.5) circle (0.5);
\draw  (9.5,-25.5) node {};
\draw [black] (26.6,-25.5) circle (3);
\draw (26.6,-25.5) node {$Y$};
\draw [black] (40.52,-12.86) -- (48.08,-23.64);
%\draw [black] (46.80,-26) -- (29.6,-25.8);
\fill [black] (48.08,-23.64) -- (48.03,-22.7) -- (47.21,-23.27);
%\fill [black] (46.08,-26.64) -- (46.03,-25.5) -- (47,-26);
%\fill [black] (30.3,-26.4) -- (30.3,-25) -- (29.5,-25.6);
\draw [black] (9.9,-25.5) -- (23.6,-25.5);
\fill [black] (23.6,-25.5) -- (22.8,-25) -- (22.8,-26);
\draw [black] (36.91,-12.73) -- (28.49,-23.17);
\fill [black] (28.49,-23.17) -- (29.38,-22.86) -- (28.6,-22.23);
\draw [black] (65.3,-26.1) -- (52.8,-26.1);
\fill [black] (52.8,-26.1) -- (53.6,-26.6) -- (53.6,-25.6);
\end{tikzpicture}
(b)
\end{center}    
    \caption{FCI outputs the graph $(b)$  with $L$ being the latent common cause with the true graph $(a)$}
    \label{fig:FCIex}
\end{figure}

\subsection{Rank Constraint, FOFC and FTFC }
    
When the underlying distribution is linear, rank constraints of submatrices of the covariance matrix, which are special cases of tensor constraints (explained in section 4.1.2), can be used to model latent variables  \cite{Sullivant_2010}.  Algorithms like Find One Factor Cluster (FOFC) and Find Two Factor Cluster (FTFC) have been invented using the rank constraint to cluster variables. However, FOFC can only identify clusters sharing one latent and needs at least three measured variables in one cluster to do so \cite{Kummerfeld:2016:CCM:2939672.2939838}; FTFC can only identify clusters sharing two latents and needs at least six variables to do so \cite{FTFC}.\\ \textcolor{white}{abc}  
Spearman (1904) established the tetrad equation in his general intelligence theory  \cite{SpearmanTetradEquation}.  If the only common cause of people's mental performance is their ``intelligence,'' the tetrad equation should hold for any combinations of measured mental performances:
\begin{center}
    $\rho_{pq}\rho_{rs}-\rho_{rq}\rho_{ps}=0$
\end{center}
where $\rho_{xy}$ is the correlation between two measured result between two mental performance $x$ and $y$.  
This tetrad constraint is a special case of the rank constraint that identifies clusters of variables sharing only one common cause.
Assuming linear models, we first define \textbf{rank faithfulness:}\\ 
\textbf{Rank Faithfulness:} A distribution $P$ is rank-faithful to a DAG $G=\mathbf{\langle V,E\rangle}$ if and only if for all $\mathbf{A,B}\subset\mathbf{V}$,  $rank(\Sigma_{\mathbf{A,B}}) = t_m(\mathbf{A,B})$, 
where $\Sigma_{\mathbf{A,B}}$ is the cross-covariance matrix of $\mathbf{A,B}$ and $t_m(\mathbf{A,B})=min\{|\mathbf{C_A}|+|\mathbf{C_B}||\mathbf{C_A, C_B\subset V}$; $\mathbf{C_A, C_B }$ t-separates $\mathbf{   A, B}\}$.\footnote{From now on we define $\Sigma_{\mathbf{A,B}}$ as the cross-covariance matrix, with rows of $\mathbf{A}$ and columns of $\mathbf{B}$.}\\ \textcolor{white}{abc}
Notice that in the parameter space (of the linear system) in the sense of Lebesgue measure, the faithfulness assumption holds almost everywhere.  \\ \textcolor{white}{abc}
Assuming the \textbf{Causal Markov Condition} and \textbf{rank faithfulness, }the rank constraint can be used to model latent common causes by checking whether submatrices of the certain covariance matrix has a zero determinant and further partially identifies the connection between the latents \cite{Sullivant_2010} \cite{CausationPredictionSearch}. It can also be used when conditioning on subsets of measured variables, called the conditional rank constraint \cite{conditionaltetrad}.
\begin{theorem}
\textbf{Rank Constraint:} The Covariance matrix $\sum_{\mathbf{A,B}}$ has rank less or equal to r for the covariance matrices of all linear model with the graph  $G=\langle \mathbf{V, E}\rangle$ if and only if there is a pair of sets $\mathbf{C_A, C_B\subset V}$ such that $\mathbf{C_A, C_B}$ t-separates $\mathbf{A, B}$ and $|\mathbf{C_A}|+|\mathbf{C_B}|\leq r$.
\end{theorem}Some algorithms using the rank constraint to identify latent variables, such as $FOFC$ and $FTFC$ \cite{Kummerfeld:2016:CCM:2939672.2939838}, assuming that there are no direct causal relations between measured variables. For four variables $X_1, X_2, X_3$ and $X_4$, if for every partition $\mathbf{A,B}$ of $\{X_1, X_2, X_3, X_4\}$ with $\mathbf{|A| = |B| =2}$,  $\Sigma_\mathbf{A,B}$ has a zero rank, the algorithm will cluster them together, put a latent common cause among them and later merge all the clusters having overlapping variables. \\ \textcolor{white}{abc}
Another application is $FTFC$. For any six variables $X_1, X_2, X_3$, $X_4$, $X_5$ and $X_6$, if $\sum_\mathbf{A,B}$ has a zero rank for every partition $\mathbf{A,B}$ of $\{X_1, X_2, X_3, X_4,  X_5, X_6\}$, i.e., the sexted constraint, the algorithm will put two latent common causes among them and later merge clusters similarly as $FOFC$. \\ \textcolor{white}{abc}
Although the clusters identified using the rank constraint are assumed to share latent common causes, the constraint will still hold if in the true graph one of the measured variable is a parent of a latent.  That is to say, the output of the cluster allows one edge between the measured and the latents to be reversed.  There is also spider model as shown in \ref{fig:spidermodel}, where the measured variables are sharing one latent variable, but it is $(\{L\},\{L\})$ that \textit{\textbf{t-separates}} partitions of the set of measured variables. In this case the sum of the cardinality of the choke set is 2, but it does not entail that there are two latent common causes.
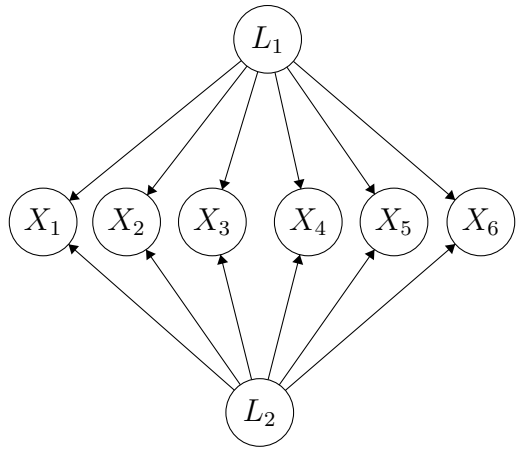
\begin{figure}[H]
    \begin{center}
\begin{tikzpicture}[scale=0.15]
\tikzstyle{every node}+=[inner sep=0pt]
\draw [black] (26.1,-24.1) circle (3);
\draw (26.1,-24.1) node {$X_2$};
\draw [black] (18.7,-24.1) circle (3);
\draw (18.7,-24.1) node {$X_1$};
\draw [black] (38.6,-7.9) circle (3);
\draw (38.6,-7.9) node {$L_1$};
\draw [black] (33.7,-24.1) circle (3);
\draw (33.7,-24.1) node {$X_3$};
\draw [black] (42.2,-24.1) circle (3);
\draw (42.2,-24.1) node {$X_4$};
\draw [black] (49.8,-24.1) circle (3);
\draw (49.8,-24.1) node {$X_5$};
\draw [black] (57.5,-24.1) circle (3);
\draw (57.5,-24.1) node {$X_6$};
\draw [black] (37.9,-41) circle (3);
\draw (37.9,-41) node {$L_2$};
\draw [black] (36.27,-9.79) -- (21.03,-22.21);
\fill [black] (21.03,-22.21) -- (21.96,-22.09) -- (21.33,-21.31);
\draw [black] (36.77,-10.28) -- (27.93,-21.72);
\fill [black] (27.93,-21.72) -- (28.82,-21.4) -- (28.03,-20.79);
\draw [black] (37.73,-10.77) -- (34.57,-21.23);
\fill [black] (34.57,-21.23) -- (35.28,-20.61) -- (34.32,-20.32);
\draw [black] (39.25,-10.83) -- (41.55,-21.17);
\fill [black] (41.55,-21.17) -- (41.86,-20.28) -- (40.89,-20.5);
\draw [black] (40.31,-10.37) -- (48.09,-21.63);
\fill [black] (48.09,-21.63) -- (48.05,-20.69) -- (47.23,-21.26);
\draw [black] (40.88,-9.85) -- (55.22,-22.15);
\fill [black] (55.22,-22.15) -- (54.94,-21.25) -- (54.29,-22.01);
\draw [black] (37.18,-38.09) -- (34.42,-27.01);
\fill [black] (34.42,-27.01) -- (34.13,-27.91) -- (35.1,-27.67);
\draw [black] (38.64,-38.09) -- (41.46,-27.01);
\fill [black] (41.46,-27.01) -- (40.78,-27.66) -- (41.75,-27.91);
\draw [black] (39.63,-38.55) -- (48.07,-26.55);
\fill [black] (48.07,-26.55) -- (47.2,-26.92) -- (48.02,-27.49);
\draw [black] (40.17,-39.04) -- (55.23,-26.06);
\fill [black] (55.23,-26.06) -- (54.3,-26.2) -- (54.95,-26.96);
\draw [black] (36.18,-38.54) -- (27.82,-26.56);
\fill [black] (27.82,-26.56) -- (27.87,-27.5) -- (28.69,-26.93);
\draw [black] (35.65,-39.02) -- (20.95,-26.08);
\fill [black] (20.95,-26.08) -- (21.22,-26.99) -- (21.88,-26.24);
\end{tikzpicture}
\end{center}
    \caption{For any six variables $X_1, X_2, X_3$, $X_4$, $X_5$ and $X_6$, if $\sum_\mathbf{A,B}$ has a zero rank for every partition $\mathbf{A,B}$ of $\{X_1, X_2, X_3, X_4,  X_5, X_6\}$, FTFC puts two latent common causes among them.}
    \label{fig:FTFC}
\end{figure}

\begin{figure}
    
\begin{center}
\begin{tikzpicture}[scale=0.12]
\tikzstyle{every node}+=[inner sep=0pt]
\draw [black] (61.1,-40.4) circle (3);
\draw (61.1,-40.4) node {$X_5$};
\draw [black] (35.1,-44.6) circle (3);
\draw (35.1,-44.6) node {$X_3$};
\draw [black] (37,-4.6) circle (3);
\draw (37,-4.6) node {$L$};
\draw [black] (6.2,-29.7) circle (3);
\draw (6.2,-29.7) node {$X_1$};
\draw [black] (20.5,-39.4) circle (3);
\draw (20.5,-39.4) node {$X_2$};
\draw [black] (45.1,-44.6) circle (3);
\draw (45.1,-44.6) node {$X_4$};
\draw [black] (74.8,-26.1) circle (3);
\draw (74.8,-26.1) node {$X_6$};
\draw [black] (45.1,-23.5) circle (3);
\draw (45.1,-23.5) node {$L_4$};
\draw [black] (18,-13.7) circle (3);
\draw (18,-13.7) node {$L_1$};
\draw [black] (23.9,-19.8) circle (3);
\draw (23.9,-19.8) node {$L_2$};
\draw [black] (32.3,-25.5) circle (3);
\draw (32.3,-25.5) node {$L_3$};
\draw [black] (52.9,-19.8) circle (3);
\draw (52.9,-19.8) node {$L_5$};
\draw [black] (58.4,-11.2) circle (3);
\draw (58.4,-11.2) node {$L_6$};
\draw [black] (34.67,-6.5) -- (8.53,-27.8);
\fill [black] (8.53,-27.8) -- (9.46,-27.69) -- (8.83,-26.91);
\draw [black] (35.71,-7.31) -- (21.79,-36.69);
\fill [black] (21.79,-36.69) -- (22.58,-36.18) -- (21.68,-35.75);
\draw [black] (36.86,-7.6) -- (35.24,-41.6);
\fill [black] (35.24,-41.6) -- (35.78,-40.83) -- (34.78,-40.78);
\draw [black] (37.6,-7.54) -- (44.5,-41.66);
\fill [black] (44.5,-41.66) -- (44.84,-40.78) -- (43.86,-40.97);
\draw [black] (38.68,-7.09) -- (59.42,-37.91);
\fill [black] (59.42,-37.91) -- (59.39,-36.97) -- (58.56,-37.53);
\draw [black] (39.61,-6.08) -- (72.19,-24.62);
\fill [black] (72.19,-24.62) -- (71.74,-23.79) -- (71.25,-24.66);
\draw [black] (45.1,-26.5) -- (45.1,-41.6);
\fill [black] (45.1,-41.6) -- (45.6,-40.8) -- (44.6,-40.8);
\draw [black] (43.92,-20.74) -- (38.18,-7.36);
\fill [black] (38.18,-7.36) -- (38.04,-8.29) -- (38.96,-7.9);
\draw [black] (20.71,-12.4) -- (34.29,-5.9);
\fill [black] (34.29,-5.9) -- (33.36,-5.79) -- (33.79,-6.69);
\draw [black] (16.22,-16.11) -- (7.98,-27.29);
\fill [black] (7.98,-27.29) -- (8.86,-26.94) -- (8.05,-26.34);
\draw [black] (25.86,-17.53) -- (35.04,-6.87);
\fill [black] (35.04,-6.87) -- (34.14,-7.15) -- (34.9,-7.8);
\draw [black] (23.39,-22.76) -- (21.01,-36.44);
\fill [black] (21.01,-36.44) -- (21.64,-35.74) -- (20.66,-35.57);
\draw [black] (32.96,-22.57) -- (36.34,-7.53);
\fill [black] (36.34,-7.53) -- (35.68,-8.2) -- (36.65,-8.42);
\draw [black] (32.74,-28.47) -- (34.66,-41.63);
\fill [black] (34.66,-41.63) -- (35.04,-40.77) -- (34.05,-40.91);
\draw [black] (54.01,-22.59) -- (59.99,-37.61);
\fill [black] (59.99,-37.61) -- (60.16,-36.68) -- (59.23,-37.05);
\draw [black] (50.73,-17.73) -- (39.17,-6.67);
\fill [black] (39.17,-6.67) -- (39.4,-7.59) -- (40.09,-6.86);
\draw [black] (60.62,-13.22) -- (72.58,-24.08);
\fill [black] (72.58,-24.08) -- (72.32,-23.17) -- (71.65,-23.91);
\draw [black] (55.53,-10.32) -- (39.87,-5.48);
\fill [black] (39.87,-5.48) -- (40.48,-6.2) -- (40.78,-5.24);
\end{tikzpicture}
\end{center}
    \caption{The Spider model which will be identified as a cluster with two common latent variables, while $(\{L\},\{L\})$ \textit{\textbf{t-separates}} partitions of $\{X_1,...X_6\}$}
    \label{fig:spidermodel}
\end{figure}
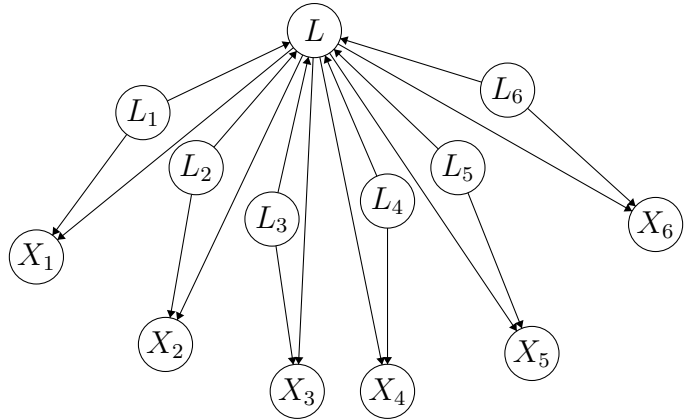

The rank constraint implies the number of latents causing the same set of measured variables but says nothing about how some of these latents are connected. For instance, in figure \ref{fig:FTFC} FTFC is able to identify the two latents $L_1$ and $L_2$, but while it can identify some causal relations between variables of different clusters, it is unable to distinguish if there are edges between $L_1$ and $L_2$ or how are those edges directed.  Furthermore, although the rank constraint is used to identify latent variables, there are cases that satisfy certain rank constraints with causal relations among measured variables.  For instance, fig \ref{fig:FTFCexpt} satisfies the sexted constraint, but there is only one latent common cause $L_1$.  Such case is ruled out by assumting that there is no direct causal relation between measured variables.\\ \textcolor{white}{abc}
%This case can be distinguished from a sexted using the conditional rank constraint and using the multitrek condition if the additive noise follows non-Gaussian distribution.  It is easy to see how the conditional rank constraint works: after conditioning only on $X_6$, any four of $X_1$ to $X_5$ form a vanishing tetrad, which suggests that $X_1$ to $X_5$ share at most one latent common cause without $X_6$; the failure of satisfying vanishing tetrad when not conditioning on $X_6$ further suggests that $X_6$ and $L_1$ \textbf{\textit{t-separate}} the other variables.  I will discuss this case using the multitrek constraint with more details in the next section. 
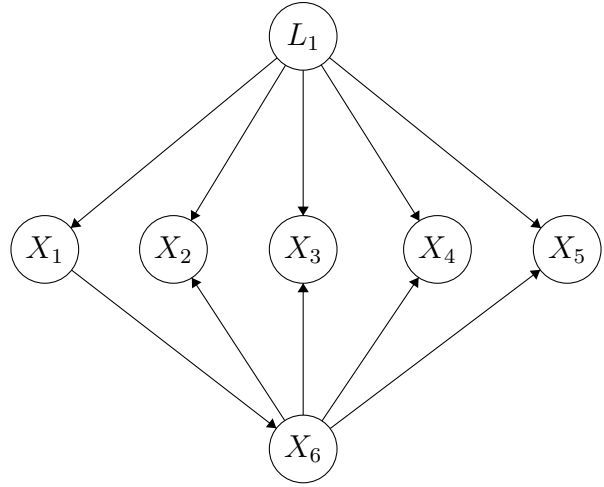
\begin{figure}
\begin{center}
\begin{tikzpicture}[scale=0.15]
\tikzstyle{every node}+=[inner sep=0pt]
\draw [black] (22.3,-29.9) circle (3);
\draw (22.3,-29.9) node {$X_2$};
\draw [black] (10.9,-29.9) circle (3);
\draw (10.9,-29.9) node {$X_1$};
\draw [black] (33.8,-11) circle (3);
\draw (33.8,-11) node {$L_1$};
\draw [black] (33.8,-29.9) circle (3);
\draw (33.8,-29.9) node {$X_3$};
\draw [black] (45.7,-29.9) circle (3);
\draw (45.7,-29.9) node {$X_4$};
\draw [black] (57.2,-29.9) circle (3);
\draw (57.2,-29.9) node {$X_5$};
\draw [black] (33.8,-47.6) circle (3);
\draw (33.8,-47.6) node {$X_6$};
\draw [black] (31.49,-12.91) -- (13.21,-27.99);
\fill [black] (13.21,-27.99) -- (14.15,-27.87) -- (13.51,-27.1);
\draw [black] (32.24,-13.56) -- (23.86,-27.34);
\fill [black] (23.86,-27.34) -- (24.7,-26.91) -- (23.85,-26.39);
\draw [black] (33.8,-14) -- (33.8,-26.9);
\fill [black] (33.8,-26.9) -- (34.3,-26.1) -- (33.3,-26.1);
\draw [black] (35.4,-13.54) -- (44.1,-27.36);
\fill [black] (44.1,-27.36) -- (44.1,-26.42) -- (43.25,-26.95);
\draw [black] (36.13,-12.89) -- (54.87,-28.01);
\fill [black] (54.87,-28.01) -- (54.56,-27.12) -- (53.93,-27.9);
\draw [black] (33.8,-44.6) -- (33.8,-32.9);
\fill [black] (33.8,-32.9) -- (33.3,-33.7) -- (34.3,-33.7);
\draw [black] (35.47,-45.11) -- (44.03,-32.39);
\fill [black] (44.03,-32.39) -- (43.16,-32.77) -- (43.99,-33.33);
\draw [black] (36.19,-45.79) -- (54.81,-31.71);
\fill [black] (54.81,-31.71) -- (53.87,-31.79) -- (54.47,-32.59);
\draw [black] (32.17,-45.08) -- (23.93,-32.42);
\fill [black] (23.93,-32.42) -- (23.95,-33.36) -- (24.79,-32.81);
\draw [black] (13.27,-31.73) -- (31.43,-45.77);
\fill [black] (31.43,-45.77) -- (31.1,-44.88) -- (30.49,-45.67);
\end{tikzpicture}
\end{center}
\caption{A example that $FTFC$ cannot identify}
\label{fig:FTFCexpt}
\end{figure}I've discussed the rank constraint in the linear, acyclic model, but the assumption of acyclicity and linear model can be further relaxed while the rank constraint can still be entailed. To use the rank constraint to identify the latent variables, only the structure of a part of the graph related to the latents matters. Spirtes has proved that the rank constraint works with only linearity and acyclicity under the choke set\cite{LAchoke}.  Later in the this chapter I will show that we can further relax the assumption and only linearity is needed under the choke set for the rank constraint to hold.  \\
\textbf{LA below the choke set}\cite{LAchoke}: Let $D(\mathbf{C_A,A,}G)$ be the set of vertices on directed paths in $G$ from $\mathbf{C_A}$ to $\mathbf{A}$ except for the members of $\mathbf{C_A}$. If $S$ is a fixed-parameter SEM $\langle\phi,\theta\rangle$ with path diagram (graph) $G$,$S$ is LA below the sets $\mathbf{C_A,C_B}$ for $\mathbf{A,B}$ if and only if each member of $\mathbf{W}= D(\mathbf{C_A,A,}G)\cup D(\mathbf{C_B,B,}G)),$
\begin{enumerate}
    \item $\mathbf{V_{ext}=V}\cup\{\epsilon_X:X\in \mathbf{W}\}$;
    \item no member of $\mathbf{W}$ lies on a cycle;
    \item $G_{ext}$ is a directed graph over $\mathbf{V_{ext}}$ with sub-graph $G$, together with an edge from $\epsilon_X$ to $X$ for each $X\in\mathbf{W}$;
    \item for each $X\in D(\mathbf{C_A,A},G_{ext}),$ \newline $X=\Sigma_{V\in\mathbf{Pa}(X, G_{ext})\cap (D(\mathbf{C_A,A},G_{ext})\cup \mathbf{C_A})}a_{X,V}V+f_x(\mathbf{Pa}(X,G_{ext})\setminus(D(\mathbf{C_A,A},G_{ext})\cup \mathbf{C_A}))$\newline and for each $X\in D(\mathbf{C_B,B},G_{ext})$,\newline $X=\Sigma_{V\in\mathbf{Pa}(X, G_{ext})\cap (D(\mathbf{C_B,B},G_{ext})\cup \mathbf{C_B})}a_{X,V}V+g_x(\mathbf{Pa}(X,G_{ext})\setminus(D(\mathbf{C_B,B},G_{ext})\cup \mathbf{C_B}))$.
\end{enumerate}
\begin{theorem}\label{thm:LArank}\cite{LAchoke} Suppose $G$ is a directed graph containing $\mathbf{C_A, A, C_B}$ and $\mathbf{B}$, and $(\mathbf{C_A, C_B})$ t-separates $\mathbf{A, B}$ in $G$.  Then for all covariance matrices entailed by a fixed parameter structural equation model $S$, defined as in \ref{def:SEM}, with path diagram $G$ that is linear and acyclic below $\mathbf{C_A}$ and $\mathbf{C_B}$ for $\mathbf{A}$ and $\mathbf{B}$ rank$(\Sigma_{\mathbf{A, B}})\leq|\mathbf{C_A}|+|\mathbf{C_B}|$.
\end{theorem} \textcolor{white}{abc} Theorem \ref{thm:LArank} claims that given sets of variables  $\mathbf{A, B}$ and their choke sets $\mathbf{C_A}$ and $\mathbf{C_B}$, as long as the paths from the choke set to $\mathbf{A, B}$ are linear and there are no cycles between $\mathbf{A, B}$ and the choke sets, the rank constraint still holds.  For instance, in figure \ref{fig:FTFC}, regrading $\{L_1, L_2\}$ as the choke set, the sextad still holds if there are cycles between $L_1$ and $L_2$, or $L_1$ and $L_2$ are causally connected in a nonlinear way; but this constraint fails if $X_1 =f(L_1)+\epsilon_1$ and $X_2=g(L_1)+\epsilon_2$ where $f$ and $g$ are nonlinear functions, or $X_1$ is causing some parents of $L_1$.\\ \textcolor{white}{abc}

\section{Recent Developments }

\subsection{Generalized Independence Noise}
The Generalized Independence Noise (GIN) condition can be used to identify latent common causes assuming that noises are non-Gaussian  \cite{xie2020generalized} .\\ 
\textbf{Generalized Independence Noise (GIN) condition: }Let $\mathbf{Y}$ and $\mathbf{Z}$ be two observed random vectors. Suppose that the variables follow the linear non-Gaussian acyclic causal models.  If there exists a non zero $\omega$ such that for $E_{\mathbf{Y}||\mathbf{Z}}:=\omega^T\mathbf{Y}$, $\omega^T\mathbb{E[\mathbf{YZ}^T]}=0$, then we say that $(\mathbf{Z,Y})$ follows the GIN condition if and only if $\mathbb{E}_{\mathbf{Y}||\mathbf{Z}}$ is independent from $\mathbf{Z}$.\\ \textcolor{white}{abc}
The GIN condition can be used to recover the causal ordering among latent variables:
\begin{theorem}\label{thm:GIN}
Suppose that random vectors $\mathbf{L,Y}$ and $\mathbf{Z}$ are related in the following way:
\begin{center}
    $\mathbf{Y=}A\mathbf{L+E_Y}$,
    
    $\mathbf{Z=}B\mathbf{L+E_Z}$.
\end{center}
Denote by $l$ the dimensionality of $\mathbf{L}$.  Assume $A$ is of full column rank.  Then, if 1) $Dim(Y)>l$,  2) $\mathbf{E_Y}\indep \mathbf{L}$, 3) $\mathbf{E_Y}\indep \mathbf{E_Z}$, 4) the cross-covariance matrix of $l$ and $\mathbf{Z}$, $\Sigma_{L Z}=\mathbb{E}[\mathbf{LZ}^T]$ has rank $l$, then $E_{Y\parallel Z}\indep \mathbf{Z}$, i.e.,$(\mathbf{Z,Y})$ satisfies the GIN condition.
\end{theorem} If $\mathbf{Z}$ and $\mathbf{Y}$ can be written as linear equations of the same vector of latent variables ($\mathbf{L}$) and $\mathbf{(Z,Y)}$ does not GIN condition, while condition $(1)(2)(4)$ are satisfied, then $\mathbf{E_Y}\nindep \mathbf{E_Z}$, which suggests that there exists common latent causes of $\mathbf{Z}$ that are descendants of $\mathbf{L}.$ \\ \textcolor{white}{abc}
If instead of having a latent common cause, measured variables have direct causal relations with each other, then they will satisfy a condition called IN:\\ 
\textbf{Independence Noise (IN) Condition: } Let $Y$ be a single variable and $\mathbf{Z}$ be a random vector. Suppose that the variables follow the linear non-Gaussian acyclic causal models.  Let $\omega$ be the vector of regression coefficients of regressing $Y$ on $\mathbf{Z}$, so $\omega=\mathbb{E}[Y\mathbf{Z}^T]\mathbb{E}^{-1}[\mathbf{ZZ}^T]=0$, then we say that $(\mathbf{Z},Y)$ follows the IN condition if and only if $\mathbb{E}_{Y||\mathbf{Z}}=Y-\omega^T\mathbf{Z}$ is independent from $\mathbf{Z}$.\\ \textcolor{white}{abc}
IN is a special case of GIN.
Let $\ddot{Y}:=(Y,\mathbf{Z})$, then the following two statements are true:
\begin{enumerate}
    \item $(\mathbf{Z},\ddot{Y} )$ follows the GIN condition if and only if $(\mathbf{Z,} Y )$ follows it.
    \item If $(\mathbf{Z,} Y )$ follows the IN condition, then  $(\mathbf{Z},\ddot{Y} )$ follows the GIN condition.
\end{enumerate}
Notice that IN can be used in practice to make sure that variables share latent common causes before applying GIN to figure out the structures among latent variables.\\ \textcolor{white}{abc}
I will describe a GIN-based algorithm that identifies the latent structures of non-Gaussian data \cite{xie2020generalized}.  Since it is not the core of the thesis, I will leave the details and provide an example to illustrate how the the algorithm works:
\begin{enumerate}
    \item for any $\mathbf{X\subset V}$, if $(\mathbf{V\setminus X, X})$ satisfies the GIN and no proper subset of $\mathbf{X}$ satisfies the GIN with its complement, label $\mathbf{X}$ as a cluster sharing $\mathbf{|X|-1}$ many latents and merge any two clusters having the same number of measure variables  overlapping variables;
    \item recover the causal ordering of latent variables by recursively checking the GIN condition: for any two clusters $\mathbf{Z,Y}$, if for  $\mathbf{Z'\subset Z}$, $(\mathbf{Z\setminus Z',Z'\cup Y})$ satisfies GIN condition, the latent cause of $\mathbf{Z}$ is no later than $\mathbf{Y}.$
\end{enumerate}
For instance, in the fig \ref{2clusters}, $(\{X_1\},\{X_2,X_3\})$ can be written as:
\begin{center}
    $\begin{bmatrix}X_2\\ \textcolor{white}{abc}X_3\end{bmatrix}=\begin{bmatrix}a_2\\ \textcolor{white}{abc}ba_3\end{bmatrix}L_1+\begin{bmatrix}\epsilon_{X_2}\\ \textcolor{white}{abc} a_3\epsilon_{L_2}+ \epsilon_{X_3}\end{bmatrix}$\textcolor{white}{*******}     $X_1=a_1L_1+ \epsilon_{X_1}$
\end{center}
since the four conditions in thereom 2.4 are satisfied (taking $X_1$ to be $\mathbf{Z}$ and $\{X_2,X_3\}$ as $\mathbf{Y}$), $(\{X_1\},\{X_2,X_3\})$ satisfies the GIN; $(\{X_4\},\{X_2,X_3\})$ is written as:
\begin{center}
    $\begin{bmatrix}X_2\\ \textcolor{white}{abc}X_3\end{bmatrix}=\begin{bmatrix}a_2\\ \textcolor{white}{abc}ba_3\end{bmatrix}L_1+\begin{bmatrix}\epsilon_{X_2}\\ \textcolor{white}{abc} a_3\epsilon_{L_2}+ \epsilon_{X_3}\end{bmatrix}$\textcolor{white}{*******}     $X_4=ba_4L_1+ \epsilon_{X_4}+a_4\epsilon_{L_2}$
\end{center}
since $a_3\epsilon_{L_2}+ \epsilon_{X_3}$ and $\epsilon_{X_4}+a_4\epsilon_{L_2}$ are not independent, condition 3) is not satisfied, so $(\{X_4\},\{X_2,X_3\})$ does not satisfy the GIN.

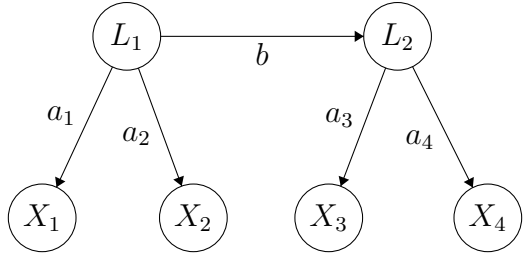
\begin{figure}
    \begin{center}
\begin{tikzpicture}[scale=0.15]
\tikzstyle{every node}+=[inner sep=0pt]
\draw [black] (42.1,-18.2) circle (3);
\draw (42.1,-18.2) node {$L_2$};
\draw [black] (50.1,-34.3) circle (3);
\draw (50.1,-34.3) node {$X_4$};
\draw [black] (36,-34.3) circle (3);
\draw (36,-34.3) node {$X_3$};
\draw [black] (10.6,-34.3) circle (3);
\draw (10.6,-34.3) node {$X_1$};
\draw [black] (24,-34.3) circle (3);
\draw (24,-34.3) node {$X_2$};
\draw [black] (18.1,-18.2) circle (3);
\draw (18.1,-18.2) node {$L_1$};
\draw [black] (21.1,-18.2) -- (39.1,-18.2);
\fill [black] (39.1,-18.2) -- (38.3,-17.7) -- (38.3,-18.7);
\draw (30.1,-18.7) node [below] {$b$};
\draw [black] (41.04,-21.01) -- (37.06,-31.49);
\fill [black] (37.06,-31.49) -- (37.81,-30.92) -- (36.88,-30.57);
\draw (38.3,-25.42) node [left] {$a_3$};
\draw [black] (16.83,-20.92) -- (11.87,-31.58);
\fill [black] (11.87,-31.58) -- (12.66,-31.07) -- (11.75,-30.64);
\draw (13.63,-25.21) node [left] {$a_1$};
\draw [black] (19.13,-21.02) -- (22.97,-31.48);
\fill [black] (22.97,-31.48) -- (23.16,-30.56) -- (22.22,-30.9);
\draw (20.29,-27.05) node [left] {$a_2$};
\draw [black] (43.43,-20.89) -- (48.77,-31.61);
\fill [black] (48.77,-31.61) -- (48.86,-30.67) -- (47.96,-31.12);
\draw (45.4,-27.35) node [left] {$a_4$};
\end{tikzpicture}
\end{center}
    \caption{$X_1$ and $X_2$ are in a cluster with latent variable causally earlier than the latent variable in the $\{X_3, X_4\}$ cluster }
    \label{2clusters}
\end{figure}

Figure \ref{2clusters}
Similarly, $(\{X_1,X_2\},\{X_4,X_3\})$ and $(\{X_3, X_4\},\{X_2,X_1\})$ both satisfy GIN, and the algorithm using GIN is able to identify the structure in fig \ref{2clusters} including $L_1\rightarrow L_2$ \cite{xie2020generalized}. Notice that in this graph only one tetrad constraint is satisfied, i.e., $\rho_{14}\rho_{23}-\rho_{13}\rho_{24}=0$. Since $\rho_{14}\rho_{23}=\rho_{13}\rho_{24}\neq\rho_{12}\rho_{34}$, the rank constraint entails that $X_1,X_2,X_3,X_4$ does not share one common cause.  Furthermore, $X_2,X_3$ do not share one common cause and $X_1,X_4$ do not share one common cause. Additionally, if we find no conditional independence among $X_1,X_2,X_3,X_4$, and background knowledge ensures that no measured variable can be a cause of the latents, the rank constraint also entails that  $X_2,X_4$ share one common cause and $X_1,X_1$ share one common cause, but how the latent variables are connected cannot be identified using only the rank constraint.\\ \textcolor{white}{abc}
In fig \ref{2latents}, $L_1$ and $L_2$ are directed connected.  If the underlying distribution is non-Gaussian, then the GIN condition can be used to find that there are two latents among $\{X_1,X_2,X_3,X_4\}$; if the probability is Gaussian, then with six measured variables, Rank Constraint can be used to find that there are two latents.  However, the direct application of these two theorems cannot tell whether there is an edge between $L_1$ and $L_2$. This is because $L_1$ and $L_2$ are causing the exact same variables. In fig \ref{2clusters}, since $L_1$ and $L_2$ are direct causes of two distinct set of measured variables, GIN can be used to identify the $L_1\rightarrow L_2$ edge.   
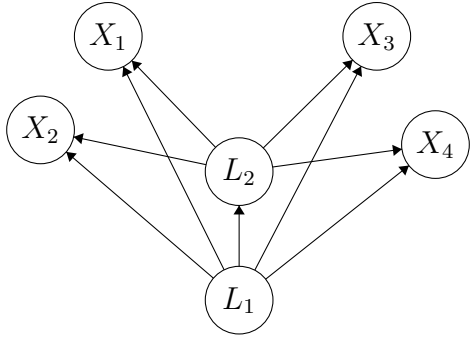
\begin{figure}
\begin{center}
\begin{tikzpicture}[scale=0.15]
\tikzstyle{every node}+=[inner sep=0pt]
\draw [black] (44.3,-25.9) circle (3);
\draw (44.3,-25.9) node {$X_4$};
\draw [black] (9.3,-24.6) circle (3);
\draw (9.3,-24.6) node {$X_2$};
\draw [black] (15.3,-16.3) circle (3);
\draw (15.3,-16.3) node {$X_1$};
\draw [black] (39.1,-16.3) circle (3);
\draw (39.1,-16.3) node {$X_3$};
\draw [black] (26.9,-39.7) circle (3);
\draw (26.9,-39.7) node {$L_1$};
\draw [black] (26.9,-28.3) circle (3);
\draw (26.9,-28.3) node {$L_2$};
\draw [black] (24.62,-37.75) -- (11.58,-26.55);
\fill [black] (11.58,-26.55) -- (11.86,-27.45) -- (12.51,-26.69);
\draw [black] (25.57,-37.01) -- (16.63,-18.99);
\fill [black] (16.63,-18.99) -- (16.54,-19.93) -- (17.44,-19.48);
\draw [black] (28.29,-37.04) -- (37.71,-18.96);
\fill [black] (37.71,-18.96) -- (36.9,-19.44) -- (37.79,-19.9);
\draw [black] (29.25,-37.84) -- (41.95,-27.76);
\fill [black] (41.95,-27.76) -- (41.01,-27.87) -- (41.63,-28.65);
\draw [black] (23.96,-27.68) -- (12.24,-25.22);
\fill [black] (12.24,-25.22) -- (12.92,-25.87) -- (13.12,-24.89);
\draw [black] (24.81,-26.14) -- (17.39,-18.46);
\fill [black] (17.39,-18.46) -- (17.58,-19.38) -- (18.3,-18.68);
\draw [black] (29.04,-26.2) -- (36.96,-18.4);
\fill [black] (36.96,-18.4) -- (36.04,-18.61) -- (36.74,-19.32);
\draw [black] (29.87,-27.89) -- (41.33,-26.31);
\fill [black] (41.33,-26.31) -- (40.47,-25.92) -- (40.6,-26.91);
\draw [black] (26.9,-36.7) -- (26.9,-31.3);
\fill [black] (26.9,-31.3) -- (26.4,-32.1) -- (27.4,-32.1);
\end{tikzpicture}
\caption{$X_1, X_2, X_3, X_4$ share two latent common causes}
\label{2latents}
\end{center}
\end{figure}

\subsection{Multi-Trek Constraint (Tensor Constraint)}\label{sec:tensorconstraint}

In addition to the GIN condition, the tensor constraint is another method to identify latent structures when the data is non-Gaussian \cite{robeva2020multitrek}.  To introduce this new constraint, we first introduce a new term for a graph structure generalizing the concept \textit{trek}.
\subsubsection{\textit{k-trek} and \textit{k-choke set}}\label{definition:ktreksep}
\textbf{\textbf{k-trek}}: a $k$-$trek$ between $k$ vertices $v_1...v_k$ is an ordered collection of $k$ directed paths $(\pi_1,...\pi_k)$ where $\tau_i$ ends at $v_i$ and $\tau_1...\tau_k$ begin at the same vertex, denoted by $top(\tau_1...\tau_k)$.   For instance, in fig \ref{fig:FOFC} $(L\rightarrow X_1,L\rightarrow X_2,L\rightarrow X_3,L\rightarrow X_4)$ is a 4-$trek$.
We say that a collection of $k$ sets of nodes $\mathbf{S_1, ..., S_k\subset V}$  such that $\mathbf{|S_1|= ...= |S_k|}=n$ is a $k$-$trek$ $system$ $T$, if there is a collection of $n$ many $k$-$treks$ between $\mathbf{S_1, ..., S_k}$ such that the ends of $T$  on the $i$-th side equal $S_i$.  A $k$-$trek$ $system$ $T$ has a sided intersection if there exits two $k$-treks $(\pi_1,..,\pi_i,...\pi_k)$ and $(\kappa_1,...,\kappa_i,...\kappa_k)$ and $\pi_i$ and $\kappa_i$ have common vertices for some $i$.\\ 
\textbf{k-choke set and k-trek-separation: }the collection of sets$(\mathbf{A_1,...,A_k})$ \textit{k-trek-separates} $\mathbf{S_1,...,S_k}$ if for every \textit{k-trek} with paths $(\pi_1,...,\pi_k)$ between $S_1,...,S_k$,there exists $j\in\{1,...,k\}$ such that $P_j$ contains a vertex from $A_j$. 
 We call $(\mathbf{A_1,...,A_k})$ \textbf{\textit{k-choke sets}}\cite{robeva2020multitrek}.\\ \textcolor{white}{abc}
 \textbf{\textit{Example.}} Consider fig \ref{fig:3-choke-sep}. $\{X_1, X_5\}, \{X_2, X_6\}, \{X_3, X_7\}$ form a \textit{3-trek-system} and $(\{L_1\}, \{L_2\}, \emptyset)$ \textit{3-trek-separates} this system.
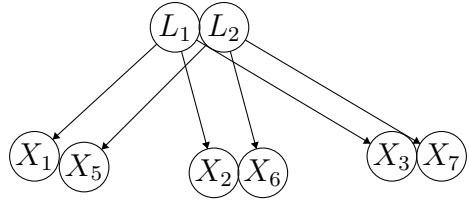
\begin{figure}
\begin{center}
\begin{tikzpicture}[scale=0.11]
\tikzstyle{every node}+=[inner sep=0pt]
\draw [black] (14.3,-27) circle (3);
\draw (14.3,-27) node {$X_1$};
\draw [black] (36,-29) circle (3);
\draw (36,-29) node {$X_2$};
\draw [black] (57.5,-27) circle (3);
\draw (57.5,-27) node {$X_3$};
\draw [black] (37.2,-11.4) circle (3);
\draw (37.2,-11.4) node {$L_2$};
\draw [black] (20.3,-28.3) circle (3);
\draw (20.3,-28.3) node {$X_5$};
\draw [black] (41.9,-29) circle (3);
\draw (41.9,-29) node {$X_6$};
\draw [black] (63.5,-27) circle (3);
\draw (63.5,-27) node {$X_7$};
\draw [black] (31.3,-11.4) circle (3);
\draw (31.3,-11.4) node {$L_1$};
\draw [black] (29.09,-13.43) -- (16.51,-24.97);
\fill [black] (16.51,-24.97) -- (17.44,-24.8) -- (16.76,-24.06);
\draw [black] (32.07,-14.3) -- (35.23,-26.1);
\fill [black] (35.23,-26.1) -- (35.5,-25.2) -- (34.54,-25.46);
\draw [black] (33.88,-12.93) -- (54.92,-25.47);
\fill [black] (54.92,-25.47) -- (54.49,-24.63) -- (53.98,-25.49);
\draw [black] (39.78,-12.93) -- (60.92,-25.47);
\fill [black] (60.92,-25.47) -- (60.49,-24.63) -- (59.98,-25.49);
\draw [black] (37.97,-14.3) -- (41.13,-26.1);
\fill [black] (41.13,-26.1) -- (41.4,-25.2) -- (40.44,-25.46);
\draw [black] (35.08,-13.52) -- (22.42,-26.18);
\fill [black] (22.42,-26.18) -- (23.34,-25.97) -- (22.63,-25.26);
\end{tikzpicture}
\end{center}
\caption{$\{X_1, X_5\}, \{X_2, X_6\}, \{X_3, X_7\}$ form a \textit{3-trek-system} and $(\{L_1\}, \{L_2\}, \emptyset)$ \textit{3-trek-separates} this system. }
\label{fig:3-choke-sep}
 \end{figure}
Now we can introduce higher order cumulant tensor and see its relation with the structure of the graph.

\subsubsection{Higher Order Cumulant Tensor}
Consider a linear graphical model $M=\langle P,G\rangle$:
\begin{center}
     $\mathbf{X =  A}^T\mathbf{ X+\epsilon}$  
\end{center}
where $\mathbf{A}_{ij}=a_{j,i}$ if $X_j\rightarrow X_i$ is in $G$ and $\mathbf{A}_{ij}=0$ otherwise.\\ \textcolor{white}{abc}  Since $G$ is acyclic, $\mathbf{I-A}$ is invertible:
\begin{center}
    $\mathbf{X = (I- A)}^{-T} \epsilon$
\end{center}
The $k$-th \textit{cumulant tensor} $Z^{(k)}$ of a random vector  $\mathbf{Z}=(Z_1,...Z_p)$ is the $p\times  ...\times p$ $(k\text{ times})$ table such that the entry at position $(i_1...i_k)$ is given by:

\begin{center}
    $cum(Z_{i_1},...,Z_{i_k})=\underset{A_1,...A_L}{\sum}(-1)^{L-1}(L-1)!\mathbb{E}[\underset{j\in A_1}{\Pi}Z_j]\mathbb{E}[\underset{j\in A_2}{\Pi}Z_j]...\mathbb{E}[\underset{j\in A_L}{\Pi}Z_j]$
\end{center}
where the sum is taken over all partitions $(A_1,...A_L)$ of the set $\{i_1...i_k\}$.\\ \textcolor{white}{abc}
\begin{lemma}\cite{bookblindsource}\label{lm:highordercum}
    The tensor $\mathcal{C}^{(k)}$ of $k-th$ order cumulants of $\mathbf{X}$ is
\begin{equation}
    \mathcal{C}^{(k)} = \mathcal{E}^{(k)}\cdot (\mathbf{I-A})^{-k}
\end{equation}
where $\mathcal{E}^{(k)}\cdot (\mathbf{I-A})^{-k}$ is the Tucker product of the $k-th$ order cumulant tensor $\mathcal{E}^{(k)}$ of the noises and the matrix $(\mathbf{I-A}^{-1})$ along each of its $k$ dimensions.  In other words,
\begin{center}
    $(\mathcal{E}^{(k)}\cdot\mathbf{(I-A)}^{-k})_{i_1,...,i_k}=\sum_{j_1,...,j_k}\mathcal{E}^{(k)}_{j_1,...,j_k}(\mathbf{(I-A)^{-1}})_{j_1,i_1}...(\mathbf{(I-A)^{-1}})_{j_k,i_k}$
\end{center}
\end{lemma}

The \textit{determinant of an order-k $n\times n\times ...\times n$ tensor T} is defined as:

\begin{center}
    $det$ $T=\underset{\sigma_2,...\sigma_k\in\mathcal{Q}(n)}{\sum}sign(\sigma_2)...sign(\sigma_k)\underset{i=1}{ \overset{n}{\prod}}T_{i,\sigma_2(i),...,\sigma_k(i)}$
\end{center}
where $\mathcal{Q}(n)$ is the set of permutations of the set $\{1,..,n\}$, $T_{i,\sigma_2(i),...,\sigma_k(i)}$ is the entry in the position determined by $i$ and the $k-1$ permutations of $i$.\\ \textcolor{white}{abc}
Before stating the tensor constraint, it is worth noticing that the higher-order cumulant is used here instead of higher order moments.  Recall that the $k$-th \textit{moment tensor} $Z^{(k)}$ of a random vector  $\mathbf{Z}=(Z_1,...Z_p)$ is the $p\times  ...\times p$ $(k\text{ times})$ table such that the entry at position $(i_1...i_k)$ is given by \cite{moment}:

\begin{center}
    $moment(Z_{i_1},...,Z_{i_k})=\mathbb{E}(Z_{i_1}Z_{i_2}...Z_{i_k})$
\end{center}
The moment function is not used because when the order $k$ is larger than 3, when variables in $\mathbf{Z}$ are independent,  then moment tensor $Z^{(k)}$ is not diagonal, which makes it less straightforward to derive the desired constraint \cite{robeva2020multitrek}.

\begin{theorem}\label{thm:TensorConstraint}
\textbf{Tensor Constraint: }Let $G=\langle\mathbf{V,E}\rangle$ be a DAG, and let $\mathbf{S_1,...,S_k\subset V}$ with the same cardinality.  Then, 
    $det$ $C^{(k)}_{\mathbf{S_1,...,S_k}}$ is identically zero
if and only if every system of k-treks between $\mathbf{S_1,...,S_k}$ has a sided intersection.
\end{theorem}
The tensor constraint is a generalized version of the rank constraint when the data is non-Gaussian, and can be used to detect the existence of latent variables when the rank constraint is not available.  For instance, in fig \ref{3trek}, the rank constraint cannot be used since there are not enough observed variables to form tetrads, but if the variables are non-Gaussian then we have $det$ $C^{(3)}_{{X_1,X_2,X_3}}\neq 0$, which implies that there exists a \textit{3-trek}.\footnote{In Chapter 4 we are going to show that there is limitation applying the tensor constraint to tensors with odd-numbered dimensions, but that limitation can only happen when there are more than one entry for each dimension.}\\ \textcolor{white}{abc}
\begin{figure}
\begin{center}
\begin{tikzpicture}[scale=0.15]
\tikzstyle{every node}+=[inner sep=0pt]
\draw [black] (37.7,-15.9) circle (3);
\draw (37.7,-15.9) node {$L$};
\draw [black] (25.6,-30.4) circle (3);
\draw (25.6,-30.4) node {$X_1$};
\draw [black] (37.7,-32.2) circle (3);
\draw (37.7,-32.2) node {$X_2$};
\draw [black] (50,-29.2) circle (3);
\draw (50,-29.2) node {$X_3$};
\draw [black] (35.78,-18.2) -- (27.52,-28.1);
\fill [black] (27.52,-28.1) -- (28.42,-27.8) -- (27.65,-27.16);
\draw [black] (37.7,-18.9) -- (37.7,-29.2);
\fill [black] (37.7,-29.2) -- (38.2,-28.4) -- (37.2,-28.4);
\draw [black] (39.74,-18.1) -- (47.96,-27);
\fill [black] (47.96,-27) -- (47.79,-26.07) -- (47.05,-26.75);
\end{tikzpicture}
\end{center}
    \caption{A structure that the Rank Constraint cannot identify and the tensor constraint can}
    \label{3trek}
\end{figure}
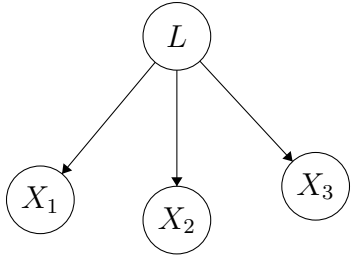When the determinant $C^{(k)}_{\mathbf{S_1,...,S_k}}$ is nonzero, the tensor constraint entails that for any $\mathbf{S_1,...,S_k\subset V}$, there is a system of $k$-$treks$ between $\mathbf{S_1,...,S_k}$ that has no sided intersection.  Assuming that measured variables cannot cause the latents, it implies that $\mathbf{S_1,...,S_k}$ have common causes.  As shown in the spider model case in \ref{fig:spidermodel}, the number of common causes cannot be concluded.\\ \textcolor{white}{abc}  
When the determinant $C^{(k)}_{\mathbf{S_1,...,S_k}}$ is zero, the tensor constraint entails that for any $\mathbf{S_1,...,S_k\subset V}$, there is no system of $k$-$treks$ between $\mathbf{S_1,...,S_k}$ that has no sided intersection, but it does not imply that the $\mathbf{S_1,...,S_k}$ have common causes fewer than $|\mathbf{S_1}|$ unless we make some additional assumptions, such as the number of direct latent common causes cannot exceed the number of measured variables.  
\subsection{Nonlinearity between Latents}
Similar to the rank constraint, for GIN and the tensor constraint to hold, only the structure of a part of the graph matters.  Theorem \ref{thm:GIN} is used to identify causal clusters, but I have proved that even if the latents that are causing different clusters are connected in a nonlinear way, the GIN condition can still be used to identify causal clusters.
\begin{theorem}\label{nonlinearGIN}
Suppose that random vectors $\mathbf{L_Y,Y,L_Z}$ and $\mathbf{Z}$ are related in the following way:
\begin{center}
    $\mathbf{Y=}A\mathbf{L_Y+E_Y}$,
    
    $\mathbf{Z=}B\mathbf{L_Z+E_Z}$.
\end{center}
Denote by $l_Y$ the dimensionality of $\mathbf{L_Y}$ and  $l_Z$ the dimensionality of $\mathbf{L_Z}$.  Assume $A$ is of full column rank.  Then, if 1) $Dim(\mathbf{Y})>l_Y$,  2) $\mathbf{E_Y}\indep \mathbf{L_Z}$, 3) $\mathbf{E_Y}\indep \mathbf{E_Z}$, 4) the cross-covariance matrix of $l_Y$ and $\mathbf{Z}$, $\Sigma_{\mathbf{L_Y Z}}=\mathbb{E}[\mathbf{L_{Y}Z}^T]$ has rank $l_Y$, then $E_{\mathbf{Y\parallel Z}}\indep \mathbf{Z}$, i.e.,$(\mathbf{Z,Y})$ satisfies the GIN condition.
\end{theorem}
Notice that for the rank constraint we discuss \textit{\textbf{linearity below the choke set}}, while for GIN we discuss the\textit{\textbf{ linearity under the latent variables}}, since GIN tests an asymmetric independence relation entailed by \textit{\textbf{d-separation}} while the rank constraint is about \textit{\textbf{t-separation}}.  In many situations, the latent parents of a causal cluster and the choke set between the same cluster and the rest of the variables contain the same latent variables.  Later we are going to show a lemma about the transition between GIN and the rank constraints, which will also show the relation between the latent parents of a cluster and the choke set. \\ \textcolor{white}{abc}
Similarly to the extension of the GIN, it can be proved that even if the latents are connected in a nonlinear way, the tensor constraint can still be used to identify the existence of $k-treks$ as long as the \textbf{LA under \textit{k-choke set}} is satisfied.  Since the structure and tools of the proof for properties of the tensor constraint are similar to that of the rank constraint, we are going to claim and prove the theorem after presenting the work of the rank constraint.

\section{Identifying Cycles with GIN and Rank Constraint}
The linearity and acyclicity between the measured variables and the latent causes enable GIN to cluster variables, while cycles between the measured and the latent variables introduce noises that cannot be separated by linear combinations that GIN aims to find.  With the presence of cycles, the rank constraint is able to perform clustering to identify latent variables.  In this section we first show that the rank constraints can be applied in partially non-linear models with cycles.  Then we propose two algorithms to identify cyclic structures in a latent causal model.\\ \textcolor{white}{abc}
Before introducing our work for identifying cycles, since we use the term `cluster' so often, we here provide an intuitive definition for \textit{causal cluster} by specifying conditions a set of variables needs to satisfy to be counted as a \textit{\textbf{causal cluster}} in our study:\\ 
\textbf{\textit{Causal Cluster: }} given a directed causal graph $G=\langle\mathbf{V,E}\rangle$, a set of variables $\mathbf{X\subset V}$ forms a cluster if there exists a parent set $\mathbf{Pa_G(X)\subset V}$ s.t. 
\begin{enumerate}
    \item $Y\in\mathbf{Pa_G(X)}$ if and only if $Y$ is a parent of some variables $X\in\mathbf{X}$ 
    \item  $X\in\mathbf{X}$ if and only if $X$ is a child of some $Y\in\mathbf{Pa_G(X)}$
    %\item $Y\in\mathbf{Pa_G(X)}$ if and only if $Y$ is a parent of every variable $X\in\mathbf{X}$ 
    %\item  $X\in\mathbf{X}$ if and only if $X$ is a child of every $Y\in\mathbf{Pa_G(X)}$
    \item there does not exist a partition of $\mathbf{X=X_1\cup X_2}$ and $\mathbf{Pa_G(X)=Pa_G(X_1)\cup Pa_G(X_2)}$ such that $\mathbf{X_1, Pa_G(X_1)}$ and $\mathbf{X_1, Pa_G(X_1)}$ both satisfies the condition 1 and 2.
\end{enumerate}\textcolor{white}{abc} 
\textit{\textbf{Example.}} Consider fig \ref{fig:detectCyclerank}. $\{X_4, X_5, X_3\}$ is a causal cluster with the parent set $\{L_1\}$; $\{X_4, X_5\}$ is not a causal cluster since there does not exist a parent set satisfying the `if' part in the condition 2; $\{X_1, X_5, X_7\}$ is not a causal cluster since the parent set that satisfies condition 1 is $\{L_1, L_2, L_3\}$, but the $\{X_1, X_5, X_7\}$ does not satisfy condition 2; $\{X_1,...,X_7\}$ is not a cluster neither, since there are partitions of this set as well as its parent set that satisfies condition 1 and 2.%there does not exists a parent set for $\{X_1, X_5, X_7\}$ to satisfy the two conditions.\\ \textcolor{white}{abc}

\subsection{Rank Constraint with Cycles under the choke sets}
We are using a lemma proved in paper \cite{Sullivant_2010} by Sullivant et al (lemma 3.7 from the original paper):\begin{lemma}\label{lm:detsysdipath}
Given a directed graph $G$ (not necessarily acyclic), and $X = (I-\Lambda)^{-T}\epsilon$, then
$det(I-\Lambda)^{-1}_{R,S}$ is identically zero iff every system of $|R|=|S|$ many directed paths from R to S has two paths which share a vertex.
\end{lemma}
We are going to prove that the rank constraint applies to the model where variables are \textbf{linear below the choke set}:\\
\textbf{Linear below the choke set:} Given a directed graph $G = \langle \mathbf{V, E}\rangle$, and $\mathbf{A, B, C_A, C_B\subset V}$, $\mathbf{A, B}$ are linear below $\mathbf{C_A}$ and $\mathbf{C_B}$ for $\mathbf{A}$ and $\mathbf{B}$ if for every directed path in $G$ (not necessarily self-avoiding) from some nodes in $\mathbf{C_A}$ to some nodes in $\mathbf{A}$, every edge $X_i\rightarrow X_j$ on this path is linear, i.e., each $X_j$ can be written as a linear equation $X_j = a_{ij}X_i+\epsilon'_{j_{\mathbf{C_A\rightarrow A}}}$ with noise $\epsilon'_{j_{\mathbf{C_A\rightarrow A}}}$ independent from $X_i$.\\ \textcolor{white}{abc} 
We further assume the model $M = \langle P_M, G_M\rangle$ to be a \textbf{generalized additive model}. 
 That is to say, given $P_M$ over $\mathbf{V}$ for $G=\langle \mathbf{V,E}\rangle$ respecting the \textbf{Causal Markov Assumption}, any $X_i\in\mathbf{V}$ can be written as:
\begin{center}
    $X_i =  \underset{X_j\in Pa_M(X_i)}{\sum} f_{i,j}(X_j)+\epsilon_i$  
\end{center}
where $f_{i,j}$ are some smooth functions and $\epsilon_i$ is independent of all $X_j \in Pa_M(X_i)$.
\begin{theorem}\label{thm:ranklinear}
Given a generalized additive model with a directed graph $G = \langle \mathbf{V, E}\rangle$ and $\mathbf{A, B}$, $\mathbf{C_A, C_B\subset V}$ such that $\mathbf{(C_A;C_B)}$ t-separates $A$ and $B$.  Given $\mathbf{|A|=|B|}=n$, for all covariance matrices entailed by a fixed parameter structural equation model with G that is linear below $\mathbf{C_A}$ and $\mathbf{C_B}$ for $\mathbf{A}$ and $\mathbf{B}$, then with $\Sigma_{A,B}$ being the submatrix of covariance matrix on $\mathbf{A}$ and $\mathbf{B}$, $rank(\Sigma_{\mathbf{A,B}})\leq\mathbf{|C_A|+|C_B|}$.

\end{theorem}

\begin{proof}
To prove the theorem, it suffices to prove that if $\mathbf{|C_A|+|C_B|<|A|}$, then $det(\Sigma_{A,B}) = 0$.\\ \textcolor{white}{abc}
Let $\mathcal{D}_{C_A\rightarrow A}\subset V$ be the set of variables on a directed path from some nodes in $\mathbf{C_A}$ to some nodes in $\mathbf{A}$, excluding $\mathbf{C_A}$ and $\mathbf{A}$. \\ \textcolor{white}{abc}
Similarly we have $\mathcal{D}_{C_B\rightarrow B}$.  \\ \textcolor{white}{abc}
We call a variable $X\in \mathbf{X}$ $A-side$ if it is on a directed path between $\mathbf{C_A}$ to $A$, so any variable in $G'$ is either $A-side$ or $B-side$, or both.\\ \textcolor{white}{abc}  
Now we consider a graph $G' $ by restrict the original graph $G=\langle\mathbf{V, E}\rangle$ to the set $\mathcal{D}_{C_A\rightarrow A}\cup\mathcal{D}_{C_B\rightarrow B}\mathbf{\cup C_A\cup C_B\cup A\cup B}$. Following the definition of \textit{parent} in a graph ($X_i$ is a parent of $X_j$ if $X_i\rightarrow X_j\in \mathbf{E}$) and the condition that it's linear under $\mathbf{C_A}$ and $\mathbf{C_B}$ for $\mathbf{A}$ and $\mathbf{B}$, we can represent $G'$ as a structural linear equation:
\begin{equation}
    X = LX+\mathcal{E}_{G'}
\end{equation}
where $X = \mathcal{D}_{C_A\rightarrow A}\cup\mathcal{D}_{C_B\rightarrow B}\cup \mathbf{C_A\cup C_B}\cup \mathbf{A\cup B}$. Having $\Phi$ as the covariance matrix of $\mathcal{E}_{G'}$, we can write the covariance matrix of $\mathbf{X}$ as:
\begin{equation}
\Sigma_{\mathbf{X,X} }= (I-L)^{-T}\Phi(I-L)^{-1}
\end{equation}\textbf{Linear below $\mathbf{C_A}$ and $\mathbf{C_B}$} only requires the edges between $\mathbf{C_A}$ to $\mathbf{A}$ and $\mathbf{C_B}$ to $\mathbf{B}$ to be linear, and it does not guarantee that every edge in $G'$ is linear.  If there are non-linear edges in $G'$, the causal influence represented by this edge is not seen in the coefficient matrix $L$ but is included in $\mathcal{E}_G'$. This suggests that $\Phi$ is not diagonal since,  unlike the original $G$,  noises in $G'$ are not necessarily independent from each other. \\ \textcolor{white}{abc}
To proceed in the proof, we use the lemma below:
\begin{lemma}\label{lemma:rankcyclic0}
For any $X\in \mathcal{D}_{C_A\rightarrow A}\cup \mathbf{A}$ and variable $Y\in \mathcal{D}_{C_B\rightarrow B}\cup \mathbf{B}$, $\Phi_{X,Y} = 0$.
\end{lemma} The proof of the lemma can be found in the Appendix, which mainly exploits the condition of $\mathbf{|C_A|+|C_B|<|A|}$ and the independence of noise from the generalized additive model.  Now we apply the Cauchy-Binet Formula twice to $\Sigma_{\mathbf{A,B}}$ and get:\\ \textcolor{white}{abc}
\begin{equation}\label{eq:dter}
det\Sigma_{\mathbf{A,B}} = \sum_{\substack{\mathbf{R,S\subset X_{G'}}\\ \textcolor{white}{abc} \mathbf{|R| = |S| = |A| = |B|}}}det(I-L)^{-1}_\mathbf{R,A}det\Phi_\mathbf{R,S}det(I-L)^{-1}_\mathbf{S,B}
\end{equation}where $\mathbf{X_{G'}}$ is the set of variables restricted to $G'$.\\ \textcolor{white}{abc}
Having $\mathbf{|C_A|+|C_B|<|A|}$, we are going to show that $det\Sigma_\mathbf{A,B}$ is identically zero.  To achieve that, we have to show that the product on the RHS is zero for every $\mathbf{R, S\subset X_{G'}}$, which is satisfied if any of the three determinants is zero.\\ \textcolor{white}{abc}
If $L$ is not a triangular matrix, then the directed graph has cycles, in which case the $\lambda^P$ can go to infinity sometimes, which makes $I-L$ not invertible.  We rule out this situation by assuming that $I-L$ is invertible(assuming eigenvalues are all nonzero).\\ \textcolor{white}{abc}
Now we want to show that no matter what type of $\mathbf{R}$ and $\mathbf{S}$ are, at least one of the three determinant is identically zero. Specifically, we prove the lemma:
\begin{lemma}\label{lemma:atleast0}
    For all $\mathbf{R,S\subset \mathbf{X_{G'}}}$, if $det(I-L)^{-1}_{\mathbf{R,A}} \neq 0$ and $det(I-L)^{-1}_{\mathbf{S,B}}\neq 0$, then $det\Phi_{\mathbf{R,S}}=0$.
\end{lemma}
The proof of the lemma can be found in the appendix.  The main idea of the proof is that based on the \textbf{Lemma}  \ref{lemma:rankcyclic0}, when the $det(I-L)^{-1}_{\mathbf{R,A}} \neq 0$ and $det(I-L)^{-1}_{\mathbf{S,B}}\neq 0$, there are zero submatrices in $\Phi_{\mathbf{R,S}}$ with dimensions large enough that leads to $det\Phi_{\mathbf{R,S}}=0$.\\ \textcolor{white}{abc}
We complete the proof using \textbf{lemma }
 \ref{lemma:atleast0}. Based on equation \ref{eq:dter}, the determinant is nonzero only if there exists some $\mathbf{R,S\subset X_{G'}}$ such that $det(I-L)^{-1}_{R,A}det\Phi_{R,S}det(I-L)^{-1}_{S,B}$ is not zero, which can only happen if every determinant in the product is nonzero.  By \textbf{lemma }\ref{lemma:rankcyclic0} it is not possible, so $det\Sigma_{\mathbf{A,B}} = 0$.
\end{proof}
\textbf{\textit{Example.}} Consider fig \ref{fig:examplecycleunderchokeset} where $\{X_1,...X_6\}$ are caused by $L$ and also causing $L$.  The causal relation between $\{X_1,...X_6\}$ and $L$ are linear, and the coefficient of edges 
$L\rightarrow X_i$ is $a_i$ and $X_i\rightarrow L$ is $b_i$. $(\{L\};\{L\})$ \textbf{\textit{t-separates}} $\{X_1, X_2\}$ and $\{X_4, X_6\}$, making $(\{L\};\{L\}) $ a choke set. Since $\{X_1,...X_6\}$ is linear under the choke set, according to the rank constraint, $rank(\Sigma_{\{X_1, X_2\}, \{X_4, X_6\}}) =2$, therefore $det\Sigma_{\{X_1, X_2\}, \{X_4, X_6\}}\neq 0$.  Similarly, since  $(\{L\};\{L\})$ \textbf{\textit{t-separates}} $\{X_1, X_2, X_3\}$ and $\{X_4, X_5, X_6\}$, $rank(\Sigma_{\{X_1, X_2, X_3\}, \{X_4, X_5, X_6\}}) =2$, therefore\newline $det\Sigma_{\{X_1, X_2, X_3\}, \{X_4, X_5, X_6\}}= 0$.
\begin{figure}[H]
    \begin{center}
\begin{tikzpicture}[scale=0.12]
\tikzstyle{every node}+=[inner sep=0pt]
\draw [black] (14.2,-14.5) circle (3);
\draw (14.2,-14.5) node {$X_1$};
\draw [black] (17.4,-21) circle (3);
\draw (17.4,-21) node {$X_2$};
\draw [black] (34.5,-25) circle (3);
\draw (34.5,-25) node {$X_4$};
\draw [black] (25.3,-25) circle (3);
\draw (25.3,-25) node {$X_3$};
\draw [black] (42.7,-22.3) circle (3);
\draw (42.7,-22.3) node {$X_5$};
\draw [black] (47,-15.3) circle (3);
\draw (47,-15.3) node {$X_6$};
\draw [black] (30.9,-10.1) circle (3);
\draw (30.9,-10.1) node {$L$};
\draw [black] (28.883,-12.32) arc (-44.08166:-58.08301:46.862);
\fill [black] (28.88,-12.32) -- (27.97,-12.55) -- (28.69,-13.24);
\draw [black] (19.465,-18.825) arc (134.91306:122.92227:54.592);
\fill [black] (19.47,-18.82) -- (20.38,-18.61) -- (19.68,-17.91);
\draw [black] (28.19,-11.386) arc (-66.82078:-83.65833:38.845);
\fill [black] (28.19,-11.39) -- (27.26,-11.24) -- (27.65,-12.16);
\draw [black] (16.957,-13.319) arc (111.41953:98.10136:48.876);
\fill [black] (16.96,-13.32) -- (17.88,-13.49) -- (17.52,-12.56);
\draw (21.35,-10.97) node [above] {$a_1$};
\draw [black] (25.918,-22.065) arc (166.09566:152.70799:42.843);
\fill [black] (29.43,-12.72) -- (28.62,-13.2) -- (29.51,-13.66);
\draw [black] (30.279,-13.034) arc (-13.9461:-27.25025:43.107);
\fill [black] (26.77,-22.38) -- (27.58,-21.9) -- (26.69,-21.44);
\draw [black] (33.325,-22.24) arc (-159.22746:-173.60661:37.603);
\fill [black] (33.33,-22.24) -- (33.51,-21.32) -- (32.57,-21.67);
\draw [black] (40.323,-20.471) arc (-129.36767:-142.54201:48.11);
\fill [black] (40.32,-20.47) -- (40.02,-19.58) -- (39.39,-20.35);
\draw [black] (44.049,-14.76) arc (-101.98246:-113.81652:53.214);
\fill [black] (44.05,-14.76) -- (43.37,-14.1) -- (43.16,-15.08);
\draw [black] (33.258,-11.954) arc (50.16113:37.92919:51.759);
\fill [black] (33.26,-11.95) -- (33.55,-12.85) -- (34.19,-12.08);
\draw [black] (33.873,-10.496) arc (80.20986:63.99116:39.051);
\fill [black] (33.87,-10.5) -- (34.58,-11.12) -- (34.75,-10.14);
\draw (40.4,-11.26) node [above] {$b_6$};
\draw [black] (31.997,-12.892) arc (19.54226:7.62367:45.203);
\fill [black] (32,-12.89) -- (31.79,-13.81) -- (32.74,-13.48);

\end{tikzpicture}
\end{center}
\caption{cycles under $L$ leading to correlation submatrix with rank 2.}
\label{fig:examplecycleunderchokeset}
\end{figure}
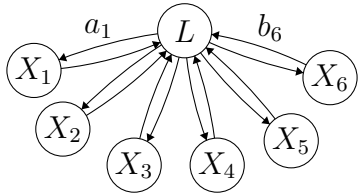\alignnewline\\ \textcolor{white}{abc}
\textbf{Theorem }\ref{thm:ranklinear} suggests that even with the presence of cycles between the measured variables and the latents, as long as there are enough measured variables sharing the same set of latents, rank constraint can still identify the cluster.  In the next section I will present a special type of latent causal structure and how the cycles within causal clusters (an exhaustive set of variables and the direct causes they share) can be identified. 
\section{Identifying Cycles}

\subsection{Linear Latent Hierarchical Cyclic Model($L^2HCM$)}
We adapt the latent hierarchical causal models with measured variables $\mathbf{X_{\mathcal{G}}}=\{X_1, ..., X_m\}$ and latent variables $\mathbf{L_{\mathcal{G}}}=\{L_1, ..., L_m\}$ with linear relationships:
\begin{equation}\label{equ:measureLLH}
    X_i=\Sigma_{L_j\in Pa(X_i)}b_{ij}L_j+\epsilon_{X_i}
\end{equation}
\begin{equation}\label{equ:latentLLH}
    L_i=\Sigma_{P^L_k\in Pa(L_i)}c_{ik}P^L_k+\epsilon_{L_i}
\end{equation}

In this particular section, we focus on the presence of cycles between children and parents.  Therefore if $L_i$ in equation \ref{equ:latentLLH} directly causes some measured variables, then its parents $P^L_k$ can either be a $X_k$ or $L_k$.  In other words, we are allowing the cyclic graph in which $X$ is a parent of $L$ only if $L$ is a parent of $X$.\\ \textcolor{white}{abc}
Huang et al. developed an algorithm identifying Irreducible Linear Latent Hierarchical ($IL^2H$) Graph using rank constraints \cite{huang2022latent}. We are going to adapt some of the conditions from $IL^2H$ graph in addition allowing the presence of cycles and more restrictions for latent parents with measured children:
\begin{enumerate}
    \item \textbf{Effective Cardinality\cite{huang2022latent}: }For a set of latent variables $\mathbf{L}$, denote by $\mathbf{C}$ the largest subset of $PCh_{\mathcal{G}}(\mathbf{L})$ \footnote{$PCh$ means pure children here, and pure children means that the set of parents of these children is $\mathbf{L}$ and nothing else; we also allow measured children causing their latent parents to still be pure children of the same set of latent parents.} such
that there is no subset $\mathbf{C'\subset C}$ satisfying $|\mathbf{C'}|>|Pa_\mathcal{G}(\mathbf{C'})|$ and $|Pa_\mathcal{G}(\mathbf{C'})|<|\mathbf{L}|$. Then, the effective cardinality of $\mathcal{L}$'s
pure children is $\mathbf{|C|}$.
\item \textbf{pc\footnote{\textit{pc} stands for `possibly cyclic.'}-Atomic latent cover:\footnote{the latent cover is named `Possibly Cyclic Atomic latent cover' to distinguish it from the atomic latent cover in \cite{huang2022latent}, and also the clusters can be cyclic here in $L^2HCM$. }}
A set of latent variables $\mathbf{L}=\{L_1,...L_k\}$ is a pc-atomic latent cover if:
\begin{enumerate}
    \item There exists a subset of  pure children $\mathbf{C}$ with the effective cardinality $\geq k+1$.
    \item It has $k+1$ neighbors in addition to $\mathbf{C}$.
    \item $\mathbf{L}$ cannot be partitioned into more subsets where each set satisfies the first and second conditions.
    \item If there are cycles between the pure children of $\mathbf{L}$ and $\mathbf{L}$, replace the $k$ in condition $(a)$ and $(b)$ as $2k$.
\end{enumerate}
\item \textbf{Extroverted Parent}: A latent is an extroverted parent if it is on treks between its measured children and every measured variable that is not among its measured children, otherwise we call the latent an \textit{\textbf{introverted parent.}}
\end{enumerate}
Introverted parents are the latents that cannot be discovered by testing constraints between its set of measured children and the rest of the measured variables.  Later in this chapter, we will show with fig \ref{fig:alternativeofcycleunderLA} that sometimes variables with introverted parents cannot be distinguished from variables on cycles with latent parents.  To guarantee the correctness of our methods, we are going to rule out introverted parents.\\ \textcolor{white}{abc}

Now we state the conditions for \textbf{Linear Latent Hierarchical Cyclic Model($L^2HCM$)}:
\begin{enumerate}
\item Every latent belongs to at least one pc-atomic latent cover and every latent is an extroverted parent.
\item for any pair of latent atomic covers $(\mathbf{L_A, L_B})$, if $PDe_\mathcal{G}(\mathbf{L_A})\cap PDe_\mathcal{G}(\mathbf{L_B})\neq\emptyset$\footnote{$PDe_\mathcal{G}(\mathbf{L})$ is a set of pure descendant of $L$ in $\mathcal{G}$, which is defined as all recursive pure children of $L$\cite{huang2022latent}.}, then either (a) $\mathbf{L_A \subset L_B} $ or (b) $ 
\mathbf{L_A \subset} PDe_\mathcal{G}(\mathbf{L_B}) $ or (c)  $\mathbf{L_B} \subset \mathbf{L_A} $ or (d) $ \mathbf{L_B \subset} PDe_\mathcal{G}(\mathbf{L_A})$\cite{huang2022latent}.
\end{enumerate}

\subsection{Detecting Cycles for Gaussian data}
We first discuss the situation where the distribution of the noise in the $L^2HCM$ model is Gaussian so only the rank constraint is applicable. Due to the existence of cycles, the latents of a cyclic cluster can be the choke set on two sides or only one side, leading to detectable change of the rank of certain correlation submatrices.  We illustrate such idea with an example first, then we prove the formal theorem.\\ \textcolor{white}{abc}
\textbf{\textit{Example.}} Consider the example in fig \ref{fig:detectCyclerank}.  There are cycles between $\{X_3, X_4, X_5\}$ and $L_1$, so that $rank(\Sigma_{\{X_3, X_4\}, \{X_5, X_6\}})= rank(\Sigma_{\{X_5, X_4\}, \{X_3, X_6\}})  = 2$ and  $rank(\Sigma_{\{X_3, X_4, X_5 \}, \{X_6, X_1\}}) = 1$.  The first equation indicates that, if there are no cycles, $\{X_3, X_4, X_5 \}$ shares at least 2 common causes, which is inconsistent with the second equation.  Such inconsistency is due to the edges from $\{X_3, X_4, X_5\}$ to $L_1$. 

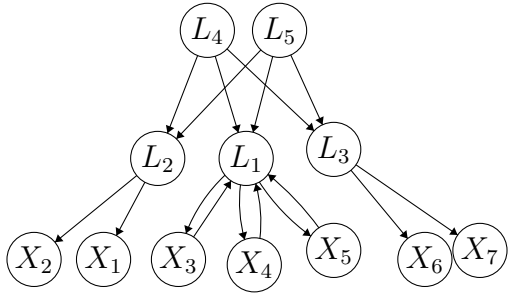
\begin{figure}[H]
\begin{center}
\begin{tikzpicture}[scale=0.12]
\tikzstyle{every node}+=[inner sep=0pt]
\draw [black] (20.2,-32.7) circle (3);
\draw (20.2,-32.7) node {$X_1$};
\draw [black] (12.5,-32.7) circle (3);
\draw (12.5,-32.7) node {$X_2$};
\draw [black] (36.9,-33.3) circle (3);
\draw (36.9,-33.3) node {$X_4$};
\draw [black] (45.7,-31.7) circle (3);
\draw (45.7,-31.7) node {$X_5$};
\draw [black] (55.8,-32.7) circle (3);
\draw (55.8,-32.7) node {$X_6$};
\draw [black] (61.9,-31.7) circle (3);
\draw (61.9,-31.7) node {$X_7$};
\draw [black] (28.5,-32.7) circle (3);
\draw (28.5,-32.7) node {$X_3$};
\draw [black] (36,-21.5) circle (3);
\draw (36,-21.5) node {$L_1$};
\draw [black] (45.7,-20.5) circle (3);
\draw (45.7,-20.5) node {$L_3$};
\draw [black] (26.2,-21.5) circle (3);
\draw (26.2,-21.5) node {$L_2$};
\draw [black] (31.7,-7.3) circle (3);
\draw (31.7,-7.3) node {$L_4$};
\draw [black] (39.7,-7.3) circle (3);
\draw (39.7,-7.3) node {$L_5$};
\draw [black] (43.158,-30.111) arc (-125.84013:-147.03845:22.479);
\fill [black] (43.16,-30.11) -- (42.8,-29.24) -- (42.22,-30.05);
\draw [black] (35.832,-30.501) arc (-164.55007:-186.72679:15.834);
\fill [black] (35.83,-30.5) -- (36.1,-29.6) -- (35.14,-29.86);
\draw [black] (29.327,-29.82) arc (159.05709:133.32699:17.456);
\fill [black] (29.33,-29.82) -- (30.08,-29.25) -- (29.15,-28.89);
\draw [black] (47.61,-22.81) -- (53.89,-30.39);
\fill [black] (53.89,-30.39) -- (53.76,-29.45) -- (52.99,-30.09);
\draw [black] (48.17,-22.21) -- (59.43,-29.99);
\fill [black] (59.43,-29.99) -- (59.06,-29.13) -- (58.49,-29.95);
\draw [black] (30.17,-30.21) -- (34.33,-23.99);
\fill [black] (34.33,-23.99) -- (33.47,-24.38) -- (34.3,-24.94);
\draw [black] (37.058,-24.302) arc (15.31709:-6.59394:16.008);
\fill [black] (37.06,-24.3) -- (36.79,-25.21) -- (37.75,-24.94);
\draw [black] (38.42,-23.271) arc (51.05393:36.06749:31.338);
\fill [black] (38.42,-23.27) -- (38.73,-24.16) -- (39.36,-23.39);
\draw [black] (24.78,-24.14) -- (21.62,-30.06);
\fill [black] (21.62,-30.06) -- (22.44,-29.59) -- (21.55,-29.11);
\draw [black] (23.88,-23.4) -- (14.82,-30.8);
\fill [black] (14.82,-30.8) -- (15.76,-30.68) -- (15.13,-29.91);
\draw [black] (32.57,-10.17) -- (35.13,-18.63);
\fill [black] (35.13,-18.63) -- (35.38,-17.72) -- (34.42,-18.01);
\draw [black] (30.62,-10.1) -- (27.28,-18.7);
\fill [black] (27.28,-18.7) -- (28.04,-18.14) -- (27.11,-17.78);
\draw [black] (37.63,-9.47) -- (28.27,-19.33);
\fill [black] (28.27,-19.33) -- (29.18,-19.09) -- (28.46,-18.4);
\draw [black] (33.88,-9.36) -- (43.52,-18.44);
\fill [black] (43.52,-18.44) -- (43.28,-17.53) -- (42.59,-18.26);
\draw [black] (38.94,-10.2) -- (36.76,-18.6);
\fill [black] (36.76,-18.6) -- (37.44,-17.95) -- (36.47,-17.7);
\draw [black] (40.94,-10.03) -- (44.46,-17.77);
\fill [black] (44.46,-17.77) -- (44.58,-16.83) -- (43.67,-17.25);
\end{tikzpicture}
\end{center}
    \caption{the existence of cycles between $\{X_3, X_4, X_5\}$ and $L_1$ can be detected with rank constraints.}
    \label{fig:detectCyclerank}
\end{figure}

\begin{theorem}\label{thm:rankdetectcycles}
 Given a $L^2HCM$ model, consider $\mathbf{Y}$ with $\mathbf{|Y|}=n$.  Let $\mathbf{Z = V\setminus Y}$. If \\ \textcolor{white}{abc}i) $\mathbf{|Z|}>\mathbf{|Y|}$, $rank(\Sigma_{\mathbf{Z, Y}}) \stackrel{(1)}{=} rank(\Sigma_{\mathbf{Z, Y^S}}) \leq  n' \stackrel{(2)} {<}rank(\Sigma_{\mathbf{V\setminus Y^S, Y^S}}) < n/2$, for every proper subset $\mathbf{Y^S_j\subset Y_j}$ s.t. $\mathbf{|Y^S_j|} \geq n'$;\\ \textcolor{white}{abc}
ii) No subset $\mathbf{Y'\subset Y}$ s.t. there exists a $\mathbf{ W\supset Y'}$ and $\mathbf{Z'= V\setminus W}$ satisfies $rank(\Sigma_{\mathbf{Z', W}}) < \mathbf{|W|}$;\\ \textcolor{white}{abc}
  then there are cycles under the common causes of $\mathbf{Y}$.
\end{theorem}
\begin{proof} We divide the proof into two parts.  In part 1 we prove that the latent parent set $\mathbf{C_{Y}}$ of $\mathbf{Y}$ has a size no larger than $n'$;  in part 2 we prove the existence of cycles between $\mathbf{Y}$ and $\mathbf{C_Y}$.
\begin{enumerate}
    \item  By $(1)$ we know that the latent parent set $\mathbf{C_{Y}}$ of $\mathbf{Y_j}$ has a size no larger than $n'_j$.  We now prove this conclusion by contradiction:\\ \textcolor{white}{abc} Assume that $|\mathbf{C_{Y}}|> n'$. Since each latent is from at least one pc-atomic latent cover, we can partition $\mathbf{C_Y=\cup_{i\in[m]}\mathbf{C^i_Y}}$ s.t. each of $\mathbf{C^i_Y}$ belongs to a different pc-atomic latent cover. We denote the corresponding pc-atomic latent cover as $\mathbf{\hat{L}^i_Y}$. Notice that $\mathbf{|\hat{L}^i_Y|\geq |C^i_Y|}$ since the latter is a subset of the former.\\ \textcolor{white}{abc}  
    By the definition of pc-atomic latent cover, each of the $\mathbf{\hat{L}^i_Y}$ has at least $|\mathbf{C^i_Y}|+1$ many pure children and $|\mathbf{C^i_Y}|+1$ many additional neighbors.  By ii) we know that for any $\mathbf{\hat{L}^i_Y}$, $\mathbf{Y}$ does not contain any pure children of such cover.  Therefore there exists at least $\sum_{i}(|\mathbf{\hat{L}}^{i}_Y| + 1) $ many pure children of each $\mathbf{\hat{L}}^{i}_Y$ that are not included in $\mathbf{Y}$.  Since $\mathbf{|\hat{L}^i_Y|\geq |C^i_Y|}$ we further have there are at least $\sum_{i}(|\mathbf{\hat{L}}^{i}_Y| + 1)\geq\sum_{i}(|\mathbf{C}^{i}_Y| + 1)>|\mathbf{C_Y}|>n'$ many pure children not included in $\mathbf{Y}$.  Specifically for each $\mathbf{C_i\subset\mathbf{\hat{L}}}^{i}_Y$, there are at least $|\mathbf{C_i}|$ many variables (pure children of $\mathbf{\hat{L}}^i_Y$) that are being connected by treks with $\mathbf{Y}$ through $\mathbf{C_i}$.  Denote these children with $Ch(\mathbf{C_Y})$.  In other words, we should at least have $(\emptyset, \mathbf{C_Y})$ $t-sep$ $\mathbf{Y}$ and $Ch(\mathbf{C_Y})$ not to mention $Ch(\mathbf{C_Y})\subset\mathbf{V\setminus Y}$.  Therefore $rank(\Sigma_{\mathbf{Z,Y}}) \geq n'$, contradiction.
    
    \item By $(2)$ we know that there exists some partition $\mathbf{Y = Y^C\cup Y^S } $ such that there exists some treks $T = \langle T^C, T^S\rangle$ between $\mathbf{ Y^C}$ and $\mathbf{ Y^S}$ such that $\mathbf{C_{Y}}$ does not contain a variable on every $T^S$, since otherwise $(\emptyset, \mathbf{C_{Y}})$ $t-sep$ $\mathbf{V\setminus\mathbf{Y^S}}$ and $\mathbf{Y^S}$ so we should have 
    $rank(\Sigma_ {\mathbf{V\setminus Y^S, Y^S}}) \leq n'$.  Assume that $T^S$ contains some latents. Then such latent should be a latent parent that is not included in $\mathbf{C_Y}$, which means that the latent parent is not an extroverted parent, since otherwise it will be on the treks between $\mathbf{Y}$ and $\mathbf{V\setminus Y}$ and included in $\mathbf{C_Y}$.  Therefore $T^S$ does not contain any latent variables.  Since we do not allow direct causal connection between measured variables, such $T^S$ can only contain some variables in $\mathbf{Y^S}$, which means that such variables are also the source of the trek, i.e. causes of the variables in the other side. Since variables are only connecting with each other through common causes, it means that there are cycles under the latent parent set of $\mathbf{Y}$.    
    
\end{enumerate}

\end{proof}

Notice that the converse of theorem \ref{thm:rankdetectcycles} is not always true.  For instance, in fig \ref{fig:rankcannotdetect}, there are cycles between $\{X_3, X_4, X_5\}$ and $L_1$, but $rank(\Sigma_{\{X_6, X_1\}, \{X_3, X_4, X_5 \}}) = rank(\Sigma_{ \{X_6, X_1\},\{X_3, X_4\}}) = rank(\Sigma_{ \{X_6, X_5, X_1\},\{X_3, X_4\}}) = 2$ since $\{L_1\}$ is in the choke set of both sides.  The rank constraint cannot identify this structure because the edge between $L_2$ and $L_1$ goes into $L_1$ while the one between $L_3$ and $L_1$ comes out of $L_1$.

\begin{figure}
\begin{center}
\begin{tikzpicture}[scale=0.12]
\tikzstyle{every node}+=[inner sep=0pt]
\draw [black] (14.7,-15.7) circle (3);
\draw (14.7,-15.7) node {$X_1$};
\draw [black] (20.2,-20.8) circle (3);
\draw (20.2,-20.8) node {$X_2$};
\draw [black] (39.8,-30) circle (3);
\draw (39.8,-30) node {$X_4$};
\draw [black] (29.8,-27.3) circle (3);
\draw (29.8,-27.3) node {$X_3$};
\draw [black] (47.6,-27.3) circle (3);
\draw (47.6,-27.3) node {$X_5$};
\draw [black] (56.7,-22.3) circle (3);
\draw (56.7,-22.3) node {$X_6$};
\draw [black] (39.2,-13.1) circle (3);
\draw (39.2,-13.1) node {$L_1$};
\draw [black] (63.6,-16.6) circle (3);
\draw (63.6,-16.6) node {$X_7$};
\draw [black] (29.8,-6.1) circle (3);
\draw (29.8,-6.1) node {$L_2$};
\draw [black] (47.6,-5.4) circle (3);
\draw (47.6,-5.4) node {$L_3$};
\draw [black] (31.079,-24.587) arc (152.99852:139.9946:48.979);
\fill [black] (37.2,-15.34) -- (36.3,-15.63) -- (37.07,-16.27);
\draw [black] (37.918,-15.812) arc (-27.0425:-39.96438:49.285);
\fill [black] (31.8,-25.06) -- (32.69,-24.77) -- (31.93,-24.13);
\draw [black] (39.237,-27.054) arc (-171.06243:-184.87094:45.679);
\fill [black] (39.24,-27.05) -- (39.61,-26.19) -- (38.62,-26.34);
\draw [black] (45.706,-24.974) arc (-142.72267:-156.06452:45.479);
\fill [black] (45.71,-24.97) -- (45.62,-24.04) -- (44.82,-24.64);
\draw [black] (41.07,-15.446) arc (36.80009:24.41273:48.925);
\fill [black] (41.07,-15.45) -- (41.15,-16.39) -- (41.95,-15.79);
\draw [black] (39.686,-16.06) arc (7.75755:-3.69092:54.942);
\fill [black] (39.69,-16.06) -- (39.3,-16.92) -- (40.29,-16.79);
\draw [black] (27.27,-7.71) -- (17.23,-14.09);
\fill [black] (17.23,-14.09) -- (18.18,-14.08) -- (17.64,-13.24);
\draw [black] (28.16,-8.61) -- (21.84,-18.29);
\fill [black] (21.84,-18.29) -- (22.7,-17.89) -- (21.86,-17.34);
\draw [black] (32.21,-7.89) -- (36.79,-11.31);
\fill [black] (36.79,-11.31) -- (36.45,-10.43) -- (35.85,-11.23);
\draw [black] (41.41,-11.07) -- (45.39,-7.43);
\fill [black] (45.39,-7.43) -- (44.46,-7.6) -- (45.14,-8.34);
\draw [black] (49.02,-8.04) -- (55.28,-19.66);
\fill [black] (55.28,-19.66) -- (55.34,-18.72) -- (54.46,-19.19);
\draw [black] (50.06,-7.12) -- (61.14,-14.88);
\fill [black] (61.14,-14.88) -- (60.77,-14.01) -- (60.2,-14.83);
\end{tikzpicture}
\end{center}
    \caption{the cycles between $L_1$ and $\{X_3, X_4, X_5\}$ cannot be detected with the rank constraints}
    \label{fig:rankcannotdetect}
\end{figure}
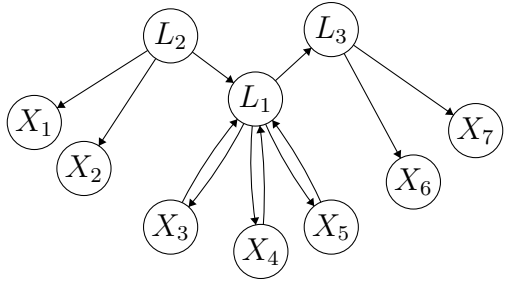
\alignnewline\\ \textcolor{white}{abc}
Generally speaking, such asymmetry cannot be identified under Gaussianity without first knowing if any cluster is cycle, since the $L_2\rightarrow L_1\rightarrow L_3$, $L_3\rightarrow L_1\rightarrow L_2$ and $L_2\leftarrow L_1\rightarrow L_3$ are markov equivalent.  The asymmetry of non-Gaussian distribution reveals more information of the structure of the graph so that cycles under the latent causes may be identified with fewer restrictions.

\subsection{Detect Cycles for non-Gaussian Data}
GIN tests if there exist linear combinations of a set of variables to be independent from some other variables, or whether the non-Gaussian components of two sets of variables can be separated.  Cycles under latents may make some measured variables ancestors of some other measured variables, and linear combinations of the ancestors cannot separate their non-Gaussian noises from the descendants.  To use GIN and the rank constraint to detect cycles, clusters can be identified by rank constraints first, then GIN can be used to identify the ancestor-descendant relation between variables that are children in different clusters.  To see the validity of this method, we need to use a lemma showing the relation between GIN and \textit{t-separation} (its proof can be found in the Appendix).\\ \textcolor{white}{abc}

\begin{lemma}
\label{GINdsep}
        Consider two sets of variables $\mathbf{Z, Y}$ in a $L^2HCM$ model. Assume faithfulness holds for the $L^2HCM$. If there are any treks between $\mathbf{Z}$ and $\mathbf{Y}$, $(\mathbf{Z, Y})$ satisfies GIN iff there exists a $\mathbf{C_Y}$, s.t. $(\emptyset, \mathbf{C_Y})$ t-separates $(\mathbf{Z, Y})$ and $\mathbf{|C_Y|<|Y|}$.
\end{lemma}\alignnewline\\ \textcolor{white}{abc}
\textbf{\textit{Example.}} Consider fig \ref{fig:rankcannotdetect} again.  For $\mathbf{Z}=\{X_6, X_7\}$, $\mathbf{Y}=\{X_3, X_4, X_5\}$, $\{\mathbf{Z, Y} \}$ does not satisfy GIN: for $(\emptyset, \mathbf{C_Y})$ to t-separate $(\mathbf{Z, Y})$, it is necessary for $\mathbf{C_Y}$ to include $\{X_3, X_4, X_5\}$ there are treks connecting $\mathbf{Y}$ and $\mathbf{Z}$ where the $\mathbf{Y}$ side directed paths only contains a variable in $\mathbf{Y}$, such as $\langle X_3, X_3\rightarrow L_1\rightarrow L_3\rightarrow X_6, \rangle$.

\begin{theorem}\label{thm:GINRankCycle}
    Let $\mathbf{Z = V\setminus Y}$. If (i) given $\mathbf{|Z|}>\mathbf{|Y|}=n$, $rank(\Sigma_{\mathbf{Z,Y}}) = n'< n$; (ii) No subset $\mathbf{Y'\subset Y}$ s.t. there exists a $\mathbf{ W\supset Y'}$ and $\mathbf{Z'= V\setminus W}$ satisfies $rank(\Sigma_{\mathbf{Z', W}}) < \mathbf{|W|}$; (iii) $(\mathbf{Z, Y})$ fails GIN, then there are cycles under the common causes of $\mathbf{Y}$.
\end{theorem}

\begin{proof}
    We proceed the proof by contradiction and use the \textbf{Lemma} \ref{GINdsep}.\\ \textcolor{white}{abc}
    By (iii) and \textbf{Lemma} \ref{GINdsep}, we know that for any $\mathbf{C_Y}$ s.t. $(\emptyset, \mathbf{C_Y})$ t-separates  $(\mathbf{Z, Y})$,  $\mathbf{|C_Y|\geq|Y|}=n$. Assume that there is such $\mathbf{|C_Y|}$ that only contains latents. 
      Since each latent is from at least one pc-atomic latent cover, we can partition $\mathbf{C_Y=\cup_{i\in[m]}\mathbf{C^i_Y}}$ s.t. each of $\mathbf{C^i_Y}$ belongs to a different pc-atomic latent cover. We denote the corresponding pc-atomic latent cover as $\mathbf{\hat{L}^i_Y}$.\\ \textcolor{white}{abc}
    Notice that $\mathbf{|\hat{L}^i_Y|\geq |C^i_Y|}$ since the latter is a subset of the former.  By the definition of pc-atomic latent cover, each of the $\mathbf{\hat{L}^i_Y}$ has at least $|\mathbf{C^i_Y}|+1$ many pure children and $|\mathbf{C^i_Y}|+1$ many additional neighbors.\\ \textcolor{white}{abc}
    %If there exists a $\mathbf{\hat{L}^i_Y}$ with a cardinality larger or equal to $n$, we know that they have at least $2n+2$ many neighbors ($n+1$ many pure children).  Therefore we are going to have at least $rank(\Sigma_{\mathbf{Z,Y}}) > n$, contradiction.\\ \textcolor{white}{abc}
    %Now we discuss the case where $|\mathbf{\hat{L}}^i_Y|<n$ for all $i$.
    By (ii) we know that for any $\mathbf{\hat{L}^i_Y}$, $\mathbf{Y}$ does not contain any pure children of such cover.  Therefore there exists at least $\sum_{i}(|\mathbf{\hat{L}}^{i}_Y| + 1) $ many pure children of each $\mathbf{\hat{L}}^{i}_Y$ that are not included in $\mathbf{Y}$.  Since $\mathbf{|\hat{L}^i_Y|\geq |C^i_Y|}$ we further have there are at least $\sum_{i}(|\mathbf{\hat{L}}^{i}_Y| + 1)\geq\sum_{i}(|\mathbf{C}^{i}_Y| + 1)>|\mathbf{C_Y}|>n$ many pure children not included in $\mathbf{Y}$.  Specifically for each $\mathbf{C_i\subset\mathbf{\hat{L}}}^{i}_Y$, there are at least $|\mathbf{C_i}|$ many variables (pure children of $\mathbf{\hat{L}}^{i}_Y$) that are being connected by treks with $\mathbf{Y}$ through $\mathbf{C_i}$.  Denote these children with $Ch(\mathbf{C_Y})$.  In other words, we should at least have $(\emptyset, \mathbf{C_Y})$ \textit{t-seps} $\mathbf{Y}$ and $Ch(\mathbf{C_Y})$ not to mention $Ch(\mathbf{C_Y})\subset\mathbf{Z}$.  Therefore $rank(\Sigma_{\mathbf{Z,Y}}) \geq n$, contradiction.
    
    %If every $\mathbf{\hat{L}^i_Y}$ has a cardinality larger or equal to $n$, we know that they have at least $2n+2$ many neighbors ($n+1$ many pure children).  Therefore we are going to have at least $rank(\Sigma_{\mathbf{Z,Y}}) > n$, contradiction.\\ \textcolor{white}{abc} 
    %Now we discuss the case where there are $|\mathbf{\hat{L}}^i_Y|<n$ for some $i$. Denote them by $\mathbf{\hat{L}}^{i_{<n}}_Y$. By (ii) we know that for any $\mathbf{\hat{L}^{i_{<n}}_Y}$, $\mathbf{Y}$ does not contain any pure children of such cover.  Therefore there exists at least $\sum_{i}(|\mathbf{\hat{L}}^{i_{<n}}_Y| + 1) $ many pure children of each $\mathbf{\hat{L}}^{i_{<n}}_Y$ that are not included in $\mathbf{Y}$. Denote by $\mathbf{Y}^{<n}$ the variables that are only descendants from some $\mathbf{\hat{L}}^{i_{<n}}_Y$. For each $\mathbf{C_i\subset\mathbf{\hat{L}}}^{i_{<n}}_Y$, there are at least $|\mathbf{C_i}|$ many variables (pure children of $\mathbf{\hat{L}}^{i_{<n}}_Y$) that are being connected through treks to $\mathbf{Y}^{<n}$ through $\mathbf{C_i}$.  For all the other variables in $\mathbf{Y\setminus Y^{<n}}$, since they are descendants from latent covers with size more than $n$, each of variables in $\mathbf{Y\setminus Y^{<n}}$ is connected with some distinct variables in $\mathbf{Z}$ with some treks through distinct latents $\mathbf{C_i}$.     Therefore $rank(\Sigma_{\mathbf{Z,Y}}) \geq n$, contradiction.\\ \textcolor{white}{abc}
    Therefore $\mathbf{C_Y}$ cannot be all latents.  Since we do not allow direct edges between measured variables, it's only possible that $\mathbf{Y}\cap\mathbf{C_Y}\neq\emptyset$, which means that some treks connecting some $Z\in\mathbf{Z}$ and $Y\in\mathbf{Y}$ have the $\mathbf{Y}$ side as $Y$, so such $Y$ is a ancestor to $\mathbf{Z}$ through its latent causes.  Based on the defition of $L^2HCM$, a measured variable cannot cause any latent variable other than its latent parents.  So we have cycles between $\mathbf{Y}$ and its latent parent set.
\end{proof}
\begin{figure}[H]
\begin{center}
\begin{tikzpicture}[scale=0.12]
\tikzstyle{every node}+=[inner sep=0pt]
\draw [black] (24.4,-25) circle (3);
\draw (24.4,-25) node {$X_1$};
\draw [black] (17,-23.1) circle (3);
\draw (17,-23.1) node {$X_2$};
\draw [black] (36.9,-33.3) circle (3);
\draw (36.9,-33.3) node {$X_4$};
\draw [black] (45.6,-32.7) circle (3);
\draw (45.6,-32.7) node {$X_5$};
\draw [black] (49.6,-25) circle (3);
\draw (49.6,-25) node {$X_6$};
\draw [black] (57,-21.1) circle (3);
\draw (57,-21.1) node {$X_7$};
\draw [black] (28.5,-32.7) circle (3);
\draw (28.5,-32.7) node {$X_3$};
\draw [black] (36.9,-17) circle (3);
\draw (36.9,-17) node {$L_1$};
\draw [black] (46.7,-9.6) circle (3);
\draw (46.7,-9.6) node {$L_3$};
\draw [black] (26.2,-9.6) circle (3);
\draw (26.2,-9.6) node {$L_2$};
\draw [black] (43.633,-30.436) arc (-141.39954:-160.61527:36.194);
\fill [black] (43.63,-30.44) -- (43.53,-29.5) -- (42.74,-30.12);
\draw [black] (36.219,-30.38) arc (-169.79322:-190.20678:29.512);
\fill [black] (36.22,-30.38) -- (36.57,-29.5) -- (35.59,-29.68);
\draw [black] (29.21,-29.786) arc (163.431:140.2726:29.888);
\fill [black] (29.21,-29.79) -- (29.92,-29.16) -- (28.96,-28.88);
\draw [black] (47.26,-12.55) -- (49.04,-22.05);
\fill [black] (49.04,-22.05) -- (49.39,-21.17) -- (48.41,-21.36);
\draw [black] (48.7,-11.83) -- (55,-18.87);
\fill [black] (55,-18.87) -- (54.84,-17.94) -- (54.09,-18.6);
\draw [black] (29.92,-30.05) -- (35.48,-19.65);
\fill [black] (35.48,-19.65) -- (34.67,-20.11) -- (35.55,-20.59);
\draw [black] (37.573,-19.922) arc (10.08281:-10.08281:29.86);
\fill [black] (37.57,-19.92) -- (37.22,-20.8) -- (38.2,-20.62);
\draw [black] (38.726,-19.38) arc (35.79806:22.18713:50.702);
\fill [black] (38.73,-19.38) -- (38.79,-20.32) -- (39.6,-19.74);
\draw [black] (25.85,-12.58) -- (24.75,-22.02);
\fill [black] (24.75,-22.02) -- (25.34,-21.28) -- (24.34,-21.17);
\draw [black] (24.51,-12.08) -- (18.69,-20.62);
\fill [black] (18.69,-20.62) -- (19.55,-20.24) -- (18.73,-19.68);
\draw [black] (28.67,-11.31) -- (34.43,-15.29);
\fill [black] (34.43,-15.29) -- (34.06,-14.43) -- (33.49,-15.25);
\draw [black] (44.31,-11.41) -- (39.29,-15.19);
\fill [black] (39.29,-15.19) -- (40.23,-15.11) -- (39.63,-14.31);
\end{tikzpicture}
\end{center}
    \caption{An example where combining GIN and rank constraint cannot identify the cycles since $L_1$ is a collider between any other two latents.}
    \label{fig:collidercycle}
\end{figure}
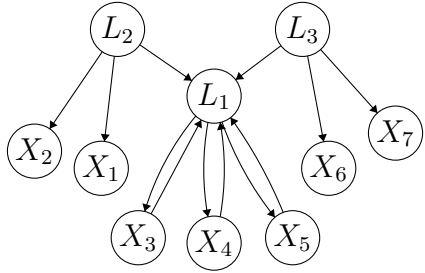
Notice that, if the latents of a cyclic cluster are colliders between the other latent variables, the variables in this cluster will still satisfy GIN with measured variables from other clusters. For instance, consider fig \ref{fig:collidercycle}.  Since $L_1$ is a collider between $L_2$ and $L_3$, $X_3, X_4, X_5$ are not ancestors of other measured variables, therefore in a linear non-Gaussian model $(\mathbf{\{X_1, X_6\}, \{X_3, X_4, X_5\}})$ satisfies GIN.  However, we still have  $rank(\Sigma_\mathbf{\{X_3, X_4, X_5\}, \{X_1, X_6\}}) = 1<rank(\Sigma_\mathbf{\{X_3, X_4\}, \{X_1, X_6\}})=2<| \{X_3, X_4, X_5\}|$, therefore the cycles between $L_1$ and $\{X_3, X_4, X_5\}$ can be identified by just rank constraints. \\ \textcolor{white}{abc}
We've shown two methods to identify cycles within causal clusters. Both use the relation of children among different causal clusters, either the change of rank of correlation submatrices or the dependence due to the asymmetric inseparability of noises. These methods all implicitly assume that latents within the same clusters have the same causal relation with latents from another cluster. In the next section we talk about identifying another type of latent causal structure: cycles between latent blocks.  
\section{Cycles between blocks using GIN}
In the original paper of GIN, Xie et al. \cite{xie2020generalized} defined Linear Non-Gaussian Latent Variable Model ($LiNGLaM$) as linear and acyclic models with additional requirements that:
\begin{enumerate}
    \item No measured variables being ancestors of latent variables.
    \item The noise terms are non-Gaussian.
    \item Every latent set $\mathbf{L}$ has at least $2\mathbf{|L|}$ pure measured children.
    \item No direct edges between measured variables.
\end{enumerate}
GIN is claimed to be able to identify causal relations between latent variables of any two considered groups of measured variables.  So far it has been shown to identify acyclic causal relations between latent variables, including hierarchical structures\cite{pmlr-v162-xie22a}. In this section we are going to adapt the $LiNGLaM$ model and discuss one situation that can occur under $LiNGLaM$: cycles between blocks of latents.
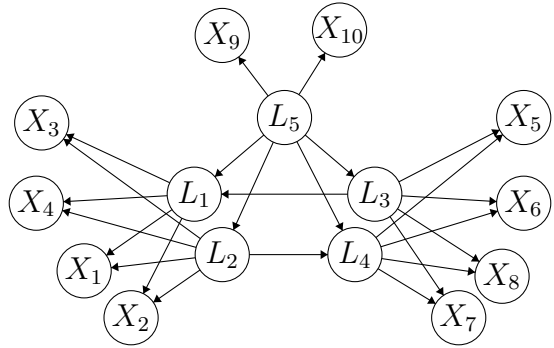
\begin{figure}[H]
\begin{center}
\begin{tikzpicture}[scale=0.12]
\tikzstyle{every node}+=[inner sep=0pt]
\draw [black] (66.1,-15.6) circle (3);
\draw (66.1,-15.6) node {$X_5$};
\draw [black] (66.1,-25.1) circle (3);
\draw (66.1,-25.1) node {$X_6$};
\draw [black] (12.1,-25.1) circle (3);
\draw (12.1,-25.1) node {$X_4$};
\draw [black] (12.7,-16.2) circle (3);
\draw (12.7,-16.2) node {$X_3$};
\draw [black] (17.4,-32.6) circle (3);
\draw (17.4,-32.6) node {$X_1$};
\draw [black] (63.7,-33.2) circle (3);
\draw (63.7,-33.2) node {$X_8$};
\draw [black] (29.6,-24.1) circle (3);
\draw (29.6,-24.1) node {$L_1$};
\draw [black] (58.9,-37.8) circle (3);
\draw (58.9,-37.8) node {$X_7$};
\draw [black] (49.6,-24.1) circle (3);
\draw (49.6,-24.1) node {$L_3$};
\draw [black] (22.6,-37.8) circle (3);
\draw (22.6,-37.8) node {$X_2$};
\draw [black] (32.7,-30.8) circle (3);
\draw (32.7,-30.8) node {$L_2$};
\draw [black] (47.4,-30.8) circle (3);
\draw (47.4,-30.8) node {$L_4$};
\draw [black] (39.6,-15.6) circle (3);
\draw (39.6,-15.6) node {$L_5$};
\draw [black] (32.7,-6.5) circle (3);
\draw (32.7,-6.5) node {$X_9$};
\draw [black] (45.7,-5.9) circle (3);
\draw (45.7,-5.9) node {$X_{10}$};
\draw [black] (26.88,-22.83) -- (15.42,-17.47);
\fill [black] (15.42,-17.47) -- (15.93,-18.26) -- (16.35,-17.36);
\draw [black] (26.6,-24.27) -- (15.1,-24.93);
\fill [black] (15.1,-24.93) -- (15.92,-25.38) -- (15.87,-24.38);
\draw [black] (27.14,-25.81) -- (19.86,-30.89);
\fill [black] (19.86,-30.89) -- (20.8,-30.84) -- (20.23,-30.02);
\draw [black] (52.12,-25.73) -- (61.18,-31.57);
\fill [black] (61.18,-31.57) -- (60.78,-30.72) -- (60.24,-31.56);
\draw [black] (51.28,-26.58) -- (57.22,-35.32);
\fill [black] (57.22,-35.32) -- (57.18,-34.38) -- (56.35,-34.94);
\draw [black] (46.6,-24.1) -- (32.6,-24.1);
\fill [black] (32.6,-24.1) -- (33.4,-24.6) -- (33.4,-23.6);
\draw [black] (28.24,-26.77) -- (23.96,-35.13);
\fill [black] (23.96,-35.13) -- (24.77,-34.64) -- (23.88,-34.19);
\draw [black] (29.81,-30) -- (14.99,-25.9);
\fill [black] (14.99,-25.9) -- (15.63,-26.6) -- (15.9,-25.63);
\draw [black] (30.28,-29.03) -- (15.12,-17.97);
\fill [black] (15.12,-17.97) -- (15.47,-18.84) -- (16.06,-18.04);
\draw [black] (29.72,-31.15) -- (20.38,-32.25);
\fill [black] (20.38,-32.25) -- (21.23,-32.65) -- (21.12,-31.66);
\draw [black] (30.23,-32.51) -- (25.07,-36.09);
\fill [black] (25.07,-36.09) -- (26.01,-36.05) -- (25.44,-35.22);
\draw [black] (50.37,-31.24) -- (60.73,-32.76);
\fill [black] (60.73,-32.76) -- (60.01,-32.15) -- (59.87,-33.14);
\draw [black] (49.96,-32.36) -- (56.34,-36.24);
\fill [black] (56.34,-36.24) -- (55.91,-35.4) -- (55.39,-36.25);
\draw [black] (50.27,-29.93) -- (63.23,-25.97);
\fill [black] (63.23,-25.97) -- (62.32,-25.73) -- (62.61,-26.69);
\draw [black] (49.73,-28.91) -- (63.77,-17.49);
\fill [black] (63.77,-17.49) -- (62.84,-17.61) -- (63.47,-18.38);
\draw [black] (35.7,-30.8) -- (44.4,-30.8);
\fill [black] (44.4,-30.8) -- (43.6,-30.3) -- (43.6,-31.3);
\draw [black] (52.27,-22.73) -- (63.43,-16.97);
\fill [black] (63.43,-16.97) -- (62.49,-16.9) -- (62.95,-17.78);
\draw [black] (52.59,-24.28) -- (63.11,-24.92);
\fill [black] (63.11,-24.92) -- (62.34,-24.37) -- (62.28,-25.37);
\draw [black] (41.2,-13.06) -- (44.1,-8.44);
\fill [black] (44.1,-8.44) -- (43.25,-8.85) -- (44.1,-9.38);
\draw [black] (37.79,-13.21) -- (34.51,-8.89);
\fill [black] (34.51,-8.89) -- (34.6,-9.83) -- (35.39,-9.23);
\draw [black] (37.31,-17.54) -- (31.89,-22.16);
\fill [black] (31.89,-22.16) -- (32.82,-22.02) -- (32.17,-21.26);
\draw [black] (38.36,-18.33) -- (33.94,-28.07);
\fill [black] (33.94,-28.07) -- (34.73,-27.55) -- (33.82,-27.13);
\draw [black] (41.89,-17.54) -- (47.31,-22.16);
\fill [black] (47.31,-22.16) -- (47.03,-21.26) -- (46.38,-22.02);
\draw [black] (40.97,-18.27) -- (46.03,-28.13);
\fill [black] (46.03,-28.13) -- (46.11,-27.19) -- (45.22,-27.65);
\end{tikzpicture}
\end{center}    
    \caption{$\{L_1, L_2\}$ and $\{L_3, L_4\}$ are two blocks of latents having distinct measured children; there are no cycles in the model but there is a cycle between the two blocks.}
    \label{fig:cycbl}
\end{figure}
\textbf{\textit{Example.}} Consider fig \ref{fig:cycbl}. Clearly it satisfies all the graphical requirement of $LiNGLaM$.  However, in addition to $L_5$ being causally earlier than $\{L_1, L_2\}$ and $\{L_3, L_4\}$, we further have $L_2\rightarrow L_4$ and $L_3\rightarrow L_1$.  Viewed as the latent common causes of $\{X_1, X_2, X_3, X_4\}$, $L_1$ and  $L_2$ are not distinguishable by GIN.  Similarly, $L_3$ and  $L_4$ are not distinguishable by GIN.  Treating $\{L_1, L_2\}$ as a latent block and $\{L_3, L_4\}$ as another, there is a cycle between blocks to be identified.
\begin{theorem}\label{thm:GINcollider1}
Consider a $G=\langle \mathbf{V, E}\rangle$ follows a linear and acyclic graph with $LiNGLaM$. Denote latents of the cluster $\mathcal{C}$ by $L_{\mathcal{C}}$ and the measured children in cluster $\mathcal{C}$ by $\{X^{\mathcal{C}}_1...X^{\mathcal{C}}_m\}$. Given clusters $\mathcal{C}_i, \mathcal{C}_a, \mathcal{C}_b$, if:
\begin{enumerate}[label=(\roman*)]
\item $\mathcal{C}_i$ is causally ealier than $\mathcal{C}_a, \mathcal{C}_b$. 
\item Neither $(\{X^{\mathcal{C}_a}_1,...X^{\mathcal{C}_a}_{|L_{\mathcal{C}_a}|},X^{\mathcal{C}_b}_{|L_{\mathcal{C}_b}|+1},...X^{\mathcal{C}_b}_{2|L_{\mathcal{C}_b}|},X^{\mathcal{C}_i}_{|L_{\mathcal{C}_i}|+1},...X^{\mathcal{C}_i}_{2|L_{\mathcal{C}_i}|}\},\{X^{\mathcal{C}_b}_1,...X^{\mathcal{C}_b}_{|L_{\mathcal{C}_b}|}\})$ \newline nor $(\{X^{\mathcal{C}_b}_1,...X^{\mathcal{C}_b}_{|L_{\mathcal{C}_b}|},X^{\mathcal{C}_a}_{|L_{\mathcal{C}_a}|+1},...X^{\mathcal{C}_a}_{2|L_{\mathcal{C}_a}|},X^{\mathcal{C}_i}_{|L_{\mathcal{C}_i}|+1},...X^{\mathcal{C}_i}_{2|L_{\mathcal{C}_i}|}\},\{X^{\mathcal{C}_a}_1,...X^{\mathcal{C}_a}_{|L_{\mathcal{C}_a|}}\})$ satisfies GIN condition.
\item $rank(\Sigma_{\{X^{\mathcal{C}_i}\},\{X^{\mathcal{C}_a}\}\cup\{X^{\mathcal{C}_b}\}})<max(|L_{\mathcal{C}_a}|,|L_{\mathcal{C}_b}|)$.
\item given a $\omega$ s.t. $[X^{\mathcal{C}_b}_1,...X^{\mathcal{C}_b}_{|L_{\mathcal{C}_b}|}]^T\omega\indep [X^{\mathcal{C}_i}_{|L_{\mathcal{C}_i}|+1},...X^{\mathcal{C}_i}_{2|L_{\mathcal{C}_i}|}]$\footnote{All vectors are column vectors by default.}, $[X^{\mathcal{C}_b}_1,...X^{\mathcal{C}_b}_{|L_{\mathcal{C}_b}|}]^T\omega\nindep [X^{\mathcal{C}_a}_1,...X^{\mathcal{C}_a}_{|L_{\mathcal{C}_a}|}]$

\end{enumerate}
then $\mathcal{C}_a$ and $\mathcal{C}_b$ has cycles between blocks of latents.
\end{theorem}
\begin{proof}
 $(i)$ and $(ii)$ indicates that in addition to common causes, it is neither simply $L_{\mathcal{C}_a}\leftarrow L_{\mathcal{C}_b}$ nor simply $L_{\mathcal{C}_a}\rightarrow L_{\mathcal{C}_b}$. $(iii)$ indicates that $\mathcal{C}_i$ is d-separated from $\mathcal{C}_b$ and $\mathcal{C}_a$ by a set with size less than $max(|L_{\mathcal{C}_a}|,|L_{\mathcal{C}_b}|)$.\\ \textcolor{white}{abc}
 WLOG assume $max(|L_{\mathcal{C}_a}|,|L_{\mathcal{C}_b}|) = |L_{\mathcal{C}_b}|$. 
 By theorem \ref{GINdsep}, $(\{X^{\mathcal{C}_i}_{|\mathcal{C}_i|+1},...X^{\mathcal{C}_i}_{2|\mathcal{C}_i|}\},\{X^{\mathcal{C}_b}_1,...X^{\mathcal{C}_b}_{|L_{\mathcal{C}_b}|}\})$ satisfies GIN.
 Now consider $(iv)$. For the vector $\omega$ satisfying $[X^{\mathcal{C}_b}_1,...X^{\mathcal{C}_b}_{|L_{\mathcal{C}_b}|}]^T\omega\indep [X^{\mathcal{C}_i}_{|L_{\mathcal{C}_i}|+1},...X^{\mathcal{C}_i}_{2|L_{\mathcal{C}_i}|}]$, we can view such the multiplication of $[X^{\mathcal{C}_b}_1,...X^{\mathcal{C}_b}_{|L_{\mathcal{C}_b}|}]$ and $\omega$  as a \textit{\textbf{d-separation}} between $L_{\mathcal{C}_b}$ and $L_{\mathcal{C}_i}$. After $L_{\mathcal{C}_b}$ \textit{\textbf{being d-separated from}} $L_{\mathcal{C}_i}$ by multiplying with $\omega$, we still have $[X^{\mathcal{C}_b}_1,...X^{\mathcal{C}_b}_{|L_{\mathcal{C}_b|}}]^T\omega$ and measured variables in $\mathcal{C}_a$ being dependent.  $\{X^{\mathcal{C}_b}\}$ are only connected to other variables through $L_{\mathcal{C}_b}$. Since any directed path from $L_{\mathcal{C}_a}$ to $L_{\mathcal{C}_b}$ will also be blocked by the \textit{\textbf{d-separation}} (these directed paths are also part of directed paths from $L_{\mathcal{C}_i}$ to $L_{\mathcal{C}_b}$) we know that such dependency between $L_{\mathcal{C}_a}$ and $L_{\mathcal{C}_b}$, can only come from the directed causal influences from $L_{\mathcal{C}_b}$ to $L_{\mathcal{C}_a}$. Therefore there are directed paths from  $L_{\mathcal{C}_b}$ to $L_{\mathcal{C}_a}$, such that $L_{\mathcal{C}_a}$ can be viewed as colliders between $L_{\mathcal{C}_b}$ and $L_{\mathcal{C}_i}$.  Therefore $L_{\mathcal{C}_a}\leftarrow L_{\mathcal{C}_b}$ exists between some latents in $\mathcal{C}_a$ and $\mathcal{C}_b$.  In this case $(ii)$ can only hold if there are also $L_{\mathcal{C}_a}\rightarrow L_{\mathcal{C}_b}$.  Therefore there are cycles between blocks.
\end{proof}
\textbf{\textit{Example.}}  Consider fig \ref{fig:cycbl} again. By the definition of $LiNGLaM$ we have the following linear SEM:
\begin{equation}\label{eq:SEMCycleBlockGIN1}
    \begin{bmatrix}X_2\\ \textcolor{white}{abc}X_3\end{bmatrix}=AL_5+C\epsilon_{L_3}+B\begin{bmatrix}\epsilon_{L_1}\\ \textcolor{white}{abc} \epsilon_{L_2}\end{bmatrix} +\begin{bmatrix}\epsilon_{X_2}\\ \textcolor{white}{abc} \epsilon_{X_3}\end{bmatrix} 
\end{equation}
where $dim(A)=dim(C)=2\times 1$ and $dim(B)=2\times 2$.  It's obvious that GIN is able to identify that the $L_5$ cluster is causally earlier than the other two clusters, while the causal relation $\{L_1, L_2\}$ cluster and $\{L_3, L_4\}$ cluster is not clear.  We further have $rank(\Sigma_{\{X_9, X_{10}\},\{X_1, ..., X_8\}}) = 1 < max(|\{L_1, L_2\}|, |\{L_3, L_4\}|)$. By equation \ref{eq:SEMCycleBlockGIN1}, given any nonzero vector $\omega$ s.t. $\omega^TA=0$, $[X_2, X_3]\omega\indep[X_9, X_{10}]$ (the causal influence of $L_5$ is canceled), while its component $\omega^TC\epsilon_{L_3}$ is dependent with $L_3$ and so its children. By the Darmois–Skitovich theorem,    $[X_2, X_3]\omega\nindep[X_7, X_8]$.\\ \textcolor{white}{abc}
Notice that there are only three clusters presented in the theorem:  $\mathcal{C}_a$ and $\mathcal{C}_b$, the causal relations between the latents of which shall be further identified; $\mathcal{C}_i$, the causally earlier one.  In practice there might be multiple clusters causally early than $\mathcal{C}_a$ and $\mathcal{C}_b$, and they can be unioned and viewed as one $\mathcal{C}_i$.  As long as condition $(iii)$ is satisfied, theorem \ref{thm:GINcollider1} can is valid to use.\\ \textcolor{white}{abc}
In fact, we can just define the $\mathcal{C}_i$ as the union of all clusters of which the latents are causally earlier than $L_{\mathcal{C}_b}$ and $L_{\mathcal{C}_a}$.  When $|L_{\mathcal{C}_b}|$ and $|L_{\mathcal{C}_a}|$ are large enough, we can have a similar procedure of checking the existence of cycles between latent blocks that identifies both the edges $L_{\mathcal{C}_b}\rightarrow L_{\mathcal{C}_a}$ and $L_{\mathcal{C}_b}\leftarrow L_{\mathcal{C}_a}$ and therefore find a cycle..

\begin{theorem}\label{thm:GINcollider2}
Consider a $G=\langle \mathbf{V, E}\rangle$ follows a linear and acyclic graph with $LiNGLaM$. Denote latents of the cluster $\mathcal{C}$ by $L_{\mathcal{C}}$ and the measured children in cluster $\mathcal{C}$ by $\{X^{\mathcal{C}}_1...X^{\mathcal{C}}_m\}$. Given clusters $\mathcal{C}_i, \mathcal{C}_a, \mathcal{C}_b$, if:
\begin{enumerate}[label=(\roman*)]
\item Let $\mathcal{C}_j$ be the union of all clusters causally ealier than $\mathcal{C}_a, \mathcal{C}_b$. \newline  Neither $(\{X^{\mathcal{C}_a}_1,...X^{\mathcal{C}_a}_{|L_{\mathcal{C}_a}|},X^{\mathcal{C}_b}_{|L_{\mathcal{C}_b}|+1},...X^{\mathcal{C}_b}_{2|L_{\mathcal{C}_b}|},X^{\mathcal{C}_j}_{|L_{\mathcal{C}_j}|+1},...X^{\mathcal{C}_j}_{2|L_{\mathcal{C}_j}|}\},\{X^{\mathcal{C}_b}_1,...X^{\mathcal{C}_b}_{|L_{\mathcal{C}_b}|}\})$ nor\newline  $(\{X^{\mathcal{C}_b}_1,...X^{\mathcal{C}_b}_{|L_{\mathcal{C}_b}|},X^{\mathcal{C}_a}_{|L_{\mathcal{C}_a}|+1},...X^{\mathcal{C}_a}_{2|L_{\mathcal{C}_a}|},X^{\mathcal{C}_j}_{|L_{\mathcal{C}_j}|+1},...X^{\mathcal{C}_j}_{2|L_{\mathcal{C}_j}|}\},\{X^{\mathcal{C}_a}_1,...X^{\mathcal{C}_a}_{|L_{\mathcal{C}_a|}}\})$ satisfies GIN condition.
\item $rank(\Sigma_{\{X^{\mathcal{C}_i}\},\{X^{\mathcal{C}_a}\}\cup\{X^{\mathcal{C}_b}\}})<min(|L_{\mathcal{C}_a}|,|L_{\mathcal{C}_b}|)$.
\item Given a $\omega$ s.t. $[X^{\mathcal{C}_b}_1,...X^{\mathcal{C}_b}_{|L_{\mathcal{C}_b}|}]^T\omega\indep [X^{\mathcal{C}_i}_{|L_{\mathcal{C}_i}|+1},...X^{\mathcal{C}_i}_{2|L_{\mathcal{C}_i}|}]$, $[X^{\mathcal{C}_b}_1,...X^{\mathcal{C}_b}_{|L_{\mathcal{C}_b}|}]^T\omega\nindep [X^{\mathcal{C}_a}_1,...X^{\mathcal{C}_a}_{|L_{\mathcal{C}_a}|}]$.
\item Given a $\theta$ s.t. $[X^{\mathcal{C}_a}_1,...X^{\mathcal{C}_a}_{|L_{\mathcal{C}_a}|}]^T\theta\indep [X^{\mathcal{C}_i}_{|L_{\mathcal{C}_i}|+1},...X^{\mathcal{C}_i}_{2|L_{\mathcal{C}_i}|}]$, $[X^{\mathcal{C}_a}_1,...X^{\mathcal{C}_a}_{|L_{\mathcal{C}_a}|}]^T\theta\nindep [X^{\mathcal{C}_b}_1,...X^{\mathcal{C}_b}_{|L_{\mathcal{C}_b}|}]$.
\end{enumerate}
then $\mathcal{C}_a$ and $\mathcal{C}_b$ has cycles between blocks of latents.
\end{theorem}
The proof can be found in the Appendix and is very similar to the proof of theorem \ref{thm:GINcollider1}.  Notice that condition $(ii)$ and the existence of $\omega$ and $\theta$ as described in the condition $(iii)$ and $(iv)$ show that $\mathcal{C}_i$ is not connected with other two clusters with all their latents.  Naturally it leads to a special case that, when $|L_{\mathcal{C}_a}|=|L_{\mathcal{C}_b}|=2$,  the edges $L_{\mathcal{C}_b}\rightarrow L_{\mathcal{C}_a}$ and $L_{\mathcal{C}_b}\leftarrow L_{\mathcal{C}_a}$ can be identified by GIN, which means that the cycle between the latent blocks can be removed:\\ \textcolor{white}{abc}  
\begin{figure}[H]
\begin{center}
\begin{tikzpicture}[scale=0.12]
\tikzstyle{every node}+=[inner sep=0pt]
\draw [black] (10.3,-25.7) circle (3);
\draw (10.3,-25.7) node {$X_1$};
\draw [black] (14.6,-17.4) circle (3);
\draw (14.6,-17.4) node {$X_2$};
\draw [black] (14.6,-40.4) circle (3);
\draw (14.6,-40.4) node {$X_4$};
\draw [black] (64,-14.9) circle (3);
\draw (64,-14.9) node {$X_5$};
\draw [black] (67.5,-23.7) circle (3);
\draw (67.5,-23.7) node {$X_6$};
\draw [black] (62.5,-39.5) circle (3);
\draw (62.5,-39.5) node {$X_8$};
\draw [black] (10.3,-32.8) circle (3);
\draw (10.3,-32.8) node {$X_3$};
\draw [black] (32.9,-32.8) circle (3);
\draw (32.9,-32.8) node {$L_2$};
\draw [black] (32.9,-21.9) circle (3);
\draw (32.9,-21.9) node {$L_1$};
\draw [black] (67.5,-31.9) circle (3);
\draw (67.5,-31.9) node {$X_7$};
\draw [black] (45.7,-21.9) circle (3);
\draw (45.7,-21.9) node {$L_3$};
\draw [black] (45.7,-32.8) circle (3);
\draw (45.7,-32.8) node {$L_4$};
\draw [black] (38.5,-9.8) circle (3);
\draw (38.5,-9.8) node {$L_5$};
\draw [black] (23,-13.1) circle (3);
\draw (23,-13.1) node {$X_9$};
\draw [black] (26.1,-4.9) circle (3);
\draw (26.1,-4.9) node {$X_{10}$};
\draw [black] (30.04,-31.9) -- (13.16,-26.6);
\fill [black] (13.16,-26.6) -- (13.78,-27.32) -- (14.08,-26.36);
\draw [black] (30.6,-30.87) -- (16.9,-19.33);
\fill [black] (16.9,-19.33) -- (17.19,-20.23) -- (17.83,-19.46);
\draw [black] (30.2,-23.2) -- (13,-31.5);
\fill [black] (13,-31.5) -- (13.94,-31.6) -- (13.51,-30.7);
\draw [black] (30.79,-24.03) -- (16.71,-38.27);
\fill [black] (16.71,-38.27) -- (17.63,-38.05) -- (16.92,-37.35);
\draw [black] (29.94,-22.4) -- (13.26,-25.2);
\fill [black] (13.26,-25.2) -- (14.13,-25.56) -- (13.96,-24.58);
\draw [black] (29.99,-21.18) -- (17.51,-18.12);
\fill [black] (17.51,-18.12) -- (18.17,-18.79) -- (18.41,-17.82);
\draw [black] (29.9,-32.8) -- (13.3,-32.8);
\fill [black] (13.3,-32.8) -- (14.1,-33.3) -- (14.1,-32.3);
\draw [black] (30.13,-33.95) -- (17.37,-39.25);
\fill [black] (17.37,-39.25) -- (18.3,-39.4) -- (17.92,-38.48);
\draw [black] (47.84,-30.7) -- (61.86,-17);
\fill [black] (61.86,-17) -- (60.93,-17.2) -- (61.63,-17.91);
\draw [black] (48.47,-31.64) -- (64.73,-24.86);
\fill [black] (64.73,-24.86) -- (63.8,-24.7) -- (64.19,-25.63);
\draw [black] (48.49,-33.91) -- (59.71,-38.39);
\fill [black] (59.71,-38.39) -- (59.16,-37.63) -- (58.79,-38.56);
\draw [black] (48.7,-32.68) -- (64.5,-32.02);
\fill [black] (64.5,-32.02) -- (63.68,-31.56) -- (63.72,-32.56);
\draw [black] (48.43,-23.15) -- (64.77,-30.65);
\fill [black] (64.77,-30.65) -- (64.25,-29.86) -- (63.84,-30.77);
\draw [black] (47.77,-24.07) -- (60.43,-37.33);
\fill [black] (60.43,-37.33) -- (60.24,-36.41) -- (59.51,-37.1);
\draw [black] (48.69,-22.15) -- (64.51,-23.45);
\fill [black] (64.51,-23.45) -- (63.75,-22.89) -- (63.67,-23.89);
\draw [black] (48.5,-20.83) -- (61.2,-15.97);
\fill [black] (61.2,-15.97) -- (60.27,-15.79) -- (60.63,-16.72);
\draw [black] (35.9,-21.9) -- (42.7,-21.9);
\fill [black] (42.7,-21.9) -- (41.9,-21.4) -- (41.9,-22.4);
\draw [black] (42.7,-32.8) -- (35.9,-32.8);
\fill [black] (35.9,-32.8) -- (36.7,-33.3) -- (36.7,-32.3);
\draw [black] (37.24,-12.52) -- (34.16,-19.18);
\fill [black] (34.16,-19.18) -- (34.95,-18.66) -- (34.04,-18.24);
\draw [black] (35.71,-8.7) -- (28.89,-6);
\fill [black] (28.89,-6) -- (29.45,-6.76) -- (29.82,-5.83);
\draw [black] (35.57,-10.42) -- (25.93,-12.48);
\fill [black] (25.93,-12.48) -- (26.82,-12.8) -- (26.61,-11.82);
\end{tikzpicture}
\end{center}
    \caption{$L_5$ is not connecting every latent in the same latent set in a cluster.}
    \label{fig:cycbl2}
\end{figure}
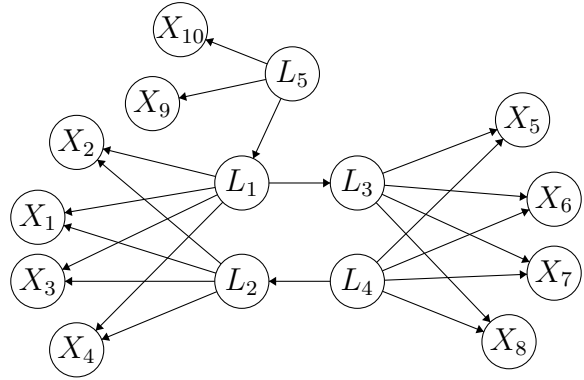
\textbf{\textit{Example. }}Consider fig \ref{fig:cycbl2}. Let $a_{ij}$ be the coefficient of $L_i\rightarrow X_j$. Let $X_{i,j}$ be a linear combination of $X_i$ and $X_j$. Then we have the following linear equations shown on the left: 

\begin{center}
    $\begin{bmatrix}X_3\\ \textcolor{white}{abc}X_4\end{bmatrix}=\begin{bmatrix}a_{13}\\ \textcolor{white}{abc}a_{14}\end{bmatrix}L_1+\begin{bmatrix}a_{23}\\ \textcolor{white}{abc}a_{24}\end{bmatrix}L_2+\begin{bmatrix}\epsilon_{X_3}\\ \textcolor{white}{abc}  \epsilon_{X_4}\end{bmatrix}$\textcolor{white}{*******}     $X_{3,4}=\begin{bmatrix}-a_{14}\\ \textcolor{white}{abc}a_{13}\end{bmatrix}^T\begin{bmatrix}X_3\\ \textcolor{white}{abc}X_4\end{bmatrix}$
\end{center}
\begin{center}
    $\begin{bmatrix}X_1\\ \textcolor{white}{abc}X_2\end{bmatrix}=\begin{bmatrix}a_{11}\\ \textcolor{white}{abc}a_{12}\end{bmatrix}L_1+\begin{bmatrix}a_{21}\\ \textcolor{white}{abc}a_{22}\end{bmatrix}L_2+\begin{bmatrix}\epsilon_{X_1}\\ \textcolor{white}{abc}  \epsilon_{X_2}\end{bmatrix}$\textcolor{white}{*******}     $X_{1,2}=\begin{bmatrix}-a_{12}\\ \textcolor{white}{abc}a_{11}\end{bmatrix}^T\begin{bmatrix}X_1\\ \textcolor{white}{abc}X_2\end{bmatrix}$
\end{center}

\begin{center}
    $\begin{bmatrix}X_5\\ \textcolor{white}{abc}X_6\end{bmatrix}=\begin{bmatrix}a_{35}\\ \textcolor{white}{abc}a_{36}\end{bmatrix}L_3+\begin{bmatrix}a_{45}\\ \textcolor{white}{abc}a_{46}\end{bmatrix}L_4+\begin{bmatrix}\epsilon_{X_5}\\ \textcolor{white}{abc}  \epsilon_{X_6}\end{bmatrix}$\textcolor{white}{*******}    $X_{5,6}=\begin{bmatrix}-a_{35}\\ \textcolor{white}{abc}a_{36}\end{bmatrix}^T\begin{bmatrix}X_5\\ \textcolor{white}{abc}X_6\end{bmatrix}$
\end{center}
\begin{center}
    $\begin{bmatrix}X_7\\ \textcolor{white}{abc}X_8\end{bmatrix}=\begin{bmatrix}a_{37}\\ \textcolor{white}{abc}a_{38}\end{bmatrix}L_3+\begin{bmatrix}a_{47}\\ \textcolor{white}{abc}a_{48}\end{bmatrix}L_4+\begin{bmatrix}\epsilon_{X_7}\\ \textcolor{white}{abc}  \epsilon_{X_8}\end{bmatrix}$\textcolor{white}{*******}    $X_{7,8}=\begin{bmatrix}-a_{38}\\ \textcolor{white}{abc}a_{37}\end{bmatrix}^T\begin{bmatrix}X_7\\ \textcolor{white}{abc}X_8\end{bmatrix}$
\end{center}
Let $\mathcal{C}_i$ be the cluster with measured variables $\mathbf{X}^{\mathcal{C}_i} = \{X_9, X_{10}\}$ and $L_{\mathcal{C}_i}=\{L_5\}$.  Similarly, we have $\mathbf{X}^{\mathcal{C}_b} = \{X_1,..., X_4\}$, $L_{\mathcal{C}_b}=\{L_1, L_2\}$ and $\mathbf{X}^{\mathcal{C}_a} = \{X_5,..., X_8\}$, $L_{\mathcal{C}_a}=\{L_3, L_4\}$. Notice that $rank(\Sigma_{\{X_9, X_{10}\},\{X_1, ..., X_8\}}) = 1 < min(|\{L_1, L_2\}|, |\{L_3, L_4\}|)=2$ so condition $(ii)$ in theorem \ref{thm:GINcollider2} is met.\\ \textcolor{white}{abc}  
  Then it is easy to see that all the combined variables $X_{i,j}$ created as shown above on the right, we have $X_{i,j }\indep\{X_9, X_{10}\}$, meeting the independence requirement described in condition $(iii)$ and $(iv)$.
Furthermore, $(\{X_{5,6}\},\{X_{3,4},X_{7,8}\})$ satisfies GIN and $(\{X_{1,2}\},\{X_{3,4},X_{7,8}\})$ does not satisfy, which indicates that there is a latent cause of $\{X_{7,8},X_{5,6}\}$ causally earlier than latent cause of $\{X_{1,2},X_{3,4}\}$;\\ \textcolor{white}{abc}
Similar procedure can be applied to detect edges with the opposite direction.  By looking for $\omega_1$ s.t. $[X_1, X_2]^T\omega\indep[X_{7,8}, X_{5,6}]$, we form a new $X_{1,2}:=[X_1, X_2]^T\omega$.  Applying the procedure for the other pairs of measured variables, we create a new set of $X_{1,2},X_{3,4},X_{7,8},X_{5,6}$, such as:
\begin{center}
        $X_{1,2}=\begin{bmatrix}-a_{22}\\ \textcolor{white}{abc}a_{21}\end{bmatrix}^T\begin{bmatrix}X_1\\ \textcolor{white}{abc}X_2\end{bmatrix}$\textcolor{white}{*******}  $X_{3,4}=\begin{bmatrix}-a_{24}\\ \textcolor{white}{abc}a_{23}\end{bmatrix}^T\begin{bmatrix}X_3\\ \textcolor{white}{abc}X_4\end{bmatrix}$
\end{center}
\begin{center}
   $X_{5,6}=\begin{bmatrix}-a_{45}\\ \textcolor{white}{abc}a_{46}\end{bmatrix}^T\begin{bmatrix}X_5\\ \textcolor{white}{abc}X_6\end{bmatrix}$ \textcolor{white}{*******} 
   $X_{7,8}=\begin{bmatrix}-a_{48}\\ \textcolor{white}{abc}a_{47}\end{bmatrix}^T\begin{bmatrix}X_7\\ \textcolor{white}{abc}X_8\end{bmatrix}$
\end{center}
among the new collection of combined variables, $(\{X_{1,2}\},\{X_{3,4},X_{7,8}\})$ satisfies GIN;$(\{X_{5,6}\},\{X_{3,4},X_{7,8}\})$ does not satisfy GIN. There is a latent cause of $\{X_{1,2},X_{3,4}\}$ causally earlier than latent cause of $\{X_{5,6},X_{7,8}\}$.\\ \textcolor{white}{abc}
Apparently the combined variables have the same latents as the original measured variables.  So we conclude that there exist cycles between blocks in the latent structure.

\section{Rank Constraint-Based and GIN Condition-Based Algorithm for Estimating Cyclic Latent Structures}

In this section we present two algorithms that learn latent causal structures with Non-Gaussian noises in the $L^2HCM$ model (allowing cycles between the measured children and their latents) and $LiNGLaM$ model (allowing cycles between latent blocks).  Both algorithm involve a stage of estimating causal order between blocks of latent variables for each causal clusters, the method of which is provided by a key proposition derived from the GIN Condition:

\begin{proposition}\label{prop:GIN4}\cite{xie2020generalized}
Suppose that $\{\mathcal{S}_1,...,\mathcal{S}_i,...,\mathcal{S}_n\}$ contains all clusters of the $LiNGLaM$.  Denote $\mathbf{T}=\{L(\mathcal{S}_1),...,L(\mathcal{S}_i)\}$ and $\mathbf{R}=\{L(\mathcal{S}_{i+1}),...,L(\mathcal{S}_n)\}$, where all elements in $\mathbf{T}$ are causally earlier than those in $\mathbf{R}$. Let $\hat{\mathbf{Z}}$ contain half of the set of children of each latent variable set in $\mathbf{T}$, and $\hat{\mathbf{Y}}$ contain the elements from the other half set of the children of each latent variable set in $\mathbf{R}$. Furthermore, Let $L(\mathcal{S}_r)$ be a latent variable set of $\mathbf{R}$ and $\mathcal{S}_r=\{R_1,R_2,...,R_{2Dim(L(\mathcal{S}_r))}\}$. If for any one of the remaining elements $L(\mathcal{S}_k)\in\mathbf{R}$, with $k\neq r$ and $\mathcal{S}_k=\{K_1,K_2,...,K_{2Dim(L(\mathcal{S}_k))}\}$ such that $(\{R_{Dim(L(\mathcal{S}_r))+1},...R_{2Dim(L(\mathcal{S}_r))},\hat{\mathbf{Z}}\},\{R_1,....,R_{Dim(L(\mathcal{S}_r))},\newline K_1,...K_{Dim(L(\mathcal{S}_k))},\hat{\mathbf{Y}}\})$ follows the GIN condition, then $L(\mathcal{S}_r)$ is a root latent variable set in $\mathbf{R}$.    
\end{proposition}

The proposition can be easily proved, as stated in the original paper \cite{xie2020generalized}, just by treating measured children of latents in $\mathbf{T}$ as a new causal cluster.  Now we want to show that this method can also be used when there are cycles between the measured children and latents in the same causal cluster:

\begin{proposition}\label{prop:GINcycle4}
Suppose that $\mathbf{C} = \{\mathcal{S}_1,...,\mathcal{S}_i,...,\mathcal{S}_n\}$ contains all clusters of the $L^2HCM$.  Denote $\mathbf{R}=\{\mathcal{S}_k,...,\mathcal{S}_n\}$ to be the set of clusters in which latents being caused by their measured children, so $\mathbf{C}\setminus\mathbf{R}$ contains clusters of a $LiNGLaM$.  Denote the causal order $\mathcal{K}=\langle L(\mathcal{S}^1),...,L(\mathcal{S}^i)\rangle$, which is a causal order over latent variables of clusters in $\mathbf{C\setminus R}$ s.t. $i<j$ iff $L(\mathcal{S}^i)$ is causally no later than $L(\mathcal{S}^j)$.  Given $L(\mathcal{S}_r)$ being the latents of its clustered measured children $\mathcal{S}_r\in\mathbf{C\setminus R}$ and $\mathcal{S}_r=\{R_1,R_2,...,R_{2Dim(L(\mathcal{S}_r))}\}$, denote $\mathbf{T}$ the set of latents causally earlier than $L(\mathcal{S}_r)$ according to $\mathcal{K}$. let $\hat{\mathbf{Z}}$ contain the elements from the half set of the children of each latent variable set in $\mathbf{T}$, and $\hat{\mathbf{Y}}$ contain the elements from half set of the children of a latent variable set $L(\mathcal{S}_k)\in\mathbf{R}$.  Let $\mathcal{S}_r=\{R_1,R_2,...,R_{2Dim(L(\mathcal{S}_r))}\}$.  If $(\{R_{Dim(L(\mathcal{S}_r))+1},...R_{2Dim(L(\mathcal{S}_r))},\hat{\mathbf{Z}}\}, \{R_1,....,R_{Dim(L(\mathcal{S}_r))},K_1,...K_{Dim(L(\mathcal{S}_k))},\hat{\mathbf{Y}}\})$ follows the GIN condition, then $L(\mathcal{S}_r)$ is causally earlier than $L(\mathcal{S}_k))$ and so is latent variables in $\mathbf{T}$ .    
\end{proposition}
It is easy to see that this proposition is just an application of \textbf{Proposition}
\ref{prop:GIN4}.\\ \textcolor{white}{abc}
\textit{\textbf{Example.}} Consider fig \ref{fig:collidercycle} again.  Denote $a_{ij}$ as the edge strength or coeffieicent of the edge $L_i\rightarrow X_j$ and $b_{ij}$ of the edge $X_i\rightarrow L_j$ and $c_{ij}$ of the edge $L_i\rightarrow L_j$.  Consider $\omega = [-c_{21}a_{15},a_{21}]$.  Having $\mathbf{Z}=\{X_2\}$ and $\mathbf{Y}=\{X_1, X_5\}$, we can see $E_{\mathbf{Y}\parallel\mathbf{Z}}\indep\mathbf{Z}$. Therefore, even though there are cycles under $L_1$, we still have $(\{X_2\}, \{X_1, X_5\})$ following the GIN condition.  Similarly, $(\{X_7\}, \{X_6, X_5\})$ follows the GIN condition.  Therefore we conclude $L_2\rightarrow L_1$ and $L_3\rightarrow L_1$.
\subsection{Rank Constraint and GIN Condition-Based Algorithm for Estimating $L^2HCM$}

We now present an algorithm of learning a $L^2HCM$ latent structure with non-Gaussian distribution.  The procedure of the algorithm is similar to the GIN condition-based algorithm that estimates $LiNGLaM$\cite{xie2020generalized}.  The first stage is to identify individual causal clusters and cycles under latents, the second stage is to learn the causal order of latent variable sets of acyclic causal orders and the third stage is to learn the partial causal order of the latent variable sets of clusters.  The theoretical justification of the algorithm is shown by \textbf{Theorem} \ref{thm:ranklinear}, \textbf{Theorem} \ref{thm:rankdetectcycles}, \textbf{Proposition} \ref{prop:GIN4} and \ref{prop:GINcycle4}.  

\begin{algorithm}
    \caption{LatentCausalCyclicStructureDiscovery}\label{alg:fullalgolatentcycunder}
    \KwIn{Data from a set of measured variables $X_\mathcal{G}$}
\KwOut{Partial causal order  $\mathcal{K}$}
$\mathcal{L}, \mathcal{L}_c\gets FindCausalCyclicClusters(X_\mathcal{G})$;\\ 
$\mathcal{K}\gets LearningtheCausalOrderofLatentVariables(\mathcal{L})$;\\ 
$\mathcal{K}\gets LearningCausalOrderForCyclicClusters(\mathcal{L}_c, \mathcal{K})$;

\end{algorithm}

\subsubsection{Stage 1: Finding Causal Cyclic Clusters}
The first stage of the algorithm is to find the causal cluster as measured variables sharing the same set of latent variables as parents.  Since we allow cycles in between the latent parents and measured children, GIN condition is not able to identify all causal clusters.  For the sake of completeness, we use rank constraint to identify clusters first then use GIN to check the existence of cycles.  By \textbf{Lemma} \ref{GINdsep} it is easy to see that for any cluster that can be identified by GIN, it can identified by the rank constraint.

After identifying (and merging) all clusters, the algorithm goes through each causal cluster  to check whether there are cycles between the latent variable set and the measured children, by checking if measured children in the cluster $\mathcal{S}$ and the rest of the measured variables $X_\mathcal{G}\setminus\mathcal{S}$ follow the GIN condition.  By \textbf{Theorem} \ref{thm:GINRankCycle}, if the GIN condition is not satisfied, then there are edges from the measured children to the latent set of the cluster.  Such cluster will be labeled as `\textit{cyclic}.' The number of latents for a cyclic cluster will be updated as the rank of the correlation matrix between $\mathcal{S}$ and $X_\mathcal{G}\setminus\mathcal{S}$.

The first stage of the algorithm is presented in algorithm \ref{alg:cycunderLAdetect}.\\ \textcolor{white}{abc} 
\begin{algorithm}
\caption{FindCausalCyclicClusters}\label{alg:cycunderLAdetect}
\KwIn{Data from a set of measured variables $X_\mathcal{G}$}
\KwOut{a set of clusters $\mathcal{L}, \mathcal{L}_c$}
 $\mathcal{L}, \mathcal{L}_c\gets\emptyset, \emptyset$\\ 
 Active Set $\mathcal{S}\gets X_\mathcal{G}$\\ 
 $k\gets 1;$\\ 
  {$\#$\textit{Identifying Causal Clusters}}\\ 
\While{$\mathcal{S}\neq \emptyset$ and $|\mathcal{S}|\geq 2k+2$}{$\mathcal{L}_k\gets\emptyset$\Comment*[r]{$\mathcal{L}_k$ keeps track clusters with $k$ many latents}

\For{$\mathbf{S}\subset\mathcal{S}, |\mathbf{S}|=k+1$}{$\mathbf{S'}\gets X_\mathcal{G}\setminus\mathbf{S}$

\If{$rank(\Sigma_\mathbf{S, S'})<k$}{$\mathcal{L}_k\gets \mathcal{L}_k\cup\mathbf{S}$}}
Merge all overlapping clusters in $\mathcal{L}_k$

\For{$\mathbf{S}\in\mathcal{L}_k$}{
$\mathcal{L}\gets \mathcal{L}\cup\{\mathbf{S}\}$

$X_\mathcal{G}\gets X_\mathcal{G}\setminus\mathbf{S}$

Initiate $L_\mathbf{S}$ as the number of latent causes shared in $\mathbf{S}$

$L_\mathbf{S}\gets k$

}
$k\gets k+1$

}
{$\#$\textit{Identifying Cycles under Latents}}\\ 
\For{$\mathbf{S}\in\mathcal{L}$}{
$\mathbf{S'}\gets X_\mathcal{G}\setminus\mathbf{S}$

\If{$E_{\mathbf{S\parallel S' }}\nindep \mathbf{S'}$}{Mark $\mathbf{S}$ as $cyclic$ and $\mathcal{L}_c\gets\mathcal{L}_c\cup \{\mathbf{S}\}$

$L_\mathbf{S}\gets \lfloor max_{\mathbf{S''\subset S}}$  $ 
   rank(\Sigma_{\mathbf{S'',  S\setminus S''}}) \rfloor/2$

}

} 
\textbf{Return } $\mathcal{L}, \mathcal{L}_c$ 
\end{algorithm}
Notice that not all cyclic clusters can be identified by this method.  When the latents of a cyclic cluster are not causally earlier than any other latent, the GIN is not able to identify its cyclicity.  For instance, consider fig \ref{fig:collidercycle} again. Having $\mathbf{Z}=\{X_2, X_1, X_6, X_7\}$ and $\mathbf{Y}=\{X_3, X_4, X_5\}$, it is easy to check that $(\mathbf{Z, Y})$ follows the GIN condition.  Fortunately this case only occurs when the cluster is the latest in the causal order, therefore such cluster can be noticed and marked at the end of the stage 3 of the algorithm.

\subsubsection{Stage 3: Learning the Partial Causal Order of Latent Variables}
The stage 2 of the algorithm takes the cluster not marked as `\textit{cyclic}' and run the step 2 of the original \textbf{GIN condition-based algorithm for estimating $LiNGLaM$} to learn the causal order of the clusters\cite{xie2020generalized}.  Using the learned causal order $\mathcal{K}$ among the latent variables in stage 2, stage 3 has two tasks:
\begin{enumerate}
    \item To check all the causal clusters that are latest in $\mathcal{K}$, to see if any of them are cyclic;
    \item To learn the partial causal order among cyclic clusters. 
\end{enumerate}

\begin{algorithm}
\caption{LearningCausalOrderForCyclicClusters}\label{alg:CAorderlearning}
\KwIn{A set of clusters $\mathcal{C}$, a causal order $\mathcal{K}$}
\KwOut{Causal Order $\mathcal{K}_c$}

$\mathcal{L}\gets$the set of latent variables for each cluster\\ 
{$\#$\textit{Further Identifying Cycles under Latents}}\\ 
\For{$L(\mathbf{S})\in\mathcal{K}$ that is causally the latest in $\mathcal{K}$}{

\If{$rank(\Sigma_{\mathbf{S,  X_\mathcal{G}\setminus S}})<|L(\mathbf{S})|$}{Mark $\mathbf{S}$ as $cyclic$ as keep track of it\\ 

$L_\mathbf{S}\gets \lfloor max_{\mathbf{S''\subset S}}$  $rank(\Sigma_{\mathbf{S'',  S\setminus S''}}) \rfloor/2$\Comment*[r]{$L_\mathbf{S}$ is the size of $L(\mathbf{S})$}

   $\mathcal{K}\gets\mathcal{K}\setminus \{L(\mathbf{S})\}$
}

}

 \For{$\mathbf{S}\in\mathcal{C}$}
 { \For{$L(\mathbf{S_r})\in\mathcal{K}$, starting from the causally latest one and move to the earlier ones}{

 $\mathbf{T}\gets\{L(\mathbf{S_i})|L(\mathbf{S_i}) \text{is causally earlier than } L(\mathbf{S_r})\}$

 $\mathbf{R}\gets\{L(\mathbf{S_r}),L(\mathbf{S})\}$
 
\If{
$L(\mathbf{S_r})$ is the root variable set in $\mathbf{R}$ according to \textbf{Proposition} \ref{prop:GINcycle4}
        }{
Include $\mathbf{S}$ in $\mathcal{K}_c$ and mark as \textit{causally later than $\mathbf{S_r}$}
        }
 
 }
 }

Include all the newly marked cyclic clusters in $\mathcal{K}_c$ and mark their causal order as the same as in line 3

\textbf{Return } $\mathcal{K}_c$ 
\end{algorithm}

The soundness of the algorithmic steps for the first task is proved in \textbf{Theorem} \ref{thm:rankdetectcycles} and the second task is proved in\textbf{ Proposition} \ref{prop:GINcycle4}.\\ \textcolor{white}{abc}
In stage 3, the causal order $\mathcal{K}$ of the acyclic clusters is used as a `ruler' to measure the causal order of cyclic clusters.  We call the causal order as `partial' since the causal relation between the two cyclic clusters cannot be identified if they are not distinguishable by the `ruler.'\\ \textcolor{white}{abc}

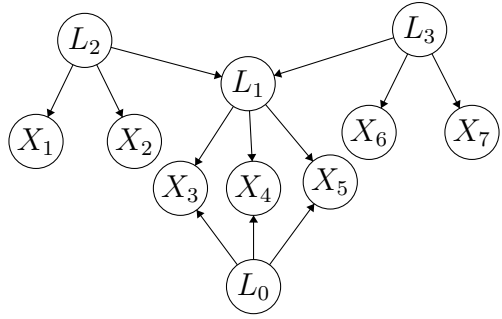
\begin{figure}[H]
\begin{center}
\begin{tikzpicture}[scale=0.12]
\tikzstyle{every node}+=[inner sep=0pt]
\draw [black] (19,-17.5) circle (3);
\draw (19,-17.5) node {$L_2$};
\draw [black] (37.1,-22.3) circle (3);
\draw (37.1,-22.3) node {$L_1$};
\draw [black] (56.1,-16.3) circle (3);
\draw (56.1,-16.3) node {$L_3$};
\draw [black] (29.6,-33.9) circle (3);
\draw (29.6,-33.9) node {$X_3$};
\draw [black] (37.7,-33.9) circle (3);
\draw (37.7,-33.9) node {$X_4$};
\draw [black] (46.2,-33.2) circle (3);
\draw (46.2,-33.2) node {$X_5$};
\draw [black] (37.7,-44.8) circle (3);
\draw (37.7,-44.8) node {$L_0$};
\draw [black] (13.6,-28.7) circle (3);
\draw (13.6,-28.7) node {$X_1$};
\draw [black] (24.5,-28.7) circle (3);
\draw (24.5,-28.7) node {$X_2$};
\draw [black] (50.5,-27.7) circle (3);
\draw (50.5,-27.7) node {$X_6$};
\draw [black] (61.9,-27.7) circle (3);
\draw (61.9,-27.7) node {$X_7$};
\draw [black] (53.24,-17.2) -- (39.96,-21.4);
\fill [black] (39.96,-21.4) -- (40.87,-21.63) -- (40.57,-20.68);
\draw [black] (21.9,-18.27) -- (34.2,-21.53);
\fill [black] (34.2,-21.53) -- (33.56,-20.84) -- (33.3,-21.81);
\draw [black] (39.47,-42.38) -- (44.43,-35.62);
\fill [black] (44.43,-35.62) -- (43.55,-35.97) -- (44.36,-36.56);
\draw [black] (37.7,-41.8) -- (37.7,-36.9);
\fill [black] (37.7,-36.9) -- (37.2,-37.7) -- (38.2,-37.7);
\draw [black] (35.91,-42.39) -- (31.39,-36.31);
\fill [black] (31.39,-36.31) -- (31.47,-37.25) -- (32.27,-36.65);
\draw [black] (35.47,-24.82) -- (31.23,-31.38);
\fill [black] (31.23,-31.38) -- (32.08,-30.98) -- (31.24,-30.44);
\draw [black] (37.25,-25.3) -- (37.55,-30.9);
\fill [black] (37.55,-30.9) -- (38,-30.08) -- (37,-30.13);
\draw [black] (39.02,-24.6) -- (44.28,-30.9);
\fill [black] (44.28,-30.9) -- (44.15,-29.96) -- (43.38,-30.6);
\draw [black] (17.7,-20.2) -- (14.9,-26);
\fill [black] (14.9,-26) -- (15.7,-25.49) -- (14.8,-25.06);
\draw [black] (20.32,-20.19) -- (23.18,-26.01);
\fill [black] (23.18,-26.01) -- (23.27,-25.07) -- (22.38,-25.51);
\draw [black] (54.78,-18.99) -- (51.82,-25.01);
\fill [black] (51.82,-25.01) -- (52.62,-24.51) -- (51.73,-24.07);
\draw [black] (57.46,-18.97) -- (60.54,-25.03);
\fill [black] (60.54,-25.03) -- (60.62,-24.09) -- (59.73,-24.54);
\end{tikzpicture}
\end{center}
    \caption{An alternative graph that is not $L^2HCM$ that is not distinguishable by the algorithm from fig \ref{fig:collidercycle}.}
    \label{fig:alternativeofcycleunderLA}
\end{figure}

\textbf{\textit{Remark.}}  Recall that we require all latent variables in $L^2HCM$ to be extroverted parents.  Here is an example demonstrating why we need this assumption.  If it is not the case that all latent variables are extroverted parents, the clusters found in line 4 of state 3 are not necessarily cyclic, where there can just be some latents of the cluster exogeneous to the latents from other clusters.  For instance, in fig \ref{fig:alternativeofcycleunderLA}, the cluster $\{X_3, X_4, X_5\}$ was initially identified as having 2 latents and passes the stage 2 of the whole algorithm identified as the causally latest cluster, but in line 3 of stage 5 we have $rank(\Sigma_{\{X_3, X_4, X_5\}, \{X_1, X_2, X_6, X_7\}})$ to be 1, but it is not cyclic.  Technically fig \ref{fig:alternativeofcycleunderLA} is a legit causal latent structure, but such case is incompatible with the one of the initial learning task: we want to learn the causal relation between the \textit{latent variable sets(blocks)}, defined by having the same cluster of measured children. The case of fig \ref{fig:alternativeofcycleunderLA} or any case where latent variables from the same causal cluster are not having the same causal order position relative to all the other latents does not fit in this quest.  Therefore, in this section we rule out such cases, but it is an interesting future research topic.
\subsubsection{Experiments}
We applied the proposed algorithm to synthetic data.  Specifically, we comparied the proposed algorithm (CGIN) with the original algorithm (GIN) about the performance in forming the correct clusters and recovering the causal order among latents of clusters. We further investigated the performance of CGIN identifying cycles under latent variables in clusters.  There are some other causal clustering algorithms that are not included in the experiment here:  we did not compare CGIN with regular rank-constraint-based causal clustering algorithms such as FOFC and FTFC, since FOFC and FTFC only identifies clusters with a fixed number of latents and not learning the causal relations between latents\cite{FOFC};  we did not compare MIMBuild, which is another algorithm that learns causal relations between latent variables shared by measured variables, because it learns a markov equivalence class of latent structures, which does not necessarily include a specific causal order\cite{JMLR:v7:silva06aMIMBuild}; we did not compare our algorithm with the tree-based method grouping algorithm CLRG, since it does not recover causal order of latent variables\cite{zhang2021robustifyingCLRG}.\\ \textcolor{white}{abc}
In the simulation study, we run three cases:fig \ref{fig:sim3Cluster}, fig \ref{fig:sim4clusters} and fig \ref{fig:sim4clusters} without the edges from measured variables to the latents.  We included the acyclic graph to focus on testing and comparing the algorithmic learning effect on the correctly clustering measured variables and recovering the causal order of latents from each cluster.  All data are generated by $L^2HCM$ such that each edge strength is generated uniformly from $[-5, -0.5]\cup[0.5, 5]$ and noise terms uniformly from $[-2, 2]$.  Each case is simulated with sample size 500, 1000, 2000.  Due to the consideration of the running time, we did not test them on larger sample sizes.  Notice that the relative width of the gap of the interval generating the edge strength in our study is significantly narrower than what was used in the simulation study of GIN in the original paper, which is $[-2, -0.5]\cup[0.5, 2]$.  We used the Kernel-based Conditional Independence test (KCI) for the non-Gaussian independence test.\\ \textcolor{white}{abc}
We applied two sets of metrics to evaluate the performance of the algorithm. The first set has four metrics evaluating the accuracy of clustering and latent causal order learning:
\begin{enumerate}
    \item $ClusterRecall = \dfrac{\sum_{i,j \in\mathbf{V}} I'_{ij}I_{ij}}{\sum_{i,j \in\mathbf{V}} I_{ij}}$
    \item $ClusterPrecision = \dfrac{\sum_{i,j \in\mathbf{V}} I'_{ij}I_{ij}}{\sum_{i,j \in\mathbf{V}} I'_{ij}}$
    \item $LatentOrderRecall = \dfrac{|\mathbf{CO'}\cap\mathbf{CO}|}{|\mathbf{CO}|}$
    \item $LatentOrderPrecision = \dfrac{|\mathbf{CO'}\cap\mathbf{CO}|}{|\mathbf{CO'}|}$
\end{enumerate}
The indicator function $I'_{ij} = 1$ if and only if $X_i$ and $X_j$ are in the same cluster in the output of the algorithm; $I_{ij} = 1$ if and only if $X_i$ and $X_j$ are in the same cluster in the true graph.  $\mathbf{CO'}$ is the set containing pairs of variables: $(X_i, X_j)$ (with order) is in $\mathbf{CO'}$ if and only if the latent parents of $X_i$ is in is causally earlier than latent parents of $X_j$ in the output of the algorithm;  $\mathbf{CO}$ is the set containing pairs of variables: $(X_i, X_j)$ is in $\mathbf{CO}$ if and only if the latent parents of $X_i$ is in is causally earlier than latent parents of $X_j$ in the true graph.  Therefore, $LatentOrderRecall$ measures the proportion of pairwise causal orders between latent parents of measured variables in the true causal graph correctly learnt by the algorithm; $LatentOrderPrecision$ measures the proportion of the pairwise causal orders between latent parents of measured variables in the output of the algorithm that are actually in the true causal graph. \\ \textcolor{white}{abc}
The result of the first set of metrics is Table 1 in \ref{table1: clusteringandorder}.  We can see that even without the existence of cyclic clusters, CGIN still generally equals or outperforms GIN except for $LatentOrderPrecision$ and $LatentOrderRecall$ at the low sample size and $clusteringRecall$ in the 3-cluster case.  The CGIN stably gives more accurate results with lower variances.  What is worth noting is that as the sample size increases, the performance of CGIN and GIN learning causal order is not improving.  It might suggest that the Kernel-based Conditional Independence test can be unreliable.\\ \textcolor{white}{abc}  
We further did significance tests for the four metrics .  Specifically, we ran 30 simulations with GIN and CGIN for each of the three graphs with the sample size of 500.  The result is in fig \ref{fig:p-val}.  The result shows that the performance difference between CGIN and GIN is significant at the 0.05 significance level, other than the $ClusteringPrecision$ and $LatentOrderRecall$ for the acyclic graph, which is compatible with the fact that the main function of CGIN is to handle the latent causal structure with cycles between the measured children and the latent parent, therefore when there are no such cycles the performance of GIN and CGIN should be similar.\\ \textcolor{white}{abc}
\begin{figure}
    \centering
\includegraphics[width=1\linewidth]{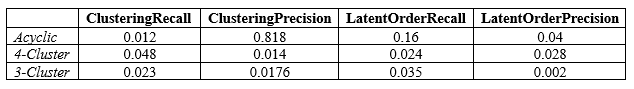}
    \caption{Result of significance test between the performance of GIN and CGIN.}
    \label{fig:p-val}
\end{figure}

The second set of metrics evaluate the accuracy of identifying cyclic clusters:
\begin{enumerate}
    \item $CyclicRecall = \dfrac{\sum_{i\in\mathbf{V}} I^{'c}_{i}I^c_{i}}{\sum_{i\in\mathbf{V}} I^c_{ij}}$
    \item $CyclicPrecision = \dfrac{\sum_{i\in\mathbf{V}} I^{'c}_{i}I^c_{i}}{\sum_{i\in\mathbf{V}} I^{'c}_{i}}$
\end{enumerate}
The indicator function $I^{'c}_{i} = 1$ if and only if there exists an edge from $X_i$ to its latent parents in the cluster, i.e., the cluster being cyclic, in the output of the algorithm; $I^c_{i} = 1$ if and only if  there exists an edge from $X_i$ to its latent parents in the cluster in the true graph.  Therefore, the $CyclicRecall$ measures the proportion of variables in cyclic clusters in the true causal graph correctly identified by the algorithm; the $CyclicPrecision$ measures the proportion of variables in cyclic clusters in the output of the algorithm that are actually in cyclic clusters in the true graph.\\ \textcolor{white}{abc}
The result of the second set of metrics is Table 2 in \ref{table2:cyclic}.  The high precision shows that CGIN has a low tendency of producing wrong edges or falsely identifying cluster as being cyclic.  The recall is comparatively lower with high variances, indicating that the algorithm tends to miss cycles.  In particular, it means that the false positive rate of GIN condition test is high, by falsely rejecting the null hypothesis that a linear combination of a (acyclic) cluster of variables and the rest of the variables are independent.

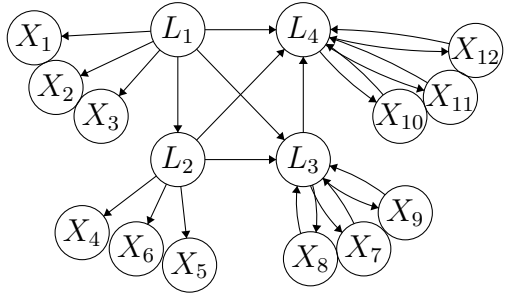
\begin{figure}
    
\begin{center}
\begin{tikzpicture}[scale=0.12]
\tikzstyle{every node}+=[inner sep=0pt]
\draw [black] (17.5,-14.9) circle (3);
\draw (17.5,-14.9) node {$X_1$};
\draw [black] (19.9,-20.4) circle (3);
\draw (19.9,-20.4) node {$X_2$};
\draw [black] (22.8,-36.5) circle (3);
\draw (22.8,-36.5) node {$X_4$};
\draw [black] (34.7,-40.1) circle (3);
\draw (34.7,-40.1) node {$X_5$};
\draw [black] (28.8,-38.3) circle (3);
\draw (28.8,-38.3) node {$X_6$};
\draw [black] (48.1,-39.4) circle (3);
\draw (48.1,-39.4) node {$X_8$};
\draw [black] (24.9,-23.6) circle (3);
\draw (24.9,-23.6) node {$X_3$};
\draw [black] (33.4,-14) circle (3);
\draw (33.4,-14) node {$L_1$};
\draw [black] (54,-38.3) circle (3);
\draw (54,-38.3) node {$X_7$};
\draw [black] (33.4,-28.4) circle (3);
\draw (33.4,-28.4) node {$L_2$};
\draw [black] (47.3,-14) circle (3);
\draw (47.3,-14) node {$L_4$};
\draw [black] (58.6,-34.3) circle (3);
\draw (58.6,-34.3) node {$X_9$};
\draw [black] (58,-23.6) circle (3);
\draw (58,-23.6) node {$X_{10}$};
\draw [black] (63.6,-21.5) circle (3);
\draw (63.6,-21.5) node {$X_{11}$};
\draw [black] (66.4,-16.3) circle (3);
\draw (66.4,-16.3) node {$X_{12}$};
\draw [black] (47.3,-28.4) circle (3);
\draw (47.3,-28.4) node {$L_3$};
\draw [black] (31.41,-16.25) -- (26.89,-21.35);
\fill [black] (26.89,-21.35) -- (27.79,-21.09) -- (27.04,-20.42);
\draw [black] (30.4,-14.17) -- (20.5,-14.73);
\fill [black] (20.5,-14.73) -- (21.32,-15.18) -- (21.27,-14.19);
\draw [black] (30.69,-15.29) -- (22.61,-19.11);
\fill [black] (22.61,-19.11) -- (23.55,-19.22) -- (23.12,-18.32);
\draw [black] (32.14,-31.12) -- (30.06,-35.58);
\fill [black] (30.06,-35.58) -- (30.85,-35.06) -- (29.95,-34.64);
\draw [black] (33.73,-31.38) -- (34.37,-37.12);
\fill [black] (34.37,-37.12) -- (34.78,-36.27) -- (33.78,-36.38);
\draw [black] (55.427,-22.06) arc (-123.76303:-140.03364:29.991);
\fill [black] (55.43,-22.06) -- (55.04,-21.2) -- (54.48,-22.03);
\draw [black] (31.02,-30.22) -- (25.18,-34.68);
\fill [black] (25.18,-34.68) -- (26.12,-34.59) -- (25.52,-33.8);
\draw [black] (60.701,-20.732) arc (-106.79958:-122.61689:43.725);
\fill [black] (60.7,-20.73) -- (60.08,-20.02) -- (59.79,-20.98);
\draw [black] (63.403,-16.418) arc (-89.42077:-104.31209:51.374);
\fill [black] (63.4,-16.42) -- (62.6,-15.93) -- (62.61,-16.93);
\draw [black] (49.852,-15.575) arc (55.65121:40.55212:32.244);
\fill [black] (49.85,-15.58) -- (50.23,-16.44) -- (50.79,-15.61);
\draw [black] (50.169,-14.874) arc (71.50417:59.07936:55.438);
\fill [black] (50.17,-14.87) -- (50.77,-15.6) -- (51.09,-14.65);
\draw [black] (50.3,-13.968) arc (89.2387:77.02844:62.479);
\fill [black] (50.3,-13.97) -- (51.09,-14.48) -- (51.11,-13.48);
\draw [black] (48.331,-31.212) arc (14.36508:-6.04579:14.874);
\fill [black] (48.71,-36.47) -- (49.3,-35.73) -- (48.3,-35.62);
\draw [black] (51.616,-36.488) arc (-133.23914:-158.58312:14.315);
\fill [black] (51.62,-36.49) -- (51.38,-35.58) -- (50.69,-36.3);
\draw [black] (55.653,-33.762) arc (-105.4688:-129.67148:16.754);
\fill [black] (55.65,-33.76) -- (55.01,-33.07) -- (54.75,-34.03);
\draw [black] (47.016,-36.609) arc (-164.92183:-186.75888:14.001);
\fill [black] (46.63,-31.32) -- (46.04,-32.05) -- (47.03,-32.17);
\draw [black] (49.492,-30.445) arc (42.74923:25.42851:20.291);
\fill [black] (49.49,-30.44) -- (49.67,-31.37) -- (50.4,-30.69);
\draw [black] (50.216,-29.092) arc (72.36837:52.49136:20.096);
\fill [black] (50.22,-29.09) -- (50.83,-29.81) -- (51.13,-28.86);
\draw [black] (33.4,-17) -- (33.4,-25.4);
\fill [black] (33.4,-25.4) -- (33.9,-24.6) -- (32.9,-24.6);
\draw [black] (35.48,-16.16) -- (45.22,-26.24);
\fill [black] (45.22,-26.24) -- (45.02,-25.32) -- (44.3,-26.01);
\draw [black] (36.4,-14) -- (44.3,-14);
\fill [black] (44.3,-14) -- (43.5,-13.5) -- (43.5,-14.5);
\draw [black] (35.48,-26.24) -- (45.22,-16.16);
\fill [black] (45.22,-16.16) -- (44.3,-16.39) -- (45.02,-17.08);
\draw [black] (36.4,-28.4) -- (44.3,-28.4);
\fill [black] (44.3,-28.4) -- (43.5,-27.9) -- (43.5,-28.9);
\draw [black] (47.3,-25.4) -- (47.3,-17);
\fill [black] (47.3,-17) -- (46.8,-17.8) -- (47.8,-17.8);
\end{tikzpicture}
\end{center}
    \caption{Simulation Graph with Four Clusters}
    \label{fig:sim4clusters}
\end{figure}
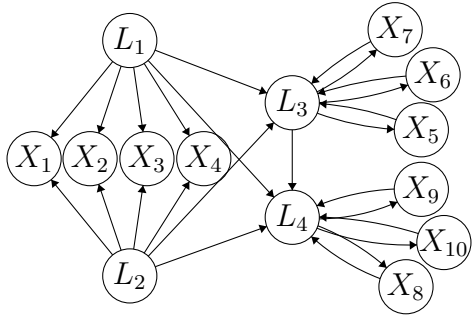
\begin{figure}
    
\begin{center}
\begin{tikzpicture}[scale=0.12]
\tikzstyle{every node}+=[inner sep=0pt]
\draw [black] (11,-27) circle (3);
\draw (11,-27) node {$X_1$};
\draw [black] (17.2,-27) circle (3);
\draw (17.2,-27) node {$X_2$};
\draw [black] (29.8,-27) circle (3);
\draw (29.8,-27) node {$X_4$};
\draw [black] (52.2,-41.2) circle (3);
\draw (52.2,-41.2) node {$X_8$};
\draw [black] (23.5,-27) circle (3);
\draw (23.5,-27) node {$X_3$};
\draw [black] (21.6,-13.9) circle (3);
\draw (21.6,-13.9) node {$L_1$};
\draw [black] (56.4,-36.3) circle (3);
\draw (56.4,-36.3) node {$X_{10}$};
\draw [black] (21.6,-40) circle (3);
\draw (21.6,-40) node {$L_2$};
\draw [black] (39.6,-20.8) circle (3);
\draw (39.6,-20.8) node {$L_3$};
\draw [black] (53.9,-30.4) circle (3);
\draw (53.9,-30.4) node {$X_9$};
\draw [black] (53.9,-23.8) circle (3);
\draw (53.9,-23.8) node {$X_5$};
\draw [black] (55.2,-17.7) circle (3);
\draw (55.2,-17.7) node {$X_6$};
\draw [black] (51,-12.7) circle (3);
\draw (51,-12.7) node {$X_7$};
\draw [black] (39.6,-33.5) circle (3);
\draw (39.6,-33.5) node {$L_4$};
\draw [black] (22.03,-16.87) -- (23.07,-24.03);
\fill [black] (23.07,-24.03) -- (23.45,-23.17) -- (22.46,-23.31);
\draw [black] (19.71,-16.23) -- (12.89,-24.67);
\fill [black] (12.89,-24.67) -- (13.78,-24.36) -- (13,-23.73);
\draw [black] (20.64,-16.74) -- (18.16,-24.16);
\fill [black] (18.16,-24.16) -- (18.88,-23.56) -- (17.94,-23.24);
\draw [black] (50.902,-23.749) arc (-93.75364:-109.94289:30.969);
\fill [black] (50.9,-23.75) -- (50.14,-23.2) -- (50.07,-24.2);
\draw [black] (23.2,-37.46) -- (28.2,-29.54);
\fill [black] (28.2,-29.54) -- (27.35,-29.95) -- (28.2,-30.48);
\draw [black] (52.421,-18.828) arc (-70.40544:-87.11599:34.458);
\fill [black] (52.42,-18.83) -- (51.5,-18.63) -- (51.84,-19.57);
\draw [black] (48.957,-14.895) arc (-46.08158:-63.12883:27.364);
\fill [black] (48.96,-14.9) -- (48.03,-15.09) -- (48.73,-15.81);
\draw [black] (42.598,-20.892) arc (85.6637:70.63977:33.297);
\fill [black] (42.6,-20.89) -- (43.36,-21.45) -- (43.43,-20.45);
\draw [black] (42.423,-19.786) arc (107.79881:94.67976:43.649);
\fill [black] (42.42,-19.79) -- (43.34,-20.02) -- (43.03,-19.06);
\draw [black] (41.722,-18.681) arc (132.37629:118.4133:33.197);
\fill [black] (41.72,-18.68) -- (42.65,-18.51) -- (41.98,-17.77);
\draw [black] (42.446,-34.444) arc (68.36646:48.77441:26.229);
\fill [black] (50.06,-39.1) -- (49.79,-38.2) -- (49.13,-38.95);
\draw [black] (53.413,-36.56) arc (-88.05784:-110.86681:28.381);
\fill [black] (53.41,-36.56) -- (52.6,-36.09) -- (52.63,-37.09);
\draw [black] (51.249,-31.8) arc (-66.08537:-89.45163:21.872);
\fill [black] (51.25,-31.8) -- (50.32,-31.67) -- (50.72,-32.58);
\draw [black] (49.34,-40.299) arc (-110.973:-131.88613:24.648);
\fill [black] (41.71,-35.63) -- (41.97,-36.54) -- (42.64,-35.8);
\draw [black] (42.599,-33.474) arc (88.38677:72.68859:40.726);
\fill [black] (42.6,-33.47) -- (43.38,-34) -- (43.41,-33);
\draw [black] (42.315,-32.228) arc (111.84372:92.61928:26.309);
\fill [black] (42.32,-32.23) -- (43.24,-32.39) -- (42.87,-31.47);
\draw [black] (23.63,-16.11) -- (37.57,-31.29);
\fill [black] (37.57,-31.29) -- (37.4,-30.36) -- (36.66,-31.04);
\draw [black] (24.4,-14.97) -- (36.8,-19.73);
\fill [black] (36.8,-19.73) -- (36.23,-18.97) -- (35.87,-19.91);
\draw [black] (23.65,-37.81) -- (37.55,-22.99);
\fill [black] (37.55,-22.99) -- (36.64,-23.23) -- (37.37,-23.91);
\draw [black] (24.42,-38.98) -- (36.78,-34.52);
\fill [black] (36.78,-34.52) -- (35.86,-34.32) -- (36.2,-35.26);
\draw [black] (19.7,-37.67) -- (12.9,-29.33);
\fill [black] (12.9,-29.33) -- (13.01,-30.26) -- (13.79,-29.63);
\draw [black] (20.64,-37.16) -- (18.16,-29.84);
\fill [black] (18.16,-29.84) -- (17.94,-30.76) -- (18.89,-30.44);
\draw [black] (22.03,-37.03) -- (23.07,-29.97);
\fill [black] (23.07,-29.97) -- (22.46,-30.69) -- (23.45,-30.83);
\draw [black] (23.19,-16.44) -- (28.21,-24.46);
\fill [black] (28.21,-24.46) -- (28.21,-23.51) -- (27.36,-24.04);
\draw [black] (39.6,-23.8) -- (39.6,-30.5);
\fill [black] (39.6,-30.5) -- (40.1,-29.7) -- (39.1,-29.7);
\end{tikzpicture}
\end{center}
    \caption{Simulation Graph with Three Clusters}
    \label{fig:sim3Cluster}
\end{figure}

\begin{figure}
    \centering
    \includegraphics[width=1\linewidth]{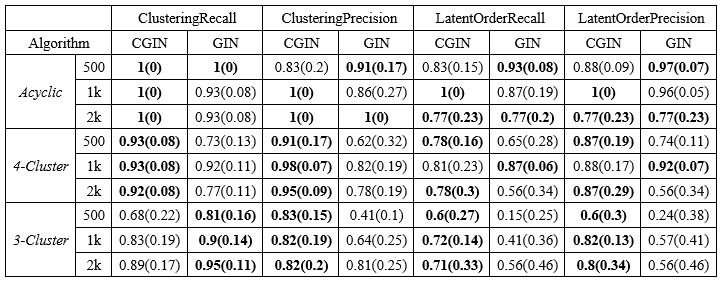}
    \caption{Table 1:  Performance of CGIN and GIN on simulation data.  The number in the parentheses is the standard deviation of the index. The better performance for each comparison is boldface.}
    \label{table1: clusteringandorder}
\end{figure}

\begin{figure}
    \centering
    \includegraphics{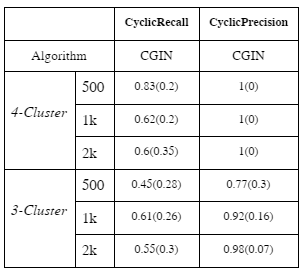}
    \caption{Table 2: Performance of identifying cycles by CGIN on simulation data.  The number in the parentheses is the standard deviation of the index.}
    \label{table2:cyclic}
\end{figure}
\newpage
\subsection{GIN Condition-Based Algorithm for Estimating $LiNGLaM$}

We now present an algorithm of learning a $LiNGLaM$ latent structure with non-Gaussian distribution. Similar to the\textbf{ GIN condition-based algorithm that estimates} $LiNGLaM$, the first stage is to identify individual causal clusters and the second stage is to learn the causal order of the latent variable sets of clusters\cite{xie2020generalized}.  The only difference is that our algorithm has an additional step in the stage 2 to identify cycles between blocks of latent variables. Throughout this chapter we have been using `block' and `set' interchangeably. To emphasize the special meaning of \textit{the group of latent variable shared by meansured variables in the same cluster}, for the rest of this section we will use `block' for this particular meaning.  The correctness of stage 2 is proved in \textbf{Theorem} \ref{thm:GINcollider2}.\\ \textcolor{white}{abc}
The stage 1 of the algorithm takes the dataset and run the step 1 of the original \textbf{GIN condition-based algorithm for estimating $LiNGLaM$} to identify causal clusters\cite{xie2020generalized}.  Here we introduce stage 2, where the algorithm learns the causal order among clusters and recursively find cyclic structures between latent blocks.

\begin{algorithm}
    \caption{CyclicLatentBlockStructureDiscovery}\label{alg:fullalgolatentcycbtblcks}
    \KwIn{Data from a set of measured variables $X_\mathcal{G}$}
\KwOut{Partial causal order  $\mathcal{K}$}
$\mathcal{L}\gets IdentifyingCausalClusters(X_\mathcal{G})$;\\ 
$\mathcal{K}\gets LearningCausalOrderForCyclicLatentBlocks(\mathcal{L})$;

\end{algorithm}
\subsubsection{Causal Order Learning Part 1: Finding Root Nodes}

\textit{Root nodes} of a set of latent variables are the node that are causally earlier than any of the other nodes in the set.  Given a set of identified causal clusters and the corresponding latent blocks, the original GIN condition-based algorithm uses \textbf{Proposition} \ref{prop:GIN4} to recursively find  the (block of) root nodes among the latent blocks from the causal clusters, one block each time.  If there are cycles between the latent blocks, then the original algorithm is not able to identify a single block of root nodes.  In this case the original algorithm will connect the latent blocks with undirected edges.\\ \textcolor{white}{abc}  
In order to provide more information about the structure between the latent blocks, we let the algorithm, instead of just looking for one latent block as root nodes, look for multiple latent blocks as root nodes.  If such multi-block root nodes are found, it is possible that there are cycles between the blocks and we call the $findCyclesBetweenBlocks$ to dissect the multi-block root notes.  $findCyclesBetweenBlocks$ returns two outputs, the second being a collection of sets of latent blocks, each set indicating the existence of cyclic structure between the latent blocks in the set.  The causal order will union with \textit{the sets in the collection}, marking latent blocks within the \textit{same set} as having \textit{block-cycles} with each other.   The complete algorithm for finding root nodes is algorithm \ref{alg:Lblcoksrootfind}.

\begin{algorithm}
\caption{LearningCausalOrderBetweenLatents}\label{alg:Lblcoksrootfind}
\KwIn{A set of clusters $\{\mathcal{S}_1...\mathcal{S}_k\}$ }
\KwOut{Causal Order $\mathcal{K}$}
$\mathcal{L}\gets $the set of latent variables for each cluster;\\ 
{$\#$\textit{Initializing $\mathbf{T}$ as the collection of clusters that are causally earlier than $\mathcal{L}$}}\\ 
{$\#$\textit{Initializing $\mathcal{K}$ as the causal order among latents}}\\ 
{$\#$\textit{Initializing $\mathcal{C}$ as the collection of cyclic blocks of latents}}\\ 
 $\mathbf{T}, \mathcal{K},\mathcal{C}\gets\emptyset, \emptyset,\emptyset$;
 
 \While{$\mathcal{L}\neq\emptyset$}
 { $k\gets 0$
 
 \While{root node not found and $k<|\mathcal{L}|$}{
$k\gets k+1$
 
Find the root node $L(\mathcal{S}_{r_1})...L(\mathcal{S}_{r_k})$ according to \textbf{Proposition } \ref{prop:GIN4} with $k$
 
}

\If{$k>1$}{
$\mathcal{C}\gets \mathcal{C}\cup findCyclesBetweenBlocks(\{\mathcal{S}_r)...\mathcal{S}_{r+k}\},\mathbf{T})[2]$
}

$\mathcal{L}\gets\mathcal{L}\setminus \cup^k_{i=1}L(\mathcal{S}_{r_i})$

$\mathbf{T}\gets\mathbf{T}\cup\{\mathcal{S}_{r_1}...\mathcal{S}_{r_k}\}$

Include $L(\mathcal{S}_{r_1})...L(\mathcal{S}_{r_k})$ in $\mathcal{K}$ as the same causal order

$k\gets 0$

 }

 \For{$\{\mathbf{S}_i..\mathbf{S}_j\}\in\mathcal{C}$}{
Mark $\{\mathbf{S}_i..\mathbf{S}_j\}$ as $block-cycles$ in $\mathcal{K}$
 
 }

\textbf{Return } $\mathcal{K}$ 
\end{algorithm}

\subsubsection{Causal Order Learning Part 2: Find Cycles between Latent Blocks}
As shown in algorithm \ref{alg:blockcycdetect}, $FindCyclesBetweenBlocks$ takes a set $\mathbf{C}$ of latent blocks, representing the latents that were just identified as root nodes all together, and another set $\mathbf{B}$ of latent blocks that are causally earlier than the first set, \textit{i.e.,} the latents that were identified as root nodes earlier than latent blocks in the first set.  It will then bisect $\mathbf{C}$ check the cycles between the two groups of latent blocks based on \textbf{Theorem} \ref{thm:GINcollider2}.  If the two groups of latent blocks are found to only be connected through their common causes in $\mathbf{B}$,  then each of the group will be treated as a new set $\mathbf{C}$ and be passed into $FindCyclesBetweenBlocks$ to be further dissected until either cycles between latent blocks are identified or the group is a singleton of latent block and cannot be dissected anymore.\\ \textcolor{white}{abc}  
Notice that \textbf{Theorem} \ref{thm:GINcollider2} only proves the existence of $block-cycles$, so we still do not know any particular directed edge between any latent variables.  For two groups of latent blocks that are connected with $block-cycles$, it is possible that the $block-cycle$ only exists between some blocks in each group.  Therefore, a pruning procedure is necessary. After learning all the groups latent blocks that are connected with $block-cycles$ and all the groups that are not, the algorithm will prune the groups of latent blocks that are connected with $block-cycles$ by removing from them the latent blocks between which the $block-cycles$ do not exist.  \\ \textcolor{white}{abc}
\begin{algorithm}
\caption{FindCyclesBetweenBlocks}\label{alg:blockcycdetect}
\KwIn{Two sets of clusters $\mathbf{C}=\{\mathcal{C}_1...\mathcal{C}_k\}$ and $\mathbf{B}=\{\mathcal{B}_1...\mathcal{B}_h\}$ s.t. clusters in $\mathbf{B}$ are all causally earlier than all clusters in $\mathbf{C}$}
\KwOut{two sets of pairs of sets of clusters}
{$\#$\textit{$\mathcal{L}_1=\{(\mathcal{C}_a, \mathcal{C}_b)|L(\mathcal{C}_a) \text{ and } L(\mathcal{C}_b) \text{ are connected only through their common causes} \}$}}\\ 
{$\#$\textit{$\mathcal{L}_2=\{(\mathcal{C}_a, \mathcal{C}_b)|\text{there are cycles between the blocks }L(\mathcal{C}_a) \text{ and } L(\mathcal{C}_b) \}$}}\\ 
 $\mathcal{L}_1, \mathcal{L}_2\gets\emptyset, \emptyset$
 
  {$\#$\textit{Identifying Pairwise Cycles}}\\ 
 \For{$\mathbf{C_1, C_2}\subsetneq\mathbf{C}$ s.t. $\mathbf{C_1\cup C_2=C}$ and $\mathbf{C_1\cap C_2=\emptyset}$}{
\For{ $(a,b)\in\{(1,2),(2,1)\}$}{\If{$rank(\Sigma_\mathbf{X^\mathbf{C_a}, X^\mathbf{B}})<min(\mathbf{|L_{C}|, |L_{B}|})$}{\eIf{$E_{X^{\mathbf{C_a}}\parallel X^{B}}\indep X^{\mathbf{C_b} }$}{$\mathcal{L}_1\gets\mathcal{L}_1\cup\{\mathbf{\{C_a, C_b\}\}}$

$\mathcal{L}_1, \mathcal{L}_2 \gets \mathcal{L}_1, \mathcal{L}_2 \cup findCyclesBetweenBlocks(\mathbf{C_a, B})$

$\mathcal{L}_1, \mathcal{L}_2 \gets \mathcal{L}_1, \mathcal{L}_2 \cup findCyclesBetweenBlocks(\mathbf{C_b, B})$
}{$\mathcal{L}_2\gets\mathcal{L}_2\cup\{\mathbf{\{C_a, C_b\}\}}$}}
 
 }
 }
{$\#$\textit{Pruning Cycles}}\\ \textcolor{white}{abc}
\For{$\mathbf{\{C^2_a, C^2_b\}\in\mathcal{L}_2}$}
{\For{$\mathbf{\{C^1_a, C^1_b\}\in\mathcal{L}_1}$}
{$\mathbf{C^2_a\gets C^2_a\setminus C^1_a}, \mathbf{C^2_b\gets C^2_b\setminus C^1_b}$

$\mathbf{C^2_a\gets C^2_a\setminus C^1_b}, \mathbf{C^2_b\gets C^2_b\setminus C^1_a}$
}
}

\textbf{Return } $\mathcal{L}_1, \mathcal{L}_2$ 
\end{algorithm}

\textbf{\textit{Example.}}  Consider fig \ref{fig:cycblockexp} of which the latents are connected as shown in fig \ref{fig:cyclatentblock}.\\$FindCyclesBetweenBlocks$ will be passed the input $\mathbf{C}=\{\{L_3, L_4\}, \{L_5, L_6\}, \{L_7, L_8\}, \{L_9, L_{10}\}\}$ and $\mathbf{B} = \{\{L_1\}\}$.  The main steps of the algorithm work as follows:
\begin{enumerate}
    \item the algorithm will identify there are no $block-cycles$ between $\mathbf{C_a} = \{\{L_3, L_4\}, \{L_5, L_6\},\}, \mathbf{C_b}=\{\{L_7, L_8\}, \{L_9, L_{10}\}\}$; it will remember it in $\mathcal{L}_1$ and call $FindCyclesBetweenBlocks(\mathbf{C_a, B})$ and $FindCyclesBetweenBlocks(\mathbf{C_b, B})$. 
    \item the algorithm will identify that there are $block-cycles$ between $\mathbf{C_a} = \{\{L_3, L_4\}, \{L_7, L_8\} \}, \mathbf{C_b}=\{\{L_5, L_6\},\{L_9, L_{10}\}\}$; it will remember it in $\mathcal{L}_2$. 
    \item For $\mathbf{C_a} = \{\{L_3, L_4\}, \{L_5, L_6\}\}$, $FindCyclesBetweenBlocks(\mathbf{C_a, B})$ will find that there are $block-cycles$ between $\{\{L_3, L_4\} \}$ and $\{\{L_5, L_6\}\}$; it will remember it in $\mathcal{L}_2$. 
    \item Similarly, for $\mathbf{C_b} = \{\{L_7, L_8\}, \{L_9, L_{10}\}\}$, $FindCyclesBetweenBlocks(\mathbf{C_b, B})$ will find that there are $block-cycles$ between $\{\{L_7, L_8\} \}$ and $\{\{L_9, L_{10}\}\}$; it will remember it in $\mathcal{L}_2$. 
    \item Since $\mathcal{L}_1$ for  $FindCyclesBetweenBlocks(\mathbf{C_a, B})$ and $FindCyclesBetweenBlocks(\mathbf{C_b, B})$ are both empty, this two calls do not go to the pruning stage and will return $(\emptyset, \{\{\{L_3, L_4\}, \{L_5, L_6\}\}\})$ and $(\emptyset, \{\{\{L_7, L_8\}, \{L_9, L_{10}\}\}\})$ to the last function call, $FindCyclesBetweenBlocks(\mathbf{C, B})$.  The $\mathcal{L}_1$ and $\mathcal{L}_2$ unions with the corresponding returned value.
    \item For $FindCyclesBetweenBlocks(\mathbf{C, B})$, $\mathcal{L}_1 = \{ \{\{\{L_3, L_4\}, \{L_5, L_6\}\}, \{\{L_7, L_8\}, \{L_9, L_{10}\}\}\}\}$ and $\mathcal{L}_2 = \{ \{ \{\{L_3, L_4\}, \{L_7, L_8\} \}, \{\{L_5, L_6\},\{L_9, L_{10}\}\}\}, \{\{L_3, L_4\}, \{L_5, L_6\}\}, \{\{L_7, L_8\}, \{L_9, L_{10}\}\} \}$.  The first element of $\mathcal{L}_2$ is $\mathcal{L}_2[1] = \{ \{\{L_3, L_4\}, \{L_7, L_8\} \}, \{\{L_5, L_6\},\{L_9, L_{10}\}\}\}$. It can be pruned by $\mathcal{L}_1$ as:
    \begin{enumerate}
        \item $\mathcal{L}_2[1][1]\gets \mathcal{L}_2[1][1]\setminus\mathcal{L}_1[1][1] = \{\{L_3, L_4\}, \{L_7, L_8\} \}\setminus \{\{L_3, L_4\}, \{L_5, L_6\}\}= \{L_7, L_8\}$\\ \textcolor{white}{abc}
         $\mathcal{L}_2[1][2]\gets \mathcal{L}_2[1][2]\setminus\mathcal{L}_1[1][2] = \{\{L_5, L_6\}, \{L_9, L_{10}\} \}\setminus \{\{L_7, L_8\}, \{L_9, L_{10}\}\}= \{L_5, L_6\}$
        \item $\mathcal{L}_2[1][1]\gets \mathcal{L}_2[1][1]\setminus\mathcal{L}_1[1][2] = \{\{L_7, L_8\} \}\setminus \{\{L_7, L_8\}, \{L_9, L_{10}\}\}= \emptyset$\\ \textcolor{white}{abc}
     $\mathcal{L}_2[1][2]\gets \mathcal{L}_2[1][2]\setminus\mathcal{L}_1[1][1] = \{\{L_5, L_6\} \}\setminus \{\{L_3, L_4\}, \{L_5, L_6\}\}= \emptyset$
    \end{enumerate}
    The other elements in $\mathcal{L}_2$ remains the same after the pruning stage.
    \item the algorithm finally returns $(\mathcal{L}_1\text{(will not be used for learning causal order)}, \mathcal{L}_2)$ where $\mathcal{L}_2 = \{ \{\{L_3, L_4\}, \{L_5, L_6\}\}, \{\{L_7, L_8\}, \{L_9, L_{10}\}\} \}$ contains two groups of latent blocks each of which has $block-cycles$ in between.
    
\end{enumerate}

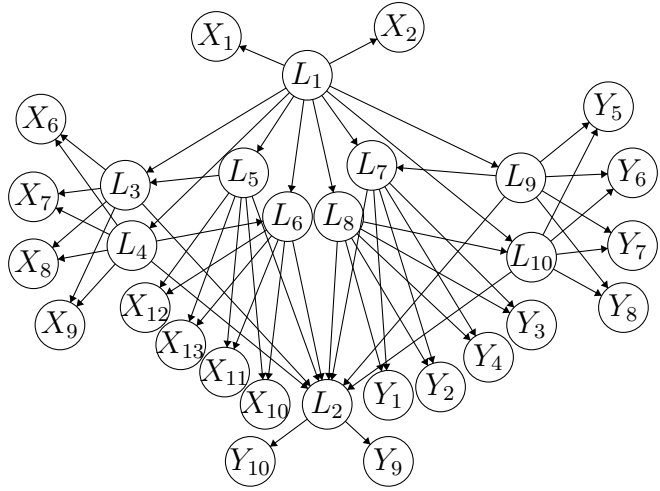
\begin{figure}
\begin{center}
\begin{tikzpicture}[scale=0.11]
\tikzstyle{every node}+=[inner sep=0pt]
\draw [black] (26.2,-4.5) circle (3);
\draw (26.2,-4.5) node {$X_1$};
\draw [black] (37.2,-9.2) circle (3);
\draw (37.2,-9.2) node {$L_1$};
\draw [black] (39.6,-49) circle (3);
\draw (39.6,-49) node {$L_2$};
\draw [black] (15.2,-22.5) circle (3);
\draw (15.2,-22.5) node {$L_3$};
\draw [black] (16,-29.8) circle (3);
\draw (16,-29.8) node {$L_4$};
\draw [black] (29.5,-21) circle (3);
\draw (29.5,-21) node {$L_5$};
\draw [black] (34.8,-26.3) circle (3);
\draw (34.8,-26.3) node {$L_6$};
\draw [black] (45,-20.1) circle (3);
\draw (45,-20.1) node {$L_7$};
\draw [black] (41,-26.3) circle (3);
\draw (41,-26.3) node {$L_8$};
\draw [black] (63,-21.6) circle (3);
\draw (63,-21.6) node {$L_9$};
\draw [black] (64.3,-31.2) circle (3);
\draw (64.3,-31.2) node {$L_{10}$};
\draw [black] (5,-14.5) circle (3);
\draw (5,-14.5) node {$X_6$};
\draw [black] (48.3,-3.4) circle (3);
\draw (48.3,-3.4) node {$X_2$};
\draw [black] (4.1,-23.9) circle (3);
\draw (4.1,-23.9) node {$X_7$};
\draw [black] (4.1,-32) circle (3);
\draw (4.1,-32) node {$X_8$};
\draw [black] (7.3,-39.4) circle (3);
\draw (7.3,-39.4) node {$X_9$};
\draw [black] (32.1,-49) circle (3);
\draw (32.1,-49) node {$X_{10}$};
\draw [black] (27.2,-45.1) circle (3);
\draw (27.2,-45.1) node {$X_{11}$};
\draw [black] (17.6,-37.3) circle (3);
\draw (17.6,-37.3) node {$X_{12}$};
\draw [black] (21.9,-41.8) circle (3);
\draw (21.9,-41.8) node {$X_{13}$};
\draw [black] (30.3,-55.9) circle (3);
\draw (30.3,-55.9) node {$Y_{10}$};
\draw [black] (47,-55.9) circle (3);
\draw (47,-55.9) node {$Y_9$};
\draw [black] (47,-47.9) circle (3);
\draw (47,-47.9) node {$Y_1$};
\draw [black] (53.3,-46.9) circle (3);
\draw (53.3,-46.9) node {$Y_2$};
\draw [black] (64.3,-39.4) circle (3);
\draw (64.3,-39.4) node {$Y_3$};
\draw [black] (59.1,-43.3) circle (3);
\draw (59.1,-43.3) node {$Y_4$};
\draw [black] (73.6,-13) circle (3);
\draw (73.6,-13) node {$Y_5$};
\draw [black] (76.5,-21) circle (3);
\draw (76.5,-21) node {$Y_6$};
\draw [black] (76.5,-29.8) circle (3);
\draw (76.5,-29.8) node {$Y_7$};
\draw [black] (75.4,-37.3) circle (3);
\draw (75.4,-37.3) node {$Y_8$};
\draw [black] (26.52,-21.31) -- (18.18,-22.19);
\fill [black] (18.18,-22.19) -- (19.03,-22.6) -- (18.93,-21.61);
\draw [black] (18.95,-29.25) -- (31.85,-26.85);
\fill [black] (31.85,-26.85) -- (30.97,-26.5) -- (31.16,-27.49);
\draw [black] (34.63,-10.75) -- (17.77,-20.95);
\fill [black] (17.77,-20.95) -- (18.71,-20.96) -- (18.19,-20.11);
\draw [black] (36.78,-12.17) -- (35.22,-23.33);
\fill [black] (35.22,-23.33) -- (35.82,-22.61) -- (34.83,-22.47);
\draw [black] (38.95,-11.64) -- (43.25,-17.66);
\fill [black] (43.25,-17.66) -- (43.2,-16.72) -- (42.38,-17.3);
\draw [black] (39.53,-11.09) -- (61.97,-29.31);
\fill [black] (61.97,-29.31) -- (61.66,-28.42) -- (61.03,-29.19);
\draw [black] (60.01,-21.35) -- (47.99,-20.35);
\fill [black] (47.99,-20.35) -- (48.75,-20.91) -- (48.83,-19.92);
\draw [black] (43.94,-26.92) -- (61.36,-30.58);
\fill [black] (61.36,-30.58) -- (60.68,-29.93) -- (60.48,-30.91);
\draw [black] (34.44,-8.02) -- (28.96,-5.68);
\fill [black] (28.96,-5.68) -- (29.5,-6.45) -- (29.89,-5.53);
\draw [black] (35.42,-29.24) -- (38.98,-46.06);
\fill [black] (38.98,-46.06) -- (39.3,-45.18) -- (38.32,-45.39);
\draw [black] (40.82,-29.29) -- (39.78,-46.01);
\fill [black] (39.78,-46.01) -- (40.33,-45.24) -- (39.33,-45.18);
\draw [black] (18.33,-31.69) -- (37.27,-47.11);
\fill [black] (37.27,-47.11) -- (36.97,-46.21) -- (36.34,-46.99);
\draw [black] (30.52,-23.82) -- (38.58,-46.18);
\fill [black] (38.58,-46.18) -- (38.78,-45.26) -- (37.84,-45.6);
\draw [black] (17.23,-24.71) -- (37.57,-46.79);
\fill [black] (37.57,-46.79) -- (37.39,-45.87) -- (36.66,-46.54);
\draw [black] (44.45,-23.05) -- (40.15,-46.05);
\fill [black] (40.15,-46.05) -- (40.79,-45.36) -- (39.81,-45.17);
\draw [black] (61.87,-32.95) -- (42.03,-47.25);
\fill [black] (42.03,-47.25) -- (42.98,-47.18) -- (42.39,-46.37);
\draw [black] (61.05,-23.88) -- (41.55,-46.72);
\fill [black] (41.55,-46.72) -- (42.45,-46.44) -- (41.69,-45.79);
\draw [black] (39.9,-10.5) -- (60.3,-20.3);
\fill [black] (60.3,-20.3) -- (59.79,-19.5) -- (59.36,-20.4);
\draw [black] (37.85,-12.13) -- (40.35,-23.37);
\fill [black] (40.35,-23.37) -- (40.66,-22.48) -- (39.69,-22.7);
\draw [black] (35.56,-11.71) -- (31.14,-18.49);
\fill [black] (31.14,-18.49) -- (32,-18.09) -- (31.16,-17.54);
\draw [black] (35.05,-11.29) -- (18.15,-27.71);
\fill [black] (18.15,-27.71) -- (19.07,-27.51) -- (18.38,-26.79);
\draw [black] (12.84,-20.65) -- (7.36,-16.35);
\fill [black] (7.36,-16.35) -- (7.68,-17.24) -- (8.3,-16.45);
\draw [black] (12.22,-22.88) -- (7.08,-23.52);
\fill [black] (7.08,-23.52) -- (7.93,-23.92) -- (7.81,-22.93);
\draw [black] (12.92,-24.45) -- (6.38,-30.05);
\fill [black] (6.38,-30.05) -- (7.31,-29.91) -- (6.66,-29.15);
\draw [black] (13.99,-32.02) -- (9.31,-37.18);
\fill [black] (9.31,-37.18) -- (10.22,-36.92) -- (9.48,-36.25);
\draw [black] (13.05,-30.35) -- (7.05,-31.45);
\fill [black] (7.05,-31.45) -- (7.93,-31.8) -- (7.75,-30.82);
\draw [black] (13.93,-25.22) -- (8.57,-36.68);
\fill [black] (8.57,-36.68) -- (9.36,-36.17) -- (8.46,-35.75);
\draw [black] (13.31,-28.47) -- (6.79,-25.23);
\fill [black] (6.79,-25.23) -- (7.28,-26.04) -- (7.73,-25.14);
\draw [black] (14.25,-27.36) -- (6.75,-16.94);
\fill [black] (6.75,-16.94) -- (6.81,-17.88) -- (7.62,-17.29);
\draw [black] (39.86,-7.81) -- (45.64,-4.79);
\fill [black] (45.64,-4.79) -- (44.7,-4.72) -- (45.16,-5.6);
\draw [black] (29.78,-23.99) -- (31.82,-46.01);
\fill [black] (31.82,-46.01) -- (32.25,-45.17) -- (31.25,-45.26);
\draw [black] (29.21,-23.99) -- (27.49,-42.11);
\fill [black] (27.49,-42.11) -- (28.06,-41.36) -- (27.06,-41.27);
\draw [black] (27.73,-23.42) -- (19.37,-34.88);
\fill [black] (19.37,-34.88) -- (20.24,-34.53) -- (19.44,-33.94);
\draw [black] (28.47,-23.82) -- (22.93,-38.98);
\fill [black] (22.93,-38.98) -- (23.67,-38.4) -- (22.73,-38.06);
\draw [black] (34.45,-29.28) -- (32.45,-46.02);
\fill [black] (32.45,-46.02) -- (33.05,-45.29) -- (32.05,-45.17);
\draw [black] (33.68,-29.08) -- (28.32,-42.32);
\fill [black] (28.32,-42.32) -- (29.09,-41.76) -- (28.16,-41.39);
\draw [black] (32.27,-27.92) -- (20.13,-35.68);
\fill [black] (20.13,-35.68) -- (21.07,-35.67) -- (20.53,-34.83);
\draw [black] (32.88,-28.61) -- (23.82,-39.49);
\fill [black] (23.82,-39.49) -- (24.72,-39.2) -- (23.95,-38.56);
\draw [black] (37.19,-50.79) -- (32.71,-54.11);
\fill [black] (32.71,-54.11) -- (33.65,-54.04) -- (33.05,-53.23);
\draw [black] (41.79,-51.05) -- (44.81,-53.85);
\fill [black] (44.81,-53.85) -- (44.56,-52.94) -- (43.88,-53.67);
\draw [black] (45.22,-23.09) -- (46.78,-44.91);
\fill [black] (46.78,-44.91) -- (47.23,-44.07) -- (46.23,-44.15);
\draw [black] (41.8,-29.19) -- (46.2,-45.01);
\fill [black] (46.2,-45.01) -- (46.46,-44.1) -- (45.5,-44.37);
\draw [black] (42.54,-28.88) -- (51.76,-44.32);
\fill [black] (51.76,-44.32) -- (51.78,-43.38) -- (50.92,-43.89);
\draw [black] (46.56,-22.66) -- (57.54,-40.74);
\fill [black] (57.54,-40.74) -- (57.55,-39.79) -- (56.7,-40.31);
\draw [black] (43.62,-27.77) -- (61.68,-37.93);
\fill [black] (61.68,-37.93) -- (61.23,-37.1) -- (60.74,-37.97);
\draw [black] (45.89,-22.97) -- (52.41,-44.03);
\fill [black] (52.41,-44.03) -- (52.65,-43.12) -- (51.7,-43.42);
\draw [black] (47.12,-22.22) -- (62.18,-37.28);
\fill [black] (62.18,-37.28) -- (61.97,-36.36) -- (61.26,-37.07);
\draw [black] (43.19,-28.35) -- (56.91,-41.25);
\fill [black] (56.91,-41.25) -- (56.67,-40.33) -- (55.99,-41.06);
\draw [black] (65.33,-19.71) -- (71.27,-14.89);
\fill [black] (71.27,-14.89) -- (70.33,-15.01) -- (70.96,-15.78);
\draw [black] (66,-21.47) -- (73.5,-21.13);
\fill [black] (73.5,-21.13) -- (72.68,-20.67) -- (72.73,-21.67);
\draw [black] (65.56,-23.16) -- (73.94,-28.24);
\fill [black] (73.94,-28.24) -- (73.51,-27.4) -- (72.99,-28.25);
\draw [black] (64.86,-23.95) -- (73.54,-34.95);
\fill [black] (73.54,-34.95) -- (73.44,-34.01) -- (72.65,-34.63);
\draw [black] (65.67,-28.53) -- (72.23,-15.67);
\fill [black] (72.23,-15.67) -- (71.43,-16.16) -- (72.32,-16.61);
\draw [black] (66.6,-29.28) -- (74.2,-22.92);
\fill [black] (74.2,-22.92) -- (73.26,-23.05) -- (73.91,-23.82);
\draw [black] (67.28,-30.86) -- (73.52,-30.14);
\fill [black] (73.52,-30.14) -- (72.67,-29.74) -- (72.78,-30.73);
\draw [black] (66.93,-32.64) -- (72.77,-35.86);
\fill [black] (72.77,-35.86) -- (72.31,-35.03) -- (71.83,-35.91);
\end{tikzpicture}
\end{center}
    \caption{$LiNGLaM$ with two cyclic latent blocks}
    \label{fig:cycblockexp}
\end{figure}

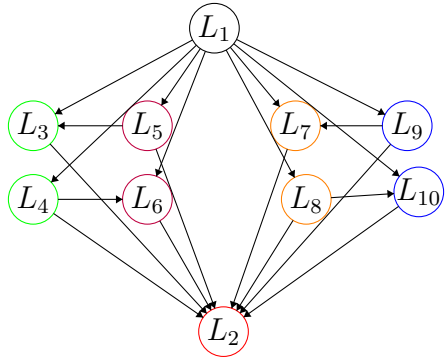
\begin{figure}
\begin{center}
\begin{tikzpicture}[scale=0.11]
\tikzstyle{every node}+=[inner sep=0pt]
\draw [black] (37.7,-10.7) circle (3);
\draw (37.7,-10.7) node {$L_1$};
\draw [red] (38.8,-47.4) circle (3);
\draw (38.8,-47.4) node {$L_2$};
\draw [green] (15.8,-22.5) circle (3);
\draw (15.8,-22.5) node {$L_3$};
\draw [green] (15.8,-31.4) circle (3);
\draw (15.8,-31.4) node {$L_4$};
\draw [purple] (29.6,-22.5) circle (3);
\draw (29.6,-22.5) node {$L_5$};
\draw [purple] (29.6,-31.4) circle (3);
\draw (29.6,-31.4) node {$L_6$};
\draw [orange] (47.5,-22.5) circle (3);
\draw (47.5,-22.5) node {$L_7$};
\draw [orange] (48.8,-31.4) circle (3);
\draw (48.8,-31.4) node {$L_8$};
\draw [blue] (61,-22.5) circle (3);
\draw (61,-22.5) node {$L_9$};
\draw [blue] (62.4,-30.5) circle (3);
\draw (62.4,-30.5) node {$L_{10}$};
\draw [black] (26.6,-22.5) -- (18.8,-22.5);
\fill [black] (18.8,-22.5) -- (19.6,-23) -- (19.6,-22);
\draw [black] (18.8,-31.4) -- (26.6,-31.4);
\fill [black] (26.6,-31.4) -- (25.8,-30.9) -- (25.8,-31.9);
\draw [black] (35.06,-12.12) -- (18.44,-21.08);
\fill [black] (18.44,-21.08) -- (19.38,-21.14) -- (18.91,-20.26);
\draw [black] (36.61,-13.49) -- (30.69,-28.61);
\fill [black] (30.69,-28.61) -- (31.45,-28.04) -- (30.52,-27.68);
\draw [black] (39.62,-13.01) -- (45.58,-20.19);
\fill [black] (45.58,-20.19) -- (45.46,-19.26) -- (44.69,-19.9);
\draw [black] (40.04,-12.58) -- (60.06,-28.62);
\fill [black] (60.06,-28.62) -- (59.75,-27.73) -- (59.12,-28.51);
\draw [black] (58,-22.5) -- (50.5,-22.5);
\fill [black] (50.5,-22.5) -- (51.3,-23) -- (51.3,-22);
\draw [black] (51.79,-31.2) -- (59.41,-30.7);
\fill [black] (59.41,-30.7) -- (58.58,-30.25) -- (58.64,-31.25);
\draw [black] (31.1,-34) -- (37.3,-44.8);
\fill [black] (37.3,-44.8) -- (37.34,-43.86) -- (36.47,-44.35);
\draw [black] (47.21,-33.94) -- (40.39,-44.86);
\fill [black] (40.39,-44.86) -- (41.24,-44.44) -- (40.39,-43.91);
\draw [black] (18.26,-33.11) -- (36.34,-45.69);
\fill [black] (36.34,-45.69) -- (35.97,-44.82) -- (35.4,-45.64);
\draw [black] (30.64,-25.31) -- (37.76,-44.59);
\fill [black] (37.76,-44.59) -- (37.95,-43.66) -- (37.01,-44.01);
\draw [black] (17.84,-24.7) -- (36.76,-45.2);
\fill [black] (36.76,-45.2) -- (36.59,-44.27) -- (35.85,-44.95);
\draw [black] (46.51,-25.33) -- (39.79,-44.57);
\fill [black] (39.79,-44.57) -- (40.53,-43.98) -- (39.58,-43.65);
\draw [black] (59.96,-32.25) -- (41.24,-45.65);
\fill [black] (41.24,-45.65) -- (42.18,-45.59) -- (41.6,-44.78);
\draw [black] (59,-24.74) -- (40.8,-45.16);
\fill [black] (40.8,-45.16) -- (41.7,-44.9) -- (40.96,-44.23);
\draw [black] (40.38,-12.06) -- (58.32,-21.14);
\fill [black] (58.32,-21.14) -- (57.84,-20.34) -- (57.38,-21.23);
\draw [black] (39.12,-13.34) -- (47.38,-28.76);
\fill [black] (47.38,-28.76) -- (47.44,-27.81) -- (46.56,-28.29);
\draw [black] (36,-13.17) -- (31.3,-20.03);
\fill [black] (31.3,-20.03) -- (32.16,-19.65) -- (31.34,-19.08);
\draw [black] (35.52,-12.76) -- (17.98,-29.34);
\fill [black] (17.98,-29.34) -- (18.91,-29.15) -- (18.22,-28.43);
\end{tikzpicture}
\end{center}
    \caption{the structure of the latents of fig \ref{fig:cycblockexp}}
    \label{fig:cyclatentblock}
\end{figure}

\chapter{From Rank Constraint to Tensor Constraint}

This chapter focuses on the tensor constraint, which is introduced in \ref{sec:tensorconstraint}.  We will first prove that the tensor constraint can hold with conditions weaker than linear model with acyclicity, then show the limitation of the tensor constraint.  Just as the tensor constraint being a generalization of the rank constraint, the first part of this chapter (including the first two sections) is to generalize some theoretical result of the rank constraint to the tensor constraint.\footnote{To be completely parallel, we should call `determinant constraint' vs `hyperdeterminant constraint', or `covariance constraint' vs `higher-order cumulant constraint'; here we just follow their names from the original papers from which the methods are discovered.} The main result presented in the first part is that the tensor constraint is applicable with \textbf{only linearity under the \textit{k-choke set}}, which is a generalization of the \textbf{Theorem} \ref{thm:ranklinear}. This result is also sufficient for showing that the tensor constraint is applicable with \textbf{linearity and acyclicity under the \textit{k-choke set}}, which is a generalization of the \textbf{Theorem} \ref{thm:LArank}.  The last section of this chapter will present the limitation of the tensor constraint due to certain mathematical properties of the hyperdeterminant.  Specifically, we will provide an example to show that when the cumulant tensor is odd-dimensional, the hyperdeterminant is not zero while the condition of the tensor constraint is met.

\section{Recursively Additive Model}
Having defined \textbf{\textit{k-choke set} } in \ref{definition:ktreksep}, we now define \textbf{linearity under the \textit{k-choke set}}.\\ 
\textbf{Linear under \textit{k-choke set}:} Given a directed graph $G = \langle \mathbf{V, E}\rangle$, and $\mathbf{|S_1|=...=|S_k|}=n$ and $\mathbf{A_1,...,A_k}$ are \textit{k-choke sets} for $ \mathbf{S_1,...S_k}$.   We say the graph $G$ is \textbf{linear under} $\mathbf{A_1,...,A_k}$ for $\mathbf{S_1,...S_k}$, if for every directed path in $G$ (not necessarily self-avoiding) from some nodes in $\mathbf{A_j}$ to some nodes in $\mathbf{S_j}$, every edge $X_i\rightarrow X_j$ on this path is linear, aka, $X_j = a_{ij}X_i+\epsilon'_{j_{\mathbf{A_j\rightarrow S_j}}}$ with noise $\epsilon'_{j_{\mathbf{C_A\rightarrow A}}}$ independent from $X_i$.  

Recall that \textbf{Theorem} \ref{thm:LArank} and \ref{thm:ranklinear} both assume that the graph $G = \langle \mathbf{V, E}\rangle$ follows a generalized additive SEM then state a local condition about the\textbf{\textit{ choke set}} under which the rank constraint can be used.  Unfortunately, the generalized additive model is not enough to prove these results for the tensor constraint.

\textbf{\textit{Example.}}  Consider fig \ref{fig:counterexampletc} with a generalized additive model with fixed parameterization as below:
\begin{itemize}
    \item $X_1 = a_1L+\epsilon_{x_1}$
    \item $X_2 = a_2L+f_2(L_1)+\epsilon_{x_2}$
    \item $X_3 = a_3L+\epsilon_{x_3}$
    \item $L_2 = f_4(L_1)+ g_4(L_3)+\epsilon_{L_2}$
    \item $X_4 = a_4L+h_4(L_2) +\epsilon_{x_4}$
    \item $X_5 = a_5L+\epsilon_{x_5}$
    \item $X_6 = a_6L+f_6(L_3)+\epsilon_{x_6}$
\end{itemize}
\begin{figure}
    \begin{center}
\begin{tikzpicture}[scale=0.12]
\tikzstyle{every node}+=[inner sep=0pt]
\draw [black] (60,-22.3) circle (3);
\draw (60,-22.3) node {$X_5$};
\draw [black] (39.4,-26.9) circle (3);
\draw (39.4,-26.9) node {$X_3$};
\draw [black] (37,-4.6) circle (3);
\draw (37,-4.6) node {$L$};
\draw [black] (16.6,-22.3) circle (3);
\draw (16.6,-22.3) node {$X_1$};
\draw [black] (23,-27.5) circle (3);
\draw (23,-27.5) node {$X_2$};
\draw [black] (47.3,-25.8) circle (3);
\draw (47.3,-25.8) node {$X_4$};
\draw [black] (66,-16.8) circle (3);
\draw (66,-16.8) node {$X_6$};
\draw [black] (39.4,-46.6) circle (3);
\draw (39.4,-46.6) node {$L_1$};
\draw [black] (50.3,-37.4) circle (3);
\draw (50.3,-37.4) node {$L_2$};
\draw [black] (66,-40.3) circle (3);
\draw (66,-40.3) node {$L_3$};
\draw [black] (34.73,-6.57) -- (18.87,-20.33);
\fill [black] (18.87,-20.33) -- (19.8,-20.19) -- (19.14,-19.43);
\draw (25.47,-12.96) node [above] {$a_1$};
\draw [black] (35.44,-7.16) -- (24.56,-24.94);
\fill [black] (24.56,-24.94) -- (25.41,-24.52) -- (24.56,-24);
\draw (29.36,-14.77) node [left] {$a_2$};
\draw [black] (37.32,-7.58) -- (39.08,-23.92);
\fill [black] (39.08,-23.92) -- (39.49,-23.07) -- (38.5,-23.18);
\draw (37.55,-15.85) node [left] {$a_3$};
\draw [black] (38.31,-7.3) -- (45.99,-23.1);
\fill [black] (45.99,-23.1) -- (46.09,-22.16) -- (45.19,-22.6);
\draw (41.45,-16.28) node [left] {$a_4$};
\draw [black] (39.38,-6.43) -- (57.62,-20.47);
\fill [black] (57.62,-20.47) -- (57.29,-19.59) -- (56.68,-20.38);
\draw (47.17,-13.95) node [below] {$a_5$};
\draw [black] (39.77,-5.76) -- (63.23,-15.64);
\fill [black] (63.23,-15.64) -- (62.69,-14.87) -- (62.3,-15.79);
\draw (50.22,-11.22) node [below] {$a_6$};
\draw [black] (37.45,-44.32) -- (24.95,-29.78);
\fill [black] (24.95,-29.78) -- (25.1,-30.71) -- (25.85,-30.06);
\draw (31.75,-35.6) node [right] {$f_2$};
\draw [black] (41.69,-44.67) -- (48.01,-39.33);
\fill [black] (48.01,-39.33) -- (47.07,-39.47) -- (47.72,-40.23);
\draw (46.15,-42.49) node [below] {$g_2$};
\draw [black] (49.55,-34.5) -- (48.05,-28.7);
\fill [black] (48.05,-28.7) -- (47.77,-29.6) -- (48.74,-29.35);
\draw (49.56,-31.12) node [right] {$h_4$};
\draw [black] (63.05,-39.76) -- (53.25,-37.94);
\fill [black] (53.25,-37.94) -- (53.95,-38.58) -- (54.13,-37.6);
\draw (57.5,-39.47) node [below] {$g_3$};
\draw [black] (66,-37.3) -- (66,-19.8);
\fill [black] (66,-19.8) -- (65.5,-20.6) -- (66.5,-20.6);
\draw (66.5,-28.55) node [right] {$f_6$};
\end{tikzpicture}
\end{center}
    \caption{The model follows a generalized additive model, while some edges are linear.}
    \label{fig:counterexampletc}
\end{figure}
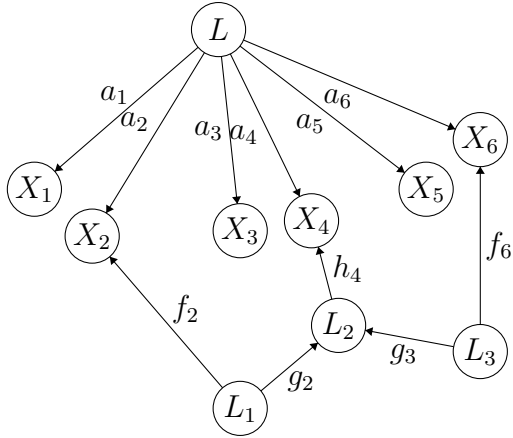
Based on the model,
We have $cum(X_2, X_4, X_6)=a_2a_4a_6\mathbb{E}(L^3)+\mathbb{E}(f_2(L_1)h_4(L_2)f_6(L_3))$.\newline Notice that $\mathbb{E}(f_2(L_1)h_4(L_2)f_6(L_3))\neq 0 $ since $L_2$ is a children of both $L_1$ and $L_3$.\\ \textcolor{white}{abc}
Take $\mathbf{S_1}=\{X_1, X_2\}$,  $\mathbf{S_2}=\{X_3, X_4\}$,  $\mathbf{S_3}=\{X_5, X_6\}$.  The choke set is $(\{L\},\emptyset,\emptyset)$, and $\mathbf{S_1, S_2, S_3}$ are linear under $\{L\}$.\\ \textcolor{white}{abc}
However, it is easy to check that $det(\mathcal{C}^{(3)}_{\mathbf{S_1, S_2, S_3}})\neq 0 $, since the second added term in $cum(X_2, X_4, X_6)$ will not be canceled out.\\ \textcolor{white}{abc}

This is the difference between the rank constraint and the tensor constraint. In the rank constraint with the generalized additive model, we still have noises independent from each other and covariances between them remain zero, and if we looked at tetrad $(\{1,2\}, \{3,4\})$, the reason that the $det(\Sigma_{\{1,2\}, \{3,4\}})\neq 0 $ is that $(\{L\},\emptyset)$ does not $t-$separate the treks, which is consistent with the theorem of the rank constraint \ref{thm:LArank}.\\ \textcolor{white}{abc}
For the tensor costraint, the generalized additive models is not enough because the noises of each variable that are initially and mutually independent become inseparable after transformation of nonlinear functions, making the higher-order cumulant tensor of the noises not diagonal, which means that the higher-order cumulant between some noises is not zero.  Intuitively, such non-zero higher-order cumulant acts like introducing additional $k-treks$ between variables. For instance, consider a simple V-structure directed graph with nonlinear edges:

\begin{center}
\begin{tikzpicture}[scale=0.12]
\tikzstyle{every node}+=[inner sep=0pt]
\draw [black] (62.5,-8.7) circle (3);
\draw (62.5,-8.7) node {$X_{b_1}$};
\draw [black] (43.7,-31.9) circle (3);
\draw (43.7,-31.9) node {$Y_0$};
\draw [black] (21.3,-34) circle (3);
\draw (21.3,-34) node {$X_{a_3}$};
\draw [black] (23.5,-9.7) circle (3);
\draw (23.5,-9.7) node {$X_{a_1}$};
\draw [black] (33.5,-21) circle (3);
\draw (33.5,-21) node {$X_{a_2}$};
\draw [black] (52.7,-21) circle (3);
\draw (52.7,-21) node {$X_{b_2}$};
\draw [black] (66.1,-34) circle (3);
\draw (66.1,-34) node {$X_{b_3}$};
\draw [black] (43.7,-45.4) circle (3);
\draw (43.7,-45.4) node {$Y_1$};
\draw [black] (25.49,-11.95) -- (31.51,-18.75);
\fill [black] (31.51,-18.75) -- (31.36,-17.82) -- (30.61,-18.49);
\draw (27.96,-16.8) node [left] {$f_1$};
\draw [black] (35.55,-23.19) -- (41.65,-29.71);
\fill [black] (41.65,-29.71) -- (41.47,-28.78) -- (40.74,-29.47);
\draw (38.07,-27.92) node [left] {$f_2$};
\draw [black] (50.79,-23.31) -- (45.61,-29.59);
\fill [black] (45.61,-29.59) -- (46.51,-29.29) -- (45.73,-28.65);
\draw (47.65,-25.02) node [left] {$g_2$};
\draw [black] (60.63,-11.05) -- (54.57,-18.65);
\fill [black] (54.57,-18.65) -- (55.46,-18.34) -- (54.68,-17.72);
\draw (57.04,-13.43) node [left] {$g_1$};
\draw [black] (31.45,-23.19) -- (23.35,-31.81);
\fill [black] (23.35,-31.81) -- (24.26,-31.57) -- (23.54,-30.89);
\draw (26.87,-26.03) node [left] {$a$};
\draw [black] (54.85,-23.09) -- (63.95,-31.91);
\fill [black] (63.95,-31.91) -- (63.72,-31) -- (63.02,-31.71);
\draw (58.36,-27.98) node [below] {$b$};
\draw [black] (43.7,-34.9) -- (43.7,-42.4);
\fill [black] (43.7,-42.4) -- (44.2,-41.6) -- (43.2,-41.6);
\draw (43.2,-38.65) node [left] {$f_3$};
\end{tikzpicture}
\end{center}

as a generalized additive model, we have: 
\begin{equation}
Y_0=f_2(X_{a_2}+\epsilon_{X_{a_2}})+g_2(X_{b_2}+\epsilon_{X_{b_2}})+\epsilon_{Y_0}
\end{equation}

$Y_1$ can be written as 
\begin{equation}
    Y_1 = f_3(Y_0)+\epsilon_{Y_1}
\end{equation}

$X_{a_3}$ can be written as 
\begin{equation}
    X_{a_3} = aX_{a_2}+\epsilon_{X_{a_3}}
\end{equation}

$X_{b_3}$ can be written as 
\begin{equation}
    X_{b_3} = bX_{b_2}+\epsilon_{X_{b_3}}
\end{equation}
The 3-order cumulant of $\langle X_{a_3}, X_{b_3}, Y_0\rangle$ can be written as:
\begin{equation}
    cum(X_{a_3}, X_{b_3}, Y_0) = ab\mathbb{E}(X_{a_3}X_{b_3}f_3(f_2(X_{a_2}+\epsilon_{X_{a_2}})+g_2(X_{b_2}+\epsilon_{X_{b_2}})+\epsilon_{Y_0}))
\end{equation}
we can see that the three independent noises $\epsilon_{a_2},\epsilon_{b_2},\epsilon_{Y_0}$ are not in an additive form anymore in the descendant of $Y$.  Even though there is no $3-trek$ on  $X_{a_3}, X_{b_3}, Y_0$, the 3-order cumulant of $\langle X_{a_3}, X_{b_3}, Y_0\rangle$ is nonzero.\\ \textcolor{white}{abc}
To make the noises or exogenous causes of each variable separable, we propose more assumptions about the structural model which is still weaker than linearity:\\ 
\textbf{Additivity Preserving Function:} A function $W$ preserves additivity if $W(\sum_i p_i) = \sum_i w_i(p_i)$ where $p_i$ are some monomials.  For instance, $f(X+Y) = sin(X)+cos(Y).$\\ 
\textbf{Recursively Additive Model: } A graph $G = \langle\mathbf{V, E}\rangle$ follows a recursively additive model, if for all $X\in \mathbf{V}$ s.t. $X= \sum a_i f_i(Pa_G(X)_i)+\epsilon_X$, for all $i$, $f_i$ is an additive preserving function. \\ \textcolor{white}{abc}

\section{Tensor Constraint with Linearity under the \textit{k-Choke Sets}}
Now we show the main result, which states a sufficient condition for the $k^{th}$-order cumulants to have a null hyperdeterminant.
\begin{theorem}\label{thm:tensorconstraintlinearunderkchokeset}
    Given recursively additive model with a directed graph $G = \langle\mathbf{V,E}\rangle$ and $k$ sets of vertices $\mathbf{S_1,\dots,S_k}$ such that the cardinatlity of each of them is $n$, if there exists subsets $\mathbf{A_1,\dots,A_k\subset V}$ such that the directed path to $\mathbf{S_1,\dots,S_k}$  are linear under $\mathbf{A_1,\dots,A_k\subset V}$ and $\mathbf{A_1,\dots,A_k\subset V}$ $k-trek$ separates $\mathbf{S_1,\dots,S_k}$, if $\sum^k_{i=1} |\mathbf{A_i}|<n$ the tensor $\mathcal{C}^{(k)}_{\mathbf{S_1,\dots,S_k}}$ of $k^{th}$-order cumulants has a null determinant.
\end{theorem}
\begin{proof}
    
For $i\in\{1,\dots,k\}$, let $\mathcal{D}_\mathbf{A_i\rightarrow S_i}\subset \mathbf{V}$ be the set of variables on a directed path from some nodes in $\mathbf{A_i}$ to some nodes in $\mathbf{S_i}$, excluding $\mathbf{A_i}$ and $\mathbf{S_i}$.

Now we consider a graph $G' $ by restrict the original graph $G=\langle \mathbf{V, E}\rangle$ to the set $\bigcup_{i\in[k]} \mathbf{{A_i}\cup\mathcal{D}_{A_i\rightarrow S_i\cup S_i}}$ as well as all the $\mathbf{A_i}$ and $\mathbf{S_i}$. Following the definition of \textit{parent} in a graph ($X_i$ is a parent of $X_j$ if $i\rightarrow j\in \mathbf{E}$) and the condition that the graph is linear under $\mathbf{A_1,\dots,A_k}$ for $\mathbf{S_1,\dots,S_k}$, we can represent $G'$ as a structural linear equation:
\begin{equation}
    \mathbf{X} = L^T\mathbf{X}+\mathcal{E}_{G'}
\end{equation}
where $\mathbf{X = \bigcup_{i\in[k]} A_i\cup\mathcal{D}_{A_i\rightarrow S_i}\cup S_i}$. By \textbf{Lemma} \ref{lm:highordercum} the tensor $\mathcal{C}^{(k)}$ of $k-$th order cumulants of $\mathbf{X}$ equals:

\begin{equation}\label{equ:tensorinproof}
    \mathcal{C}^{(k)} = \mathcal{E}_{G'}^{(k)}\cdot(I-L)^{-k}
\end{equation}

Denote $\mathbf{A_1,\dots,A_k}$ as $\langle \mathbf{A_i}\rangle$ and $\mathbf{S_1,\dots,S_k}$ as $\langle \mathbf{S_i}\rangle$. \textit{Linear below $\langle \mathbf{A_i}\rangle$ and $\langle \mathbf{S_i}\rangle$} only requires the edges between $\mathbf{A_i}$ to $\mathbf{S_i}$ to be linear, and it does not guarantee that every edge in $G'$ is linear.  If there are non-linear edges in $G'$, the causal influence represented by this edge is not in seen the coefficient matrix $L$ but is included $\mathcal{E}_G'$. This suggests that $\Phi$ is not diagonal since,  unlike the original $G$,  noises in $G'$ are not necessarily independent from each other.

Applying \textit{Cauchy-Binet Formula} $k$ many times to equation \ref{equ:tensorinproof}, we get:
\begin{equation}\label{equ:dtertensor}
    det\mathcal{C}^{(k)}_{\mathbf{S_1,\dots,S_k}} = \sum_{\substack{\mathbf{R_1, \dots, R_k \subset X_{G'}}\\ \textcolor{white}{abc}\mathbf{|R_i|} = n, \forall i\in 1,...k}}det\mathcal{E}^{(k)}_{G'_{ \mathbf{R_1,\dots,R_k}}}det(I-L)^{-1}_{\mathbf{R_1,S_1}}\dots det(I-L)^{-1}_{\mathbf{R_k,S_k}}
\end{equation}

To show that $det\mathcal{C}^{(k)}_{\mathbf{S_1,\dots,S_k}}$ is zero, it suffices to show that each term being added together is zero.  Therefore we prove the lemma:

\begin{lemma}\label{lm:atleast0tensor}
     For all $i$, if $det(I-L)^{-1}_{\mathbf{R_i,S_i}}\neq 0 $, then $det\mathcal{E}^{(k)}_{G'_{\mathbf{R_1,\dots,R_k}}} = 0$\\ \textcolor{white}{abc}
\end{lemma}

\begin{proof}

Let $|\mathbf{S_i}| = |\mathbf{R_i}| = n$.  For each $\mathbf{R_i}$, let $\mathbf{R_{A_i}}$ be the set of variables in $\mathbf{R_i}$ that $only$ have some directed path to $\mathbf{S_i}$ that intersect with $\mathbf{A_i}$ and $\mathbf{R_{U_i}}$ have directed paths to $\mathbf{S_i}$ that does not intersect with $\mathbf{A_i}$.  By \textbf{Lemma} \ref{lm:detsysdipath} we know $\mathbf{R_{A_i}}$ contains no more variables than $\mathbf{A_i}$ since otherwise there will be a system of directed paths from $\mathbf{R_i}$ to $\mathbf{S_i}$ with intersecting vertices.\\ \textcolor{white}{abc}  
Let $a_i = |\mathbf{A_i}|$. Then $|\mathbf{R_{U_i}}| \geq n - a_i$.\\ \textcolor{white}{abc}
Now we have the following inequalities:
\begin{enumerate}
    \item $\sum_i|\mathbf{R_{A_i}}|< \sum_i|\mathbf{{A_i}}| < n$
    \item $\forall i, |\mathbf{R_{A_i}|+|R_{U_i}|} = n$
\end{enumerate}
Combining the inequality $1$ and $2$ and we get:
\begin{equation}\label{ineq:R}
    \forall i, \mathbf{|R_{U_i}|} > \sum_{j\neq i} \mathbf{|R_{A_j}|}
\end{equation}
Since $\langle \mathbf{A_i}\rangle$ $k-trek$ separates $\langle \mathbf{S_i}\rangle$, there cannot be a $k-trek$ between $\langle \mathbf{R_{U_i}}\rangle$ since that will introduce a $k-trek$ between $\langle\mathbf{S_i}
\rangle$ that is not $k-trek$ separated by $\langle \mathbf{A_i} \rangle$.\\ \textcolor{white}{abc}
Now we prove that the $k-th$ order joint cumulant among noises of some variables is zero:

\begin{lemma}\label{lm:det0tensor}
    $\mathcal{E}^{(k)}_{G'_{u_1,\dots,u_k}} = 0$ for any $u_1,\dots,u_k\in \mathbf{R_{U_1},\dots,R_{U_k}}$
\end{lemma}

The proof of the lemma can be found in the Appendix, which uses the vanishing property and multilinearity of the higher order joint cumulant. 

With \textbf{Lemma} \ref{lm:det0tensor}, we prove the $det\mathcal{E}^{(k)}_{G'_{\mathbf{R_1,\dots,R_k}}} = 0$.\\ \textcolor{white}{abc}

Recall the definition of \textit{combinatorial hyperdeterminant} of an order-k $n\times n\times ...\times n$ tensor $T$:

\begin{center}
    $det T = \sum_{\sigma_2,...\sigma_k\in\mathfrak{S}(n)}sign(\sigma_2)...sign(\sigma_k)\prod_{i=1}^n T_{i, \sigma_2(i), ..., \sigma_k(i)}$
\end{center}
where $\mathfrak{S}(n)$ is the set of permutations of the set $\{1, ..., n\}$.\\ \textcolor{white}{abc}
If we interpret the meaning of the summed monomials  in the \textit{hyperdeterminant} by its definition while ignoring the sign of permutation, it simply describes a general procedure of selecting entries to make the product (denoted by $P_{\pi}$):
\begin{enumerate}
    \item freeze the ordering of one axis (in our case $\mathbf{R_i}$ for some $i$, then we call the frozen one $\mathbf{R_1}$ for convenience)
    \item $P_{\pi} = 1$
    \item for each $i\in\mathbf{R_1}$:
    \begin{itemize}
    \item for each axis $j\in\{2,...k\}$:
    \begin{itemize}
        \item randomly select a coordinate $c^i_j\in\mathbf{R_j}$
        \item $\mathbf{R_j} = \mathbf{R_j}\setminus\{c^i_j\}$ 
    \end{itemize}
    \item $P_{\pi} = P_{\pi}\times T_{i, c^i_2, ..., c^i_k}$
    \end{itemize}
\end{enumerate}
Since the hyperdeterminant is a sum over the products $P_{\pi}$ after every random selection procedure timed with its corresponding permutation sign, to show that the $det T =0$, it suffices to show that $P_{\pi}=0$.\\ \textcolor{white}{abc}
Notice that $P_{\pi}=0$ iff some entries selected is zero, recall in our proof the tensor $T=\mathcal{E}^{(k)}_{G'_{\mathbf{R_1,\dots,R_k}}}$, it suffices to show the claim:
\begin{center}
    \textit{there exists some $i$ s.t. $i\in \mathbf{R_{U_1}}, (\forall j) c^i_j\in\mathbf{R_{U_j}}$.  }
\end{center}
The claim can be seen as a \textit{pigeonhole} case.  Based on the \textbf{Inequality} \ref{ineq:R}:
\begin{center}
    $\mathbf{|R_{U_1}|} > \sum_{j\neq 1} \mathbf{|R_{A_j}|}$
\end{center}
since $c^i_j$ is either in $\mathbf{R_{A_j}}$ or $\mathbf{R_{U_j}}$, after selecting all the elements as coordinates in $\mathbf{R_{A_j}}$ for each axis, there are still some $i\in \mathbf{R_{U_1}}$ left.\\ \textcolor{white}{abc} By \textbf{Lemma} \ref{lm:det0tensor} we know $\mathcal{E}^{(k)}_{i, c^i_2\dots,c^i_k} = 0$ for all such $i\in\mathbf{R_{U_1}}$ since $c^i_j\in \mathbf{R_{U_j}}$. As a product that times with such $\mathcal{E}^{(k)}_{i, c^i_2\dots,c^i_k}$, we have $P_\pi=0$.  Therefore $det T$ is zero. \\ \textcolor{white}{abc}
Replacing $T$ with $\mathcal{E}^{(k)}_{G'_{\mathbf{R_1,\dots,R_k}}}$, we have that $det\mathcal{E}^{(k)}_{G'_{\mathbf{R_1,\dots,R_k}}}=0$.

\end{proof}
We complete the proof of the theorem using \textbf{Lemma }
 \ref{lm:atleast0tensor}. Based on equation \ref{equ:dtertensor}, the hyperdeterminant is nonzero only if there exists some $\mathbf{R_i,S_i\subset X_{G'}}$ for $i=1\dots k$ such that $det\mathcal{E}^{(k)}_{G'_{ \mathbf{R_1,\dots,R_k}}}det(I-L)^{-1}_{\mathbf{R_1,S_1}}\dots det(I-L)^{-1}_{\mathbf{R_k,S_k}}$ is not zero, which can only happen if every hyperdeterminant and determinant in the product is nonzero.  By \textbf{Lemma }\ref{lm:atleast0tensor} it is not possible, so $det\mathcal{C}^{(k)}_{\mathbf{S_1,\dots,S_k}}  = 0$.
\end{proof}

The \textbf{Theorem} \ref{thm:tensorconstraintlinearunderkchokeset} suggests that the tensor constraint can be used under cases that are not as strict as directed acyclic graph with linear structural models.  However, unlike the computation of a determinant, the computation of a hyperdeterminant of a higher-dimensional tensor is NP-hard \cite{amanov2021tensorrank}.  In addition to this difficulty, in the next section we are going to show that the definition of combinatorial hyperdeterminant is problematic for odd-numbered dimensions, leading to an inconsistency of the original theorem of the tensor constraint, \textbf{Theorem} \ref{thm:TensorConstraint}, equating the structure of the graph with the hyperdeterminant of the high order cumulant tensor.

\section{Inconsistency of Combinatorial Hyperdeterminant}
We first prove a consequence of \textbf{Theorem} \ref{thm:TensorConstraint} stated in \cite{robeva2020multitrek}, then provide a counterexample of the lemma which indicates a problem with combinatorial hyperdeterminant with odd many dimensions.

\begin{lemma}\label{lemma:rankentailstensor}
   Consider a DAG $G=\langle \mathbf{V, E}\rangle$ with linear structural model. Given $\mathbf{S_1,...,S_k\subset V}$ all with cardinality $n$, if for some $k'<k$,  $det\mathcal{C}^{k'}_\mathbf{S_1,...S_{k'}} = 0$ then $det\mathcal{C}^{k}_\mathbf{S_1,...S_{k}} = 0$.
\end{lemma}
\begin{proof}
    By \textbf{Theorem} \ref{thm:TensorConstraint} we know that $det\mathcal{C}^{k'}_\mathbf{S_1,...S_{k'}} = 0$ only if there does not exist a $k'-trek$ system on $\mathbf{S_1, ..., S_{k'}}$ that has no sided intersections.  Since every $k-trek$ system on $\mathbf{S_1,...,S_k}$ include some $k-trek$ system on $\mathbf{S_1, ..., S_{k'}}$, we know that there does not exist a $k-trek$ system on $\mathbf{S_1, ..., S_{k}}$ that has no sided intersections. Therefore $det\mathcal{C}^{k}_\mathbf{S_1,...S_{k}} = 0$.
\end{proof}

\textbf{Lemma} \ref{lemma:rankentailstensor} shows a logical relation between the hyperdeterminants higher-order cumulants subtensor with different dimensions.  Now we provide a counterexample of this lemma.

\begin{figure}
\begin{center}
\begin{tikzpicture}[scale=0.15]
\tikzstyle{every node}+=[inner sep=0pt]
\draw [black] (11.5,-29.4) circle (3);
\draw (11.5,-29.4) node {$X_1$};
\draw [black] (18.8,-29.4) circle (3);
\draw (18.8,-29.4) node {$X_2$};
\draw [black] (32.9,-29.4) circle (3);
\draw (32.9,-29.4) node {$X_3$};
\draw [black] (40.5,-29.4) circle (3);
\draw (40.5,-29.4) node {$X_4$};
\draw [black] (61.6,-30) circle (3);
\draw (61.6,-30) node {$X_5$};
\draw [black] (71.1,-28.1) circle (3);
\draw (71.1,-28.1) node {$X_6$};
\draw [black] (35.6,-8.8) circle (3);
\draw (35.6,-8.8) node {$L_1$};
\draw [black] (50.8,-50) circle (3);
\draw (50.8,-50) node {$L_2$};
\draw [black] (52.6,-17.9) circle (3);
\draw (52.6,-17.9) node {$L_3$};
\draw [black] (33.32,-10.75) -- (13.78,-27.45);
\fill [black] (13.78,-27.45) -- (14.71,-27.31) -- (14.06,-26.55);
\draw (22,-18.61) node [above] {$a_{11}$};
\draw [black] (35.21,-11.77) -- (33.29,-26.43);
\fill [black] (33.29,-26.43) -- (33.89,-25.7) -- (32.9,-25.57);
\draw (33.57,-18.95) node [left] {$a_{13}$};
\draw [black] (38.24,-10.22) -- (49.96,-16.48);
\fill [black] (49.96,-16.48) -- (49.49,-15.67) -- (49.01,-16.55);
\draw (42.51,-13.85) node [below] {$b_{13}$};
\draw [black] (54.39,-20.31) -- (59.81,-27.59);
\fill [black] (59.81,-27.59) -- (59.73,-26.65) -- (58.93,-27.25);
\draw (56.52,-25.34) node [left] {$a_{35}$};
\draw [black] (48.28,-48.38) -- (21.32,-31.02);
\fill [black] (21.32,-31.02) -- (21.72,-31.88) -- (22.27,-31.04);
\draw (36.34,-39.2) node [above] {$a_{22}$};
\draw [black] (49.46,-47.32) -- (41.84,-32.08);
\fill [black] (41.84,-32.08) -- (41.75,-33.02) -- (42.65,-32.58);
\draw (46.35,-38.59) node [right] {$a_{24}$};
\draw [black] (50.97,-47) -- (52.43,-20.9);
\fill [black] (52.43,-20.9) -- (51.89,-21.67) -- (52.89,-21.72);
\draw (52.28,-33.98) node [right] {$b_{23}$};
\draw [black] (55.23,-19.35) -- (68.47,-26.65);
\fill [black] (68.47,-26.65) -- (68.01,-25.83) -- (67.53,-26.7);
\draw (60.31,-23.5) node [below] {$a_{36}$};
\end{tikzpicture}
\end{center}
    \caption{given (right to left) $\mathbf{S_1}=\{X_5, X_6\}, \mathbf{S_2}=\{X_3, X_4\}, \mathbf{S_3}=\{X_1, X_2\}$, there does not exist a $3-trek$ system on $\mathbf{S_1, S_2, S_3}$ that has no sided intersection. }
    \label{fig:counterexampletensorconstraint}
\end{figure}
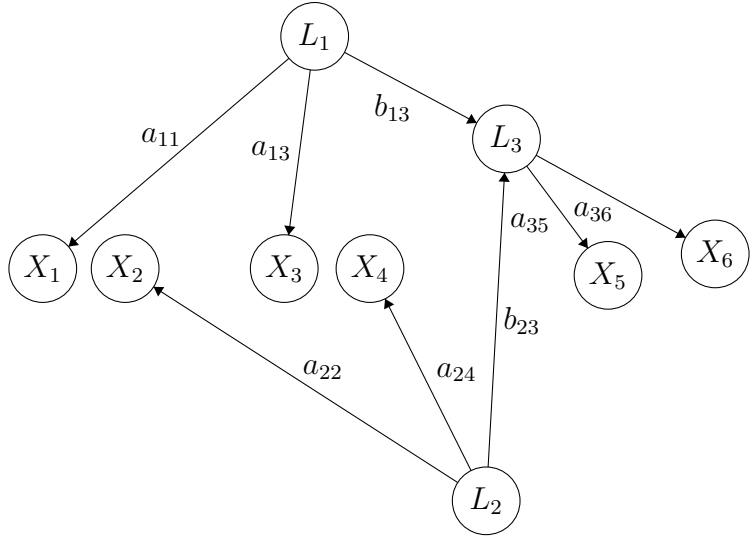
\textbf{\textit{Counterexample.}} Consider fig \ref{fig:counterexampletensorconstraint} with the linear edge coefficients shown in the figure. \\ \textcolor{white}{abc} Take
$\mathbf{S_1}=\{X_5, X_6\}, \mathbf{S_2}=\{X_3, X_4\}, \mathbf{S_3}=\{X_1, X_2\}$.
\\ \textcolor{white}{abc}
Denote $\Tilde{\mathcal{C}} = \mathcal{C}^{(2)}_\mathbf{S_1, S_2}$, which is the covariance submatrix $\Sigma_\mathbf{S_1, S_2}$.  Since every trek system on $\mathbf{S_1}$ and $\mathbf{S_2}$ has a sided intersection (at $L_3$), we get $det\Tilde{\mathcal{C}}=0$, which is also a direct application of the rank constraint.\\ \textcolor{white}{abc}
Now we consider the third order cumulant.
Denote $\hat{\mathcal{C}} = \mathcal{C}^{(3)}_{\mathbf{S_1, S_2, S_3}}$ with the order is $\mathbf{S_1, S_2, S_3}$, we want to compute the hyperdeterminant $det\hat{\mathcal{C}}$ according to the formula:

\begin{center}
    $det T = \sum_{\sigma_2,...\sigma_k\in\mathfrak{S}(n)}sign(\sigma_2)...sign(\sigma_k)\prod_{i=1}^n T_{i, \sigma_2(i), ..., \sigma_k(i)}$
\end{center}

Here $k=3$ and $n=2$. For $n=2$, there are two permutations:
\begin{itemize}
    \item $\sigma_1(1)=1, \sigma_1(2)=2$, $sign(\sigma_1)=1$
    \item $\sigma_2(1)=2, \sigma_2(2)=1$, $sign(\sigma_2)=-1$
\end{itemize}
Plugging in the permutations, we get:
\begin{flalign}
    det\hat{\mathcal{C}}&= sign(\sigma_1)^2\hat{\mathcal{C}}_{X_5, X_3, X_1}\hat{\mathcal{C}}_{X_6, X_4, X_2}+sign(\sigma_2)^2\hat{\mathcal{C}}_{X_6, X_3, X_1}\hat{\mathcal{C}}_{X_5, X_4, X_2}+0+0\\ \textcolor{white}{abc} 
    &= \hat{\mathcal{C}}_{X_5, X_3, X_1}\hat{\mathcal{C}}_{X_6, X_4, X_2}+\hat{\mathcal{C}}_{X_6, X_3, X_1}\hat{\mathcal{C}}_{X_5, X_4, X_2}\\ \textcolor{white}{abc}
    &= 2a_{11}a_{13}b_{13}a_{35}a_{22}a_{42}b_{23}a_{36}\mathbb{E}(L_1^3)\mathbb{E}(L_2^3)
\end{flalign}
The two 0s in the equation is due to the fact that there are no treks on $X_1, X_4$ nor $X_2, X_3$.\\ \textcolor{white}{abc}  
$det\hat{\mathcal{C}}$ is identically nonzero as long as none of the edge coefficient is zero.  This is a violation of \textbf{Lemma} \ref{lemma:rankentailstensor}. \\ \textcolor{white}{abc}
This counterexample is also a direct violation of the \textbf{Theorem} \ref{thm:TensorConstraint}: every $3-trek$ system in the figure has a sided intersection.  We introduce \textbf{Lemma} \ref{lemma:rankentailstensor} here to draw connections between the number of dimensions.  It is easy to see that this counterexample can be generalized to any cases with odd number of sets or dimensions.  In fact, when Cayley first introduced this definition of hyperdeterminant \cite{cayley1844theory}, he only defined it for tensors with an even number of dimensions, and many people have found that the hyperdeternimant with the odd number of dimensions does not contain desirable properties (examples of desirable properties can be found in the appendix) and only discuss the even number case. \\ \textcolor{white}{abc} 
 This violation can be avoided if $\mathbf{S_1}$ is not the first in the order of the $\hat{\mathcal{C}}$.  In other words, the hyperdeterminant with the odd number dimensions is sensitive to the first dimension.  It adds to the difficulty of using the tensor constraints and makes the correctness of \textbf{Theorem} \ref{thm:TensorConstraint} questionable without some further restriction.  

\chapter{Conclusion}

\section{Summary}
In this thesis I discussed two questions in causal discovery:
defining a faithfulness assumption more general than $k$-Triangle Faithfulness that can be applied to nonparametric distributions, and under the assumption that the modified version of $k$-Triangle Faithfulness holds, can be used to show the uniform consistency of a modified causal discovery algorithm; relaxing the linearity and acyclicity assumption to learn causal structures with latent variables with some cycles and nonlinearity.  I believe that the work in this thesis of relaxing various simplification assumptions will extend the causal discovery method to be applicable in a wider range with diversified causal mechanism and statistical phenomena.  I now summarize the novel contributions of this thesis.

\subsection{ Generalized version of $k$-Triangle Faithfulness with Uniform Consistency and Bounding the Probability of Violation of $k$-Triangle Faithfulness}
I presented the Generalized version of $k$-Triangle Faithfulness, which can be applied to any smooth distribution with the Edge Estimation Algorithm that provides uniformly consistent estimators of causal effects and the \textit{Very Conservative }$SGS$ Algorithm that is a uniformly consistent estimator of the Markov equivalence class of the true DAG. I then provided an investigation of the probability of the violation of $k$-Triangle-Faithfulness in the linear Gaussian model and compared the $k$-Triangle-Faithfulness with the Strong Faithfulness assumption\cite{JMLR:v8:kalisch07a} to quantify how much weaker the $k$-Triangle-Faithfulness assumption is, both by a mathematical analysis and a simulation study.
\subsection{Learning Latent Causal Structure with Cyclicity and Nonlinearity}
I applied GIN and rank constraints to learn causal structure with latent variables with models with cycles and nonlinearity. I designed causal clustering algorithms with GIN and rank constraints to identify cycles with models with the noises of variables following non-Gaussianity distribution. I implemented this algorithm and show its simulation result.  Finally I returned to acyclic latent causal structure with cycles between latent blocks and describe the pseudo-code identifying the cycles between latent blocks using GIN and rank constraints and an example of how the pseudo-algorithm works.
\subsection{Tensor Constraint for Latent Causal Structure with Cyclicity and Nonlinearity}
I first defined \textbf{recursively additive model}, which is more restricted than the generalized additive model but more general than the linear model, and show that tensor constraint, as a generalization of rank constraint under non-Gaussian distribution\cite{robeva2020multitrek}, is preserved with only linearity between the observed variables and their latent common causes. I further showed that the tensor constraint may fail when the dimension number of the high order cumulant tensor is odd. 

\section{Future Work}

\subsection{Hierarchical Latent Structure with Cyclic Hierarchical Clusters }

In Chapter 3 I showed the algorithm CGIN that can identify causal clusters with cycles between the measured children and latent parents as well as learning their causal order.  I conjecture that this algorithm can be extended to latent causal structures with hierarchies and cycles between parents and children that are all latent.  The\textbf{ Theorem} \ref{thm:rankdetectcycles}, \ref{thm:GINcollider1} and \ref{thm:GINcollider2} can be easily extended to hierarchical latent structures.  At least two directions are worth pursuing: learning latent causal structures such hierarchical structures with cyclic clusters for models with Gaussian distribution using the rank constraint and the same learning task under non-Gaussian distribution combining rank constraints and GIN.

\subsection{Identifying Causal Clusters with Cycles between Latent Blocks}
In Chapter 3 I described an algorithm learning causal orders between latent parents of different causal clusters and identifying the existence of cycles between latent blocks.  Unfortunately, I found out that it is both difficult for GIN and rank constraint to accurately identify individual causal clusters where the connection between the latent blocks are cyclic.  It will be interesting to develop methods or use different tests to identify clusters with cycle between latent blocks.

\subsection{Learning Latent Causal Structure with Tensor Constraints}
In Chapter 4 I showed that when the dimension of the high dimensional cumulant tensor is odd, the tensor constraint can be nonzero even if every $k-$ trek system has sided intersections. The \textbf{Theorem} \ref{thm:TensorConstraint} \cite{robeva2020multitrek} for tensor constraint still works when the dimension number is even.  In \textbf{Theorem} \ref{thm:tensorconstraintlinearunderkchokeset} I show that this constraint is preserved assuming the model with only linearity under \textbf{k-choke set}.  It would be interesting to determine if there are causal learning algorithm using the tensor constraints, maybe only in the even-dimensional situation, to identify the existence of latent variables and clusters.

\chapter{Appendix}

This chapter gives the proofs of some theoretic results in previous chapters that are not presented directly in the chapter. The the notations follow what are given in the corresponding chapters.

\section{Proof in Chapter 1}
\begin{proposition}
Recall the Total variation smoothness:\\ 
 \textbf{ TV (Total Variation) Smoothness(L): } Let $\mathcal{P}_{[0,1],TV(L)}$ be the collection of distributions $p_{Y,\mathbf{A}}$, such that for all $\mathbf{a},\mathbf{a'}\in [0,1]^{|\mathbf{A}|}$, we have:
\begin{center}
    $||p_{Y|\mathbf{A=a}}-p_{Y|\mathbf{A=a'}}||_1\leq L||\mathbf{a-a'}||_1$
\end{center}
Consider two vectors of variables $\mathbf{X}$ and $\epsilon$, where $\epsilon$ follows a multivariate standard normal distribution: 
\begin{center}
$\epsilon\sim\mathcal{N}\left(0,\mathbf{\Sigma_\epsilon}\right)$
\end{center} and $\mathbf{X}$ can be written as:  
\begin{center}
    $\mathbf{X}=\mathbf{B^TX}+\epsilon$
\end{center}
Denote $\mathbf{(I-B)^{-T}\mathbf{\Sigma_\epsilon}(I-B)}^{-1}$ with $\Sigma$. Given $Y$ as a variable in $\mathbf{X}$ and $\mathbf{A}$ as a set of variables in $\mathbf{X}$, consider 
\begin{center}
$L\geq\dfrac{2\phi(0)||\Sigma_{Y,\mathbf{A}}\Sigma_{\mathbf{A}\mathbf{A}}^{-1}||_1}{\sqrt{var(Y|\mathbf{A})}}$
\end{center}
where $\Sigma_{Y,\mathbf{A}}$ is a submatrix of $\Sigma$ and $\phi(0)$ is the pdf of the standard normal distribution at $0$.\\ \textcolor{white}{abc}
Then $p_{Y,\mathbf{A}}$ satisfies TV(L) smoothness.

\end{proposition}

\begin{proof}
Based on the structure equation it is easy to derive that $\mathbf{X}$ follows the distribution:
\begin{center}
$\mathbf{X}\sim\mathcal{N}\left(0,\left[\mathbf{(I-B)\mathbf{\Sigma_\epsilon}^{-1}(I-B)}^T\right]^{-1}\right)$
\end{center}

 Recall that $L_1$ distance between two Gaussian distribution with only different means is:
\begin{center}
    $||Y_1-Y_2||_1 = 2\left(2\Phi(\dfrac{|\mu_1-\mu_2|}{2\sqrt{var(Y)}})-1\right)$
\end{center}
where $Y_1\sim\mathcal{N}\left(\mu_1,var(Y)\right)$ and $Y_2\sim\mathcal{N}\left(\mu_2,var(Y)\right)$.\\ \textcolor{white}{abc}  Now we plug in $Y_1$ as $Y|\mathbf{A=a_1}$ and $Y_2$ as $Y|\mathbf{A=a_2}$.  Fortunately when conditioning on the same set of variable, the resulted Gaussian distribution always have the same covariance. we see:
\begin{center}
    $||Y_1-Y_2||_1 = 2\left(2\Phi(\dfrac{|\Sigma_{Y,\mathbf{A}}\Sigma_{\mathbf{A}\mathbf{A}}^{-1}(\mathbf{a_1-a_2})|}{2\sqrt{var(Y|\mathbf{A})}})-1\right)$
\end{center}
It is commonly know that $\Phi'(x) = \phi(x)\leq \phi(0)$ (which can also be derived with some high school math).  Then we have:
\begin{flalign}
    ||Y_1-Y_2||_1 &= 2\left(2\Phi(\dfrac{|\Sigma_{Y,\mathbf{A}}\Sigma_{\mathbf{A}\mathbf{A}}^{-1}(\mathbf{a_1-a_2})|}{2\sqrt{var(Y|\mathbf{A})}})-1\right)\\ \textcolor{white}{abc}
    &<2\left(2\Phi(\dfrac{||\Sigma_{Y,\mathbf{A}}\Sigma_{\mathbf{A}\mathbf{A}}^{-1}||_1||\mathbf{a_1-a_2}||_1}{2\sqrt{var(Y|\mathbf{A})}})-1\right)\label{inequ:cw}\\ \textcolor{white}{abc}
    &<4\phi(0)\dfrac{||\Sigma_{Y,\mathbf{A}}\Sigma_{\mathbf{A}\mathbf{A}}^{-1}||_1||\mathbf{a_1-a_2}||_1}{2\sqrt{var(Y|\mathbf{A})}}\label{inequ:slope}\\ \textcolor{white}{abc}
    &\leq 2\dfrac{\phi(0)||\Sigma_{Y,\mathbf{A}}\Sigma_{\mathbf{A}\mathbf{A}}^{-1}||_1||\mathbf{a_1-a_2}||_1}{\sqrt{var(Y|\mathbf{A})}}\\ \textcolor{white}{abc}
    &\leq L||\mathbf{a_1-a_2}||_1
\end{flalign}
The step \ref{inequ:cw} is derived according to Cauchy-Schwart Inequality.  As shown in plot \ref{fig:gaussianplot}, the step \ref{inequ:slope} is derived by the fact that $\Phi'(x)<\phi(0)$ for $x>0$.
\begin{figure}
    \centering
    \includegraphics[width=0.6\linewidth]{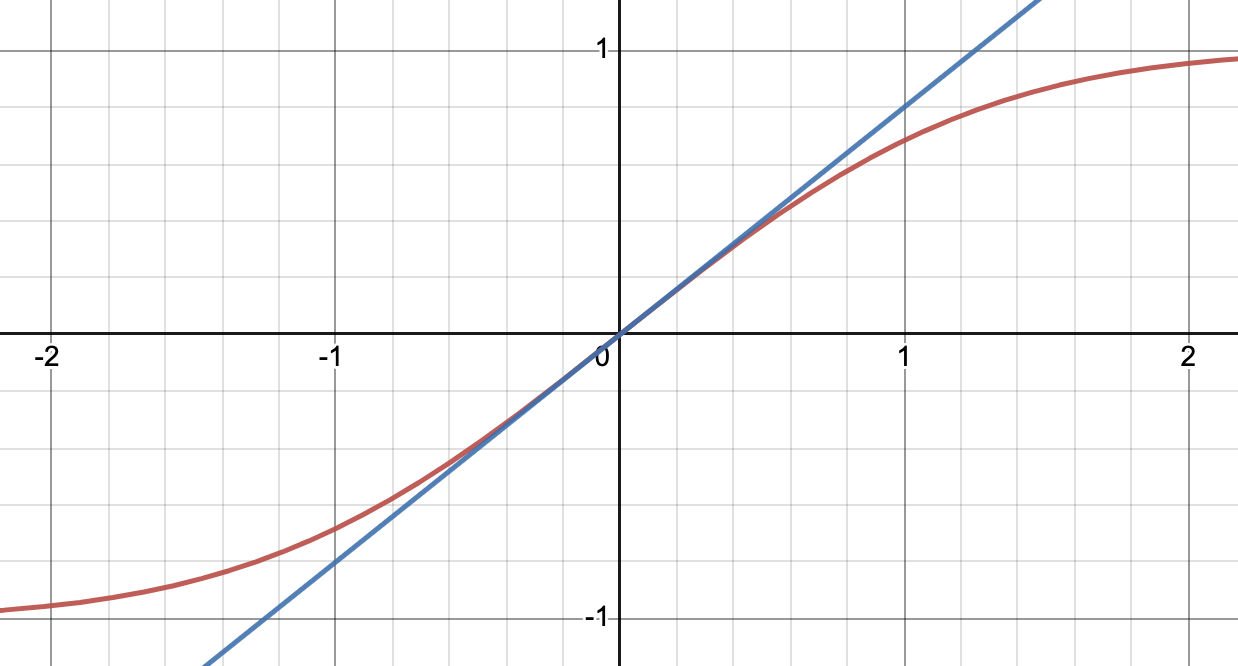}
    \caption{Plot of $f(x)=2\Phi(x)-1$ as the red line and $f(x)=2\phi(0)(x)$ as blue}
    \label{fig:gaussianplot}
\end{figure}
\end{proof}

\section{Proofs in Chapter 2}
\subsection{Proof of Lemma \ref{lemma:deg_cov}}
\textbf{Lemma }\ref{lemma:deg_cov}
\textit{ Following the definition of $i,j$ and $\mathbf{S_\Delta}$, we have:
    \begin{center}
        $deg(cov(X_i,X_j|X_{S_{\Delta}}))\leq 2(|\mathbf{V}|-|\mathbf{S_\Delta}|)$
    \end{center}
    where $\mathbf{Q} =\{i,j\}\bigcup\mathbf{S_\Delta}$.}
\begin{proof}
Recall that the polynomial $P_{ij|S_\Delta}$ of the conditional covariance $cov(X_i, X_j|\mathbf{X_{s_\Delta}})$ has a path description:
\begin{center}
    $P_{ij|S_\Delta} = det(K_{Q^cQ^c})K_{ij}-K_{iQ^c}C(K_{Q^c}K_{Q^c})K_{Q^cj}$
\end{center}
where $K_{ij}=\underset{k:i\rightarrow k\leftarrow j}{\sum}a_{ik}a_{jk}-a_{ij}$.\\ \textcolor{white}{abc}
By Ponstein's theorem, $max\{deg(det(K_{Q^c}K_{Q^c})),deg(C(K_{Q^c}K_{Q^c})_{ij})\}\leq|\mathbf{Q^c}|.$\\ \textcolor{white}{abc}
Therefore $deg(P_{ij|S_\Delta})\leq |\mathbf{Q^c}|+3=|\mathbf{V}|-|\mathbf{S_\Delta}|-2+3< 2(|\mathbf{V}|-|\mathbf{S_\Delta}|)$.
\end{proof}
\subsection{Proof of Theorem \ref{lemma:probtrainglevio}}
\textbf{Theorem} \ref{lemma:probtrainglevio}
\textit{
Let $G_{\{\Delta\}} = \langle \mathbf{V, E}\rangle$ be a DAG where each edge is in at most one triangle and density $d$.  Then,
    \begin{center}
$\mathbb{P}\left((a_{st})_{(s,j)\in\mathbf{E}}\in\mathcal{Q}_{G,K}\right)\geq (1-(1-d^3)^{\lfloor\frac{|\mathbf{V}|}{3}\rfloor}) (1-(1-\dfrac{K}{2})^4)$    
    \end{center}}
\begin{proof}
\begin{align*}
    \mathbb{P}\left((a_{st})_{(s,j)\in\mathbf{E}}\in\mathcal{Q}_{G,K}\right) &= \mathbb{P}(\text{violation of $k$-Triangle-Faithfulness in $G$ with $K$})\\ \textcolor{white}{abc}
    &\geq \mathbb{P}(\text{triangle exists in $G$}) \mathbb{P}(\text{violation of $k$-Triangle-Faithfulness exists in a triangle})
\end{align*}
Interpreting density $d$ as the probability of an edge exists between two vertices, it is obvious that:
\begin{center}
    $\mathbb{P}(\text{triangle exists in $G$})\geq(1-(1-d^3)^{\lfloor\frac{|\mathbf{V}|}{3}\rfloor})$
\end{center}
By lemma \ref{lemma:triangle}: $\mathbb{P}(\text{violation of $k$-Triangle-Faithfulness exists in a triangle})=1-(1-\dfrac{K}{2})^4.$\\ \textcolor{white}{abc}
Therefore $\mathbb{P}\left((a_{st})_{(s,j)\in\mathbf{E}}\in\mathcal{Q}_{G,K}\right)\geq (1-(1-d^3)^{\lfloor\frac{|\mathbf{V}|}{3}\rfloor}) (1-(1-\dfrac{K}{2})^4)$.    
\end{proof}
\section{Proofs in Chapter 3}
\subsection{Proof of Theorem \ref{nonlinearGIN}}
\textbf{Theorem }\ref{nonlinearGIN}\textit{ Suppose that random vectors $\mathbf{L_Y,Y,L_Z}$ and $\mathbf{Z}$ are related in the following way:}
\begin{center}
    $\mathbf{Y=}A\mathbf{L_Y+E_Y}$,
    
    $\mathbf{Z=}B\mathbf{L_Z+E_Z}$.
\end{center}
\textit{Denote by $l_Y$ the dimensionality of $\mathbf{L_Y}$ and  $l_Z$ the dimensionality of $\mathbf{L_Z}$.  Assume $A$ is of full column rank.  Then, if 1) $Dim(\mathbf{Y})>l_Y$,  2) $\mathbf{E_Y}\indep \mathbf{L_Z}$, 3) $\mathbf{E_Y}\indep \mathbf{E_Z}$, 4) the cross-covariance matrix of $\mathbf{L_Y}$ and $\mathbf{Z}$, $\Sigma_{\mathbf{L_Y Z}}=\mathbb{E}[\mathbf{L_{Y}Z}^T]$ has rank $l_Y$, then $E_{\mathbf{Y\parallel Z}}\indep \mathbf{Z}$, i.e.,$(\mathbf{Z,Y})$ satisfies the GIN condition.}

\begin{proof}
Without loss of generality, assume that each component of $\mathbf{L_Y}$ and $\mathbf{L_Z}$ has a zero mean, and that both $\mathbf{E_Y}$ and $\mathbf{E_Z}$ are zero-mean.  If we can find a non-zero vector $\omega$ such that $w^TA=0$, then $\omega^T\mathbf{Y}=\omega^TA\mathbf{L_Y}+\omega^T\mathbf{E_Y}=\omega^T\mathbf{E_Y}$, which will be independent from $\mathbf{Z}$ in light of condition 2) and 3), so the GIN condition of $\mathbf{Y}$ given $\mathbf{Z}$ holds true.\\ \textcolor{white}{abc}
If condition 2) and 3) hold, we have $\mathbb{E}[\mathbf{YZ}^T]=A\Sigma_{\mathbf{L_YZ} }$, which is determined by $(\mathbf{Y,Z} )$.   Now we want to show for any non-zero vector $\omega$, $\omega^TA=0$ if and only if $\omega^TA\Sigma_{\mathbf{L_YZ^T} }=0$.\\ \textcolor{white}{abc}
$\Rightarrow$: trivially true, since $\omega^TA=0$.\\ \textcolor{white}{abc}
$\Leftarrow$: notice that by condition 1), rank$(A\Sigma_{L_YZ})\leq l_Y$ because rank$(A\Sigma_{L_YZ})\leq min(rank(A),rank(\Sigma_{L_YZ}))$ and rank$(A)=l_Y$.  Further according to the Sylvester Rank Inequality, if $A$ is an $m\times n$ matrix and $B$ is $n\times k$ then: 
\begin{center}
    $rank(A)+rank (B)-n\leq rank(AB)$,
\end{center}
 we have rank$(A\Sigma_{L_YZ})\geq$ rank$(A)+$rank$(\Sigma_{L_YZ})-l_Y=l_Y$.  Therefore, rank$(A\Sigma_{L_YZ})=l_Y$. Because of condition 1), there exists a non-zero vector $\omega$, determined by $\mathbf{(Y,Z)}$, such that $\omega^T\mathbb{E}[\mathbf{YZ}^T]=\omega^T A\Sigma_{L_YZ}=0$, which implies $\omega^TA=0$ because $\Sigma_{L_YZ}$ has $l_Y$ rows and has rank $l_Y$.  With this $\omega$, so we have $E_{Y\parallel Z}=\omega^T\mathbf{E_Y}$ and is independent from $\mathbf{Z}$.
\end{proof}
\subsection{Proof of Lemma \ref{lemma:rankcyclic0}}
\textbf{Lemma }\ref{lemma:rankcyclic0}
\textit{
For any $X\in \mathcal{D}_{C_A\rightarrow A}\cup \mathbf{A}$ and variable $Y\in \mathcal{D}_{C_B\rightarrow B}\cup \mathbf{B}$, $\Phi_{X,Y} = 0$.}
\begin{proof}
Recall that:
    \begin{equation}
    X = L_{X^.}\mathbf{X} + \epsilon _{G'_X}
    \end{equation}
    \begin{equation}
    Y = L_{Y^.}\mathbf{X} + \epsilon_{G'_Y}
    \end{equation}
where $ L_{X^.}$ is the row of $L$ corresponding to $X$. Since $G'$ is a subgraph of $G$ and $L$ is the coefficient matrix of edges in $G'$ that are linear, we can further write:
    \begin{equation}
        \epsilon_{G'_X} = \epsilon_X + f(Pa_{G\setminus G'}(X))
    \end{equation}
    \begin{equation}
        \epsilon_{G'_Y} = \epsilon_Y + f(Pa_{G\setminus G'}(Y))
    \end{equation}
where $\epsilon_X$ is the additive noise of $X$ in the original graph $G$ and $Pa_{G\setminus G'}$ denotes the parents of $X$ in $G$ that are not parents of $X$ in $G'$.\\ \textcolor{white}{abc}
By definition $\Phi_{X,Y} = cov(\epsilon_{G'_X}, \epsilon_{G'_Y})$. Assuming that $cov(\epsilon_{G'_X},\epsilon_{G'_Y})\neq 0$, we know that the $\epsilon_{G'_X}$ and $\epsilon_{G'_Y}$ are only dependent in one of these cases:

\begin{enumerate}
    \item $f(Pa_{G\setminus G'}(X))$ and $f(Pa_{G\setminus G'}(Y))$ are dependent
    \item $\epsilon_X$ and $f(Pa_{G\setminus G'}(Y))$ are dependent 
    \item $\epsilon_Y$ and $f(Pa_{G\setminus G'}(X))$ are dependent 
\end{enumerate}
\textit{Case 1} is only possible if $X$ and $Y$ have some parents that are not $A-side$ and $B-side$ and their parents are dependent. Since the non $A-side$ parent of $X$ and non $B-side$ parent of $Y$ are dependent, we know that they are connect by a trek that is not having $(C_A;C_B)$ as a choke set.  Since $X$ and $Y$ are treks to $A$ and $B$, we now have a new trek that is not having $(C_A;C_B)$ as a choke set. Contradiction.\\ \textcolor{white}{abc} 
\textit{Case 2} and $3$ are symmetric so we just analyze \textit{Case 2}.  This case only happens if $X$ is an ancestor of some parents of $Y$ that are non $B-side$ .  This introduces a new trek that is not having $(C_A;C_B)$ as a choke set.
Contradiction.
\end{proof}
\subsection{Proof of Lemma \ref{lemma:atleast0}}

\textbf{Lemma }\ref{lemma:atleast0}
 \textit{   For all $\mathbf{R,S\subset \mathbf{X_{G'}}}$, if $det(I-L)^{-1}_{\mathbf{R,A}} \neq 0$ and $det(I-L)^{-1}_{\mathbf{S,B}}\neq 0$, then $det\Phi_{\mathbf{R,S}}=0$.}
\begin{proof}
Let $|\mathbf{A}| = n$.  Let $\mathbf{R_{C_A}}$ be the set of variables in $\mathbf{R}$ that $only$ have directed paths to $\mathbf{A}$ that intersect with $\mathbf{C_A}$ and $\mathbf{R_A}$ have directed paths to $\mathbf{A}$ that does not intersect with $\mathbf{C_A}$.  Similarly, we have $\mathbf{S_B}$  and $\mathbf{S_{C_B}}$.  By \textbf{Lemma \ref{lm:detsysdipath}} we know $\mathbf{R_{C_A}}$ contains no more variables than $\mathbf{C_A}$ since otherwise there will be a system of directed paths from $\mathbf{R}$ to $\mathbf{A}$ with intersecting vertices. \\ \textcolor{white}{abc} 
  Let $a_1 = |\mathbf{C_A}|$. Then $|\mathbf{R_A}| \geq n - a_1$.\\ \textcolor{white}{abc}
Similarly, having $b_1 = |\mathbf{C_B}|$ and $\mathbf{S_B}$ the set of variables that have directed paths to $\mathbf{B}$ that does not intersect with $\mathbf{C_B}$, we have $|\mathbf{S_B}|\geq n - b_1$.\\ \textcolor{white}{abc}
Now we have the following inequalities:
\begin{enumerate}
    \item $a_1 + b_1<n$
    \item $|\mathbf{S_B}|\geq n - b_1$
    \item $|\mathbf{R_A}| \geq n - a_1$
    \item $|\mathbf{R_A|+|R_{C_A}|} = n$
    \item $|\mathbf{S_B|+|S_{C_B}|} = n$
    \item $\mathbf{|R_{C_A}|}\leq a_1$
    \item $\mathbf{|S_{C_B}|}\leq b_1$
\end{enumerate}

Combining the inequality $1-7$ about and we get:
\begin{enumerate}
    \item $|\mathbf{R_A|} > \mathbf{|S_{C_B}|}$ 
    \item$\mathbf{|S_B}|>|\mathbf{R_{C_A}}|$
\end{enumerate}
Since singularity is preserved under the row and column permutation, we order columns and rows in $\Phi_{\mathbf{R,S}}$ as:

\[
\left(\begin{array}{@{}c c@{}}
  \bigA
  & \bigB \\ 
  \bigC 
  & \bigD \\ 
\end{array}\right)
\]where block $A=\Phi_{\mathbf{R_{A},S_{B}}}$ and $D=\Phi_{\mathbf{R_{C_A},S_{C_B}}}$.
By \textbf{Lemma \ref{lemma:rankcyclic0}}, we know that for every $X\in\mathbf{R_{C_A}}$
and $Y\in \mathbf{S_{C_B}}$,
$cov(\epsilon_{G'_X},\epsilon_{G'_Y})=0$ so $\mathbf{D}$ is a zero matrix. 
Notice that $A$ is also a zero matrix: recall that $A$ as a submatrix of $\Phi$, for any $A_{ij}$ to be nonzero it has to be the case that there are some common causes between members in $\mathbf{R_A}$ and $\mathbf{S_B}$ in the original graph $G$, which introduces a new trek that does not have $(C_A;C_B)$ as a choke set.\\ \textcolor{white}{abc}
Having top left and bottom right submatrix zero, and $|\mathbf{R_A|} > \mathbf{|S_{C_B}|}$, $\Phi_{\mathbf{R,S}}$ has linearly dependent rows, therefore has a determinant of 0.

\end{proof}

\subsection{Proof of Lemma \ref{GINdsep}}
\textbf{Lemma }\ref{GINdsep}
\textit{
         Consider two sets of variables $\mathbf{Z, Y}$ in a $L^2HCM$ model. Assume faithfulness holds for the $L^2HCM$. If there are any treks between $\mathbf{Z}$ and $\mathbf{Y}$, $(\mathbf{Z, Y})$ satisfies GIN iff there exists a $\mathbf{C_Y}$, s.t. $(\emptyset, \mathbf{C_Y})$ t-separates $(\mathbf{Z, Y})$ and $\mathbf{|C_Y|<|Y|}$.
}
\begin{proof}
    $\leftarrow:$ Since $\mathbf{C_Y}$ is a set of ancestors of $\mathbf{Y}$, we can write $\mathbf{Y}=A\mathbf{C_Y}+\epsilon$.  We further know that $\epsilon$ is independent from $\mathbf{Z}$ since otherwise there is another trek connecting $(\mathbf{Z, Y})$ that does not include $\mathbf{C_Y}$, which means $(\emptyset,\mathbf{C_Y})$ does not t-separate $(\mathbf{Z, Y})$. Since $\mathbf{|C_Y|<|Y|}$, $A$ has a shape of $\mathbf{|Y|\times|C_Y|}$ and there exists a vector $\omega$ s.t. $\omega^T A=0$.  Therefore $\omega^T\mathbf{Y}\indep\mathbf{Z}$ so $(\mathbf{Z, Y})$ satisfies GIN.\\ \textcolor{white}{abc}
    $\rightarrow:$ We pursue the proof by contrapositive. Assume that for all $\mathbf{C_Y}$, s.t. $(\emptyset, \mathbf{C_Y})$ t-separates $(\mathbf{Z, Y})$, $\mathbf{|C_Y|\geq|Y|}$.  Since $\mathbf{Y}$ is a vector of descendants of $\mathbf{C_Y}$, we can write $\mathbf{Y}$ as $\mathbf{Y}=A\mathbf{C_Y}+\epsilon$ with $A$ has a shape of $\mathbf{|Y|\times|C_Y|}$ where $\mathbf{|Y|\leq|C_Y|}$ with full column rank.  There does not exists nonzero vector $\omega$ s.t. $\omega^TA=0$.  Since there are treks connecting $\mathbf{C_Y}$ and $\mathbf{Z}$, they share commen non-Gaussian variables and by Darmois-Skitovitch Theorem we know for all nonzero $\omega$, $\omega^T\mathbf{Y}\nindep\mathbf{Z}$, so $(\mathbf{Z, Y})$ cannot satisfy GIN. 
\end{proof}

\subsection{Proof of Theorem \ref{thm:GINcollider2}}

\textbf{Theorem }\ref{thm:GINcollider2}
\textit{Consider a $G=\langle \mathbf{V, E}\rangle$ follows a linear and acyclic graph with $LiNGLaM$. Denote latents of the cluster $\mathcal{C}$ by $L_{\mathcal{C}}$ and the measured children in cluster $\mathcal{C}$ by $\{X^{\mathcal{C}}_1...X^{\mathcal{C}}_m\}$. Given clusters $\mathcal{C}_i, \mathcal{C}_a, \mathcal{C}_b$, if:
\begin{enumerate}[label=(\roman*)]
\item Let $\mathcal{C}_j$ be the union of all clusters causally ealier than $\mathcal{C}_a, \mathcal{C}_b$. \newline  Neither $(\{X^{\mathcal{C}_a}_1,...X^{\mathcal{C}_a}_{|L_{\mathcal{C}_a}|},X^{\mathcal{C}_b}_{|L_{\mathcal{C}_b}|+1},...X^{\mathcal{C}_b}_{2|L_{\mathcal{C}_b}|},X^{\mathcal{C}_j}_{|L_{\mathcal{C}_j}|+1},...X^{\mathcal{C}_j}_{2|L_{\mathcal{C}_j}|}\},\{X^{\mathcal{C}_b}_1,...X^{\mathcal{C}_b}_{|L_{\mathcal{C}_b}|}\})$ nor\newline  $(\{X^{\mathcal{C}_b}_1,...X^{\mathcal{C}_b}_{|L_{\mathcal{C}_b}|},X^{\mathcal{C}_a}_{|L_{\mathcal{C}_a}|+1},...X^{\mathcal{C}_a}_{2|L_{\mathcal{C}_a}|},X^{\mathcal{C}_j}_{|L_{\mathcal{C}_j}|+1},...X^{\mathcal{C}_j}_{2|L_{\mathcal{C}_j}|}\},\{X^{\mathcal{C}_a}_1,...X^{\mathcal{C}_a}_{|L_{\mathcal{C}_a|}}\})$ satisfies GIN condition.
\item $rank(\Sigma_{\{X^{\mathcal{C}_i}\},\{X^{\mathcal{C}_a}\}\cup\{X^{\mathcal{C}_b}\}})<min(|L_{\mathcal{C}_a}|,|L_{\mathcal{C}_b}|)$.
\item given a $\omega$ s.t. $[X^{\mathcal{C}_b}_1,...X^{\mathcal{C}_b}_{|L_{\mathcal{C}_b}|}]^T\omega\indep [X^{\mathcal{C}_i}_{|L_{\mathcal{C}_i}|+1},...X^{\mathcal{C}_i}_{2|L_{\mathcal{C}_i}|}]$, $[X^{\mathcal{C}_b}_1,...X^{\mathcal{C}_b}_{|L_{\mathcal{C}_b}|}]^T\omega\nindep [X^{\mathcal{C}_a}_1,...X^{\mathcal{C}_a}_{|L_{\mathcal{C}_a}|}]$
\item given a $\theta$ s.t. $[X^{\mathcal{C}_a}_1,...X^{\mathcal{C}_a}_{|L_{\mathcal{C}_a}|}]^T\theta\indep [X^{\mathcal{C}_i}_{|L_{\mathcal{C}_i}|+1},...X^{\mathcal{C}_i}_{2|L_{\mathcal{C}_i}|}]$, $[X^{\mathcal{C}_a}_1,...X^{\mathcal{C}_a}_{|L_{\mathcal{C}_a}|}]^T\theta\nindep [X^{\mathcal{C}_b}_1,...X^{\mathcal{C}_b}_{|L_{\mathcal{C}_b}|}]$.
\end{enumerate}
then $\mathcal{C}_a$ and $\mathcal{C}_b$ has cycles between blocks of latents.}
\begin{proof} 
$(i)$ indicates that in addition to common causes, it is neither simply $L_{\mathcal{C}_a}\leftarrow L_{\mathcal{C}_b}$ nor simply $L_{\mathcal{C}_a}\rightarrow L_{\mathcal{C}_b}$.  $(ii)$ indicates that $\mathcal{C}_i$ is d-separated by $\mathcal{C}_b$ and $\mathcal{C}_a$ by a set with size less than $min(|L_{\mathcal{C}_a}|,|L_{\mathcal{C}_b}|)$. By \textbf{Theorem \ref{GINdsep}}, both $(\{X^{\mathcal{C}_i}_{|\mathcal{C}_i|+1},...X^{\mathcal{C}_i}_{2|\mathcal{C}_i|}\},\{X^{\mathcal{C}_b}_1,...X^{\mathcal{C}_b}_{|L_{\mathcal{C}_b}|}\})$ and\newline  $(\{X^{\mathcal{C}_i}_{|\mathcal{C}_i|+1},...X^{\mathcal{C}_i}_{2|\mathcal{C}_i|}\},\{X^{\mathcal{C}_a}_1,...X^{\mathcal{C}_a}_{|L_{\mathcal{C}_a}|}\})$ satisfy GIN .   $(iv)$ shows that $[X^{\mathcal{C}_b}_1,...X^{\mathcal{C}_b}_{|L_{\mathcal{C}_b}|}]^T\omega$ functions as removing edges between $L_{\mathcal{C}_b}$ and $L_{\mathcal{C}_i}$ and there are colliders between $L_{\mathcal{C}_i}$ and $[X^{\mathcal{C}_b}_1,...X^{\mathcal{C}_b}_{|L_{\mathcal{C}_b}|}]^T\omega$.  Since $\{X^{\mathcal{C}_b}\}$ are only connected to other variables through $L_{\mathcal{C}_b}$, we know that there are colliders between $L_{\mathcal{C}_b}$ and $L_{\mathcal{C}_i}$, which can only be $L_{\mathcal{C}_a}$.  Therefore $L_{\mathcal{C}_a}\leftarrow L_{\mathcal{C}_b}$ exists between some latents in $\mathcal{C}_a$ and $\mathcal{C}_b$.  Similarly, $(v)$ indicates that some $L_{\mathcal{C}_b}$ are colliders between  $L_{\mathcal{C}_a}$ and $L_{\mathcal{C}_i}$.  Therefore there are cycles between blocks.
\end{proof}

\section{ Proofs in Chapter 4}
\subsection{Proof of Lemma \ref{lm:det0tensor}}
\textbf{Lemma }\ref{lm:det0tensor}\textit{ $\mathcal{E}^{(k)}_{G'_{u_1,\dots,u_k}} = 0$ for any $u_1,\dots,u_k\in \mathbf{R_{U_1},\dots,R_{U_k}}$
}

\begin{proof}
    We use these two properties of higher order joint cumulant:
\begin{enumerate}
    \item vanishing property 
    \item the multilinearity 
\end{enumerate}
Recall that $\mathcal{E}^{(k)}_{u_1,\dots,u_k}$ is the $k-th$ cumulant of the vector $(\epsilon_{G'_{u_1}},\dots,\epsilon_{G'_{u_k}})$. Notice that  $\epsilon_{G'_{u_i}}$ consists of two parts: the noise to $u_i$ in $G$ and the function of parents of $u_i$ that is not included in $G'$ or included in $G'$ but the edge is not linear.  In particular, $\epsilon_{G'_{u_i}} = \sum_j f_j(Pa_{G\setminus G'_L}(u_i)_j)+\epsilon_{u_i}$ where $Pa_{G\setminus G'_L}(u_i)_j$ is a parent of $u_i$ of which the edge to $u_i$ is not included in the coefficient matrix $L$.

With a little work we can see that
\begin{equation}
 \epsilon_{G'_{u_i}} = \sum_j f_j(\epsilon_{Anc_{G\setminus G'_L}(u_i)_j})+\epsilon_{u_i}
\end{equation}
where $Anc_{G\setminus G'_L}(u_i)_j$ is either an ancestor of parents of $u_i$ of which the edge to $u_i$ is not included in $L$ or a parent as $Pa_{G\setminus G'_L}(u_i)_j$.

Denote $\kappa$ as the k-th order cumulant. By the multilinearity of joint cumulant, we further have 
\begin{equation}
\mathcal{E}^{(k)}_{u_1,\dots,u_k} = \kappa(\epsilon_{G'_{u_1}},\dots,\epsilon_{G'_{u_k}}) = \sum_{\substack{w_1,\dots,w_k\in\\ \textcolor{white}{abc} \mathbf{Anc_{G\setminus G'_L}(u_1),\dots,\mathbf{Anc_{G\setminus G'_L}(u_k)}}}}\kappa(\epsilon_{w_1},\dots,\epsilon_{w_k})    
\end{equation}
Since there is no $k-trek$ on $u_1,\dots,u_k$, there is at least one $w_i$ for each $\kappa(\epsilon_{w_1},\dots,\epsilon_{w_k})$ that is different from other $w_{j\neq i}$, which means that $\epsilon_{w_i}$ is independent from all the other noise variables in the cumulant. By the vanishing property of joint cumulant, we have $\kappa(\epsilon_{w_1},\dots,\epsilon_{w_k})=0$ for every $\epsilon_{w_1},\dots,\epsilon_{w_k}$, which means that $\mathcal{E}^{(k)}_{u_1,\dots,u_k}$ is identically zero.\\ \textcolor{white}{abc}
\end{proof}

\subsection{Example of Desirable Property of Hyperdeterminant}
The example in this section comes from the paper \cite{amanov2021tensorrank} discussing the tensor rank and hyperdeterminant.\\ 
\textit{\textbf{Slices:}}  Consider a $d-$dimensional tensor $T$ with shape $n^d$.  For each direction $k\in[d]$, a \textit{parallel} $(d-1)-$dimensional \textit{slices} $T_1^{(k)},...T_n^{(k)}$ of $T$ given by fixing the $k-$th coordinate:
\begin{center}
    $T_l^{(k)}(i_1,...,i_{k-1},i_{k+1},...,i_d)=T(i_1,...,i_{k-1},i_{k+1},...,i_d)$ for $l=1,...,n$.
\end{center}
$T$ is a diagonal identity tensor if $T(i_1,...,i_d)=1$ if $i_1 = ...= i_d$ and $T(i_1,...,i_d)=0$, otherwise.\\ 
\textit{\textbf{Desirable Property of hyperdeterminant.}} Let $T$ be the a $d-$dimensional tensor with shape $n^d$.  The hyperdeterminant of $T$, $detT$, is expected to satisfy:
\begin{enumerate}
    \item Skew-symmetry:  If $T'$ is obtained from $T$ by swapping any two parallel slices. then
    \begin{center}
        $detT=-detT'$
    \end{center}
    \item Normalization: If $T$ is a diagonal identity tensor, then \begin{center}
        $detT = 1$.
    \end{center}
\end{enumerate}
If $d$ is even, the combinatorial hyperdeterminant of $T$ satisfies both properties.  If $d$ is odd, then the combinatorial hyperdeterminnat of $T$ does not satisfy both properties.
\backmatter

%\renewcommand{\baselinestretch}{1.0}\normalsize

% By default \bibsection is \chapter*, but we really want this to show
% up in the table of contents and pdf bookmarks.
\renewcommand{\bibsection}{\chapter{\bibname}}
\bibliographystyle{plainnat}
\bibliography{bib} %your bib file

\end{document}